\documentclass[12pt]{book} 
\usepackage[twoside, margin=1.2in]{geometry}
\usepackage{graphicx, setspace, xcolor, amssymb, amsmath, amsthm, subcaption, enumitem,multirow, graphbox, array, mathtools, graphbox, circuitikz, fancyhdr, pdfpages, titlesec, moresize}
\usepackage{hyperref}
\usepackage[T1]{fontenc}
\usepackage{palatino}

\usetikzlibrary{circuits.ee.IEC}
\usepackage{tikzit}

\tikzstyle{dot}=[fill=white, draw=black, shape=circle]
\tikzstyle{box}=[fill=white, draw=black, shape=rectangle]
\tikzstyle{RandomVar}=[fill=white, draw=black, shape=circle]
\tikzstyle{elemt}=[fill={rgb,255: red,0; green,128; blue,128}, draw=black, shape=circle]
\tikzstyle{discard}=[fill=black, draw=black, shape=circle]
\tikzstyle{codiscard}=[fill=red, draw=red, shape=circle]
\tikzstyle{coelemt}=[fill={rgb,255: red,255; green,128; blue,0}, draw=black, shape=circle]

\tikzstyle{arrow}=[->]
\tikzstyle{dashedline}=[-, dashed]
\tikzstyle{doubleline}=[-, double]
\tikzstyle{thickarrow}=[->, very thick]
\tikzstyle{thickline}=[-, very thick]

\newcommand{\boxmargin}{5pt}

\tikzset{discard/.style={ground, scale=1.5, rotate=-90}}
\tikzset{codiscard/.style={ground, scale=1.5, rotate=90}}

\tikzset{elemt/.style={append after command={
   \pgfextra
        \draw[sharp corners, fill=white]%
    ([shift={(-\boxmargin, 0)}]\tikzlastnode.south west)-- ([shift={(+\boxmargin, 0)}]\tikzlastnode.south east) -- ([shift={(0, 10pt)}]\tikzlastnode.north)-- ([shift={(-\boxmargin, 0)}]\tikzlastnode.south west);
   \endpgfextra}}}

\tikzset{coelemt/.style={append after command={
   \pgfextra
        \draw[sharp corners, fill=white]%
         ([shift={(-\boxmargin, 0)}]\tikzlastnode.north west)-- ([shift={(+\boxmargin, 0)}]\tikzlastnode.north east) -- ([shift={(0, -10pt)}]\tikzlastnode.south)-- ([shift={(-\boxmargin, 0)}]\tikzlastnode.north west);
   \endpgfextra}}}

\tikzset{projection/.style={append after command={
   \pgfextra
        \draw[sharp corners, fill=white]%
    ([shift={(-\boxmargin, 0)}]\tikzlastnode.south west)-- ([shift={(+\boxmargin, 0)}]\tikzlastnode.south east) -- ([shift={(0, 10pt)}]\tikzlastnode.north east) -- ([shift={(0, 10pt)}]\tikzlastnode.north west) -- ([shift={(-\boxmargin, 0)}]\tikzlastnode.south west);
   \endpgfextra}}}

\tikzset{tinycircleBlack/.style={append after command={
   \pgfextra
        \draw[fill=black]%
    (\tikzlastnode.center) circle (2mm);
   \endpgfextra}}}

\tikzstyle{dot}=[tinycircleBlack]

\theoremstyle{plain}
\newtheorem{thm}{Theorem}[chapter]
\newtheorem{prop}[thm]{Proposition}
\newtheorem{lemma}[thm]{Lemma}
\newtheorem{cor}[thm]{Corollary}
\newtheorem{claim}[thm]{Claim}
\theoremstyle{definition}
\newtheorem{defs}[thm]{Definition}
\theoremstyle{remark}
\newtheorem{rmk}[thm]{Remark}
\newtheorem{ex}[thm]{Example}

\newenvironment{abstract}{\newpage 

\vspace*{5cm}

{\fontfamily{qag}\huge\bfseries Abstract}

\vspace*{2.5cm}

}{\newpage}

\newenvironment{acknowledgements}{\newpage 

\vspace*{5cm}

{\fontfamily{qag}\huge\bfseries Acknowledgements}

\vspace*{2.5cm}

}{\newpage}

\titleformat{\part}[display]{\huge\fontfamily{qag}\bfseries}{\partname\quad\thepart}{1pc}{\thispagestyle{empty}\HUGE\bfseries\scshape}

\titleformat{\chapter}[block]{\Huge\fontfamily{qag}\bfseries}{\begin{center}\chaptername~\thechapter\end{center}}{1pc}{\centering\huge\bfseries\scshape}[\vspace*{1cm}\hrule]

\titleformat{\section}{\LARGE\fontfamily{qag}\bfseries}{\thesection}{1pc}{\LARGE\bfseries}

\titleformat{\subsection}{\Large\fontfamily{qag}\bfseries}{\thesubsection}{1pc}{\Large\bfseries}

\titleformat{\subsubsection}{\large\fontfamily{qag}\bfseries}{\thesubsubsection}{1pc}{\large\bfseries}

\newcommand{\ket}[1]{\left|#1\right>}
\newcommand{\bra}[1]{\left<#1\right|}

\newcommand{\DR}{\mathcal{D}_{\mathbb{R}_+}}

\title{A Quantum-inspired Analysis of Human Disambiguation Processes}
\author{Daphne Pauline Wang}
\date{}

\begin{document}

\begin{titlepage}
    {\fontfamily{qag}\centering
    \hrule
    \vspace*{1cm}
    
    \textbf{\huge A Quantum-Inspired Analysis of Human Disambiguation Processes}\\
    \vspace*{.5cm}
    \textit{\Large Foundational Theory and Applications}

    \vspace*{1cm}
    \hrule

    \vspace*{3cm}

    {\Large Daphne Pauline Wang}
    
    \vspace*{1cm}

    {\normalsize Thesis submitted as part of:}
    
    \textit{PhD in Computer Science}

    \vfill

    \includegraphics[width=40mm]{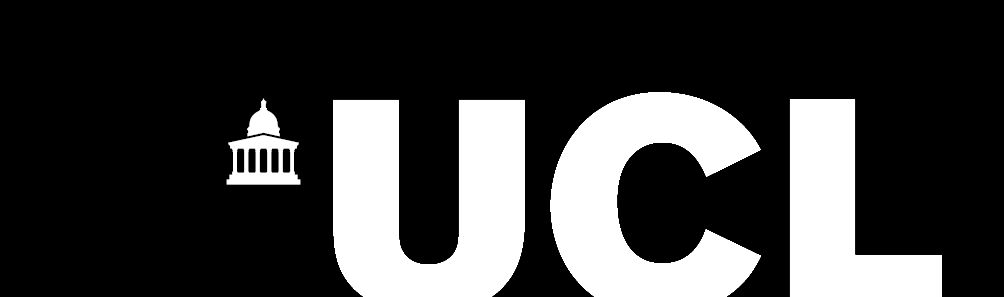}\\
    {\normalsize University College London}

    \vspace*{.5cm}

    }
\end{titlepage}

\pagenumbering{roman}

\begin{abstract}
    Formal languages are essential for computer programming and are constructed to be easily processed by computers. In contrast, natural languages are much more challenging and instigated the field of  Natural Language Processing (NLP). One major obstacle is the ubiquity of ambiguities. Recent advances in NLP  have led to the development of large language models, which can resolve ambiguities with high accuracy. At the same time, quantum computers have gained much attention in recent years as they can solve some computational problems faster than classical computers. This new computing paradigm has reached the fields of machine learning and NLP, where hybrid classical-quantum learning algorithms have emerged. However, more research is needed to identify which NLP tasks could benefit from a genuine quantum advantage. In this thesis, we applied formalisms arising from foundational quantum mechanics, such as contextuality and causality, to study ambiguities arising from linguistics. By doing so, we also reproduced psycholinguistic results relating to the human disambiguation process. These results were subsequently used to predict human behaviour and outperformed current NLP methods. 
\end{abstract}



\begin{acknowledgements}

\paragraph{}First and foremost, I would like to express my deepest gratitude to my supervisor, Prof Mehrnoosh Sadrzadeh, for her invaluable help throughout my PhD. I am incredibly lucky to benefit from knowledge and experience, and she has always encouraged me to achieve more. I could not have wished to have a better supervisor. I would also like to thank my examiners Fabio Zanasi and Rui Soares-Barbosa for their very helpful comments on the thesis. In addition, I would like to give special thanks to the many people who have also guided me through this journey, including Samson Abramsky, Ehtibar Dzhafarov, Shane Mansfield, Ruth Kempson, Wing Yee Chow, and Richard Breheny. I am also grateful to my office mates, Lachlan McPheat, Kin Ian Lo, and Hadi Wazni, for the many discussions we have had (and the countless hours of procrastination).

I also could not have done it without my husband, Joel Ingram, who kept me sane throughout the years. He has been my constant source of comfort and happiness. Additionally, I would like to thank my parents for their incredible influence on my life and emotional support. Finally, I would like to mention my friends and family from home, who have always believed in me more than I believed in myself.

\end{acknowledgements}
\thispagestyle{empty}

\vspace*{5cm}

\begin{flushright}
    \textit{To Popo.}
\end{flushright}

\tableofcontents

\pagestyle{fancy}
\fancyhead{}
\fancyfoot{}
\fancyhead[RE]{\nouppercase{\it\sffamily \rightmark}}
\fancyhead[LE]{\thepage}
\fancyhead[LO]{\nouppercase{\it\sffamily \leftmark}}
\fancyhead[RO]{\thepage}

\chapter*{Introduction}
\pagenumbering{arabic}
\addcontentsline{toc}{chapter}{Introduction}
\markboth{Introduction}{Introduction}
\paragraph{}In this research program, we investigate some properties of the English language using mathematical tools from Quantum Mechanics. We created quantum-inspired models of human disambiguation processes from linguistic data. Using these models, we provide promising evidence that this method leads to novel quantum computing methods for Natural Language Processing.

\subsection*{Computational Linguistics}
\paragraph{} Artificial Intelligence (AI), or how to perform intelligent tasks done by humans algorithmically, is a longstanding challenge in Computer Science. Amongst one of the most widely studied areas in AI is the field of Natural Language Processing (NLP), whose goal is to understand natural language. This field is currently widely dominated by the use of Large Language Models (LLMs), such as \texttt{BERT} or \texttt{GPT}, which consist of artificial neural networks trained over large corpora. These LLMs are incredibly successful in various NLP tasks, such as text generation or knowledge extraction. There are, however, several criticisms of them, including:
\begin{itemize}
    \item \textbf{Lack of reproducibility}, due to the immense resources needed for training;
    \item \textbf{Transparency}, i.e. the impossibility of tracing back the decision process of the neural network;
    \item \textbf{Cognitive plausibility}, i.e. whether these neural networks reproduce the way humans learn and process natural language data.
\end{itemize}

In parallel, computational linguistics aims to use computational tools to study human cognitive processes and natural languages. Computational linguistics and NLP are highly related, and their distinction is increasingly blurred. 

The aim of this thesis is more in line with computational linguistics since we investigate features of the natural language using tools from NLP (e.g. LLMs). The thesis focuses on English, although our approach could be replicated in other languages.

\subsection*{Quantum Computing}
\paragraph{}In 1994, Peter Shor published his famous article describing a quantum algorithm that can factorise an integer in polynomial time~\cite{shor1999polynomial}, thus demonstrating the use of quantum systems to solve hard computational tasks. The discovery of this algorithm sparked the interest of the computer science research community in quantum information theory and quantum computing. The idea behind quantum computations is straightforward. Instead of using bits -- as in classical computing -- information is encoded as \emph{qubits}. Qubits can not only take values the values $\ket{0}$ or $\ket{1}$ but can also be expressed as the (complex) linear combination:
\begin{equation}
    \alpha \ket{0} + \beta \ket{1} \quad \text{such that} \quad \left|\alpha\right|^2 + \left|\beta\right|^2 = 1
\end{equation} 
Similarly, instead of considering operations between strings of bits as computations, we use operations between systems of qubits, which are subject to the laws of quantum mechanics. 

\paragraph{}Since Shor's algorithms, it has been shown that quantum systems are capable of achieving speed-ups in various computational tasks in theory, in tasks such as database search~\cite{Grover}, optimisation tasks~\cite{QAOAAdvantage} or simulation of physical systems~\cite{berry2007efficient,babbush2023exponential}, and more recently in practice, in tasks such as sampling from a random quantum circuit~\cite{Arute2019} or boson sampling~\cite{zhong2020quantum,Madsen2022}. In addition, although quantum advantages are more difficult to prove, quantum computations have started to find applications in various fields of computer science, including optimisation problems (e.g.~\cite{QAOA,QAOAAdvantage}) and AI (e.g.~\cite{Flamini2020, Zeng2016, Trugenberger2002,Rebentrost2014}).

\subsection*{This project}
\paragraph{}

\paragraph{}Most research on quantum computing applications to AI and NLP consists of creating a quantum version of existing algorithms. Therefore, these approaches also suffer from the same issues as the classical approaches, including the need for more transparency and cognitive plausibility. In addition, the advantage of using quantum computing resources is not always clear and usually relies on heuristics. 

In this work, we aim to address these problems for quantum NLP. By studying linguistic data using the formalisms of quantum mechanics, we create a parallel between linguistic phenomena and quantum systems, from which we can identify which features of natural languages would benefit from simulations on quantum hardware. We also show that, by using the mathematical frameworks developed to study quantum mechanics, we can uncover properties of the human disambiguation process.

\paragraph{}This thesis uses the mathematics of \emph{category theory}, in particular the notions of \emph{sheaves} and \emph{presheaves}. The main reason for doing so is the level of abstraction allowed by category theory, allowing us to draw parallels between quantum and linguistic systems. A second motivation is the \emph{categorical quantum mechanics} research project which originated from the seminal paper of Samson Abramsky and Bob Coecke~\cite{Abramsky2004}. Indeed, the line of research showed that quantum mechanics can very elegantly be described in categorical terms (see more details in Chapter~\ref{chap:BackQM}). 

On top of that, this description has also been applied to various aspects of linguistics, notably in the Distributional Compositional Categorical models of meanings (also known as DisCoCat), which emanated from~\cite{DisCoCat}, or as semantics of Discourse Representation Theory~\cite{Abramsky2014semanticUnification}. 

\paragraph{} This thesis is an additional example of the application of categorical quantum mechanics in linguistics to ambiguities in the English language. Ambiguities in English occur at different levels, from words to discourses. This project investigates the disambiguation process of two types of ambiguities, namely:
\begin{itemize}
    \item \textbf{Lexical ambiguity} which happens when a single word has multiple interpretations. For example, the word \textit{bank}, which could refer to a financial institution or the bank of a river;
    \item \textbf{Syntactic ambiguity} which happens when a phrase can have multiple grammatical structures. For example, in the sentence \textit{She saw a man with binoculars}, the phrase \textit{with binoculars} can either attach to the verb \emph{saw} or to the noun-phrase \textit{a man}.
\end{itemize}

In the case of lexical ambiguity, we have studied the statistics of the meaning activations of subject-verb and verb-object phrases, where each word is lexically ambiguous. We then used the mathematical framework arising from the study of quantum contextuality and causality to study these statistics. We were able to show that the observed statistics from lexical ambiguity data show are capable to exhibit quantum-like contextuality. In addition, we were able to rederive some psycholinguistics results that were originally based on eye-tracking data (which is not easily reproducible and expensive to obtain). Using these results, we produced quantum simulations of the disambiguation process of subject-verb and verb-object phrases, which could then be applied to NLP tasks.

Regarding syntactic ambiguity, we also used the sheaf-theoretic frameworks originating from quantum contextuality and causality to create a model of the syntactic parsing process. This model was then used to make reading time predictions in special sentences, known as garden-path sentences. These predictions were closer to the human baseline than the ones obtained from state-of-the-art methods of computational linguistics.


\subsection*{Outline of the thesis}
\paragraph{}The aim of Part~\ref{part:background} is to introduce the concept we will use in the rest of the thesis.
\begin{itemize}
    \item In Chapter~\ref{chap:BackQM}, we introduce the different branches of categorical quantum mechanics that we will use in the subsequent parts. 
    \item In Chapter~\ref{chap:ambiguities}, we introduce the main literature regarding lexical (Section~\ref{sec:lexicalBack}) and syntactic (Section~\ref{sec:SyntacticBackground}) ambiguities. In particular, we will describe the psycholinguistic theories relating to these ambiguity types and the computational tools that have been used in the past in NLP and computational linguistics. 
    
    Elements of Section~\ref{sec:lexicalBack} will be used in Part~\ref{part:Lexical}, while theories introduced in Section~\ref{sec:SyntacticBackground} will be used in Part~\ref{part:Syntactic}.
\end{itemize}

\paragraph{}Parts~\ref{part:Lexical} and~\ref{part:Syntactic} correspond to the original contributions of the thesis. These parts can be read independently.

\paragraph{}In Part~\ref{part:Lexical}, we focus on the lexical disambiguation process. 
\begin{itemize}
    \item We start by studying the features of lexically ambiguous phrases in terms of quantum \emph{contextuality} and \emph{causality} in Chapter~\ref{chap:lexicalFeatures}. 
    \item In Chapter~\ref{chap:lexicalCircuits}, we use the conclusions of Chapter~\ref{chap:lexicalFeatures} to generate a quantum model of the human disambiguation process using variational quantum circuits.
\end{itemize}

\paragraph{}In Part~\ref{part:Syntactic}, we study the human parsing process by looking at \emph{garden-path sentences}.
\begin{itemize}
    \item Inspired by the psycholinguistic theories described in Section~\ref{sec:SyntacticBackground}, we describe our sheaf-theoretic model of the human parsing process in Chapter~\ref{chap:SyntacticModel}.
    \item In Chapter~\ref{chap:GPPred}, we evaluate the models from Chapter~\ref{chap:SyntacticModel} empirically. We then compare the model's predictions with those from state-of-the-art computational linguistics models.
\end{itemize}

\subsection*{Published contributions}

\paragraph{}Several original contributions presented in thesis were published before the submission of this thesis. These are the following:
\begin{itemize}
    \item \underline{Title}: \textit{On the Quantum-like Contextuality of Ambiguous Phrases}\\
    \underline{Authors} (publication order): \textbf{D. Wang}, M. Sadrzadeh, S. Abramsky and V. H. Cervantes\\
    \underline{Published in:} \textit{ACL Anthology} as part of the \textit{Proceedings of the 2021 Workshop on Semantic Spaces at the Intersection of NLP, Physics, and Cognitive Science (SemSpace)}\\
    \underline{Publication date:} 2021\\
    \underline{Material presented in thesis in:} Section~\ref{sec:lexicalContext}
    \item  \underline{Title}: \textit{Analysing Ambiguous Nouns and Verbs with Quantum Contextuality Tools}\\
    \underline{Authors} (publication order): \textbf{D. Wang}, M. Sadrzadeh, S. Abramsky and V. H. Cervantes\\
    \underline{Published in:} \textit{Journal of Cognitive Science}\\
    \underline{Publication date:} 2021\\
    \underline{Material presented in thesis in:} Section~\ref{sec:lexicalSignal}
    \item  \underline{Title}: \textit{Causality and Signalling of Garden-Path Sentences}\\
    \underline{Authors} (publication order): \textbf{D. Wang} and M. Sadrzadeh\\
    \underline{Published in:} \textit{Philosophical Transaction of Royal Society A}\\
    \underline{Publication date:} 2024\\
    \underline{Material presented in thesis in:} Part~\ref{part:Syntactic} (Chapters~\ref{chap:SyntacticModel} \&~\ref{chap:GPPred})
\end{itemize}


\part{Background}\label{part:background}
\chapter{Quantum theory and applications}\label{chap:BackQM}
\paragraph{}This chapter introduces the quantum physics formalisms we will use in Parts~\ref{part:Lexical} and~\ref{part:Syntactic}. In particular, we will make use of the categorical description of quantum mechanics. We will, therefore, start by introducing the mathematics of category theory in Section~\ref{sec:categories}. In Section~\ref{sec:QDescription}, we will describe quantum correlations, notably in terms of sheaf theory, and in Section~\ref{sec:QProcess}, we will give a categorical description of quantum processes. 
\section{A crash course in Category Theory}\label{sec:categories}
\paragraph{}Category theory originated in the work of Eilenberg and MacLane for homological algebra~\cite{Eilenberg1945}. The field of category theory then rapidly evolved and reached other areas of mathematics such as algebraic geometry~\cite{Grothendieck1957}, set theory~\cite{Lawvere1964}, as well as computer science~\cite{lambek1980lambda,fong2018seven} and physics~\cite{Abramsky2004,BaezLauda2011}.

In this section we introduce the basic concepts at the heart of category theory. This introduction is by no means comprehensive, and we will only present the elements that will be useful in subsequent chapters.

\subsection{Basics of Category Theory}

\paragraph{}In this subsection, we start by defining the main notions of category theory, which will be subsequently expanded in Section~\ref{subsec:sheaves} and Section~\ref{subsec:monoidal}.

\begin{defs}[Category]
    A \emph{category} $\mathcal{C}$ consists of:
    \begin{enumerate}
        \item A collection of \emph{objects} denoted as $ob\left(\mathcal{C}\right)$
        \item A set\footnote{We are here only considering locally small categories; in general categories, $\mathcal{C}(A,B)$ coulbe be a proper class.} of \emph{morphisms} for each pair of objects $A,B\in ob\left(\mathcal{C}\right)$, denoted as $\mathcal{C}(A, B)$ equipped with:
        \begin{enumerate}
            \item Sequential composition: for morphisms $f\in \mathcal{C}(A,B)$ and $g \in \mathcal{C}(B, C) $ we have $g \circ f \in \mathcal{C}(A, C)$.
            \item Identity morphism: for each object $A\in ob\left(\mathcal{C}\right)$, there exists a unique morphism $id_A\in \mathcal{C}(A,A)$.
        \end{enumerate}
        satifsfying the following properties:
        \begin{enumerate}
            \item For any $f\in \mathcal{C}(A,B)$, $g \in \mathcal{C}(B, C)$ and $k \in \mathcal{C}(C, D)$:
            \begin{equation}\label{eq:morphTrans}
                k \circ (g\circ f) = (k\circ g) \circ f
            \end{equation} 
            In other words, sequential composition is associative.
            \item For any $f\in \mathcal{C}(A, B)$ and $g\in \mathcal{C}(C, A)$:
            \begin{equation}\label{eq:idUnit}
                f \circ id_A = f \qquad id_A\circ g = g
            \end{equation}
        \end{enumerate}
    \end{enumerate}
\end{defs}

\begin{ex} Here are standard examples of categories:
    \begin{enumerate}[label=\alph*.]
        \item The category of sets and functions denoted as $\mathbf{Sets}$. The objects of $\mathbf{Sets}$ are sets, and morphisms are functions between sets. Composition is the standard functional composition, and identity morphisms are identity functions:
        \begin{equation}
            id_A\in  \mathbf{Sets}(A,A) = a \mapsto a
        \end{equation}
        \item The category of sets and relations denoted as $\mathbf{Rel}$. The objects are the same as the ones of $\mathbf{Sets}$, and the morphisms are binary relations, i.e. $\mathbf{Rel}(A, B) = \mathcal{P}(A\times B)$. The composition of relations $u\in \mathbf{Rel}(A,B)$ and $v\in \mathbf{Rel}(B, C)$ is defined as:
        \begin{equation}
            v\circ u\in \mathbf{Rel}(A,C)  = \left\{(a,c)~\middle|~ \exists b\in B. ~ (a,b)\in u\wedge (b,c)\in v\right\}
        \end{equation}
        The identity morphisms in $\mathbf{Rel}$ are then defined as:
        \begin{equation}
            id_A\in \mathbf{Rel}(A,A) = \left\{(a,a)~\middle|~ a\in A\right\}
        \end{equation}
        \item The category of vector spaces denoted as $\mathbf{Vect}$. Its objects are vector spaces, and morphisms between two vector spaces are linear maps. Composition is the standard composition of (linear) maps. We can also define the subcategory $\mathbf{FdVect}$ of $\mathbf{Vect}$ in which the objects are finite-dimensional vector spaces, and morphisms are defined as in $\mathbf{Vect}$. In $\mathbf{FdVect}$, morphisms can be seen as matrices (by fixing a basis), and composition as matrix multiplication. Similarly, the identity morphisms become identity matrices.
        \item A category is said to be a \emph{preorder} if there is at most one morphism between two objects. This notion generalises the notion of order by taking the existence of a morphism between $A$ and $B$ to represent $A\leq B$. The presence of composition morphisms in a preorder means that the relation $\leq$ is transitive, i.e. if $A\leq B$ and $B\leq C$ then $A\leq C$. Similarly, the existence of identity morphism indicates that the relation $\leq$ is reflexive, i.e. $A\leq A$ for any object $A$.
        \item A preorder is said to be a \emph{partial order} (or a \emph{poset}) iff $A\leq B$ and $B\leq A$ implies that $A = B$. Furthermore, a partial order is a \emph{total order} iff for any pair of distinct objects $A$ and $B$, either $A\leq B$ or $B\leq A$.
    \end{enumerate}
\end{ex}

It is often convenient to denote morphisms $f\in \mathcal{C}(A,B)$ as arrows:
\begin{equation*}
    A \xrightarrow{f} B \qquad\text{or equivalently}\qquad f:A\to B
\end{equation*}
Then, from associativity equation \eqref{eq:morphTrans}, the composition of several morphisms, say $f\in \mathcal{C}(A,B)$, $g\in \mathcal{C}(B,C)$, $h\in \mathcal{C}(C, D)$ and $k\in \mathcal{C}(D, E)$ can unambiguously be written as:
\begin{equation*}
    A \xrightarrow{f} B \xrightarrow{g} C \xrightarrow{h} D \xrightarrow{k} E
\end{equation*}
In addition, identity morphisms can be omitted from equation \eqref{eq:idUnit}.

From this, we can define the notion of a \emph{diagram} in a category $\mathcal{C}$, which corresponds to a (labelled) directed graph such that nodes are objects in $\mathcal{C}$ and (labelled) arrows $A\xrightarrow{f} B$ correspond to the morphism $f\in \mathcal{C}(A,B)$. Paths, therefore, correspond to the composition of morphisms. We then say that a diagram \emph{commutes} iff the paths having the same endpoints are equal.
For example, the commutativity of the following diagram represents the equation $g \circ f = l \circ k\circ h$ :
\begin{equation*}
    \begin{tikzpicture}
        \node (A) at (-1, 1) {$A$};
        \node (B) at (1, 1) {$B$};
        \node (C) at (-1, 0){$C$};
        \node (D) at (1, 0){$D$};
        \node (E) at (0,0){$E$};
        \draw[->] (A) to node[above]{$f$} (B);
        \draw[->] (B) to node[right]{$g$} (D);
        \draw[->] (A) to node[left]{$h$} (C);
        \draw[->] (C) to node[below]{$k$} (E);
        \draw[->] (E) to node[below]{$l$} (D);
    \end{tikzpicture}
\end{equation*}

\paragraph{}In addition, it can be seen that given a set of arrows of a category, reversing all of the arrows still leads to a valid category. This category is known as the \emph{opposite category}.
\begin{defs}
    Given a category $\mathcal{C}$, its \emph{opposite category}, denoted as $\mathcal{C}^{op}$, is defined as follows:
    \begin{itemize}
        \item The objects of $\mathcal{C}^{op}$ are the same as the objects of $\mathcal{C}$;
        \item Each morphism $f\in \mathcal{C}(A,B)$ gives a morphism $\tilde{f}\in \mathcal{C}^{op}(B,A)$. The composition $g\circ f \in \mathcal{C}(A, C)$, where $f\in \mathcal{C}(A,B)$ and $g\in \mathcal{C}(B,C)$, then gives $\widetilde{g\circ f} = \tilde{f}\circ \tilde{g} \in \mathcal{C}^{op}(C,A)$.
    \end{itemize} 
\end{defs}

So far, given a category $\mathcal{C}$, the notion of ``sameness'' of objects is only captured as the equality of objects. However, equality might be too restrictive in general. For example, in $\mathbf{Sets}$, the two sets $A_1 = \left\{0, 1, 2\right\}$ and $A_2 = \left\{1,2,3\right\}$ are different, but they have the same expressive power in the sense that any map $f_1:A_1 \to B$ can be translated into a map $f_2: A_2 \to B$ and conversely. In the category $\mathbf{Sets}$, this notion is conveyed by the existence of a \emph{bijection} between the sets $A_1$ and $A_2$, and this notion extends to an arbitrary category as the notion of \emph{isomorphism}.

\begin{defs}
    A morphism $f: A\to B$ in a category $\mathcal{C}$ is an \emph{isomorphism} if there exists a morphism $f^{-1}: B\to A$ such that:
    \begin{equation}
        f \circ f^{-1} = id_{B} \qquad\text{and}\qquad f^{-1}\circ f = id_A
    \end{equation}
\end{defs}

\begin{ex}
    \begin{enumerate}[label=\alph*.]
        \item As expected, the isomorphisms in $\mathbf{Sets}$ are the bijections.
        \item In $\mathbf{Vect}$, the isomorphisms are the isomorphisms of vector spaces.
        \item In partial orders, the only isomorphisms are the identity morphisms.
    \end{enumerate}
\end{ex}


\paragraph{}Up to now, we have studied categories individually. We now look at relationships \emph{between categories}.

\begin{defs}
    A \emph{functor} $\mathcal{F} : \mathcal{C}\to \mathcal{D}$ between two categories $\mathcal{C}$ and $\mathcal{D}$ is defined as follows:
    \begin{itemize}
        \item For each object $A$ in the category $\mathcal{C}$ gives an object $\mathcal{F}A$ in $\mathcal{D}$ under the action of $\mathcal{F}$.
        \item Similarly, each morphism $f\in \mathcal{C}(A,B)$ gives a morphism $\mathcal{F}f\in \mathcal{D}(\mathcal{F}A, \mathcal{F}B)$ satisfying:
        \begin{align}
            \mathcal{F}id_A =& id_{\mathcal{F}A}\\
            \mathcal{F}(g\circ f) =& \mathcal{F}g\circ \mathcal{F}f
        \end{align}
        for every object $A$ in $\mathcal{C}$, and any arrows $f,g$ in $\mathcal{C}$.
    \end{itemize}
\end{defs}

\begin{ex}
    Let us look at some examples of functors.
    \begin{enumerate}[label=\alph*.]
        \item We can define a functor $F: \mathbf{Sets} \to \mathbf{Rel}$ which sends any set $A\in ob \left(\mathbf{Sets}\right)$ to the same set $A\in ob\left(\mathbf{Rel}\right)$. The action on morphisms in $\mathbf{Sets}$ (i.e. functions), will simply become the equivalent relation on $\mathbf{Rel}$. Namely, for any $f: A\to B$ in $\mathbf{Sets}$, we get:
        \begin{equation*}
            Ff = \left\{(a,f(a))~\middle|~a\in A\right\}: A\to B
        \end{equation*}
        in $\mathbf{Rel}$.
        \item We can define a functor $\mathcal{D}_{\mathbb{R}_+}:\mathbf{Sets} \to \mathbf{Sets}$ which is defined on object as:
        \begin{equation*}
            \mathcal{D}_{\mathbb{R}_+}::
        U \mapsto \left\{d:U\to \mathbb{R}_{+}~\middle|~d\text{ is a probability distribution over }U\right\}
        \end{equation*}
        For any morphism $f: U\to V$, we then define $\mathcal{D}_{\mathbb{R}_+}f$ as:
        \begin{equation*}
            \mathcal{D}_{\mathbb{R}_+}f:: d_U\mapsto d_V\text{ such that }d_V(v) = \sum_{u\in f^{-1}(v)} d_U(u)
        \end{equation*}
        The functor $\mathcal{D}_{\mathbb{R}_+}$ is called the \emph{distribution monad}. 
        \item For any set $A$, we can define the \emph{free monoid} over $A$, denoted as $A^\star$ which correspond to the set of lists of elements in $A$. In other words, $A^\star \cong \amalg_{n\in \mathbb{N}} A^n$. This correspondence can be extended to a functor $F: \mathbf{Sets} \to \mathbf{Mon}$, where $\mathbf{Mon}$ is the category of monoids and monoid homomorphisms. In this functor, any function $f:A\to B$ in $\mathbf{Sets}$ is mapped to:
        \begin{equation*}
            Ff: A^\star \to B^\star :: (a_1, a_2, \ldots) \mapsto (f(a_1), f(a_2), \ldots)
        \end{equation*}
    \end{enumerate}
\end{ex}

\paragraph{}We can, moreover, study relationships between functors by looking at \emph{natural transformations}.
\begin{defs}
    Given two functor $\mathcal{F}, \mathcal{G}: \mathcal{C} \to \mathcal{D}$, a \emph{natural transformation} $\eta: \mathcal{F} \Rightarrow \mathcal{G}$ is a family of maps $\left\{\eta_A: \mathcal{F}A \to \mathcal{G}A\right\}_{A\in ob\left(\mathcal{C}\right)}$ such that for any morphism $f:A\to B$ in $\mathcal{C}$ the following (naturality) square commutes:
    \begin{equation}
        \begin{tikzpicture}
            \node(FA) at (-1,1){$\mathcal{F}A$};
            \node(FB) at (1,1){$\mathcal{F}B$};
            \node(GA) at (-1,-1){$\mathcal{G}A$};
            \node(GB) at (1,-1){$\mathcal{G}B$};
            \draw [->] (FA) to node[above]{$\mathcal{F}f$} (FB);
            \draw [->] (GA) to node[below]{$\mathcal{G}f$} (GB);
            \draw [->] (FA) to node[left]{$\eta_A$} (GA);
            \draw [->] (FB) to node[right]{$\eta_B$} (GB);
        \end{tikzpicture}
    \end{equation}
    In addition, a natural transformation is a \emph{natural isomorphism} if the morphisms $\eta_A$ are isomorphisms for all $A\in ob\left(\mathcal{C}\right)$.
\end{defs}

Using the natural transformations as ``morphisms'' between functors also leads to the notion of functor category.

\begin{defs}[Functor category]
    Given two categories $\mathcal{C}$ and $\mathcal{D}$, we define the \emph{functor category} $\left[\mathcal{C}, \mathcal{D}\right]$ (also written $\mathcal{D}^\mathcal{C}$) where:
    \begin{itemize}
        \item Objects of $\left[\mathcal{C}, \mathcal{D}\right]$ are functors $\mathcal{F}: \mathcal{C} \to \mathcal{D}$.
        \item Morphisms are natural transformations, and compositions are defined pointwise, i.e. $\mu \circ \eta = \left\{\mu_A\circ \eta_A ~\middle|~ A\in ob\left(\mathcal{C}\right)\right\}$.
    \end{itemize}
\end{defs}

\subsection{Sheaf theory}\label{subsec:sheaves}

\paragraph{}In this project, we will make use of sheaves and presheaves. The general idea is that presheaves and sheaves define a mathematical notion of \emph{consistency}. These concepts will be at the core of description of sheaf-theoretic contextuality described in Section~\ref{subsec:BackgroundContext}, and subsequently used in the models developed in Part~\ref{part:Lexical} and~\ref{part:Syntactic}. Here, we review the main definitions of sheaf theory.

\begin{defs}
    Given a category $\mathcal{C}$, we define a \emph{presheaf} over $\mathcal{C}$ as a functor $\mathcal{F}: \mathcal{C}^{op} \to \mathbf{Set}$.
\end{defs}
\begin{rmk}
    The above definition corresponds to the notion of set-valued presheaf. Depending on the school of thought, the term presheaf may also refer to the more general abelian presheaves or presheaves of modules, which take values in abelian groups or modules respectively instead of sets. Here, we will only consider set-valued sheaves and presheaves.
\end{rmk}

\paragraph{}This work will mainly look at presheaves and sheaves over topological spaces. A topological space is usually defined as a tuple $(X, \tau)$ where $X$ is a set of \emph{points} and $\tau\subseteq \mathcal{P}(X)$ is the set of \emph{open sets} which contains the empty set and is closed under arbitrary unions and finite intersections. 
\begin{rmk}
    The class of topological spaces then extends to the category $\mathbf{Top}$ where objects are topological spaces and morphisms are continuous functions between them.
\end{rmk}

Given a topological space $\mathcal{X} = (X, \tau)$, we can define a preorder category $\mathcal{T}(\mathcal{X})$ where $ob\left(\mathcal{T}(\mathcal{X})\right) = \tau$ and morphisms are inclusion relations, i.e. for any two open subsets $U$ and $V$ of $X$, there exists a morphism $V\to U$ iff $V\subseteq U$. A \emph{presheaf over a topological space} $\mathcal{X}$ is then defined as a functor $P: \mathcal{T}(\mathcal{X})^{op}\to \mathbf{Sets}$. These inclusion morphisms $V\to U$ are mapped under the presheaf to \emph{restriction morphisms} $res_V^U: PU\to PV$. For each element of $s\in PU$, we will denote the action of the restriction morphism on $s$ as $\left.s\right|_V$, i.e.:
\begin{equation}
    \begin{matrix}
        res_V^U &:& PU &\to &PV\\
        &::& s &\mapsto& \left.s\right|_V
    \end{matrix}
\end{equation}

\begin{figure}[htb!]
    \centering
    \includegraphics[scale=.4]{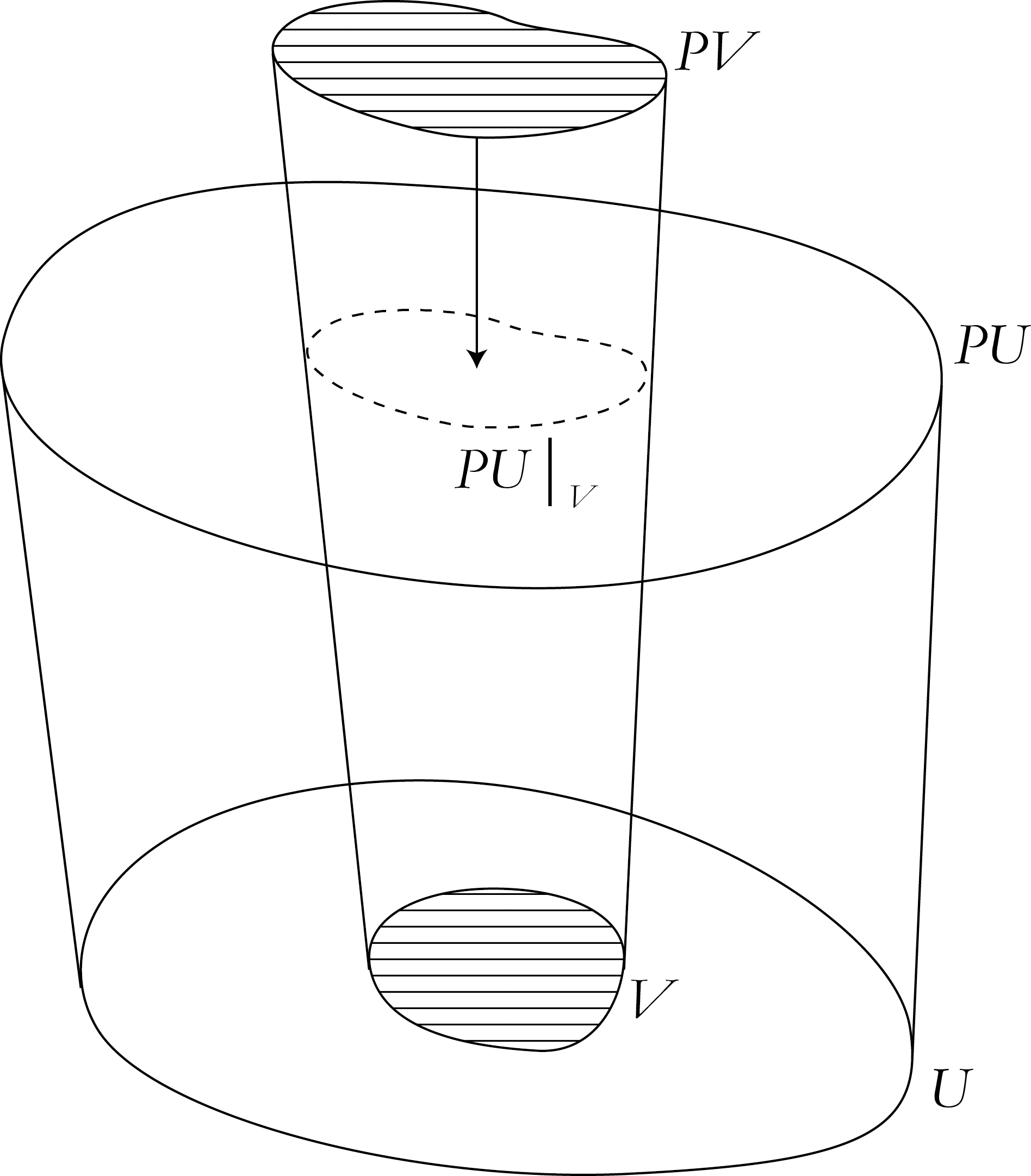}
    \caption{Illustration of the restriction morphsims of a presheaf.\label{fig:presheafRes}}
\end{figure}

Given an open subset $U$, we will define \emph{sections of $P$ at $U$} as the elements of the set $PU$. The intuition is that sections of a presheaf encode \emph{data} over the topological space $\mathcal{X}$. The existence of the restriction morphisms signifies that data of a smaller subset $V$ can be retrieved from the data of a larger set $U$ (see Fig.~\ref{fig:presheafRes}).

\begin{rmk}
    The notion of \emph{section} (or \emph{cross section}) is usually defined in the literature in terms of a \emph{bundle} (i.e. map) $p:E\to X$ as maps $\sigma: U \to E$ such that $U\xrightarrow{\sigma} E \xrightarrow{p} X$ is the inclusion map $U\hookrightarrow X$~\cite{maclane2012sheaves,topoi}. There is, in fact, a one-to-one correspondence between elements of $PU$ and sections of a canonical bundle~\cite{maclane2012sheaves} (see Appendix~\ref{app:sections} for more details).
\end{rmk}


\paragraph{}So far, we have looked at the consistency of data between subsets via restriction maps. We now describe the notion of consistency ``across'' different open sets. Given a presheaf $P$ over a topological space $\mathcal{X}$, we say that there is \emph{a gluing} between two sections $s_U \in PU$ and $s_V\in PV$ of the open sets $U$ and $V$, or equivalent that $s_U$ and $s_V$ are \emph{locally consistent} or \emph{compatible}, iff:
\begin{equation}
    \left.s_U\right|_{U\cap V} = \left.s_V\right|_{U\cap V}
\end{equation}
This gluing condition is illustrated in Fig.\ref{fig:sheafGlue}. The existence of a gluing represents consistency of data at a local level. 

In terms of \emph{global consistency}, we want to be able to define a gluing for every pair of open subsets. The notion of \emph{sheaves} encodes this global consistent condition.

\begin{figure}[htb!]
    \centering
    \includegraphics[scale=.4]{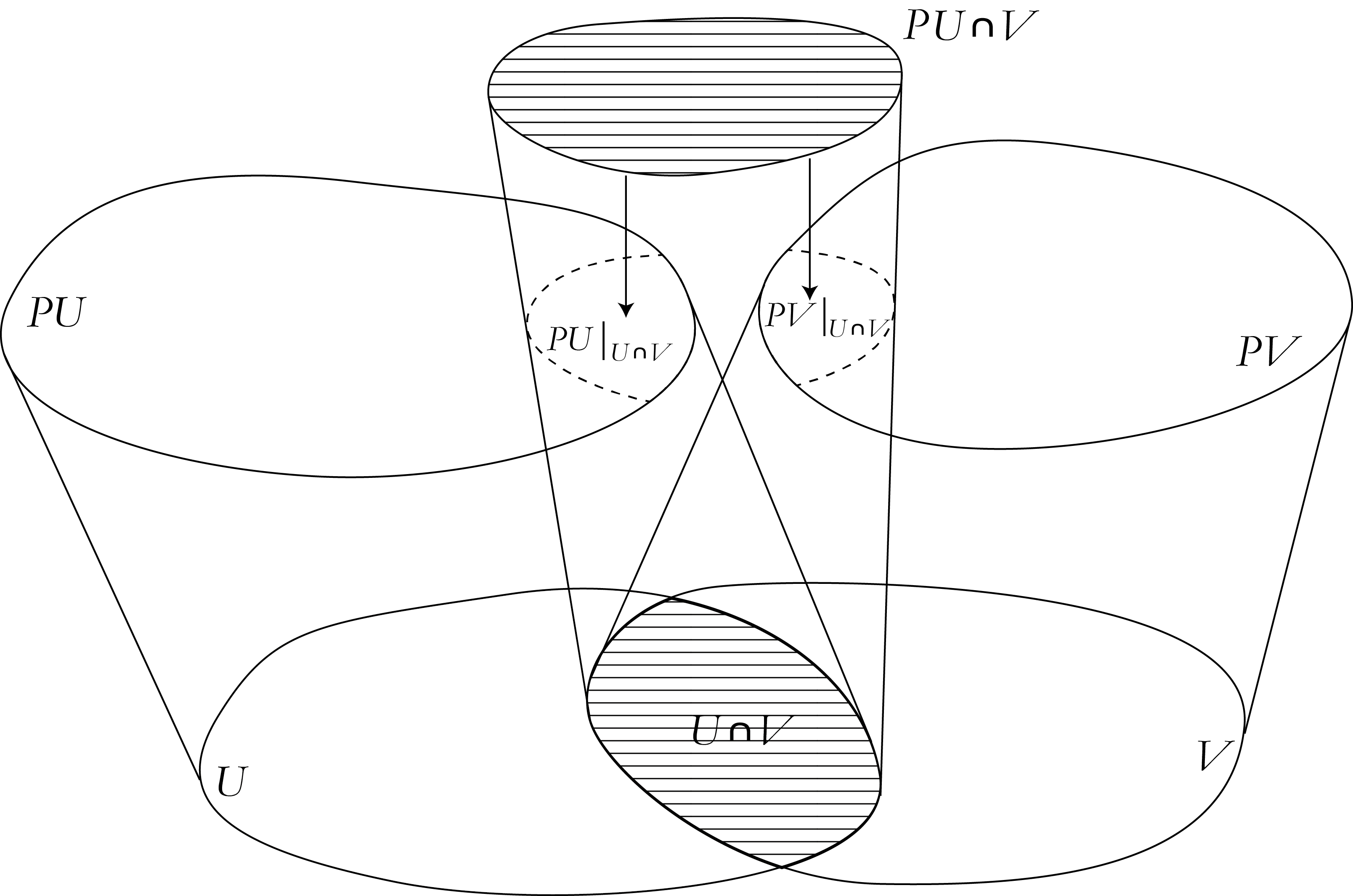}
    \caption{Illustration of the general presheaf structure over intersecting sets. If there exists a gluing between two sections in $PU$ and $PV$, then there will be an intersection between they will coincide in the two dashed regions $\left.PU\right|_{U\cap V}$ and $\left.PV\right|_{U\cap V}$.\label{fig:sheafGlue}}
\end{figure}

\begin{defs}\label{def:sheaf}
    A presheaf $P:\mathcal{T}(\mathcal{X})^{op}\to \mathbf{Sets}$ is a \emph{sheaf} iff for any covering $\left\{U_i\right\}_{i\in I}$ of the space $\mathcal{X}=(X, \tau)$, i.e. $\bigcup_{i\in I} U_i = X$, for every family of pairwise compatible sections $\left\{s_i\in U_i\right\}_{i\in I}$, then there exists a section $s\in PX$ such that:
    \begin{equation}
        s|_{U_i} = s_i
    \end{equation}
    for all $i\in I$.
\end{defs}

\paragraph{}Although many had used concepts similar to sheaves before, the exact notion originates from the work of Jean Leray~\cite{Leray1945}, who introduced them to study equations and transformations from a purely topological perspective -- by getting rid of notions he found unnecessary. Subsequent work of Cartan and Serre then exported the ideas from sheaf theory to algebraic geometry~\cite{cartan1950,serre1955faisceaux}, which was then made categorical by Grothendieck in the seminal article~\cite{Grothendieck1957}. The machinery of sheaf theory was then notably used to prove Weil's conjectures in algebraic geometry~\cite{Grothendieck1965,Deligne}.

The use of sheaves also arose, somewhat independently, from a logical perspective, notably from the work of Lawvere and Tierney~\cite{lawvere1970quantifiers,Tierney2011}. In particular, the category of sheaves (as well as the category of presheaves) forms a \emph{topos}, which can be associated with a logic~\cite{topoi,LambekScott,PTJelephant}. In foundations of mathematics, topos theory has provided alternative (and much simpler) proofs of results from set theory, notably the independence of the Continuum Hypothesis~\cite{TierneyContinuum} and the Axiom of Choice~\cite{FREYDchoice} from the Zermelo-Fr\ae{}nkel axioms. In addition, categories of presheaves and sheaves were found to provide semantics for \emph{intiutionistic logic}, usually referred to as \emph{Kripke-Joyal semantics}~\cite{Kripke1965}. 

Sheaf theory has recently been applied to topological data analysis~\cite{curry2015topological} and quantum mechanics~\cite{AbramskyBrad}. We will discuss the latter in more detail in Section~\ref{subsec:BackgroundContext}.

\subsection{Monoidal categories}\label{subsec:monoidal}

\paragraph{}We now turn our attention to another type of category which we will use in Chapter~\ref{chap:lexicalCircuits}, namely \emph{monoidal categories}. As we will see in Section~\ref{sec:QProcess}, these categories are particularly useful in modeling process theories.

\begin{defs}
    A \emph{monoidal category} is category $\mathcal{C}$ which is equipped with the structure $(\otimes, I, \alpha, \lambda, \rho)$ where:
    \begin{itemize}
        \item The \emph{tensor product} or \emph{monoidal product} $\otimes$ is a bifunctor: $\otimes: \mathcal{C}\times \mathcal{C} \to \mathcal{C}$
        \item $I\in ob\left(\mathcal{C}\right)$ is the \emph{unit object}. The morphisms $e: I\to A$, where $A\in ob\left(\mathcal{C}\right)$, will be called the \emph{elements} of $A$
        \item The \emph{associator} $\alpha$ is a natural isomorphism with elements $\alpha_{A,B,C}: \left(A\otimes B\right)\otimes C \to A \otimes \left(B\otimes C\right)$ for any $A,B,C\in ob\left(\mathcal{C}\right)$
        \item The \emph{left unitor} $\lambda$ is an natural isomorphism with elements $\lambda_A: I\otimes A \to A$
        \item The \emph{right unitor} $\rho$ is an natural isomorphism with elements $\rho_A: A\otimes I \to A$
    \end{itemize}
    such that the following diagrams commute:
    \begin{gather}
        \begin{tikzpicture}
            \node(left) at (-2, 0.75){$(A\otimes I)\otimes B$};
            \node (right) at (2, 0.75){$A\otimes(I\otimes B)$};
            \node (mid) at (0, -0.75){$A\otimes B$};
            \draw [->] (left) to node[above]{$\alpha_{A,I,B}$} (right);
            \draw [->] (left) to node[left]{$\rho_A\otimes id_B$} (mid);
            \draw [->] (right) to node[right]{$id_A\otimes\lambda_B$} (mid);
        \end{tikzpicture}\label{eq:triangleEq}\\
        \begin{tikzpicture}
            \node (1) at (-2.75, 0){$((A\otimes B)\otimes C)\otimes D$};
            \node (2a) at (0, 1.25){$(A\otimes B) \otimes (C\otimes D)$};
            \node (final) at (2.75, 0){$A\otimes (B\otimes (C\otimes D))$};
            \node (2b) at (-2.75, -1.25){$(A\otimes (B\otimes C))\otimes D$};
            \node (3b) at (2.75, -1.25){$A\otimes ((B\otimes C)\otimes D)$};
            \draw [->] (1) to node[left]{$\alpha_{A\otimes B, C, D}$} (2a);
            \draw [->] (2a) to node[right]{$\alpha_{A, B, C\otimes D}$} (final);
            \draw [->] (1) to node[left]{$\alpha_{A,B,C}\otimes id_D$} (2b);
            \draw [->] (2b) to node[below]{$\alpha_{A, B\otimes C, D}$} (3b);
            \draw[->] (3b) to node[right]{$id_A\otimes \alpha_{B,C,D}$} (final);
        \end{tikzpicture}\label{eq:pentagonEq}
    \end{gather}
    The equations \eqref{eq:triangleEq} and \eqref{eq:pentagonEq} are referred to as the \emph{triangle} and \emph{pentagon} equations respectively.
\end{defs}

\begin{ex}Here are some examples of monoidal categories
    \begin{enumerate}[label=\alph*.]
        \item The category $\mathbf{Sets}$ is a monoidal category where the tensor product is the standard cartesian product of sets, and the unit is the singleton set $I=\{\star\}$. Since for any sets $A, B, C$, we have $(A\times B) \times C = A \times (B\times C)$, the associator consists of identity morphisms. In addition, we have $\left\{\star\right\}\times A\simeq A$ for any set $A$ and the the left unitor is defined as:
        \begin{equation*}
            \lambda_A:: \left(\star, a\right)\in I\times A \mapsto a \in A \qquad \lambda_A^{-1}:: a\in A \mapsto (\star, a) \in I\times A 
        \end{equation*}
        The right unitor can be defined similarly. Elements of an object $A$ will correspond to the elements $a$ of the set $A$. 
        \item Similarly to the category of $\mathbf{Sets}$, the category of sets and relations $\mathbf{Rel}$ is also monoidal, where the monoidal product is also the cartesian product. The associator is, as in $\mathbf{Sets}$, simply consisting of identities and the left and right unitors are defined as:
        \begin{align}
            \lambda_A=\left\{\left(\left(\star, a\right), a\right)\right\} \subseteq \left(I\times A\right) \times A \qquad \rho_A = \left\{(\left(a,\star\right), a)\right\} \subseteq \left(A\times I\right) \times A
        \end{align}
        Elements of an object $A \in ob\left(\mathbf{Rel}\right)$ will be isomorphic to subsets of $A$.
        \item The category of finite dimensional vector spaces $\mathbf{FdVect}$ is also a monoidal category. The standard choice of monoidal product is the \emph{tensor product} $\otimes$. For example, given two vector spaces $V$ and $W$ of dimension $n$ and $m$ respectively, the vector space $V\otimes W$ will have dimension $n\times m$. Furthermore, for each pair of linear maps $\mathbf{A}: U\to V$ and $\mathbf{B}: W\to X$ which can be seen as matrices:
        \begin{equation*}
            \mathbf{A} = \begin{pmatrix}
                a_{11} & a_{12} &\ldots &a_{1k}\\
                a_{21} & a_{22} &\ldots &a_{2k}\\
                \vdots & \vdots & \ddots& \vdots\\
                a_{l1} & a_{l2} &\ldots &a_{lk}
            \end{pmatrix}\qquad\mathbf{B} = \begin{pmatrix}
                b_{11} & b_{12} &\ldots &b_{1m}\\
                b_{21} & b_{22} &\ldots &b_{2m}\\
                \vdots & \vdots & \ddots& \vdots\\
                b_{n1} & b_{n2} &\ldots &b_{nm}
            \end{pmatrix}
        \end{equation*} 
        then the matrix corresponding to $\mathbf{A}\otimes \mathbf{B}$ will correspond to:
        \begin{equation*}
            \mathbf{A}\otimes \mathbf{B} = \begin{pmatrix}
                a_{11}\mathbf{B} & a_{12}\mathbf{B} &\ldots &a_{1k}\mathbf{B}\\
                a_{21}\mathbf{B} & a_{22}\mathbf{B} &\ldots &a_{2k}\mathbf{B}\\
                \vdots & \vdots & \ddots& \vdots\\
                a_{l1}\mathbf{B} & a_{l2}\mathbf{B} &\ldots &a_{lk}\mathbf{B}
            \end{pmatrix} = \begin{pmatrix}
                a_{11}b_{11} & a_{12}b_{11} &\ldots & a_{11}b_{1m}\\
                a_{21}b_{21} & a_{22}b_{11} &\ldots & a_{21}b_{2m}\\
                \vdots & \vdots & \ddots& \vdots\\
                a_{l1}b_{n1} & a_{l2}b_{n1} &\ldots &a_{lk}b_{nm}\end{pmatrix}
        \end{equation*}
    \end{enumerate}
\end{ex}

\paragraph{}It is very convenient to use \emph{string diagrams} to represent objects and morphisms in a monoidal category. Generally speaking, each string diagram corresponds to an equivalence class of morphisms with respect to certain\footnote{In general, equality of string diagrams is defined in terms of strict monoidal categories or strict symmetric monoidal categories, see~\cite{CategoriesWorkingMathematician} and~\cite{PiedeleuZanasi} for respective definitions. We also note that this is reasonable as every (symmetric) monoidal category is monoidally equivalent to a strict (symmetric) monoidal category~\cite{CategoriesWorkingMathematician,WilsonGhicaZanasi}.} isomorphisms in $\mathcal{C}$. The objects of a monoidal category, as well as the identity morphisms, will be represented by wires:
\begin{equation*}
    \begin{tikzpicture}
	\begin{pgfonlayer}{nodelayer}
		\node [style=none] (0) at (0, 1) {};
		\node [style=none] (1) at (0, -1) {};
		\node [style=none] (2) at (0.5, -0.75) {$A$};
	\end{pgfonlayer}
	\begin{pgfonlayer}{edgelayer}
		\draw (0.center) to (1.center);
	\end{pgfonlayer}
\end{tikzpicture}

\end{equation*}
with the exception of the unit object $I$, which will generally not be represented (i.e., corresponds to the empty wire). Morphisms will be represented by boxes, e.g.:
\begin{equation*}
    \begin{tikzpicture}
	\begin{pgfonlayer}{nodelayer}
		\node [style=none] (0) at (0, 2) {};
		\node [style=none] (1) at (0, -1.75) {};
		\node [style=none] (2) at (0.5, 1.75) {$A$};
		\node [style=box] (3) at (0, 0) {$f$};
		\node [style=none] (4) at (0.5, -1.5) {$B$};
	\end{pgfonlayer}
	\begin{pgfonlayer}{edgelayer}
		\draw (0.center) to (1.center);
	\end{pgfonlayer}
\end{tikzpicture}
 = f: A\to B
\end{equation*}
As we mentioned before, the identity morphisms will be represented by the wires themselves. Sequential composition of morphisms will be represented by ``stacking'' boxes, e.g.:
\begin{equation*}
    \input{Background/tikzit/sequentialComp.tikz} \quad=\quad A\xrightarrow{f} B \xrightarrow{g} C
\end{equation*}

\begin{rmk}
    Many conventions can be adopted regarding the direction of composition in string diagrams. Here, we will assume that sequential composition happens from top to bottom.
\end{rmk}

Let us now look at the diagrammatic view of the monoidal structure. We will represent the tensor product of two objects as stacking wires \emph{in parallel}:
\begin{equation*}
    \input{Background/tikzit/AotimesB.tikz}
\end{equation*}
Similarly, the action of the tensor product on morphisms will also be represented by concatenating the boxes in parallel, e.g.:
\begin{equation*}
    \input{Background/tikzit/fotimesg.tikz}
\end{equation*}
In particular, from the associator being a natural isomorphism coupled with the pentagon equation \eqref{eq:pentagonEq}, it is not necessary to specify in which order the objects are being tensored as they will all lead to isomorphic objects (and therefore the same string diagram); consequently, we also do not need to represent morphisms $\alpha_{A,B,C}$. Similarly, the representation of the unit object $I$ as the empty wire is motivated by the existence of the left and right unitors, as well as the triangle equation \eqref{eq:triangleEq}, i.e.:
\begin{equation*}
    \begin{tikzpicture}
	\begin{pgfonlayer}{nodelayer}
		\node [style=none] (0) at (0, 1) {};
		\node [style=none] (1) at (0, -1) {};
		\node [style=none] (2) at (0.5, -0.75) {$A$};
	\end{pgfonlayer}
	\begin{pgfonlayer}{edgelayer}
		\draw (0.center) to (1.center);
	\end{pgfonlayer}
\end{tikzpicture}

\end{equation*}
would represent objects $A$, $I\otimes A$ and $A\otimes I$ alike.

In addition, \emph{elements} $e: I \to A$ will be denoted as triangles in string diagrams, namely:
\begin{equation*}
    \begin{tikzpicture}
	\begin{pgfonlayer}{nodelayer}
		\node [style=elemt] (0) at (0, 1) {$e$};
		\node [style=none] (1) at (0, -1) {};
		\node [style=none] (2) at (0.5, -0.75) {$A$};
	\end{pgfonlayer}
	\begin{pgfonlayer}{edgelayer}
		\draw (0) to (1.center);
	\end{pgfonlayer}
\end{tikzpicture}
= e: I\to A
\end{equation*}

\paragraph{}Monoidal categories can be seen as process theories. Indeed, taking objects of the category to be systems or entities and morphisms to be processes, we can see the sequential composition of morphisms as the evolution of the entities through the sequence of processes. Similarly, the tensor product of morphisms will correspond to the execution of independent processes in parallel. Then, the string diagrams are graphical representations of the interactions of different systems under the action of some processes.

\paragraph{Symmetric monoidal categories}So far, the definition of a monoidal category only describes a very general process theory. In order to describe process theories satisfying some extra properties, we define more and more specialised monoidal categories by adding additional structure or axioms. 

We start by defining a \emph{symmetric monoidal category}, which is a monoidal category $\mathcal{C}$ equipped with a natural isomorphism with elements $\sigma_{A, B}: A\otimes B \to B\otimes A$ such that :
\begin{equation}
    \sigma_{B,A}\circ \sigma_{A, B} = id_{A\otimes B} 
\end{equation}

In terms of string diagrams, we represent the isomorphisms $\sigma_{A,B}$ as follows:
\begin{equation*}
    \input{Background/tikzit/crossing.tikz}
\end{equation*}
In addition, by naturality of $\sigma$, this implies that ``morphisms slides through the crossings'', i.e.:
\begin{equation}
    \input{Background/tikzit/naturality_cross_LHS.tikz} = \input{Background/tikzit/naturality_cross_RHS.tikz} \quad \iff \quad \begin{tikzpicture}[baseline=-.2pt]
        \node (A) at (-1.25, 1){$A\otimes B$};
        \node (B) at (1.25, 1){$C\otimes D$};
        \node (C) at (-1.25, -1){$B\otimes A$};
        \node (D) at (1.25, -1){$C\otimes D$};
        \draw [->] (A) to node[above]{$f\otimes g$} (B);
        \draw [->] (B) to node[right]{$\sigma_{C,D}$} (D);
        \draw [->] (A) to node[left]{$\sigma_{A,B}$} (C);
        \draw [->] (C) to node[below]{$g\otimes f$} (D);
    \end{tikzpicture}
\end{equation}
As these morphisms are isomorphisms, two diagrams are equivalent in a symmetric monoidal category if they can be transformed into each other by crossing or uncrossing wires. 

\paragraph{Dual objects}We now introduce another property that monoidal categories can have which will become very useful in Section~\ref{sec:QProcess}, namely dual objects. 

\begin{defs}(Dual objects)
    An object $R$ in a monoidal category $\mathcal{C}$ has a \emph{left-dual} $L\in ob\left(\mathcal{C}\right)$, or equivalently $L$ has the \emph{right-dual} $R$ iff there exists morphisms $\eta: I \to R\otimes L$, known as the \emph{unit}, and $\epsilon: L\otimes R\to I$, known as the \emph{counit}, such that the following diagrams commute:
    \begin{align}
        \begin{tikzpicture}[baseline=-.2pt]
            \node(ini) at (-3, 1){$L$};
            \node (a) at (0, 1){$L\otimes I$};
            \node (b) at (3, 1){$L\otimes (R\otimes L)$};
            \node (c) at (3, -1){$(L\otimes R)\otimes L$};
            \node (d) at (0, -1){$I\otimes L$};
            \node (fin) at (-3, -1){$L$};
            \draw [->] (ini) to node[above]{$\rho^{-1}_R$} (a);
            \draw [->] (a) to node[above]{$id_L\otimes \eta$} (b);
            \draw [->] (b) to node[right]{$\alpha^{-1}_{L,R,L}$} (c);
            \draw [->] (c) to node[below]{$\epsilon\otimes id_L$} (d);
            \draw [->] (d) to node[below]{$\lambda_L$} (fin);
            \draw [->] (ini) to node[left]{$id_L$} (fin);
        \end{tikzpicture}\label{eq:snakeL}\\
        \begin{tikzpicture}[baseline=-.2pt]
            \node(ini) at (-3, 1){$R$};
            \node (a) at (0, 1){$I\otimes R$};
            \node (b) at (3, 1){$(R\otimes L) \otimes R$};
            \node (c) at (3, -1){$R\otimes (L\otimes R)$};
            \node (d) at (0, -1){$R\otimes I$};
            \node (fin) at (-3, -1){$R$};
            \draw [->] (ini) to node[above]{$\lambda^{-1}_L$} (a);
            \draw [->] (a) to node[above]{$\eta\otimes id_R$} (b);
            \draw [->] (b) to node[right]{$\alpha_{R,L,R}$} (c);
            \draw [->] (c) to node[below]{$id_R\otimes \epsilon$} (d);
            \draw [->] (d) to node[below]{$\rho_R$} (fin);
            \draw [->] (ini) to node[left]{$id_R$} (fin);
        \end{tikzpicture}\label{eq:snakeR}
    \end{align} 
    In terms of string diagrams, it is useful to represent dual objects by decorating the wires with arrows going in the opposite direction, e.g.:
    \begin{equation}
        L = \begin{tikzpicture}
	\begin{pgfonlayer}{nodelayer}
		\node [style=none] (0) at (0, 0) {};
		\node [style=none] (1) at (0, 1) {};
		\node [style=none] (2) at (0, -1) {};
		\node [style=none] (3) at (0.5, -0.75) {$L$};
	\end{pgfonlayer}
	\begin{pgfonlayer}{edgelayer}
		\draw [style=arrow] (1.center) to (0.center);
		\draw (0.center) to (2.center);
	\end{pgfonlayer}
\end{tikzpicture}
 \qquad R = \begin{tikzpicture}
	\begin{pgfonlayer}{nodelayer}
		\node [style=none] (0) at (0, 0) {};
		\node [style=none] (1) at (0, -1) {};
		\node [style=none] (2) at (0, 1) {};
		\node [style=none] (3) at (0.5, -0.75) {$R$};
	\end{pgfonlayer}
	\begin{pgfonlayer}{edgelayer}
		\draw [style=arrow] (1.center) to (0.center);
		\draw (0.center) to (2.center);
	\end{pgfonlayer}
\end{tikzpicture}

    \end{equation}
    In addition, by representing the morphisms $\eta$ and $\epsilon$ as:
    \begin{equation}
        \eta = \input{Background/tikzit/eta.tikz} \qquad \epsilon = \input{Background/tikzit/epsilon.tikz} 
    \end{equation}
    then equations \eqref{eq:snakeL} and \eqref{eq:snakeR} become respectively:
    \begin{align}
        \input{Background/tikzit/snakeL.tikz} = \begin{tikzpicture}
	\begin{pgfonlayer}{nodelayer}
		\node [style=none] (0) at (0, 0) {};
		\node [style=none] (1) at (0, 2) {};
		\node [style=none] (2) at (0, -2) {};
		\node [style=none] (3) at (0.5, -1.75) {$L$};
	\end{pgfonlayer}
	\begin{pgfonlayer}{edgelayer}
		\draw [style=arrow] (1.center) to (0.center);
		\draw (0.center) to (2.center);
	\end{pgfonlayer}
\end{tikzpicture}
\\
        \input{Background/tikzit/snakeR.tikz}  = \begin{tikzpicture}
	\begin{pgfonlayer}{nodelayer}
		\node [style=none] (0) at (0, 0) {};
		\node [style=none] (1) at (0, -2) {};
		\node [style=none] (2) at (0, 2) {};
		\node [style=none] (3) at (0.5, -1.75) {$R$};
	\end{pgfonlayer}
	\begin{pgfonlayer}{edgelayer}
		\draw [style=arrow] (1.center) to (0.center);
		\draw (0.center) to (2.center);
	\end{pgfonlayer}
\end{tikzpicture}

    \end{align}
\end{defs}

\begin{ex} The dual of a finite-dimensional vector space $V\in ob\left(\mathbf{FdVect}\right)$ is the standard \emph{dual space} of functions to the field $V$ is defined over, i.e. $V^* = \left\{f: V\to I \right\}$ or equivalently, looking at $V = span\left(\left\{\mathbf{e}\in \mathcal{B}\right\}\right)$ as the span vectors $\mathbf{e}$ in an orthonormal basis $\mathcal{B}$, the space $V^* = span\left(\left\{\mathbf{e}^T\mid \mathbf{e}\in \mathcal{B}\right\}\right)$ is spanned by the transpose of all of the vectors in $\mathcal{B}$. This space is both the left- and right-dual of $V$ and is isomorphic to $V$. Having fixed the orthonormal basis $\mathcal{B} = \left\{\mathbf{e}_i| i\in \left\{1, \ldots, dim V\right\}\right\}$, we can define the unit and counit as:
    \begin{equation}
        \eta:: x \mapsto x \sum_i \mathbf{e}_i\otimes \mathbf{e}^T_i \qquad \epsilon:: \sum_{i,j} \alpha_{i,j} \mathbf{e}^T_i \otimes \mathbf{e}_j \mapsto \sum_i \alpha_{i,i}
    \end{equation}
\end{ex}

\begin{rmk}
If the left- or right-dual of an object $A$ in a monoidal category is isomorphic to $A$ itself, we will drop the arrow decoration on the wire corresponding to $A$.
\end{rmk}

If every object in a category has a right-dual, the category is said to be a \emph{rigid} or \emph{autonomous category}. A rigid category $\mathcal{C}$ moreover gives rise to a \emph{dualising functor} $\_^*: \mathcal{C}\to \mathcal{C}^{op}$, such that the action of objects $A\in ob\left(\mathcal{C}\right)$ gives the right-dual $A^*\in ob\left(\mathcal{C}\right)$ of $A$, and the action on morphisms $f:A\to B$ is defined as:
\begin{equation}
    \input{Background/tikzit/boxarrow.tikz} = \input{Background/tikzit/dualbox.tikz}
\end{equation}

If a monoidal category $\mathcal{C}$ is both symmetric and rigid, it is said to be a \emph{compact-closed category}.

\paragraph{Dagger structures} The dualising functor is not the only interesting example of a functor $\mathcal{C}\to \mathcal{C}^{op}$. Indeed, the notion of \emph{dagger functor} $\dagger: \mathcal{C}\to \mathcal{C}^{op}$ will become useful in the description of Hilbert spaces to encode the notion of inner products.

\begin{defs}
    For any category $\mathcal{C}$, we define a \emph{dagger functor} as a functor $\dagger: \mathcal{C}\to \mathcal{C}^{op}$ such that for any morphim $f$ in $\mathcal{C}$, $\left(f^\dagger\right)^\dagger = f$. A category with a dagger functor is a \emph{dagger category}. Respectively, a monoidal category endowed with a dagger functor is called a \emph{dagger monoidal category}.
\end{defs}

\begin{ex}
    We finally introduce the category of Hilbert spaces $\mathbf{Hilb}$ and its subcategory $\mathbf{FdHilb}$ of finite-dimensional Hilbert spaces. In these categories, the objects are Hilbert spaces $\mathcal{H}$, i.e. complex vector spaces equipped with an inner product $\left<\_\middle|\_\right>_{\mathcal{H}}: \mathcal{H}\times \mathcal{H} \to \mathbb{C}$ such that $d(x, y) = \sqrt{\left<x-y\middle|x-y\right>}_{\mathcal{H}}$ is a metric. The morphisms in $\mathbf{Hilb}$ and $\mathbf{FdHilb}$ will be taken to be \emph{bounded} (or equivalently continuous with respect to the topology induced by the metric) linear maps between Hilbert spaces. Then, we define the action of $\dagger$ on objects to return the same Hilbert space, i.e. $\mathcal{H}^\dagger = \mathcal{H}$ for any $\mathcal{H}\in ob\left(\mathbf{Hilb}\right)$ (resp. $\mathbf{FdHilb}$) and for any bounded linear map $f: \mathcal{H}\to \mathcal{K}$, we take $f^\dagger : \mathcal{K}\to \mathcal{H}$ to be the unique map such that:
    \begin{equation}
        \left<f(\mathbf{v})\middle|\mathbf{w}\right>_\mathcal{K} = \left<\mathbf{v}\middle|f^\dagger (\mathbf{w})\right>_\mathcal{H}
    \end{equation}
    for any $\mathbf{v}\in \mathcal{H}$ and $\mathbf{w}\in \mathcal{K}$. Treating morphisms in $\mathbf{FdHilb}$ as complex matrices, taking the dagger correspond to taking the Hermitian conjugate of the matrix, i.e. given $\mathbf{A}: \mathcal{H} \to \mathcal{K}$, $\mathbf{A}^\dagger = \left(\mathbf{A}^T\right)^* = \left(\mathbf{A}^*\right)^T$.

    A similar definition of a dagger functor can also be obtained for $\mathbf{FdVect}$, but will then depend on a choice of inner product.
\end{ex}

\paragraph{}Finally, in monoidal categories coming with both a dualising functor and a dagger functor, we may require some additional conditions on the interaction between the two structures.

\begin{defs}
    A compact closed category (i.e. symmetric and autonomous) with a dagger functor is a \emph{dagger-compact category} whenever the following hold for any pair of duals $(L,R)$:
    \begin{equation}
        \left(\input{Background/tikzit/epsilon.tikz}\right)^\dagger = \input{Background/tikzit/epsilondagger.tikz} \qquad \left(\input{Background/tikzit/eta.tikz}\right)^\dagger = \input{Background/tikzit/etadagger.tikz}
    \end{equation}
\end{defs}

\begin{ex}
    The categories $\mathbf{Hilb}$ and $\mathbf{FdHilb}$ can be seen to be dagger-compact by taking the duals to be the same as the ones defined earlier for vector spaces and the dagger functor defined above. 
\end{ex}

\section{Describing Quantum Correlations}\label{sec:QDescription}
\paragraph{}The aim of this section is to introduce the concept of quantum contexutality, and in particular the framework of the sheaf-theoretic contextuality (introduced in details in Section~\ref{subsec:sheafContextuality}), which will be a recurrent theme of the work described in Part~\ref{part:Lexical} and~\ref{part:Syntactic}. The framework of Contextuality-by-Default is also introduced in Section~\ref{subsec:CbD}, which will be widely used in Chapter~\ref{chap:lexicalFeatures}.

\paragraph{}The inherent probabilistic nature of quantum mechanics has been a longstanding source of debate. Namely, is nature non-deterministic, or is the apparent randomness due to our lack of knowledge about the observed system? This question has led to the development of theory-agnostic descriptions of observations from quantum systems. That is, only assuming classical probability theory, is it possible to describe the observed statistics? Or are the statistical correlations intrinsically non-classical (i.e. different from probabilistic classical physical systems)? These questions are answered by studying the \emph{contexutality} of quantum systems.

We first introduce the standard formalisms of contextuality (Section~\ref{subsec:BackgroundContext}), and then describe its categorical equivalent and its various extensions (Section~\ref{subsec:sheafContextuality}). In Section~\ref{subsec:CbD}, we introduce an alternative framework to the one of Section~\ref{subsec:sheafContextuality}.

In terms of notation, we will also use the standard \emph{Dirac notation} where vectors in a Hilbert space will be denoted as $\ket{\psi}\in \mathcal{H}$, their Hermitian conjugate will be denoted as $\bra{\psi}$, and the inner product of two vectors $\ket{\psi}, \ket{\phi} \in \mathcal{H}$ will be denoted as $\left<\phi\middle|\psi\right>$.

\subsection{Contextuality}\label{subsec:BackgroundContext}
\paragraph{}The initial criticism of Einstein, Podolsky, and Rosen~\cite{EPR} was that the probabilistic nature of quantum mechanics can only be due to the \emph{incompleteness} of quantum theory. This is known as the EPR paradox. Hence, by ``completing'' the description of the system with additional (unobserved) variables, one could obtain a deterministic system from which we can recover the observed statistics~\cite{EPR}. The main argument from~\cite{EPR} is that any physical theory should satisfy \emph{realism}, i.e. every physical quantity, such as the position or the momentum of a particle, should possess a definite value at any given time which should not depend on whether it is observed or not. 

\subsubsection{Non-locality}
\paragraph{}The first and most widely known counterargument of the EPR paradox is attributed to John Bell~\cite{Bell1966}. In reality, Bell's theorem uses an assumption not made explicit in~\cite{EPR}, namely that spatially separated systems cannot influence each other. This requirement is known as \emph{no-signalling}. 

\paragraph{}We here describe an operational view of Bell's theorem due to Fine~\cite{Fine1982}. Let's consider an experiment such that a party Charlie prepares a bipartite state $\ket{\Psi}$ where one subsystem is sent to Alice, the other to Bob, such that Alice and Bob are assumed to be so far away that they cannot influence each other in the time frame of the experiment (see Fig.~\ref{fig:causDiagBell}). Then, both Alice and Bob randomly (and independently) choose to measure a physical quantity on their subsystems, say in the respective sets $\{a,a'\}$ and $\{b,b'\}$. Finally, each physical quantity will take values in $\pm 1$. 

\begin{figure}[htb!]
    \centering
    \includegraphics[width=.75\linewidth]{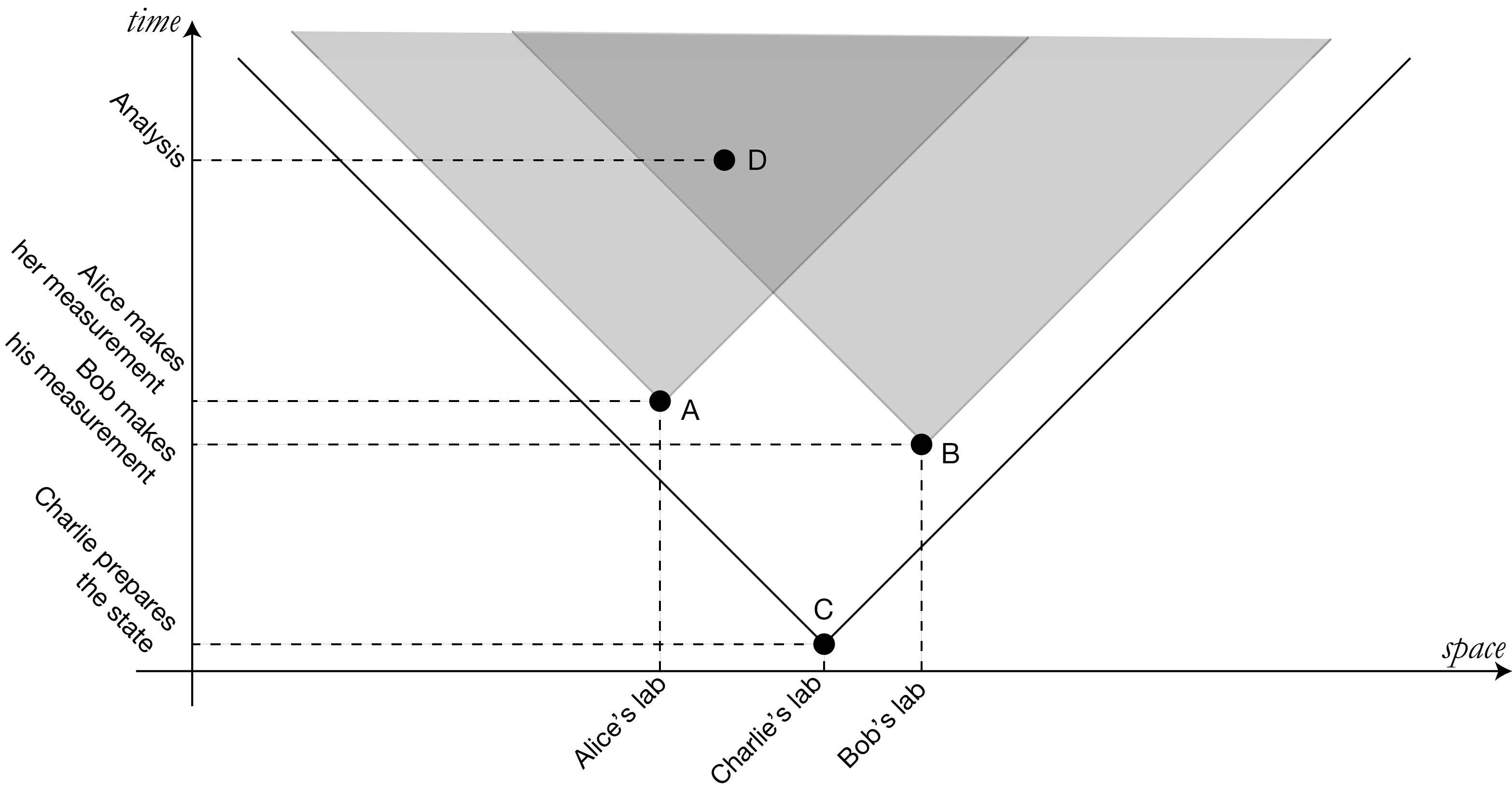}
    \caption{Causal diagram of a Bell experiment. Events are represented as dot and future light-cones are represented as triangles.\label{fig:causDiagBell}}
\end{figure}

\paragraph{}From the realism condition, we will require that, given a hidden-variable $\lambda\in\Lambda$, the values of physical quantities $a, a', b, b'$ are uniquely (and deterministically) determined. In addition, from the no-signalling condition, we will require that the outcomes of $a,a'$ will not depend on the choice of Bob, and respectively, the outcomes of $b$ or $b'$ are independent of Alice's choice. Therefore, using both of these requirements, we can define functions $A: \{a,a'\}\times \Lambda\to \left\{\pm 1\right\}$ and $B: \{b,b'\}\times \Lambda\to \left\{\pm 1\right\}$ which associate the value of the physical quantities accessible from Alice and Bob's measurements given the value of a hidden variable $\Lambda$\footnote{Note that we can assume without loss of generality that we only have a single hidden variable $\Lambda$. The case of multiple hidden variables $\Lambda_1, \Lambda_2,\ldots$ can be reduced to a single hidden-variable model by taking the joint distributions over $\prod_{i} \Lambda_i$.}. 

Using these conditions, it is possible to show that:
\begin{equation}\label{eq:CHSH}
    \left|\left<ab\right> + \left<ab'\right> + \left<a'b\right> - \left<a'b'\right>\right| \leq 2 
\end{equation}
The proof of this can be found in Appendix~\ref{app:CHSH}. This inequality is widely known as the \emph{CHSH inequality} (for Clauser-Horne-Shimony-Holt who first proved it)~\cite{Clauser1969}. Generally speaking, any inequality that provides a sufficient condition for the existence of a hidden variable model is known as a \emph{Bell inequality}\footnote{The original inequality proved by Bell in~\cite{Bell1966} corresponds to a different experiment in which it is harder to obtain a violation~\cite{Clauser1969}.}.

\paragraph{}However, we know that quantum theory predicts violations of such Bell inequalities, and this has now been verified experimentally~\cite{hensen2015loophole,giustina2015significant,shalm2015strong,rosenfeld2017event,Storz2023}. One example of such violation of \eqref{eq:CHSH} can be achieved by taking the state:
\begin{equation*}
\ket{\Psi} = \frac{1}{\sqrt{2}} \left(\ket{0}\otimes\ket{1} - \ket{0}\otimes\ket{1}\right)
\end{equation*}
as the state prepared by Charlie, and $a,a',b,b'$ being one-qubit measurements along the respective basis~\cite{nielsen2010quantum}:
\begin{center}
    \begin{tabular}{|c|c|c|}
        \hline Measurement & Outcome $-1$ & Outcome $+1$\\\hline
        $a$ & $\ket{0}$ & $\ket{1}$\\\hline
        $a'$ & $\frac{1}{\sqrt{2}}\left(\ket{0}+\ket{1}\right)$ & $\frac{1}{\sqrt{2}}\left(\ket{0}-\ket{1}\right)$\\\hline
        $b$ & $\cos\left(\frac{-\pi}{8}\right) \ket{0} + \sin\left(\frac{-\pi}{8}\right)\ket{1}$ & $\cos\left(\frac{3\pi}{8}\right)\ket{0} + \sin\left(\frac{3\pi}{8}\right)\ket{1}$\\\hline
        $b'$ & $\cos\left(\frac{\pi}{8}\right) \ket{0} + \sin\left(\frac{\pi}{8}\right)\ket{1}$ & $\cos\left(\frac{5\pi}{8}\right)\ket{0} + \sin\left(\frac{5\pi}{8}\right)$\\\hline
    \end{tabular}
\end{center}
From these sets of measurements, it can be shown that the observed probability distributions are given by:
\begin{center}
    \begin{tabular}{|c|c|c|c|c|}
        \hline & $(-1,-1)$ & $(-1, +1)$ & $(+1,-1)$ & $(+1, +1)$\\\hline
        $(a,b)$ & $\frac{1}{2}\sin^2\left(\frac{\pi}{8}\right)$ & $\frac{1}{2}\cos^2\left(\frac{\pi}{8}\right)$ & $\frac{1}{2}\cos^2\left(\frac{\pi}{8}\right)$ & $\frac{1}{2}\sin^2\left(\frac{\pi}{8}\right)$\\\hline
        $(a,b')$ & $\frac{1}{2}\sin^2\left(\frac{\pi}{8}\right)$ & $\frac{1}{2}\cos^2\left(\frac{\pi}{8}\right)$ & $\frac{1}{2}\cos^2\left(\frac{\pi}{8}\right)$ & $\frac{1}{2}\sin^2\left(\frac{\pi}{8}\right)$\\\hline
        $(a',b)$ & $\frac{1}{4}\left(1 + \frac{1}{\sqrt{2}}\right)$ & $\frac{1}{4}\left(1 - \frac{1}{\sqrt{2}}\right)$ & $\frac{1}{4}$ & $\frac{1}{4}$\\\hline
        $(a',b')$ & $\frac{1}{4}\left(1 - \frac{1}{\sqrt{2}}\right)$ & $\frac{1}{4}\left(1 + \frac{1}{\sqrt{2}}\right)$ & $\frac{1}{4}$ & $\frac{1}{4}$\\\hline
    \end{tabular}
\end{center}
Hence, by some simple calculations, it is possible to obtain the expectation values of the product variables $ab$, $ab'$, $a'b$, and $a'b'$ as:
\begin{center}
    \begin{tabular}{|c|c|c|c|}
        \hline$\left<ab\right>$ & $\left<ab'\right>$ & $\left<a'b\right>$ & $\left<a'b'\right>$ \\\hline 
        $\frac{1}{\sqrt{2}}$ & $\frac{1}{\sqrt{2}}$ & $\frac{1}{\sqrt{2}}$ & $\frac{-1}{\sqrt{2}}$\\\hline
    \end{tabular}
\end{center}
Which leads to:
\begin{equation}
    \left|\left<ab\right> + \left<ab'\right> + \left<a'b\right> - \left<a'b'\right>\right| = 2\sqrt{2} > 2 
\end{equation}

\paragraph{}The implications of the violation of \eqref{eq:CHSH} is that the statistics of quantum systems cannot be explained by a model consistent with realism (i.e. all physical quantities have values in a system at a given time) and locality (i.e. an event can only influence other events in its future light-cone). Which is the correct assumption to drop, i.e. realism or locality, is still highly debated in quantum foundations.

A theory without locality would imply that the physical property of a system (e.g. its position) can instantly and non-locally be altered. The leading example of such interpretation is known as \emph{Bohmian mechanics}~\cite{BohmI,BohmII}. We should also note that, even in non-local theories, \emph{information} cannot travel faster than light. For example, in the previously described experiment, Alice cannot infer which measurement Bob has done from her observed local statistics, and conversely.

On the other hand, non-realistic but local theories can also be formulated. Examples of such interpretations are the so-called \emph{Copenhagen interpretation}~\cite{Bohr}, which promotes the idea of the ``collapse of the wavefunction'' upon measurement, or \emph{QBism}~\cite{Caves2002} in which a quantum state is not the ``reality'' of the system, but correspond to a more subjective description of it.  

\subsubsection{Contextuality}
\paragraph{}The proof of the non-existence of hidden variables in quantum mechanics above highly depends on the locality or no-signalling condition, which restricts the situations it can describe. In~\cite{Kochen1967}, the so-called \emph{Kochen-Specker theorem} proposed a more general criterion for the non-existence of hidden variables. Indeed, instead of relying on spatially separated measurements, we only require that measurements be done ``at the same time'', meaning that we can know the values of the measurements simultaneously. For example, if the value of an observable $A$ is found to be $3$, then the value of the operator $A^2$ is automatically known to be $9$. In the standard quantum mechanics formalism, this condition is expressed as having commuting observables, i.e. $A$ and $B$ as compatible iff $AB = BA$. If two observables are compatible we should be able to observe their joint statistics, whereas it does not make sense to talk about the joint statistics of observables which are not compatible. A system will then be said to be non-contextual if we can extend the system to one where all of the observables are compatible, i.e. a system in which the values of all the observables can be known at the same time. This extended system can be seen as a hidden-variable model of the system. The Kochen-Specker theorem~\cite{Kochen1967} then states that this is not possible for the observables of quantum mechanics. We now describe their proof, as well as subsequent proofs, of quantum contextuality.

\paragraph{}The Kochen-Specker argument was originally abstractly formulated in terms of partial algebra as follows. We start with a set of observables $\mathcal{O}$ endowed with a \emph{co-measurability relation} $\mathfrak{R}\subseteq \mathcal{O}\times \mathcal{O}$; namely, for two operators $A,B\in \mathcal{O}$, we have $A~\mathfrak{R}~B$ iff $A$ and $B$ are comaptible. This co-measurability relation is required to be reflexive (i.e. every observable is co-measurable with itself) and symmetric (i.e. if $A$ is co-measurable with $B$, then $B$ is co-measurable with $A$). In addition, it is also desirable that, if $A$ and $B$ are both compatible with an observable $C$, then any function\footnote{Strictly speaking, we mean a Borel function here.} of $A$ and $B$ will remain compatible with $C$. Formally, the co-measurability relation corresponds to a \emph{partial algebra} defined as follows.

\begin{defs}[Partial algebra]
    A \emph{partial algebra} over a field $K$ is a tuple $(X, \mathfrak{R})$ where $X$ is a set endowed with addition, multiplication, and scalar multiplication over $K$, and $\mathfrak{R}$ is a binary relation $\mathfrak{R} \subseteq X\times X$ such that:
    \begin{enumerate}[label=\alph*.]
        \item $\mathfrak{R}$ is symmetric and reflexive;
        \item There exists an element $1\in X$ such that $A~\mathfrak{R} ~1$ for all $A\in X$
        \item For any $A_1, A_2, A_3\in X$ such that $A_i~\mathfrak{R}~ A_j$ for all $i,j \in \{1,2,3\}$ and $\alpha\in K$:
        \begin{align}
            \left(A_1 + A_2 \right) \quad\mathfrak{R}&\quad A_3\\
            A_1\cdot A_2 \quad\mathfrak{R}&\quad A_3\\
            \alpha A_1 \quad\mathfrak{R}&\quad A_2    
        \end{align}
        \item For any $A_1, A_2, A_3$ such that $A_i~\mathfrak{R}~A_j$ for all $i,j\in \{1,2,3\}$, the polynomials over $A_1, A_2, A_3$ form a commutative algebra.
    \end{enumerate}
    In particular, in the case where the field $K$ is the field $\mathbb{F}_2$, we talk of partial Boolean algebras.
\end{defs}

\paragraph{}In addition, we can define morphisms between partial algebras as follows.
\begin{defs}
    A \emph{partial algebra homomorphism} is a map $h:\left(X,\mathfrak{R}_X\right) \to \left(Y, \mathfrak{R}_Y\right)$ between two partial algebras over the same field $K$ such that:
    \begin{align}
        a~\mathfrak{R}_X~b \implies& h(a)~\mathfrak{R}_Y~ h(b)\\
        h(\alpha a + \beta b) =& \alpha h(a) + \beta h(b) \qquad \forall \alpha,\beta\in K\\
        h(a\cdot b) =& h(a)\cdot h(b)\\
        h(1) =& 1
    \end{align}
\end{defs}
Therefore, if there exists a homomorphism $h: (X, \mathfrak{R}_X)\to (Y, \mathfrak{R}_Y)$, then the theory over observables in $X$ can be simulated using observables in $Y$ such that the functional relations between the observables of $X$ are preserved.

\paragraph{}Then, a partial algebra $\mathfrak{R}$ is said to be \emph{non-contextual} iff there exists a partial algebra homomorphism $h: \mathfrak{R}\to \mathfrak{A}$ where $\mathfrak{A}$ is a \emph{total} (commutative) algebra. Accordingly, a partial Boolean algebra is non-contextual iff there is a partial algebra homomorphism into a Boolean algebra.

This is indeed the case for classical mechanics where observables are functions $f: \Omega \to \mathbb{R}$ with $\Omega$ being the set of possible states of the system. This indeed forms a commutative algebra over $\mathbb{R}$, where all of the observables are co-measurable, and we have the following:
\begin{align*}
    1: \Omega \to \mathbb{R} &::= \omega \mapsto 0\\
    f + g: \Omega \to \mathbb{R} & ::= \omega \mapsto f(\omega) + g(\omega)\\
    f \cdot g: \Omega \to \mathbb{R} & ::= \omega \mapsto f(\omega) \times g(\omega)\\
    \alpha f: \Omega \to \mathbb{R} & ::= \omega \mapsto \alpha f(\omega)
\end{align*}

On the other hand, the Kochen-Specker theorem states that quantum mechanical observables are \emph{contextual}, i.e. that there is no homomorphism to a total commutative algebra. Indeed, quantum observables are represented as self-adjoint operators on a Hilbert space $\mathcal{H}$. Addition, multiplication, and scalar multiplication of operators is defined as the standard matrix operations. As mentioned previously, the co-measurability $\mathfrak{R}$ is defined such that $A~\mathfrak{R}~B$ iff $AB = BA$. This structure forms a partial algebra. The obtained partial algebra is not total, as in general, not all observables will commute. Therefore, a set of quantum observables $\mathcal{O}$ admits a hidden variable model iff there exists a homomorphism $h: \left(O,\mathfrak{R}\right) \to (\mathcal{O}', \mathfrak{R}')$ where $\left(\mathcal{O},\mathfrak{R}'\right)$ is the partial algebra associated with a total (commutative) algebra. 

In the original article~\cite{Kochen1967}, the proof of contextuality for quantum theory was achieved by looking at a set of 117 observables in a 3-dimensional Hilbert space (this could represent the angular momentum of a single particle along 117 different directions). Later on, simpler proofs of quantum contextuality have been proposed, using smaller sets of observables, and provided less involved geometric arguments~\cite{Peres1991,Mermin1990,Cabello1996}. 

In addition, a major flaw in the original proof of~\cite{Kochen1967} is that it is not easily checked experimentally. In~\cite{KCBS}, Klyachko, Can, Binicio\u{o}lu and Shumovsky (KCBS) provided a contextuality proof on a 3-dimensional quantum system by deriving a non-contextual inequality, in the same vein as the CHSH inequality for non-locality, and showing its violation by fixing a set of 5 projection operators and the state which is being measured.

To see this, we start with the derivation of the classical bound. Suppose that we have 5 observables $\left\{A_k\right\}_{k=1,\ldots,4}$ such that $A_{k}$ and $A_{k\oplus_5 1}$ are co-measurable for all $k$ (where $\oplus_5$ denotes the addition modulo $5$). Now, suppose that all of the $A_k$'s take value in $\pm 1$, therefore, the products $A_kA_{k\oplus_5 1}$ test whether their values are correlated (i.e. $A_kA_{k\oplus_5 1}=1$) or anticorrelated (i.e. $A_kA_{k\oplus_5 1}=-1$). Then, if there exists a hidden-variable model, all of the $A_k$ gets assigned a value, regardless of which pair of observables is measured (see Table~\ref{tab:KCBSassign}). Moreover, since there is an odd number of pairs $(A_k, A_{k\oplus_5 1})$, then the number of anticorrelated pairs has to be even in a global assignment of values, and is at most $4$. Hence, this gives the KCBS inequality:

\begin{equation}\label{eq:KCBS}
    \sum_{k=1}^5 \left<A_kA_{k\oplus_5 1}\right> \geq -3
\end{equation}

It turns out that the assignment of Table~\ref{tab:KCBSassign} (seen as a deterministic hidden-variable model) saturates this inequality.

\begin{table}[htb]
    \centering
    \begin{tabular}{|c|c|c|c|c|c|}
        \hline & $A_1$ & $A_2$ & $A_3$ & $A_4$ & $A_5$\\\hline
        Value & $+1$ & $-1$ & $+1$ & $+1$ & $-1$\\\hline
    \end{tabular}
    \caption{Example of a total assignment of values to the observables $A_k$ in the KCBS experiment.\label{tab:KCBSassign}}
\end{table}

Now, we will describe a specific instance of such an experiment demonstrating contextuality in 3-dimensional quantum systems, which is taken from~\cite{cabello2010noncontextuality,ahrens2013fundamental}. Starting from the 5 states (up to normalisation factors):
\begin{align}
    \ket{v_1} \propto& \ket{0} + \sqrt{\cos\left(\frac{\pi}{5}\right)} \ket{2}\\
    \ket{v_2} \propto& \cos\left(\frac{4\pi}{5}\right)\ket{0} + \sin\left(\frac{4\pi}{5}\right)\ket{1} + \sqrt{\cos\left(\frac{\pi}{5}\right)} \ket{2}\\
    \ket{v_3} \propto& \cos\left(\frac{2\pi}{5}\right)\ket{0} - \sin\left(\frac{2\pi}{5}\right)\ket{1} + \sqrt{\cos\left(\frac{\pi}{5}\right)} \ket{2}\\
    \ket{v_4} \propto& \cos\left(\frac{2\pi}{5}\right)\ket{0} + \sin\left(\frac{2\pi}{5}\right)\ket{1} + \sqrt{\cos\left(\frac{\pi}{5}\right)} \ket{2}\\
    \ket{v_5} \propto& \cos\left(\frac{4\pi}{5}\right)\ket{0} - \sin\left(\frac{4\pi}{5}\right)\ket{1} + \sqrt{\cos\left(\frac{\pi}{5}\right)} \ket{2}
\end{align}
It can be checked that these states are pairwise orthogonal, i.e. they satisfy $\left<v_k\middle|v_{k\oplus_5 1}\right> = 0$ for each $k$. These states give a set of projections operators with eigenvalues (i.e. outcomes) $\pm 1$, namely:
\begin{equation}
    P_k = 2\ket{v_k}\bra{v_k}  - \mathbb{I}
\end{equation}
where $\mathbb{I}$ is the identity. We can then check that the pairs of projectors $P_k$ and $P_{k\oplus_5 1}$ commute since:
\begin{equation*}
    P_kP_{k\oplus_5 1} = P_{k\oplus_5 1}P_k =  - 2\ket{v_k}\bra{v_k} - 2\ket{v_{k\oplus_5 1}}\bra{v_{k\oplus_5 1}}  + \mathbb{I}
\end{equation*}
using the fact that the pairs $\ket{v_k}$ and $\ket{v_{k\oplus_5 1}}$ are orthogonal. We recall that this means that changing the order of the projections will not change the values of the individual observables. Now, taking the state to be measured to be the state $\ket{\psi} = \ket{2}$, it can be shown that the expectation value $\left<P_kP_{k\oplus_5 1}\right>$ is:
\begin{equation}
    \left<P_kP_{k\oplus_5 1}\right> = \frac{1 - 3\cos\left(\frac{\pi}{5}\right)}{2\cos^2\left(\frac{\pi}{10}\right)}
\end{equation}
for each $k=1,\ldots, 5$. Therefore leading to the violation of the KCBS inequality~\eqref{eq:KCBS}:
\begin{equation}
    \sum_{k=1}^5 \left<A_kA_{k\oplus_5 1}\right> = 5~\frac{1 - 3\cos\left(\frac{\pi}{5}\right)}{2\cos^2\left(\frac{\pi}{10}\right)} \simeq -3.944 < -3
\end{equation}

The advantages of this proof is that it provides clear experiments which needs to be performed for showing contextuality of quantum mechanics, and the inequality derived is minimal in terms of number of observables and dimension of the quantum system~\cite{KCBS}. In addition, we should emphasize that this proof of contextuality \emph{does not depend on locality assumptions}, as measurements are done on a single system. This then shows that contextuality is strictly more general than non-locality.

The KCBS inequality was generalised for $n$-dimensional quantum systems with $n\geq 3$, by considering $n$ observables $\left\{P_i\right\}_{i=1,\ldots,n}$, where the only compatible measurements are $P_i, P_{i\oplus_n 1}$, where $\oplus_n$ is the addition modulo $n$~\cite{Araujo2013}. Then, the KCBS inequality arises as the special case $n=3$, whilst the CHSH inequality corresponds to the case $n=4$.

\subsubsection{Contextuality and quantum computations}
\paragraph{}Contextuality provides a fundamental distinction between classical and quantum theories and has also been shown to be an essential resource in quantum computing. It has famously been demonstrated that quantum systems can solve computational problems exponentially faster than any known classical algorithm, such as factoring~\cite{shor1999polynomial} or simulation of physical systems~\cite{babbush2023exponential}. Where the advantage comes from has historically been unclear. Recent studies have shown that contextuality is a crucial ingredient for obtaining a quantum advantage, more so than superposition or entanglement, which can be efficiently simulated using classical computers~\cite{Steane2003,VandeNest2013}.

\paragraph{}One of the first demonstrations of the role of contextuality in computation relates to fault-tolerant stabiliser quantum computing. 

Fault tolerance is vital to achieve reliable computation on real quantum computers. One of the promising avenues to achieve quantum fault tolerance relies on \emph{stabiliser codes}, where a specific set of measurements (usually generalisations of the Pauli gates) is used to correct noise introduced in a quantum circuit. However, these gates or measurements are part of the Clifford group, which can only generate circuits that are simulable on classical computers~\cite{gottesman1998}. The full power of quantum computations can be achieved from \emph{magic state distillation}. 

For magic state distillation, we start from an input state $\rho$, which can be noisy, and aim at distilling it into a target ``magic state'' $\ket{m}$ using stabiliser measurements on some subsystem. This target state is defined so that non-Clifford gates can be performed using it. 

Now, not all initial states $\rho$ can be distilled into a magic state $\ket{m}$. In~\cite{Howard2014}, the authors showed that the set of states that can be distilled into magic states are precisely the ones that can exhibit contextuality. Since quantum circuits using only Clifford gates are efficiently simulable using classical resources, this result shows that contextuality is essential to obtain a quantum advantage.

\paragraph{}Contextuality has also been studied from a resource theoretical point of view. One result is that the amount of contextuality of a system cannot increase with classical operations such that classical pre- and post-processing, classical control over measurements or probabilistic mixing of experiments~\cite{Abramsky2019,Duarte2018,Amaral2019,Wagner2023}. This result implies that any computational advantage coming from contextuality, e.g. in magic state distillation, cannot be created from classical operations. Using quantum systems is, therefore, necessary to obtain a quantum advantage.

\subsection{The sheaf-theoretic view of contextuality}\label{subsec:sheafContextuality}
\paragraph{}In~\cite{AbramskyBrad}, the authors showed that contextuality corresponds to the impossibility of finding a global section given a consistent family of local sections of a presheaf. 

In this framework, the possible local measurements form a set $X$, which will become the base space of our presheaves (with suitable topology). We then impose a compatibility relation on $X$, which, in turn, gives us a cover of this space. This compatibility relation corresponds to the co-measurability relation described in Section~\ref{subsec:BackgroundContext}. 
\begin{ex}
    Let's consider the standard (2,2,2)-Bell scenario consisting of 2 parties, each choosing between 2 measurements, and each measurement having two possible outcomes. The set of possible measurements is $X = \left\{a_1, a_2, b_1, b_2\right\}$ and we will denote as $I_A = \{a_i\}_{i=1,2}$ the set of measurements available to Alice, and $I_B =\{b_i\}_{i=1,2}$ the set of measurements available to Bob. Alice's measurements in $I_A$ are compatible with all of Bob's in $I_B$. However, the measurements $a_1$ and $a_2$ are incompatible, as they cannot be performed simultaneously, and similarly for Bob's measurements. 
\end{ex}
Each of these measurements comes with a set of possible outcomes $O$\footnote{Without loss of generality, we can assume that $O$ is the same for any choice of measurement.}. Then, given a set of compatible measurements $U$, an event associates outcomes with the measurements selected in $U$. An event is, therefore, modelled as a function:
\begin{equation*}
    s: U\to O
\end{equation*}
\begin{ex}In the (2,2,2)-Bell scenario, if Alice chooses to perform the measurement $a_1$ and obtains outcome $x \in O$, and Bob the measurement $b_2$ and obtains the outcome $y \in O$, then the event could be represented as the function:
    \begin{equation*}
    s:U\to O :: a_1 \mapsto x; b_2\mapsto y 
    \end{equation*}
\end{ex}

\subsubsection{Presheaves and empirical models}
\paragraph{}Formally speaking, these functions are modelled as the \emph{presheaf of events}. The presheaf of events is defined as $\mathcal{E} \colon \mathcal{P}(X)^{op} \to \mathbf{Sets}$, where the morphisms in $\mathcal{P}(X)$ are inclusion relations. In other words, we are taking a presheaf over the set of measurements $X$ endowed with the \emph{discrete topology}.

The action of this \emph{presheaf} on objects $U$ gives the set of all possible \emph{assignments} or \emph{functions} $s: U\to O$. The action on morphisms $U\xrightarrow{\subseteq} V$ in $\mathcal{C}$, gives us the \emph{restrictions} of these assignments, namely:
\begin{equation}\label{eq:restrictE}
    \begin{matrix}
        \mathcal{E}(U\subseteq V): &&\mathcal{E}(V) &\to &&\mathcal{E}(U)\\
    &s:&V\to O &\mapsto &s|_U:& U\to O\\
    &&v\in V\mapsto o_v\in O& &&v\in U\mapsto o_{v} \in O
    \end{matrix}
\end{equation}

\paragraph{}In quantum mechanics, however, the outcomes of measurements are not generally deterministic, so instead of looking at events, it is more relevant to look at \emph{the probability distributions} over all of the possible events. Therefore, we post-compose the event presheaf $\mathcal{E}$ with the distribution monad $\mathcal{D}_{\mathbb{R}_+}:\mathbf{Sets}\to \mathbf{Sets}$ as defined in Section~\ref{sec:categories}. The obtained functor is once again a presheaf.

\paragraph{}In a given experiment, we will not observe \emph{all} of the possible probability distributions for each set of co-measurable measurements, but instead, we will see only a \emph{single} probability distribution per global measurement choice, which will correspond to the \emph{observed probability distribution}. Hence, in terms of the presheaf $\DR\mathcal{E}$, this means that, when selecting a set of measurements $U\subseteq X$ to perform, we will only observe a single section $e_U\in \DR\mathcal{E}(U)$. 

Similarly, we can only access the probability distributions of specific combinations of compatible measurements in a given quantum experiment. For example, suppose Alice can either measure $a_1$ or $a_2$. In that case, we cannot observe the joint statistics of $a_1$ and $a_2$ as these measurements cannot be performed simultaneously. So, instead of looking at sections of $\DR\mathcal{E}(U)$ for each of the subsets $U\subseteq X$, we will instead consider a collection $\mathcal{U}=\left\{U_i\right\}_{i\in I}$, such that for each of the collections of $U\in \mathcal{U}$, the elements of $U$ correspond to compatible measurements. Without loss of generality, we will moreover assume that the set $\mathcal{U}$ is a cover of the space $X$, i.e. $\bigcup_{U\in \mathcal{U}} = X$, so that all of the measurements are possible. This gives rise to the notion of \emph{measurement scenario}.
\begin{defs}[Measurement scenario]
    A \emph{measurement scenario} will consist on a tuple $(\mathcal{X}, \mathcal{U})$, where $\mathcal{X}$ is a topological space and $\mathcal{U}$ is an (open) cover of $X$.
\end{defs}

We then define the data of an experiment as follows.
\begin{defs}[Empirical model]
    Given measurement scenario $(\mathcal{X}, \mathcal{U})$, we define an \emph{empirical model} as a set of sections $e = \left\{e_U\in \DR\mathcal{E}(U)\middle| U\in \mathcal{U} \right\}$ of the presheaf $\DR\mathcal{E}(U): \mathcal{T}(\mathcal{X})^{op}\to \mathbf{Sets}$, where $\mathcal{E}$ can be any presheaf of events.
\end{defs}

\begin{ex} In a (2,2,2)-Bell scenario as described previously, we would have $X = \left\{a_1, a_2, b_1, b_2\right\}$, with associated cover $\mathcal{U} = \left\{\left\{a_1, b_1\right\}, \left\{a_1, b_2\right\}, \left\{a_2, b_1\right\}, \left\{a_2, b_2\right\}\right\}$. Then, an empirical model could be represented as in Table~\ref{fig:basechsh} where each of the rows is labelled by the choice of measurements $U\in \mathcal{U}$ and correspond to the selected section of $\DR\mathcal{E}(U)$. More specifically, the cell at the intersection of the row labelled by $\{a_i, b_j\}$ and column labelled by $(o_k, o_l)\in O^2$ corresponds to the probability $e_{\{a_i, b_j\}}(s::a_i\mapsto o_k; b_j\mapsto o_l)$. 
\begin{table}
    \begin{center}
 \begin{tabular}{r|ccccc}
 & $(0, 0)$ & $(0, 1)$ & $(1, 0)$ & $(1, 1)$ \\ \hline
 $(a_1, b_1)$ & $1 / 2$ & $0$ & $0$ & $1 / 2$ \\
 $(a_1, b_2)$ & $3 / 8$ & $1 / 8$ & $1 / 8$ & $3 / 8$ \\
 $(a_2, b_1)$ & $3 / 8$ & $1 / 8$ & $1 / 8$ & $3 / 8$ \\
 $(a_2, b_2)$ & $1 / 8$ & $3 / 8$ & $3 / 8$ & $1 / 8$ \\
 \end{tabular}
\end{center}
 \caption{An empirical model for the (2,2,2)-Bell scenario\label{fig:basechsh}}
\end{table}
\end{ex}

\begin{rmk}
    To simplify the notation, we will denote the probabilities:
    \begin{equation}
        e_{\{a_i, \ldots, a_k\}}(s:: a_j\mapsto o_j) \equiv e_{(a_i, \ldots a_k)}(o_i, \ldots, o_k)
    \end{equation}
\end{rmk}


\subsubsection{Sheaf-theoretic contextuality}

\paragraph{}In the standard contextuality experiments, we are interested in studying the source of the correlations between contexts, i.e. choices of measurements and their observed statistics. In order to isolate the source of potential correlations between the contexts and the outcomes, the standard practice is to limit the overall number of possible sources of such correlations. One type of correlation which can be eliminated in quantum experiments is communication, i.e. the \emph{signalling} between Alice and Bob in the above example. In practice, we can achieve this by spatially isolating these parties. 

The consequence of such isolation, or lack of signalling, is that the marginal probability distributions do not depend on the choice of measurements of the other parties. In other words, for any set of inputs $U$, and any two sets of measurements $V, V'$ compatible with all elements of $U$, we should have:
\begin{equation}\label{eq:consistencyEM}
    e_{U\cup V}|_{U}(\underline{o}_U) = e_{U\cup V'}|_{U}(\underline{o}_U)
\end{equation}
for all joint outcomes $\underline{o}_U$ over the measurements of $U$, where $e_{W}$ corresponds to the joint probability distribution corresponding with the choices of inputs $W$ for any set $W$.

\begin{ex}The (2,2,2)-Bell scenario depicted in Table~\ref{fig:basechsh} indeed satisfies this so-called \emph{no-signalling} condition, since, for instance:
\begin{equation}
    e_{(a_1,b_1)}|_{a_1}(0) = e_{(a_1,b_2)}|_{a_1}(0) = \frac{1}{2}
\end{equation}
\end{ex}

\paragraph{}We then define the notion of (non-)contextuality as follows.
\begin{defs}
    A system is said to be \emph{non-contextual} iff there exists a joint probability distribution over $X$ which correctly restricts to all of the $e_U$'s, i.e., if there exists a global section $e\in P(X)$ such that $\left.e\right|_U = e$ for all $U\in \mathcal{P}(X)$.
\end{defs}  

We note that this condition is reminiscent of the definition of a sheaf described in Section~\ref{subsec:sheaves} (Definition~\ref{def:sheaf}). If an empirical model is non-contextual, the global section acts as a hidden-variable model for the observed statistics.

\begin{ex}The example of Fig.~\ref{fig:basechsh} is \emph{contextual}, i.e. a global probability distribution cannot be defined.\end{ex} 

\begin{rmk}
    This notion of contexutality can be shown to be equivalent to the notion of contextuality defined in terms of non-existence of a homomorphism from a partial Boolean algebras to the Boolean algebra $\mathbf{2}$~\cite{EMpBA}.
\end{rmk}

\subsubsection{On the no-signalling property}
\paragraph{}In realistic experiments, the no-signalling condition does not usually hold; this can be due to the unsharpness of the instruments~\cite{Emeriau2022} or simply the finiteness of the measurements~\cite{Emeriau2022,Dzhafarov2016contextcontent}. As a result, different frameworks have been developed to study contextuality in the presence of signalling. Examples of these are the Contextuality-by-Default framework~\cite{Dzhafarov2016contextcontent} and the corrected Bell inequalities of the sheaf-theoretic model~\cite{Emeriau2022}, both of which create a measure of the signalling property of the system. We will describe the Contextality-by-Default framework in the subsequent subsection, but first, let's look at way of dealing with signalling in the sheaf-theoretic framework~\cite{Emeriau2022}.

The intuition is that a signalling system is said to be contextual if the amount of signalling is not enough to make the system ``classically explainable''. In sheaf-theoretic terminology, the empirical model is said to be \emph{no-signalling} or \emph{consistent} if every pair of sections in an empirical model satisfies the compatibility condition of \eqref{eq:consistencyEM}. Given an empirical model $e$, which is not necessarily compatible, we define the \emph{no-signalling fraction} $\mathsf{NSF}\in [0, 1]$ as the maximal possible value of $\lambda$ across all of the decompositions of $e$:
\begin{equation}
    e = \lambda \cdot e_{NS} + (1-\lambda) \cdot e' 
\end{equation}
where $e_{NS}$ is a no-signalling empirical model (the multiplication is here point-wise multiplication), and $e'$ can be any empirical model. We then define the \emph{signalling fraction} as:
\begin{equation}
    \mathsf{SF} = 1 - \mathsf{NSF}
\end{equation}

The signalling fraction can be seen as the degree of incompatibility of an empirical model, as it measures the departure from a no-signalling, locally compatible model. 

\paragraph{}Similarly, for any arbitrary empirical model $e$, we can define the \emph{non-contextual fraction} $\mathsf{NCF}$~\cite{Emeriau2022,Abramsky2017,Amaral2019} as the maximal $\lambda\in [0,1]$ such that:
\begin{equation}
    e = \lambda \cdot e_{NC} + (1-\lambda) \cdot e'
\end{equation}
where, this time, $e_{NC}$ is a non-contextual (and no-signalling) empirical model. In addition, we will also define the \emph{contextual fraction} $\mathsf{CF}$ as:
\begin{equation}
    \mathsf{CF} = 1 - \mathsf{NCF}
\end{equation} 

Then, a possibly signalling empirical model is said 
to be contextual iff:
\begin{equation}
    \mathsf{CF} > \mathsf{SF}
\end{equation}
\begin{rmk}
The contextual fraction $\mathsf{CF}$ can also quantify a resource from which a quantum advantage can be obtained~\cite{Abramsky2019}.
\end{rmk}

\subsubsection{Extending to causality}
\paragraph{}In~\cite{MansfieldSequential,sheafcausality,sheafcausalityB,abramskyCausality}, this formulation of contextuality has been extended to scenarios where structured signalling is allowed, first by allowing sequential operations in~\cite{MansfieldSequential}, then by allowing \emph{definitite causal orders}~\cite{sheafcausality,abramskyCausality} and even \emph{indefinite causal structures}~\cite{sheafcausality,sheafcausalityB}.

Here, we will focus on the case of definite causal order and use the formulation of~\cite{sheafcausality}, although the one of~\cite{abramskyCausality} is equivalent on the situations of interest in Part~\ref{part:Lexical} and~\ref{part:Syntactic}\footnote{Although, the formulation of~\cite{abramskyCausality} is applicable to strictly more scenarios than the one of~\cite{sheafcausality}.}. We start by defining the notion of \emph{party} corresponding to a point in space and time. For example, it could represent a lab, as in the contextuality scenarios, or a sequence of operations done in the same lab. 

Each party $A$ will be associated with a set of possible inputs or measurements $I_A$. The measurements of $I_A$ are assumed to be pairwise incompatible. $ X = \amalg_A I_A$ will denote the set of all possible measurements. And as in contextuality scenarios, each of the inputs $x\in I_A$ will have an associated set of outcomes $O$, which we will take to be the same for all possible measurements. 

\paragraph{}Given a set of parties $\Omega$, we define a \emph{causal order} over $\Omega$ as a partial order $\Sigma = \left(\Omega, \preceq\right)$ over $\Omega$. This partial order should be interpreted as follows: for any two parties $A,B\in \Omega$, if $A\preceq B$, then the input of $A$ can influence outputs of any measurement chosen by $B$, but not the other way around. A \emph{causal scenario} is therefore taken to be $\left(\Sigma = \left(\Omega, \preceq\right), X = \amalg_{A\in \Omega} I_A, O\right)$. 

Given a set of parties $\omega\subseteq \Omega$, we define its \emph{causal past} as the downward-closed set (with respect to $\preceq$):
\begin{equation}
    \omega_\downarrow = \left\{B\in \Omega~\middle|~\exists A\in \omega. ~ B\preceq A\right\}
\end{equation}
Then, we define the set of all lowersets $\Lambda_\Sigma$ as:
\begin{equation}
    \Lambda_\Sigma = \left\{\omega_\downarrow~\middle|~\omega\in \mathcal{P}(\Omega)\right\}
\end{equation}
Roughly speaking, each set $\lambda\in \Lambda_\Sigma$ corresponds to a set of parties for which a complete history, i.e. sets of inputs and outcomes, can be defined. 

We now recall that in contextuality scenarios, for any set of measurements $U\subseteq X$, each measurement $x\in U$ is assumed to be made independently. However, in the case of causal scenarios, the measurements are allowed to depend on the inputs and outcomes of the preceding events. The approach of~\cite{sheafcausality} is to encode this causal structure within the topology of the base space of the presheaf. Given a causal order $\Sigma$, we will then define the topological space $\mathcal{L}_\Sigma$ as having open sets abstractly defined as:
\begin{equation}
    \underline{U}\in \mathcal{L}_\Sigma = \left(\lambda\in \Lambda_\Sigma, \left(U_A\subseteq I_A\right)_{A\in \lambda}\right)
\end{equation}
where we also require that $U_A\neq \emptyset$ for all $A\in \lambda$. 

The idea is that each of these sets $\underline{U}$ gives rise to a well-defined sub-scenario $\left((\lambda, \preceq), \amalg_{A\in \lambda}U_A, O\right)$ of the full causal scenario. The condition that $U_A\neq \emptyset$ for all $A$ states that every party in the sub-scenario can ``do something'' so that every party in its future can use their local experiment. 

In addition, we can order these sub-scenarios as follows:
\begin{equation}
    \begin{matrix}
        &\underline{U} = \left(\lambda_U\in \Lambda_\Sigma, \left(U_A\subseteq I_A\right)_{A\in \lambda_U}\right) \subseteq \underline{V} = \left(\lambda_V\in \Lambda_\Sigma, \left(V_A\subseteq I_A\right)_{A\in \lambda_V}\right)\\
        \iff&\lambda_U \subseteq \lambda_V \quad \wedge \quad \forall \omega\in \lambda_U.~ U_\omega \subseteq V_\omega 
    \end{matrix}
\end{equation}
This means that if $\underline{U}\subseteq \underline{V}$, then everything that can happen in $\underline{U}$ is also possible in $\underline{V}$. We can moreover define the \emph{union} and \emph{intersection} of sub-scenarios $\underline{U}$ and $\underline{V}$ as follows:
\begin{align}
    \underline{U}\cap \underline{V} =& \left(\lambda=\left\{A\in \lambda_U \cap \lambda_V~\middle|~ U_A\cap V_A \neq \emptyset\right\}, (U_A\cap V_A)_{A\in \lambda}\right)\label{eq:localeInter}\\
    \underline{U}\cup \underline{V} =& \left(\lambda_U\cup \lambda_ V, \left(U_A\cup V_A\right)_{A\in \lambda_U\cup \lambda_V}\right)
\end{align}

\begin{rmk}
    These definitions are directly taken from~\cite{sheafcausality}. However, although they are well-motivated from a physical point of view, the intersection defined in \eqref{eq:localeInter} is \emph{not always defined}. For instance, let us look at the causality scenario $(\Sigma, X, O)$ where $\Sigma = \left(\left\{A, B,C\right\}, \preceq\right)$ is the total order:
    \begin{equation}
        A\preceq B\preceq C
    \end{equation}
    and where:
    \begin{equation}
        I_A = I_B = I_C = \left\{0,1\right\}
    \end{equation}
    Then, taking:
    \begin{align}
        \underline{U} =& \left(\left\{A, B, C\right\}, \left(I_A, \left\{0\right\}, I_C\right)\right)\\
        \underline{V} =& \left(\left\{A, B, C\right\}, \left(I_A, \left\{1\right\}, I_C\right)\right)
    \end{align}
    Then:
    \begin{equation}
        \left\{A\in \lambda_U \cap \lambda_V~\middle|~ U_A\cap V_A \neq \emptyset\right\} = \left\{A, C\right\} \notin \Lambda_\Sigma
    \end{equation}
    which is not a lowerset, so $\underline{U}\cap\underline{V}\notin \mathcal{L}_\Sigma$. However, this issue is not easily fixable, and we will leave the task of formulating a better framework as future work.
\end{rmk}

It is then claimed that $\mathcal{L}_\Sigma$ forms a locale~\cite[Proposition 5]{sheafcausality} (and hence a topological space)\footnote{Note that this does not follow from the definitions of~\cite{sheafcausality} as, from the above remark, $\mathcal{L}_\Sigma$ does not define a meet. This could still lead to a locale, but with a different choice of meet.}.

\paragraph{}We say that a function $s:\prod_{A\in \lambda}U_A \to O^{\left|\lambda\right|}$ over a lower set $\lambda$ respects the causal order $\Sigma$ iff for all $\left(i_A\right)_{A\in \lambda}, \left(i'_A\right)_{A\in \lambda}\in \prod_{A\in \lambda} U_A$:
\begin{equation}
    \left.\left(i_A\right)_{A\in \lambda}\right|_{B_\downarrow} = \left.\left(i'_A\right)_{A\in \lambda}\right|_{B_\downarrow} \implies \left.s\left(\left(i_A\right)_{A\in \lambda}\right)\right|_{\{B\}} = \left.s\left(\left(i'_A\right)_{A\in \lambda}\right)\right|_{\{B\}}
\end{equation}
where we write the (strict) past of $B$ as $B_\downarrow \equiv \{B\}_\downarrow-\{B\}$. This condition states that the past of $B$ is unchanged by its future. 
\begin{ex}
    Let's consider the simple causal scenario $\Sigma = \left(\left\{A,B\right\}, \preceq\right)$ where the only non-trivial causal relation is $A\preceq B$, and the causal scenario:
    \begin{equation*}
        \left(\Sigma, \left\{\left(A, I_A=\left\{a_1, a_2\right\}\right), \left(B, I_B=\left\{b_1, b_2\right\}\right)\right\}, O=\left\{0,1\right\}\right)
    \end{equation*}
    In addition, let's consider the open subset $\underline{U} = \left(\left\{A, B\right\} \left(I_A, I_B\right) \right)$. Then, the function $s: I_A\times I_B \to O^2 \in \mathcal{E}_\Sigma$ defined as:
    \begin{center}
        \begin{tabular}{c|c}
            $s$ & Outcomes\\\hline
            $(a_1, b_1)$ & $(0,0)$\\
            $(a_1, b_2)$ & $(0,1)$\\
            $(a_2, b_1)$ & $(1,1)$\\
            $(a_2, b_2)$ & $(1,0)$
        \end{tabular}
    \end{center}
    is a causal function with respect to $\Sigma$, since:
    \begin{align}
        \left.s\left(a_1, b_1\right)\right|_{\{A\}} = \left.s\left(a_1, b_2\right)\right|_{\{A\}} = 0\\
        \left.s\left(a_2, b_1\right)\right|_{\{A\}} = \left.s\left(a_2, b_2\right)\right|_{\{A\}} = 1
    \end{align}
\end{ex}

We then define the (pre-)sheaf of causal events $\mathcal{E}_\Sigma: \mathcal{L}_\Sigma^{op} \to \mathbf{Sets}$ as:
\begin{equation}
    \begin{matrix}
        \mathcal{E}_\Sigma:& \mathcal{L}_\Sigma^{op} &\to&\mathbf{Sets}\\
        & \left(\lambda, \left(U_A\right)_{A\in \lambda}\right)& \mapsto& \left\{s ~\middle|~s \text{ respects the causal order }\Sigma\right\}\\
        & \underline{U}\subseteq \underline{V} & \mapsto & \left(s::(i_A)_{A}\mapsto (o_A)_{A}\right)\mapsto \left(\left.s\right|_{\underline{U}}:: (i_A)_{A}\mapsto (o_A)_{A}\right)
    \end{matrix}
\end{equation}
Each section $s$ of $\mathcal{E}_\Sigma(\underline{U})$ therefore corresponds to a set of consistent histories over the sub-scenario associated with $\underline{U}$. In this case, the consistency condition expresses the consistency with respect to the causal order $\Sigma$.

\paragraph{}As in the contextuality case, we then want to consider a probabilistic mixture of possible histories, and hence consider sections of the presheaf $\DR\mathcal{E}_\Sigma: \mathcal{L}_\Sigma^{op}\to \mathbf{Sets}$. Similarly, we will define a \emph{causal empirical model} as a family of sections $e = \left\{e_{\underline{U}}~\middle|~\underline{U}\in \mathcal{M}\right\}$, where $\mathcal{M}$ is a cover of $\mathcal{L}_\Sigma$, i.e. $\bigcup_{\underline{U}\in \mathcal{M}} \underline{U} = \left(\Omega, \left(I_A\right)_{A\in \Omega}\right)$. 

A standard choice of cover is the following:
\begin{equation}
    \mathcal{M}_{local} = \left\{(\lambda, \left(\left\{i_A\right\}\right)_{A\in \lambda})~\middle|~ \lambda\in \Lambda_\Sigma, i_A\in I_A\right\}
\end{equation}
This cover will record the statistics of observing the outputs at each stage for any choice of inputs. We will, for instance, use this cover in Chapter~\ref{chap:SyntacticModel} when looking at the grammatical parsing process. For this cover, an empirical model is said to be causal or consistent with $\Sigma$ if the restriction of a section to an earlier stage correspond to the choice of section at this earlier stage, i.e.:
\begin{equation}
    \underline{U} \subseteq \underline{V} \quad \wedge\quad \underline{U},\underline{V}\in \mathcal{M}_{local} \quad \implies \quad \left.e_{\underline{V}}\right|_{\underline{U}} = e_{\underline{U}}
\end{equation}

Another choice of cover, which we will adopt in Section~\ref{sec:lexicalCausal}, is the following:
\begin{equation}
    \mathcal{M}_{global} = \left\{\left(\Omega, \left(\left\{i_A\right\}\right)_{A\in \Omega}\right)~\middle|~i_A\in I_A\right\}
\end{equation}
In these empirical models, we can only access the final probability distributions given a global choice of inputs. 

\begin{ex}
    Let's consider once again a causal scenario defined over the causal order $\left(\left\{A,B\right\}, \preceq\right)$, where the only non-trivial causal relation is $A\preceq B$, and where $I_A = \left\{a_1, a_2\right\}$, $I_B = \left\{b_1, b_2\right\}$ and $O=\left\{0,1\right\}$. Let's moreover consider the global cover:
    \begin{equation}
        \mathcal{M} = \left\{\left(\{A,B\}, \left(\left\{a_i\right\}, \left\{b_j\right\}\right)\right)~\middle|~i, j= 1,2\right\}
    \end{equation}
    Then, Table~\ref{fig:exCausal} depicts an example of an empirical model, where each of the rows corresponds to a probability distribution associated with the global choice of input $\left(a_i,b_j\right)$, and the columns are labelled with respect to the observed outcome. This model can moreover be found to be causal with respect to $\Sigma$ as (removing the curly brackets around singletons and the index $\Omega$ for the sake of clarity):
    \begin{align*}
        \left.e_{(a_1,b_1)}\right|_{a_1}(0) = \left.e_{(a_1,b_2)}\right|_{a_1}(0) = 6/13\\
        \left.e_{(a_2,b_1)}\right|_{a_2}(0) = \left.e_{(a_2,b_2)}\right|_{a_2}(0) = 23/65
    \end{align*} 
    \begin{table}[htb!]
    \begin{center}
        \begin{tabular}{c|c|c|c|c|}
            & $(0,0)$ & $(0,1)$ &$(1,0)$ & $(1,1)$ \\\hline
             $(a_1,b_1)$ & $0$ & $6/13$ & $0$ & $7/13$\\\hline
             $(a_1,b_2)$ & $24/65$ & $6/65$ & $7/13$ & $0$\\\hline
             $(a_2,b_1)$ & $23/65$ & $0$ & $14/65$ & $28/65$\\\hline
             $(a_2,b_2)$ & $23/260$ & $69/260$ & $42/65$ & $0$\\\hline
       \end{tabular}
\end{center}
 \caption{Example of an empirical model causal with respect to to the causal order $A\preceq B$.\label{fig:exCausal}}
\end{table}
\end{ex}

\begin{rmk}
    The notation can quickly become very complex in causal empirical models. Hence, as done in the previous example, any redundant information will be removed in the subsequent chapters whenever it is clear from the context what each of the quantities refers to.
\end{rmk}

\paragraph{The causal fraction}As for the no-signalling property in contextuality scenarios, a generic empirical model will not necessarily be consistent with a given causal order, notably when the probability distributions are obtained empirically. Hence, we will define the notion of the \emph{causal fraction} $\mathsf{CausF}_\Sigma$ with respect to to a causal order $\Sigma$ which will quantify how much of the observed statistics is compatible with the causal order $\Sigma$. This fraction will be defined as the maximal $\lambda\in [0,1]$ such that:
\begin{equation}
    e = \lambda \cdot e_{\Sigma} + (1-\lambda)\cdot e'
\end{equation}
where $e_\Sigma$ is consistent with the causal order $\Sigma$.

\subsection{The Contextuality-by-Default framework}\label{subsec:CbD}

\paragraph{}We have previously seen that the no-signalling condition imposed on the probability distributions is often too restrictive in practice. Solutions on the sheaf-theoretic side included allowing a small enough amount of signalling into the system or studying systems with a well-defined causal structure. Here, we will describe an alternative way of doing the former, i.e. taking signalling into account, using the framework of Contextuality-by-Default (CbD).

\paragraph{}One of the ideas behind the Contextuality-by-Default approach is to extend the notion of contextuality by allowing \emph{direct influence} of the context on the results of measurements. However, for every system in which changing the context results in a change of probability distribution, there is some \emph{contextual influence}. Therefore, one question is to distinguish what counts as ``direct influence'', and what is ``truly contextual influence''.

In CbD, non-contextual systems are the ones for which one can find a ``global explanation'' of the system which maximises the probability that distributions corresponding to the same contents coincide. We refer to this minimal amount of contextual influence allowed by the observed probability distributions as \emph{direct influence}, while \emph{contextual influences} will designate any influence due to the context. A system will be contextual if the direct influences are \emph{not enough} to describe the observed system.


\paragraph{}We now introduce the standard formalism of Contextuality-by-Default (CbD) (see also~\cite{Dzhafarov2016contextcontent} for a more general introduction). In this setting, a \emph{content} is a measurement, or more generally, a question with a known set of answers. The \emph{context} gathers all the conditions under which one or several of these questions are asked.

Formally, we start with the concept of a \emph{probability space} $(\Omega, \Sigma, \mu)$, where $\Omega$ is called the sample space, and will correspond to the set of possible outcomes (e.g. set of possible answers to a question), $\Sigma$ is a $\sigma$-algebra over $\Omega$ (i.e. set of subsets closed under complementation, countable unions and countable intersections), which we will usually take to be $\Sigma = \mathcal{P}\left(\Omega\right)$, and $\mu: \Sigma\to \mathbb{R}_+$ is a probability distribution. Now, the sample space $\Omega$ consists of an abstract collection of objects from which we cannot, for example, calculate expecation value. We then define a \emph{random variable}\footnote{Here we only consider real-valued random variables} over a probability space $(\Omega, \Sigma, \mu)$ as a (measurable) function $X: \Omega\to \mathbb{R}$, where $X(\omega)$ can be seen as the (real-)value of the outcome $\omega$. Then, for any $v\in \mathbb{R}$, we define the probability:
\begin{equation}
    P\left[X=v\right] = \mu\left(\left\{\omega\in\Omega~\middle|~X(\omega)=v\right\}\right)
\end{equation}
Similarly, for any $I\subseteq \mathbb{R}$, we define:
\begin{equation}
    P\left[X\in I\right] = \mu\left(\left\{\omega\in\Omega~\middle|~X(\omega)\in I\right\}\right)
\end{equation}

Every content $q_i$ in a context $c^j$ gives rise to a random variable $R^j_i$ that takes values from the possible answers to $q_i$ and gives the probability of each answer in the context $c^j$. So, to make the parallel with the sheaf-theoretic frameowrk introduced in the previous section, for a given measurement scenario we would have:
\begin{equation}
    P\left[R^j_i = v\right] = \left.e_{c^j}\right|_{q_i} (v)
\end{equation}

All random variables in a given context are jointly distributed, i.e. they are defined over the same probability space. However, random variables from different contexts are not: they are \emph{stochasitically unrelated}. This is the main difference with the sheaf-theoretic framework of contextuality, since here, it does not make sense to question the equality or inequality of marginal probability distributions arising from different contexts, since they are not defined over the same probability space. To talk about random variables that are not jointly distributed, we introduce the concept of \emph{probabilistic coupling}.

\begin{defs}
    A \emph{probabilistic coupling} of random variables $X_1,\ldots, X_n$ is a set of random variables $Y_1, \ldots, Y_n$ which are jointly distributed, and for which the probability distribution of each $Y_i$ agrees with the probability distributions of $X_i$.
\end{defs}

\begin{ex} Here is an example taken from~\cite{Dzhafarov2016contextcontent}. Consider two unrelated random variables $X_1$ and $X_2$ taking values in $\left\{1,2,3\right\}$ and $\left\{1,2\right\}$ respectively with the probability distributions are given by:
\begin{center}
    \begin{tabular}{c|ccc}
        & $X_1 = 1$ & $X_1 = 2$ & $X_1 = 3$\\\hline
        $P$& $0.3$ & $0.3$ & $0.4$
    \end{tabular}
\end{center}
and :
\begin{center}
    \begin{tabular}{c|cc}
        & $X_2 = 1$ & $X_2 = 2$\\\hline
        $P$& $0.7$ & $0.3$
    \end{tabular}
\end{center}
Then, we could create a probabilistic coupling $Y_1, Y_2$ such that the joint probability distribution of $Y_1$ and $Y_2$ is given by:
\begin{center}
    \begin{tabular}{c|ccc}
        $P$ & $Y_1 = 1$ & $Y_1 = 2$ & $Y_1 = 3$\\\hline
        $Y_2 = 1$ & $0.3$ & $0.2$ & $0.2$\\
        $Y_2 = 2$ & $0$ & $0.1$ & $0.2$\\
    \end{tabular}
\end{center}
It can be checked that the marginals of the above joint probability distribution do indeed reduce to the probability distributions of $X_1$ and $X_2$.
\end{ex}

Then, given a set of random variables $R^j_i$, a probabilistic coupling over them will correspond to a hidden variable model of the observed statistics. Note, however, that it is not an analogue of a global section in the sheaf-theoretic framework of contextuality, since the probabilities do not only depend on the content (observable) but also the context it is measured in.

Now, given any set of random variables $R^j_i$, it is always possible to define (infinitely many) couplings $S^j_i$~\cite{Dzhafarov2016contextcontent}. Hence, instead of requiring the existence of a coupling compatible with the observed distributions, we require a ``classical-like system'' to be a coupling that satisfies certain properties. 

In particular, let's consider a probabilistic coupling $S^j_i$ associated with the observed statistics of a system recorded in $R^j_i$. Then, since the $S^j_i$ are jointly distributed, the following probability is well-defined for any fixed content $q_i$ and pairs of contexts $c^j, c^{j'}$:
\begin{equation}
    P\left[S^j_i = S^{j'}_i\right] = \sum_{v\in V} P\left[S^j_i = v,  S^{j'}_i=v\right]
\end{equation}
where $V$ is the set of values the content $q_i$ can take. It can be shown that the above probability is bounded above for any choice of coupling, as:
\begin{equation}
    P\left[S^j_i = S^{j'}_i\right] \leq \sum_{v\in V} \min \left(P\left[R^j_i = v\right], P\left[R^{j'}_i = v\right]\right)
\end{equation}
This inequality is, in fact, saturated, i.e. for any pair of random variables $R^j_i$ and $R^{j'}_i$, there exists a coupling $\left\{S^j_i\right\}_{c^j, q_i}$ such that:
\begin{equation}\label{eq:minDICDcoupling}
    P\left[S^j_i = S^{j'}_i\right] = \sum_{v\in V} \min \left(P\left[R^j_i = v\right], P\left[R^{j'}_i = v\right]\right)
\end{equation}

A system $\left\{R^j_i\right\}_{c^j, q_i}$ is then said to be \emph{contextual} in the CbD framework iff there exits a coupling $\left\{S^i_j\right\}_{c^j,q_i}$ such that for any pair of random variables $R^j_i, R^{j'}_i$, \eqref{eq:minDICDcoupling} is satisfied. If a system is said to be \emph{consistently connected}, i.e. if:
\begin{equation}
    P\left[R^j_i =v\right] =  P\left[R^{j'}_i =v\right]
\end{equation}
for any pair of variables $R^j_i, R^{j'}_i$. In most widely studied scenarios\footnote{See Remark~\ref{rmk:signalling} for a discussion about the scenarios in which consistent connectedness and no-signalling are the same notion.}, this notion of contextuality collapses to the standard definition of contextuality in no-signalling systems~\cite{Dzhafarov2016contextcontent,Dzhafarov2015}.

\paragraph{}For a generic system, it is computationally hard to prove the existence or the non-existence of such a coupling, as it requires solving many linear inequalities. We will now focus on a specific type of context-content system, namely \emph{cyclic systems}, for which contextuality can be checked more easily. 

In a cyclic system, each context has exactly two contents, and every content is exactly in 2 contexts. The number of contents (or equivalently, the number of contexts) is the rank $n$ of the system. Moreover, again following normal practice in CbD, we will assume that all random variables take values in $\{\pm 1\}$\footnote{Every general system can be rewritten as a system with binary variables only~\cite{Dzhafarov2016contextcontent}; however, in the general case, by making such transformation on a system, it will cease to be cyclic, and the following inequality will no longer apply. There are, however, ways to study the contextuality of such a system~\cite{Dzhafarov2016contextcontent}.}. 

A cyclic system is known to be contextual in CbD iff~\cite{Kujala2016, Dzhafarov2016contextcontent}:
\begin{equation}\label{eq:BellInequality}
    s_{odd} \left(\left\{\left<R^{j}_{i_j}  R^{j}_{i'_j}\right>\right\}_{j=1,\ldots,n}\right) > n-2 + \Delta
\end{equation}
where $i_j\neq i'_j$ for all $j$ and when $R^{j}_{i_j}, R^{j}_{i'_j}$ are well-defined for all $j$. The $s_{odd}$ function and the quantity $\Delta$ are defined below. 
\begin{itemize}
    \item $s_{odd}: \mathbb{R}^n \to \mathbb{R}$ is defined as:
    \begin{equation}
        s_{odd}\left(\underline{x}\right) = \max_{\substack{\underline{\sigma}\in \{\pm1\}^n \\ \mathfrak{p}(\underline{\sigma})=-1}}\underline{\sigma}\cdot \underline{x}
    \end{equation}
    where both $\underline{\sigma}$ and $\underline{x}$ are $n$-dimensional (real) vectors and where $\mathfrak{p}(\underline{\sigma}) = \prod_{i=1}^n \sigma_i$ ($\mathfrak{p}$ can be seen as the parity function of $\underline{\sigma}$). In other words, $s_{odd}$ returns the maximal sum of all its arguments weighted with $\pm1$ coefficients under the condition that an odd number of negative coefficients are attributed.
    \item  $\Delta \in \mathbb{R}$ is defined as:
    \begin{equation}\label{eq:truecontext}
        \Delta = \sum_{i=1}^n \left|\left<R^{j_i}_{i}\right> - \left<R^{j'_i}_{i}\right>\right|
    \end{equation}
    where once again $j_i\neq j'_i$ $\forall i$ and $R^{j_i}_{i}, R^{j'_i}_{i}$ should be well-defined. The quantity $\Delta$ measures a system's ``degree of signalling'', and a system is consistently connected iff $\Delta=0$.
\end{itemize}

We note that equation~\eqref{eq:BellInequality} is a generalisation of the inequalities derived in~\cite{Araujo2013} for no-signalling cyclic systems (although they were both proven independently).

\subsubsection{Quantifying contextuality}
\paragraph{}Similarly to the contextual fraction $\mathsf{CF}$ defined in the previous subsection, we can define a quantification of the contextuality from the CbD framework. In fact, several measures have been proposed~\cite{Kujala2019}, including the non-contextual measure denoted as $\mathsf{NCNT2}$ for a given set of probability distributions $\left\{R^j_i\right\}$, defined as: 
\begin{equation}
    \mathsf{NCNT2} = \min \left(\Delta - s_{odd} \left(\left\{\left<R^{j}_{i_j}  R^{j}_{i'_j}\right>\right\}_{j=1,\ldots,n}\right), m\right)
\end{equation}
In the above equation, the quantity $m$ is defined as:
\begin{equation}
    m= \min_j \min \left(\left<R^{j}_{i_j}  R^{j}_{i'_j}\right> - 2 \left|p^j_1 + p^j_2 - 1\right| +1, 1 - \left|p^j_1 - p^j_2\right| - \left<R^{j}_{i_j}  R^{j}_{i'_j}\right>\right)
\end{equation}
where $p^j_1$ and $p^j_2$ are shorthands for respectively:
\begin{align}
    p^j_1 &= P\left[R^{j}_{i_j} = + 1\right]\\
    p^j_2 &= P\left[R^{j}_{i'_j} = + 1\right]\\
\end{align}

The advantage of this measure is that we can compare the contextuality of empirical models which are not contextual as the measure can be positive or negative. A negative $\mathsf{NCNT2}$ implies that the model is CbD-contextual, whereas a positive value implies non-contextuality. This measure will be used in Chapter~\ref{chap:lexicalFeatures}.

\subsubsection{Quatifying direct influences}
\paragraph{}As interesting as it is to have criteria for contextuality, we will see in the following Chapters that the amount of signalling in empirical models will be of interest for studying natural language data. We have introduced the signalling fraction from~\cite{Emeriau2022} in Section~\ref{subsec:sheaves}, as well as the ``degree of signalling'' $\Delta$ defined above, but only for cyclic systems.

\paragraph{}To obtain a more generic quantification of direct influence within the CbD framework, we introduce another related framework known as M-contextuality (model-contextuality), first introduced in~\cite{Jones2019}. This framework was inspired by the causal analysis of contextuality of Cavalcanti~\cite{Cavalcanti2018} and the (classical) theory of causality of Pearl~\cite{Pearl2011}. 

In~\cite{Jones2019}, the author showed that every system of random variables observed in the different contexts can be expressed as a Bayesian network in the form of Fig.~\ref{fig:baysianNet}. Here, we treat the contexts as a single random $C$, and the contents are each modelled by a random variable $F_q$ which is \emph{deterministically} determined by the context variable $C$ and some other latent variable $\Lambda$, which corresponds to background knowledge of the system. Such a Bayesian network is called a \emph{canonical model} in~\cite{Jones2019}. Note that it is important that the latent variable $\Lambda$ is independent of the context variable $C$ (otherwise, any part of the variable $\Lambda$ correlated with $C$ can be without loss of generality encompassed by $C$). 

Now, given a canonical model successfully describing a set of observed probabilities, we quantify the \emph{direct influence} of the context variable $C$ on a given content $q$ as:
    \begin{equation}\label{eq:M-directInfluence}
        \Delta_{c,c'}\left(F_q\right) = P\left[\Lambda \in \left\{\lambda| F_q(\lambda,c)\neq F_q(\lambda, c')\right\}\right]
    \end{equation}
    
In turn, a system is said to be \emph{M-contextual} if these direct influences cannot attain their respective minima in a single canonical model compatible with the empirical model. The main result of~\cite{Jones2019} was to show that this notion of contextuality is, in fact, equivalent to the CbD definition of contextuality.
    \begin{figure}[ht]
        \centering
            \includegraphics[width=.4\linewidth]{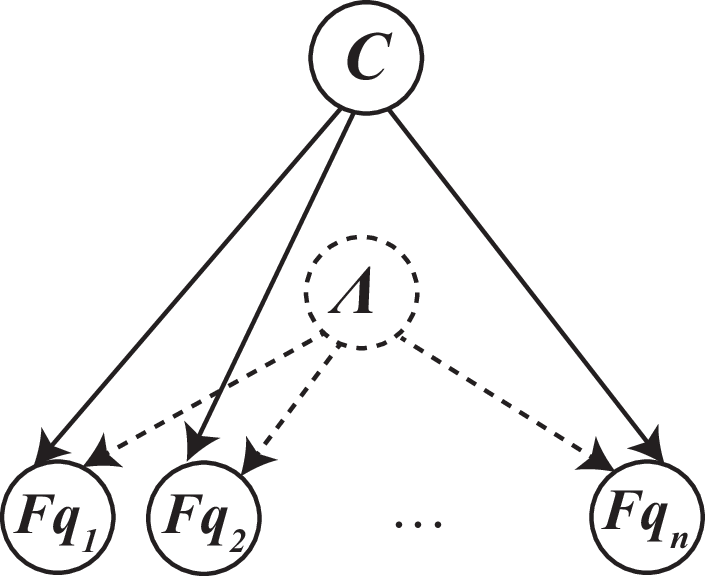}
        \caption{Bayesian Network representation of a canonical causal model.\label{fig:baysianNet}}
    \end{figure}    
    
\paragraph{M-contextuality and degree of signalling}We now have two ways of quantifying the ``direct influence'' of a system, namely using the ``degree of signalling'' $\Delta$ from CbD or by using the minimum amount of contextual direct influence $\Delta_{c,c'}(F_q)$ allowed for each content $q$. As it turns out, these quantities are intrinsically related, and the following is true:
    
\begin{prop}\label{prop:deltas}
    For a cyclic system with binary random variables taking values in $\{\pm1\}$, we have:
    \begin{equation}\label{eq:directInfluence}
        \Delta = 2 \sum_{q} \Delta^*_{c_q,c_q'} \left(F_q\right)
    \end{equation}
    where $\Delta^*_{c_q,c'_q}$ is the minimum direct influence of the contexts $c_q,c'_q$ associated with content $q$ across all canonical models compatible with the observed distributions.
\end{prop}

The proof of this proposition can be found in Appendix~\ref{app:deltas}.
    
We note that $\Delta$ is only defined for a small class of systems, while the RHS of \eqref{eq:directInfluence} applies to more general systems.

\paragraph{Direct influences and the signalling fraction}The notion of direct influence stems from similar motivations as the signalling fraction $\mathsf{SF}$ defined in Section~\ref{subsec:sheaves}, namely, what is the minimal ``part'' of the observed statistics which can be explained by a no-signalling system. We here show that the signalling fraction gives an upper bound of the degree of signalling $\Delta$ in some particular circumstances (including in cyclic systems) and that in the general case, the signalling fraction gives us an upper bound for all of the degree of direct influences in a system\footnote{The results represented here are original at the time of submission of the thesis.}. First, we start with an important remark.
\begin{rmk}[No-signalling and consistent-connectedness]\label{rmk:signalling}
    Somewhat surprisingly, although the notion of consistent-connectedness, as defined in~\cite{Dzhafarov2016contextcontent}, is claimed to be equivalent to no-signalling, this is not generally the case. Indeed, a system is said to be consistently-connected iff, for any content $X$ in contexts $C, C'$, the marginals over $X$ of the probability distributions associated with $C$ and $C'$ are the same. However, in general, the notion of no-signalling is stated as, for any subset $\left\{X_i\right\}_{i\in I}$ of contents such that $X_i\in C$ and $X_i\in C'$ for all $i\in I$, then the marginal distribution restricted to all of the $X_i$ coincides. Hence, this distinction only applies if there exist contexts $C, C'$ such that $\left|C\cap C'\right|>1$.

    Here is an example of an empirical model which is consistently-connected but signalling.
    \begin{center}
        \begin{tabular}{cc|cccc}
             & & (0,0) & (0,1) & (1,0) & (1,1) \\\hline
            $a$ & $b$ & 1/2 & 0 & 0 & 1/2\\ 
            $a$ & $b$ & 0 & 1/2 & 1/2 & 0\\
        \end{tabular}    
        \end{center}
\end{rmk}

Having cleared up this distinction, we then state the following results.
\begin{prop}[Signalling fraction and degrees of direct influences]\label{prop:SFDelta}
    Given statistics of a system for the contexts $\left\{C_i\subseteq X\right\}_{i\in I}$ for individual measurements (i.e. contents) $X$, we have:
    \begin{enumerate}[label=\alph*.]
        \item For any system, we have:
        \begin{equation}
            \max_{x\in X} \Delta^*_{C,C'}(x) \leq  \mathsf{SF}
        \end{equation}
        \item If the choices of contexts satisfies $\left|C_i\cap C_j\right|\leq 1$ for all $i,j\in I$ and $i\neq j$, then:
        \begin{equation}
            \max_{x\in X} \Delta^*_{C,C'}(x) =  \mathsf{SF}
        \end{equation}
    \end{enumerate}
\end{prop}

The proof of these claims can be found in Appendix~\ref{app:SFDelta}. Moreover, since we already had a relationship between the degrees of direct influence $\Delta^*_{C, C'}(x)$ and the overall degree of signalling $\Delta$ (Proposition~\ref{prop:deltas}), the next results follows.
\begin{cor}
    In a cyclic system of rank $n$, we have:
    \begin{equation}
        \Delta \leq 2~\mathsf{SF}
    \end{equation}
\end{cor}

\subsubsection{From Contextuality-by-Default to sheaf-theoretic contextuality}
\paragraph{}In~\cite{Szhafarov2023sheaf}, the author proposed a way of describing signalling empirical models in terms of no-signalling ones within the sheaf-theoretic framework, such that the notion of contextuality in these generated empirical models is equivalent to the notion of contextuality within the Contextality-by-Default framework. This mechanism for creating no-signalling models was coined as \emph{consistentification}. We here briefly describe this procedure.
    
Recall that in CbD, a cyclic system is contextual is non-contextual whenever it is possible to impose a global probability distribution on the system such that the probabilities $P\left[S^i_q = S^{i'}_q\right]$ are simultaneously maximised. This condition can expressed as the possibility of imposing a joint probability distribution on pairs of variables of different contexts that share a content, such that$P\left[S^i_q = v, S^{i'}_q = v\right]$ are minimal for every outcome $v$, and for which marginals coincides with the marginals of the observed variables. 

Hence, the process of consistentification consists of creating a new system for which both the contexts and contents of the original system are measurement contexts. The set of observable $X$ is therefore defined as $X = \left\{\left(q_i,c^j\right)\middle|q_i\in c^j\right\}$, and CbD-contexts and CbD-content correspond to the following set of measurement contexts:
\begin{align}
    \mathcal{M}_c =& \left\{\left\{\left(q_i,c^j\right)\in X\right\}\middle|c^j \text{ is a CbD-context}\right\}\\
    \mathcal{M}_q =& \left\{\left\{\left(q_i,c^j\right)\in X\right\}\middle|q_i \text{ is a CbD-content}\right\}
\end{align}
The probability distributions over the measurement contexts of $\mathcal{M}_c$ are defined as before, i.e. correspond to observed probability distributions. On the other hand, the probability distributions over the measurement contexts of $\mathcal{M}_q$ will be obtained by imposing minimal direct influences on each of the individual contents. This correspondance is illustrated in Fig~\ref{fig:consistentification}. By definition of the $S^i_q$ from above, this system is no-signalling, i.e., consistently connected. 

\begin{figure}[htb]
    \centering
    \begin{tikzpicture}
        \node (old) at (-4, 0) {\includegraphics[scale=.5]{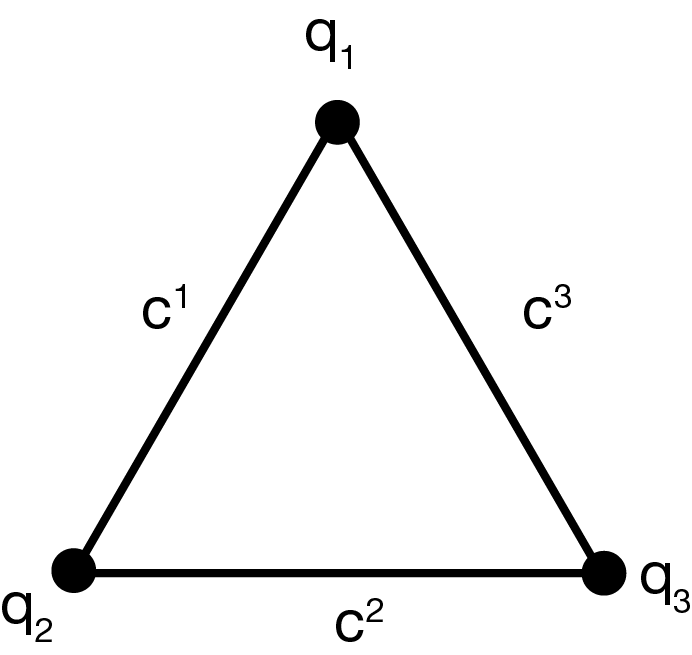}};
        \node (new) at (4, 0) {\includegraphics[scale=.5]{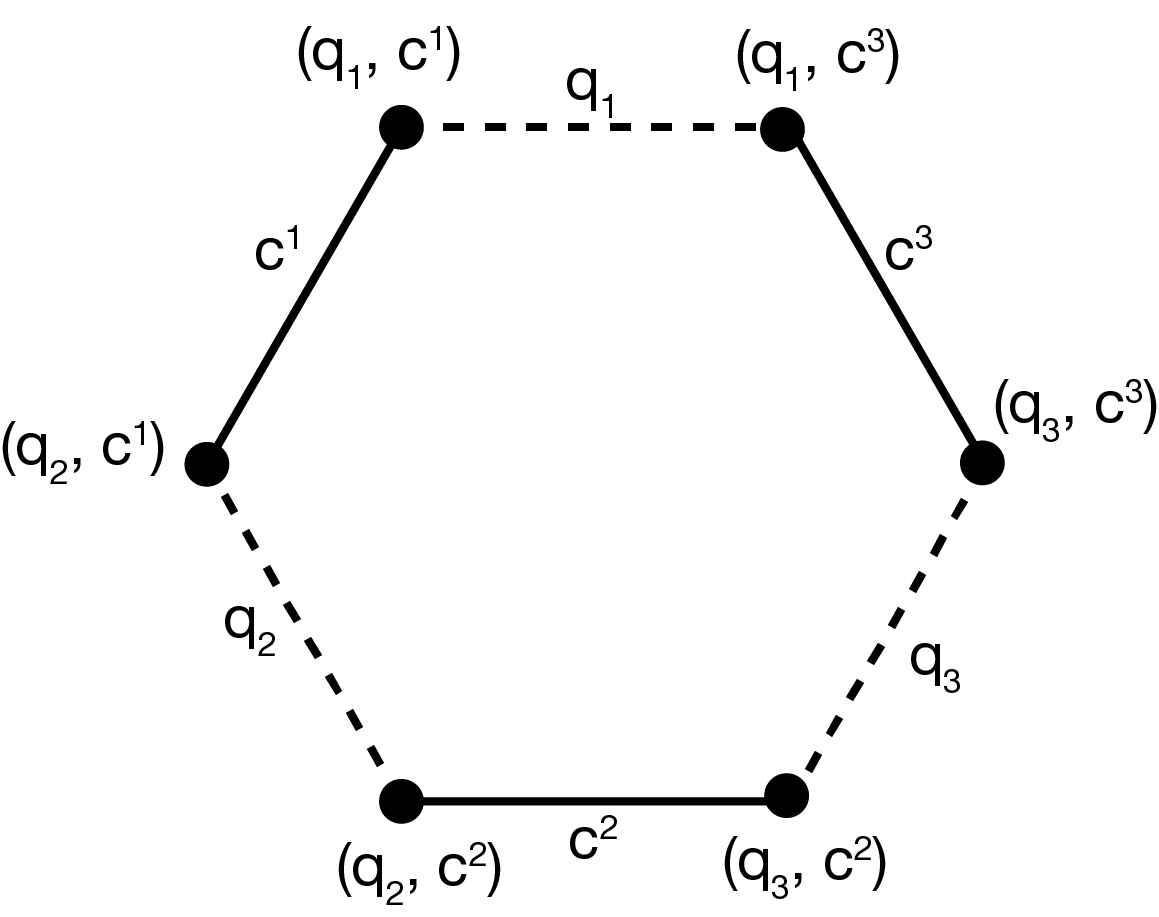}};

        \draw[->] (old) edge node[above]{\textit{consistentification}} (new);
    \end{tikzpicture}
    \caption{Correspondance between the original measurement scenario (left), and the consistentified one (right). On the latter, the solid measurement contexts are the ones inherited from the left-hand measurement scenario, whilst the dashed ones are the ones created from the minimal direct influence condition.\label{fig:consistentification}}
\end{figure}

Moreover, the criterion of CbD-contextuality in the original system is, by design, the same as the sheaf-theoretic criterion of contextuality in the generated one.
    
\section{Quantum Mechanics as a Process Theory}\label{sec:QProcess}
\paragraph{}The goal of this section is to motivate and introduce the notation of Chapter~\ref{chap:lexicalCircuits}. Unlike the previous section, we will now assume the standard Hilbert space formalism of quantum mechanics and provide a categorical description of quantum states and operations. 


\subsection{Features of quantum processes}
\paragraph{} In Section~\ref{subsec:monoidal}, we have seen that monoidal categories are very useful in describing processes that can be composed sequentially (modelled using the sequential composition of morphisms) and in parallel (using the monoidal product $\otimes$). In particular, in the case of quantum processes, we will mainly focus on the category of Hilbert spaces $\mathbf{Hilb}$. Moreover, in the case of quantum computing, we can restrict ourselves to the category of finite-dimensional Hilbert spaces $\mathbf{FdHilb}$, which also has some additional ``nice'' properties, such as the existence of orthonormal bases for all objects of $\mathbf{FdHilb}$.

\paragraph{}We start by describing the intepretation of objects of $\mathbf{FdHilb}$ in terms of quantum mechanics. A Hilbert space $\mathcal{H}\in ob\left(\mathbf{FdHilb}\right)$ will correspond to a quantum system, and the dimension of $\mathcal{H}$ will correspond to the dimension of the quantum system in question. For example, a qubit will live in the 2-dimensional Hilbert space $\mathbb{C}^2$. 

A (pure) quantum state living in the Hilbert space $\mathcal{H}$ will then be represented as a vector $\ket{\psi}: I\to \mathcal{H}$ (where we recall that the monoidal unit for Hilbert spaces is $I = \mathbb{C}$), while their duals will be represented as $\bra{\psi}: \mathcal{H}\to I$. The latter corresponds to the outcome of some measurement. 

These states are not necessarily normalised, as they should in standard quantum mechanics. Therefore, the set of \emph{physical states} will consists of states $\ket{\psi}: I\to \mathcal{H}$ such that:
\begin{equation}\label{eq:QuantumNorm}
    \left|\begin{tikzpicture}
	\begin{pgfonlayer}{nodelayer}
		\node [style=elemt] (0) at (0, 1) {$\psi$};
		\node [style=coelemt] (1) at (0, -1) {$\psi$};
	\end{pgfonlayer}
	\begin{pgfonlayer}{edgelayer}
		\draw (0) to (1);
	\end{pgfonlayer}
\end{tikzpicture}
\right|^2 = 1
\end{equation}

\paragraph{}Intuitively, we would then want morphisms in $\mathbf{FdHilb}$ to represent quantum operations. This is not quite the case as, physically, the valid operations of (pure) quantum states will only be unitary operations, whereas the morphisms of $\mathbf{FdHilb}$ are in general bounded linear maps. The unitarity condition ensures that the operations preserve the normalisation of states \eqref{eq:QuantumNorm}. Hence, the set of \emph{physical operations} on a system $\mathcal{H}\in ob\left(\mathbf{FdHilb}\right)$ will consists on the morphisms $U: \mathcal{H}\to \mathcal{H}$ such that:
\begin{equation}
    \input{Background/tikzit/unitarityA.tikz} = \input{Background/tikzit/unitarityB.tikz} = \begin{tikzpicture}
	\begin{pgfonlayer}{nodelayer}
		\node [style=none] (2) at (0, -2) {};
		\node [style=none] (3) at (0, 2) {};
		\node [style=none] (5) at (0.5, -1.75) {$\mathcal{H}$};
	\end{pgfonlayer}
	\begin{pgfonlayer}{edgelayer}
		\draw (2.center) to (3.center);
	\end{pgfonlayer}
\end{tikzpicture}
 
\end{equation}
\begin{ex}\label{ex:stdGates}
    \begin{enumerate}[label=\alph*]
        \item The following are standard examples of one-qubit gates:
        \begin{gather*}
            \begin{tikzpicture}
	\begin{pgfonlayer}{nodelayer}
		\node [style=box] (0) at (0, 0) {$X$};
		\node [style=none] (1) at (0, 1) {};
		\node [style=none] (2) at (0, -1) {};
	\end{pgfonlayer}
	\begin{pgfonlayer}{edgelayer}
		\draw (1.center) to (2.center);
	\end{pgfonlayer}
\end{tikzpicture}
 = \begin{pmatrix}
                0 & 1 \\ 1 & 0
            \end{pmatrix} \qquad \begin{tikzpicture}
	\begin{pgfonlayer}{nodelayer}
		\node [style=box] (0) at (0, 0) {$Y$};
		\node [style=none] (1) at (0, 1) {};
		\node [style=none] (2) at (0, -1) {};
	\end{pgfonlayer}
	\begin{pgfonlayer}{edgelayer}
		\draw (1.center) to (2.center);
	\end{pgfonlayer}
\end{tikzpicture}
 = \begin{pmatrix}
                0 & -i\\i& 0
            \end{pmatrix}\\
            \begin{tikzpicture}
	\begin{pgfonlayer}{nodelayer}
		\node [style=box] (0) at (0, 0) {$Z$};
		\node [style=none] (1) at (0, 1) {};
		\node [style=none] (2) at (0, -1) {};
	\end{pgfonlayer}
	\begin{pgfonlayer}{edgelayer}
		\draw (1.center) to (2.center);
	\end{pgfonlayer}
\end{tikzpicture}
 = \begin{pmatrix}
                1 & 0 \\ 0 & -1
            \end{pmatrix} \qquad \begin{tikzpicture}
	\begin{pgfonlayer}{nodelayer}
		\node [style=box] (0) at (0, 0) {$Y$};
		\node [style=none] (1) at (0, 1) {};
		\node [style=none] (2) at (0, -1) {};
	\end{pgfonlayer}
	\begin{pgfonlayer}{edgelayer}
		\draw (1.center) to (2.center);
	\end{pgfonlayer}
\end{tikzpicture}
 = \frac{1}{\sqrt{2}}\begin{pmatrix}
                1 & 1\\1& -1
            \end{pmatrix}
        \end{gather*}
        \item For any unitary $U$ over $\mathcal{H}$, we can also define the controlled unitary $cU$ acting on $\mathbb{C}^2\otimes \mathcal{H}$ as follows:
        \begin{equation}
            \input{Background/tikzit/cU.tikz} = \begin{pmatrix}
                1 & 0 & \underline{0}\\
                0 & 1 & \underline{0}\\
                \underline{0} & \underline{0} & \mathbf{U}
            \end{pmatrix}
        \end{equation}
        The intuition is that if the control qubit is in the state $\ket{0}$, the identity on $\mathcal{H}$ is applied, whereas if the control qubit is in the state $\ket{1}$, the unitary $U$ is applied to the target space $\mathcal{H}$.
    \end{enumerate}
\end{ex}
Using the monoidal structure, we can also compose morphisms in parallel using the monoidal product $\otimes$. If two operations $U$ and $V$ are composed in parallel as $U\otimes V$, then it is understood that they are done independently. Now, we have already seen some special (non-unitary) morphisms $\ket{\psi}: I \to \mathcal{H}$ which represent quantum states. Two states $\ket{\psi}: I\to \mathcal{H}_1$ and $\ket{\phi}: I\to \mathcal{H}_2$ can also be composed in parallel as $\ket{\psi}\otimes \ket{\phi}: I\otimes I \to \mathcal{H}_1\otimes \mathcal{H}_2$. As with processes, these states are considered independent and known as \emph{product states}. However, given the compound system $\mathcal{H} = \mathcal{H}_1\otimes \mathcal{H}_2$, not all of the states in $\mathcal{H}$ will be product states. For instance, taking $\mathcal{H}_1 = \mathcal{H}_2 = \mathbb{C}^2$, the Bell state $\ket{\Psi} = \frac{1}{\sqrt{2}}\left(\ket{0}\otimes \ket{0} + \ket{1}\otimes \ket{1}\right)$ cannot be decomposed as $\ket{\Psi} = \ket{\psi}\otimes\ket{\phi}$. States that cannot be decomposed as a product state are called \emph{entangled states}.

\paragraph{}In addition, we have seen in Section~\ref{subsec:monoidal} that the category $\mathbf{FdHilb}$ also satisfies some extra properties. For instance $\mathbf{FdHilb}$ is a symmetric monoidal cateogry, i.e. for any two $\mathcal{H}_A, \mathcal{H}_B\in ob\left(\mathbf{FdHilb}\right)$, we have $\mathcal{H}_A\otimes \mathcal{H}_B \cong \mathcal{H}_B \otimes \mathcal{H}_A$. The interpretation of this property in terms of quantum systems is quite natural, namely that swapping quantum systems leads to an ``essentially equivalent'' system in the sense that both swapped and unswapped compound systems share the same physical properties.

Furthermore, the category $\mathbf{FdHilb}$ is compact-closed, and the duality structure proved very useful. In particular, it is widely known in the quantum mechanics literature that the set of operations $U: \mathcal{H}_A\to \mathcal{H}_B$ is isomorphic to a set of states $\ket{\Psi}: I \to \mathcal{H}_A \otimes \mathcal{H}_B$; this correspondence is known as the \emph{Choi-Jamio\l{}kowsky isomorphism}\footnote{Here, the isomorphism is understood at the level of sets, i.e. there exists a bijection between the set of quantum states, and the set of quantum operations.}. This equivalence can easily be seen in terms of string diagrams, using the units of the duality as:
\begin{equation}
    \input{Background/tikzit/CJLHS.tikz} \quad \xrightarrow{\simeq} \quad \input{Background/tikzit/CJRHS.tikz}
\end{equation}
If $U$ is a unitary map, the obtained state under this isomorphism is a \emph{maximally entangled state}, as they can maximally violate the Bell inequalities introduced in Section~\ref{sec:QDescription}.

\begin{ex}[Bell states]
    In two-qubit systems, certain maximally entangled states are important in quantum protocols such as \emph{quantum teleportation}. These are known as \emph{Bell states} and are defined as:
    \begin{equation*}
        \begin{matrix}
            \ket{\Phi_+}=& \input{Background/tikzit/PhiP.tikz}&\qquad&\ket{\Phi_-}=& \input{Background/tikzit/PhiM.tikz}\\
        \ket{\Psi_+}=& \input{Background/tikzit/PsiP.tikz}&\qquad&\ket{\Psi_-}=& \input{Background/tikzit/PsiM.tikz}
        \end{matrix}
    \end{equation*}
    (up to normalisation and global phase factors).
\end{ex}

\subsection{Mixed states and density matrices}
\paragraph{}So far, we have only considered pure quantum states subject to unitary transformation. In realistic systems, however, the quantum states will sometimes interact with their environment in a manner that is not always known or controlled. To deal with this situation, we use \emph{density matrices} and \emph{quantum channels} instead of pure states and unitaries. These are obtained by ``forgetting'' about the subsystem corresponding to the environment. This construction leads to quantum states which are a probabilistic mixture of pure states. We will here describe the categorical way of defining these states and operations.

\paragraph{}Before looking at density matrices, we introduce the categorical notion of \emph{trace} of a morphism.
\begin{defs}
    In a compact closed category $\mathcal{C}$, the \emph{trace} of a morphism $f: A\to A$, with $A\in ob\left(\mathcal{C}\right)$, is given by:
    \begin{equation}
        Tr(f) = \input{Background/tikzit/Trf.tikz}
    \end{equation}
\end{defs}
Using the unit and counit definitions from Section~\ref{subsec:monoidal} for vector spaces, it can easily be verified that the trace corresponds to the standard trace of square matrices. Given a morphism on a compound space $f: A\otimes B \to A\otimes B$, we can also take the \emph{partial trace} of $f$ w.r.t to a subsystem, i.e. $A$ or $B$. This is defined categorically as:
\begin{align}
    Tr_A(f) =& \input{Background/tikzit/TrAf.tikz}\\
    Tr_B(f) =& \input{Background/tikzit/TrBf.tikz}
\end{align}
The intuition is that taking the partial trace of a subsystem corresponds to averaging out the behaviour of the traced-out subsystem.

\paragraph{}Now, for a pure state $\ket{\psi}: I \to \mathcal{H}$, we define its associated density matrix as $\rho = \ket{\psi}\bra{\psi}$. In string diagrammatic notation, this gives us:
\begin{equation}
    \begin{tikzpicture}
	\begin{pgfonlayer}{nodelayer}
		\node [style=elemt] (0) at (0, -1.25) {$\psi$};
		\node [style=coelemt] (1) at (0, 1.25) {$\psi^\dagger$};
		\node [style=none] (2) at (0, -2.25) {};
		\node [style=none] (3) at (0, 2.25) {};
	\end{pgfonlayer}
	\begin{pgfonlayer}{edgelayer}
		\draw (0) to (2.center);
		\draw (3.center) to (1);
	\end{pgfonlayer}
\end{tikzpicture}
 \quad\xrightarrow{\simeq}\quad \begin{tikzpicture}
	\begin{pgfonlayer}{nodelayer}
		\node [style=elemt] (0) at (-1, 0.75) {$\psi^*$};
		\node [style=none] (2) at (-1, -0.75) {};
		\node [style=elemt] (3) at (1, 0.75) {$\psi$};
		\node [style=none] (4) at (1, -0.75) {};
	\end{pgfonlayer}
	\begin{pgfonlayer}{edgelayer}
		\draw (0) to (2.center);
		\draw (3) to (4.center);
	\end{pgfonlayer}
\end{tikzpicture}

\end{equation} 
(where the isomorphism is the Choi-Jamio\l{}kowsky isomorphism). Then, if a state in $\mathcal{H}$ is entangled (i.e. correlated) to an environment modelled as a system $\mathcal{H}_E$, then the global quantum state is represented by a (possible entangled) state $\ket{\Psi}: I \to \mathcal{H}_E\otimes \mathcal{H}$. If the details of the interaction with the environment are not known, then the state will be represented as a density matrix where the environment subsystem is traced out, i.e.:
\begin{equation}
    \rho = \input{Background/tikzit/densityMatrix.tikz} \quad\xrightarrow{\simeq}\quad \input{Background/tikzit/densityMatrixIso.tikz}
\end{equation}
Any state of this form will be known as a \emph{density matrix}. In particular, the analogue of the normalisation condition of \eqref{eq:QuantumNorm} becomes:
\begin{equation}\label{eq:DensityNorm}
    Tr(\rho) = \input{Background/tikzit/trRho.tikz} = 1
\end{equation}
As for operations on pure states, we will also restrict permissible operations on density matrices. In particular, we would want these operations to send density matrices to density matrices. Moreover, if an operation is only applied to a subsystem of a larger quantum state, we would also like the resulting state to remain a density matrix. The operators satisfying these conditions are known as \emph{completely positive operators}. From a well-known theorem known as \emph{Stinespring dilation theorem}~\cite{Stinespring1955}, every completely positive map $V: \mathcal{H}_A\otimes \mathcal{H}_A \to \mathcal{H}_B\otimes \mathcal{H}_B$ can be decomposed as:
\begin{equation}
    \input{Background/tikzit/Stinespring.tikz}
\end{equation}
where $U$ is a unitary transformation, and $\mathcal{K}$ is a Hilbert space. In addition, in order to preserve the normalisation condition of \eqref{eq:DensityNorm}, we will say that a \emph{quantum channel} is a complete positive map which is \emph{trace preserving}.

\paragraph{}The above construction of density matrices and completely positive operators can be generalised to any dagger compact closed categories, using the so-called $\mathbf{CPM}$ construction~\cite{Selinger2007}. This construction works as follows. Given a compact closed category $\mathcal{C}$, we create a new (also dagger compact-closed) category $\mathbf{CPM}(\mathcal{C})$ such that:
\begin{itemize}
    \item Objects of $\mathbf{CPM}(\mathcal{C})$ are the same as the ones of $\mathcal{C}$
    \item Morphisms $f:A\to B$ in $\mathbf{CPM}(\mathcal{C})$ are morphisms $\tilde{f}:A^*\otimes A \to B^*\to B$ in $\mathcal{C}$ such that there exists $C\in ob\left(\mathcal{C}\right)$ and $g: A\otimes C \to B$ such that:
    \begin{equation}
        \input{Background/tikzit/CPMmorph.tikz} = \input{Background/tikzit/CPMmorphdecomp.tikz}
    \end{equation}
\end{itemize}
The categorical, monoidal, dagger, and compact closed structure of $\mathbf{CPM}(\mathcal{C})$ will be inherited from $\mathcal{C}$.

To emphasize the distinction between the string diagrams in the two categories $\mathbf{FdHilb}$ and $\mathbf{CPM}\left(\mathbf{FdHilb}\right)$, we will denote the wires in $\mathbf{CPM}\left(\mathbf{FdHilb}\right)$ as thicker than the ones of $\mathbf{FdHilb}$, e.g.:
\begin{equation*}
    \begin{tikzpicture}
	\begin{pgfonlayer}{nodelayer}
		\node [style=none] (0) at (0, 1) {};
		\node [style=none] (1) at (0, -1) {};
		\node [style=none] (2) at (0.5, -0.75) {$A$};
	\end{pgfonlayer}
	\begin{pgfonlayer}{edgelayer}
		\draw (0.center) to (1.center);
	\end{pgfonlayer}
\end{tikzpicture}
 \in ob\left(\mathbf{FdHilb}\right) \qquad\qquad \begin{tikzpicture}
	\begin{pgfonlayer}{nodelayer}
		\node [style=none] (0) at (0, 1) {};
		\node [style=none] (1) at (0, -1) {};
		\node [style=none] (2) at (0.5, -0.75) {$A$};
	\end{pgfonlayer}
	\begin{pgfonlayer}{edgelayer}
		\draw [style=thickline] (0.center) to (1.center);
	\end{pgfonlayer}
\end{tikzpicture}
\in ob\left(\mathbf{CPM}\left(\mathbf{FdHilb}\right)\right)
\end{equation*}
Therefore, the following will represent density matrices and quantum channels respectively:
\begin{equation*}
    \begin{tikzpicture}
	\begin{pgfonlayer}{nodelayer}
		\node [style=elemt] (0) at (0, 0.75) {$\rho$};
		\node [style=none] (1) at (0, -0.75) {};
	\end{pgfonlayer}
	\begin{pgfonlayer}{edgelayer}
		\draw [style=thickline] (0) to (1.center);
	\end{pgfonlayer}
\end{tikzpicture}
 \qquad \qquad \begin{tikzpicture}
	\begin{pgfonlayer}{nodelayer}
		\node [style=box] (0) at (0, 0) {$U$};
		\node [style=none] (1) at (0, 1) {};
		\node [style=none] (2) at (0, -1) {};
	\end{pgfonlayer}
	\begin{pgfonlayer}{edgelayer}
		\draw [style=thickline] (1.center) to (2.center);
	\end{pgfonlayer}
\end{tikzpicture}

\end{equation*}
Finally, the unit and counit from duality in $\mathbf{FdHilb}$ will be denoted in $\mathbf{CPM}\left(\mathbf{FdHilb}\right)$ as:
\begin{equation}
    \begin{tikzpicture}[circuit ee IEC]
	\begin{pgfonlayer}{nodelayer}
		\node [style=codiscard] (0) at (0, 0.75) {};
		\node [style=none] (1) at (0, -0.75) {};
		\node [style=none] (2) at (0.5, -0.5) {$\mathcal{H}$};
	\end{pgfonlayer}
	\begin{pgfonlayer}{edgelayer}
		\draw [style=thickline] (0) to (1.center);
	\end{pgfonlayer}
\end{tikzpicture}
 = \input{Background/tikzit/etaHilb.tikz} \qquad \begin{tikzpicture}[circuit ee IEC]
	\begin{pgfonlayer}{nodelayer}
		\node [style=discard] (0) at (0, -0.75) {};
		\node [style=none] (1) at (0, 0.75) {};
		\node [style=none] (2) at (0.5, 0.5) {$\mathcal{H}$};
	\end{pgfonlayer}
	\begin{pgfonlayer}{edgelayer}
		\draw [style=thickline] (0) to (1.center);
	\end{pgfonlayer}
\end{tikzpicture}
 = \input{Background/tikzit/epsilonHilb.tikz}
\end{equation}
Notably, the map \begin{tikzpicture}[circuit ee IEC]
	\begin{pgfonlayer}{nodelayer}
		\node [style=discard] (0) at (0, -0.75) {};
		\node [style=none] (1) at (0, 0.75) {};
		\node [style=none] (2) at (0.5, 0.5) {$\mathcal{H}$};
	\end{pgfonlayer}
	\begin{pgfonlayer}{edgelayer}
		\draw [style=thickline] (0) to (1.center);
	\end{pgfonlayer}
\end{tikzpicture}
 will be referred to as the \emph{discard map}, as we have seen that taking the (partial) trace corresponds to forgetting about the exact behaviour of the traced-out subsystem. We will investigate this property in more detail in the following subsection.

\subsection{Causality in quantum processes}\label{subsec:Qcausal}
\paragraph{}So far, we have not described any notion of temporal order, as the roles of inputs and outputs in any diagram can be reversed employing the duality and dagger structures. We here want to further restrict the set of physical diagrams by imposing a principle of causality. It turns out that this can be done using the discarding process described above.

\paragraph{}We start with a relatively simple intuition. If discarding the entire output of a process would correspond to physically ignoring the outcome of a process, this should be equivalent to ignoring the process in question~\cite{KissingerHobanCoecke,Kissinger2019}. In terms of string diagrams, this translates as:
\begin{equation}
    \input{Background/tikzit/boxdiscard.tikz} = \begin{tikzpicture}[circuit ee IEC]
	\begin{pgfonlayer}{nodelayer}
		\node [style=discard] (0) at (0, -1) {};
		\node [style=none] (1) at (0, 0.75) {};
		\node [style=none] (2) at (0.5, -0.25) {$A$};
	\end{pgfonlayer}
	\begin{pgfonlayer}{edgelayer}
		\draw [style=thickline] (1.center) to (0);
	\end{pgfonlayer}
\end{tikzpicture}

\end{equation}
Processes that satisfy this property will be called \emph{causal processes}.

\paragraph{}It has been shown that restricting to causal processes is enough to be able to encode the notion of causal relations as described in Section~\ref{subsec:BackgroundContext}~\cite{KissingerHobanCoecke}. First, we will describe the analogues of a party as a pair of input and output systems $A_{in}, A_{out}$. Then, choosing an input will correspond to selecting an input state in $\ket{\psi_{in}}: I \to A_{in}$, and similarly, observing an outcome will correspond to $\bra{\psi_{out}}: A_{out}\to I$. Now, given two parties $A = \left(A_{in}, A_{out}\right)$ and $B = \left(B_{in}, B_{out}\right)$, the interaction between $A$ and $B$ (regardless of a potential causal order) will be modelled as a causal morphism $f: A_{in}\otimes B_{in}\to A_{out}\otimes B_{out}$.

Then, we will say that a process $f$ is compatible with $A\preceq B$ iff:
\begin{equation}\label{eq:terminalityA}
    \input{Background/tikzit/ApreceqB.tikz} = \input{Background/tikzit/ApreceqBRHS.tikz}
\end{equation}
Intuitively, the above equation means that ignoring the subsystem $B$ does not change what happens in the party $A$. The generic form of such processes is given by~\cite{KissingerHobanCoecke,lorenz2023causal,jacobs2019causal}:
\begin{equation}
    \input{Background/tikzit/ApreceqBGeneric.tikz}
\end{equation}
where $f_A$ and $f_B$ are causal processes. When considering the statistics of the measurements of such processes, this notion of causality is indeed equivalent to the notion of compatibility with the causal order $A\preceq B$ as defined in Section~\ref{subsec:BackgroundContext} (see Appendix~\ref{app:causality} for more details). 

Similarly, we can also say that a process is compatible with $B\preceq A$ iff:
\begin{equation}\label{eq:terminalityB}
    \input{Background/tikzit/BpreceqA.tikz} = \input{Background/tikzit/BpreceqARHS.tikz}
\end{equation}
And the generic structure of such processes will be given by:
\begin{equation}
    \input{Background/tikzit/BpreceqAGeneric.tikz}
\end{equation}
for causal processes $f_A$ and $f_B$. These measurement statistics of these processes are also  compatible with $B\preceq A$ with respect to the causal relation defintion from Section~\ref{subsec:BackgroundContext} (see Appendix~\ref{app:causality}).

Finally, a process will be said to be no-signalling iff it is compatible with both $A\preceq B$ and $B\preceq A$. The generic structure of no-signalling processes consists of the (monoidal) product of two processes:
\begin{equation}
    \input{Background/tikzit/NSGeneric.tikz}
\end{equation}
We will use these process structures in Chapter~\ref{chap:lexicalCircuits}.


\chapter{Ambiguities in Natural Languages}\label{chap:ambiguities}
\paragraph{}The English language is ambiguous in several different ways. Examples of the different types of ambiguities are:
\begin{itemize}
    \item \textbf{Lexical ambiguity}\\
    A word is lexically ambiguous whenever we can interpret it in at least two ways. For example, the word \textit{bank} can either mean a financial institution or the side of a body of water in~\ref{item:lexicalAmb}.
    \begin{enumerate}[label=(\arabic*), series=ambiguities]
        \item \label{item:lexicalAmb} The \underline{bank} is far away.
    \end{enumerate}
    \item \textbf{Syntactic ambiguity}\\
    A phrase or a sentence is syntactically ambiguous iff it has at least two possible grammatical structures. An example of a syntactically ambiguous sentence is as follows:
    \begin{enumerate}[label=(\arabic*), resume=ambiguities]
        \item \label{item:syntacticAmb} She saw a man \underline{with binoculars}.
    \end{enumerate}
    In~\ref{item:syntacticAmb}, all of the words have definite meanings, but the phrase \textit{with binoculars} can be attached to either \textit{She}, i.e. there is a woman who used binoculars to see a man, or \textit{a man}, i.e. what the woman saw is a man who was using binoculars.
    \item \textbf{Coreference ambiguity}\\
    Texts usually include references to previously mentioned elements. For example, every pronoun such as \textit{he, she, it} or \textit{they} refers to entities defined from the context. These references can also lead to ambiguous utterances, such as the following sentence:
    \begin{enumerate}[label=(\arabic*), resume=ambiguities]
        \item \label{item:corefAmb} I put the CD in the computer before \underline{it} broke.
    \end{enumerate}
    In~\ref{item:corefAmb}, the pronoun \textit{it} can equally refer to either \textit{the CD} or \textit{the computer}. 
\end{itemize}

The existence of these ambiguities poses some challenges in NLP, as many of them require knowledge of the world to be disambiguated. For instance, co-reference ambiguities have led to the Winograd Schema challenge, which consists of sentences such as:
\begin{enumerate}[label=(4\alph*)]
    \item The trophy didn't fit in the suitcase because \underline{it} was too big.
    \item The trophy didn't fit in the suitcase because \underline{it} was too small.
\end{enumerate}
The challenge is then to identify which of the trophy or the suitcase is referred to by the pronoun \textit{it} in each sentence.

In this work, we will focus on studying lexical and syntactic ambiguity. However, similar work has been done regarding co-reference ambiguity (see~\cite{Lo2022,Lo2023}).

\paragraph{}In Section~\ref{sec:lexicalBack}, we describe how computers and humans process lexically ambiguous words (Sections~\ref{subsec:WSD} and~\ref{subsec:lexicalPsycho} respectively). We focus on syntactic ambiguities in Section~\ref{sec:SyntacticBackground}. In particular, we will introduce the theories of human parsing and the significance of garden-path sentences in Section~\ref{subsec:syntacticPsycho}, and in Section~\ref{subsec:surprisal}, we look at computational approaches to model human behaviour regarding the parsing of garden-path sentences.

\section{Lexical ambiguity in linguistics}\label{sec:lexicalBack}
\paragraph{}It is common for words to have several interpretations or multiple entries in a dictionary. When this is the case, the word is \emph{lexically ambiguous}. For example, the word \textit{charge} can be a verb or a noun. As a noun, it has, according to the Oxford Dictionary, the following main possible meanings:
\begin{enumerate}
    \item A material load; that which can be borne, taken, or received.
    \item A load of trouble, expense, responsibility, blame, etc.
    \item An impetuous attack
\end{enumerate}
Similarly, as a verb, the word \textit{charge} has the following possible meanings:
\begin{enumerate}
    \item To load; to cause to bear, hold, or receive.
    \item To load heavily; to burden, put anything onerous, troublesome, hateful upon.
    \item To attach weight to.
    \item To attack impetuously: and senses leading up to it.
\end{enumerate}
Each of these meanings could be further fine-grained, e.g. \textit{charge} in \textit{electrical charge} and in \textit{heavy charge} would both be included in the first definition of the noun \textit{charge} as above, but refer to different things.

Most words commonly used in English are ambiguous, and 99.6\% of the words in the British National Corpus~\cite{BNC} are ambiguous. However, this does not create a considerable obstacle for \emph{humans} to understand English. On the other hand, this problem constitutes a significant obstacle for machines in Natural Language Processing.

\paragraph{}In this section, we start by reviewing approaches to automatically disambiguate lexically ambiguous words in NLP (Section~\ref{subsec:WSD}) and then compare it with the process of human disambiguation as theorised in psycholinguistics (Section~\ref{subsec:lexicalPsycho}).
\subsection{The challenge of word-sense disambiguation}\label{subsec:WSD}
\paragraph{}Word Sense Disambiguation (WSD) is an NLP task that identifies which meaning of an ambiguous word is activated in a given context. WSD was one of the first challenges of NLP as it was crucial in Machine Translation~\cite {weaver1952translation}. To see this, consider a word that is ambiguous in the source language but not the target language, then it needs to be disambiguated before one can translate it, e.g. \textit{spring} the season is \textit{printemps} in French, but \textit{spring} the coil is translated as \textit{ressort}. Furthermore, it was shown that lexical disambiguation improves the accuracies of other NLP tasks such as Information Retrieval~\cite{Zhong2010,Bovi2015} and Question Answering~\cite{Ramakrishnan2003}.

WSD can be defined as follows. Given a text $T$ containing the target word $w$, the aim is to associate $w$ with its intended interpretation, usually taken from a set of definitions or labels of the different possible meanings of $w$. This task is particularly hard due to the apparent amount of knowledge required and the difficulty of obtaining annotated data.

We will now describe the main approaches in NLP aimed at the WSD task.

\subsubsection{Historic approaches}
\paragraph{} In~\cite{Schutze1998}, Sch\"utze divides the task of WSD into two subtasks: 
\begin{enumerate}[label=\arabic*.]
    \item \textbf{Word-sense discrimination} which aims to classify the contexts in which the intended interpretations are the same;
    \item \textbf{Word-sense labelling} which aims to label the different classes with a definition.
\end{enumerate}
In~\cite{Schutze1998}, the author proposes a completely unsupervised algorithm for solving the problem of word-sense discrimination. The idea behind the approach is similar to the one behind \emph{distributional vectors}. 

Indeed, the distributional hypothesis~\cite{Firth1957,Joos1950,Harris1954} dictates that similar words are found in similar contexts. From there, we can obtain vectorial representations of words, known as distributional vectors, by recording how often a target word $w$ co-occurs with other words in a corpus. These figures correspond to \emph{first-order co-occurences}~\cite{Schutze1998}. Similarly, we can obtain a representation of a context $c$ by collecting the vector representations of the words in the context $c$. These vectors correspond to a \emph{second-order co-occurences}. For first-order occurrences, it has been widely verified that semantically close words are associated with distributional vectors that are close in the vector space~\cite{jurafskyspeech}. Similarly, we will expect second-order co-occurrence vectors to be close whenever the contexts they represent are semantically close. 

The idea of~\cite{Schutze1998} is to identify clusters of context-vectors containing $w$ with different senses. Each sense will then be represented abstractly as the centroid of the associated cluster, i.e. by the average vector of all of the points in the cluster. Then, given a test context $c'$, we start by creating its context embedding and then select the sense associated with the cluster it is closest to. We, therefore, select the appropriate sense by calculating the distance with all of the centroids. This method resulted in fairly high accuracies ($77.9\%$)~\cite{Schutze1998}, but the obtained clusters do not align with the classifications made by humans, e.g. in dictionaries, and the results are therefore hard to interpret.

\begin{figure}[htb!]
    \centering
    \includegraphics[width=.5\linewidth]{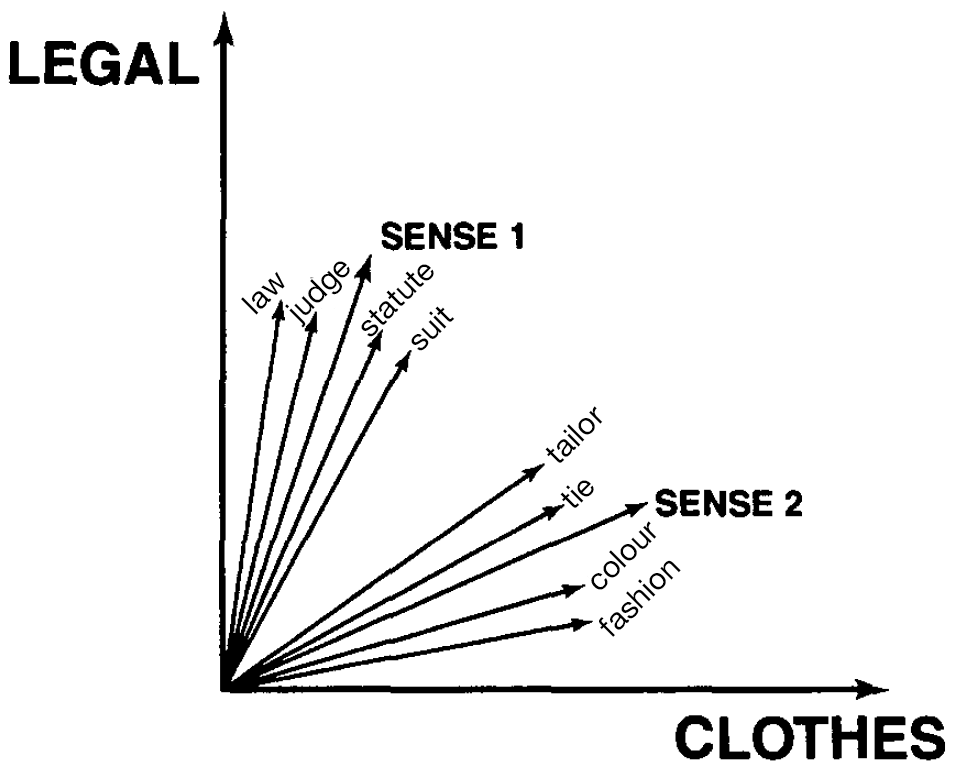}
    \caption{Clustering of the contexts for two senses of the word \textit{suit} (adapted from~\cite{Schutze1998}).\label{fig:Schutze}}
\end{figure}

\paragraph{}Most of the word-sense labelling approaches rely on supervised methods, i.e. dependent on human input or human annotations. 

Amongst the most successful supervised settings are \emph{Support Vector Machines} (SVMs)~\cite{Lee2002,Zhong2010,Iacobacci2016}, which were first introduced in~\cite{Boser1992}. SVMs are trained to discriminate positive and negative data points on a vector space by learning the linear hyperplane equation separating positively-labeled and negatively-labelled points. 
\begin{figure}[htb!]
    \centering
    \includegraphics[width=.5\linewidth]{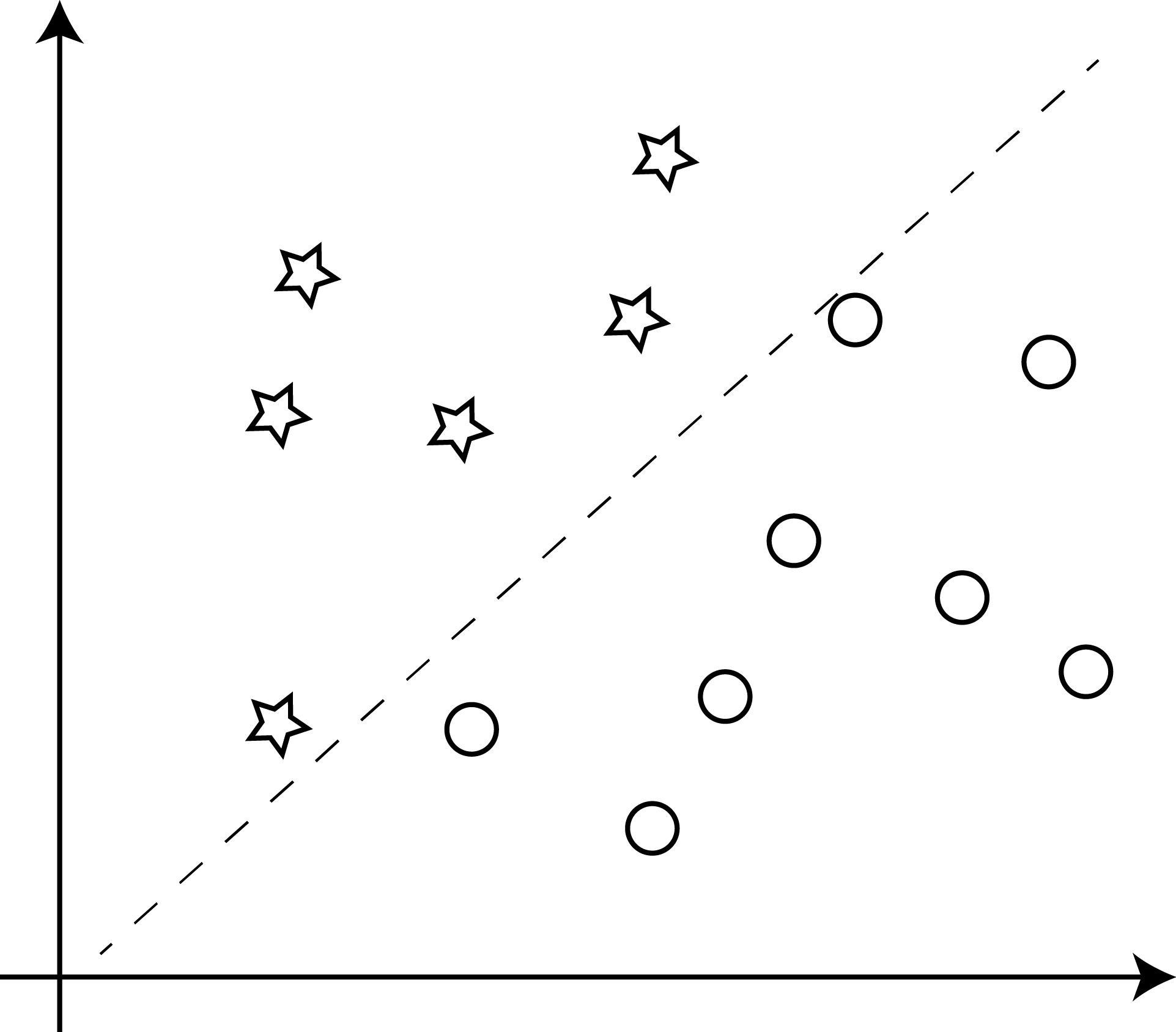}
    \caption{Illustration of Support Vector Machines. The two classes of data points (e.g. positive and negative) are depicted in different shapes (e.g. stars and circles). The aim of the SVM is to learn the equation of the dashed line.\label{fig:SVM}}
\end{figure}

In the case of the WSD task, we generally want to classify the data points between more than two classes (as a word may have more than two possible meanings). Hence, if a word $w$ has $k$ different senses, the disambiguating task is separated into $k$ binary classification tasks, where the $i$th task aims to identify whether the intended meaning of $w$ is its $i$th sense or not. The algorithm gives a confidence score for each sense of $w$, and the sense with the highest confidence score is then selected. 

The dimensions of the vector space correspond to features, e.g. the $n$-nearest neighbours of $w$ can be encoded in a 5-dimensional vector space. In particular, it has been shown that using features of different nature (e.g. neighbours $w$, part-of-speech of the neighbours, etc.) is beneficial in WSD tasks~\cite{Lee2002,Zhong2010}, see Fig.~\ref{fig:SVMexVector} for an example. 

The SVMs are then trained using annotated corpora, such as the SemCor corpus~\cite{Miller1993}, a subset of the Brown corpus where each word is annotated by its \texttt{WordNet} sense. \texttt{WordNet}~\cite{WordNet} is a database that contains definitions and examples of the different senses of a word, as well as semantic relations such as hypernymy/hyponymy or word similarity (see Fig.~\ref{fig:WordNet}). 

Many WSD systems were evaluated over the SensEval benchmark. Four different versions of the SensEval tasks have been published, namely SensEval~\cite{Edmonds2002}, SensEval-2~\cite{Edmonds2001}, SensEval-3~\cite{senseval3} and SensEval-2007~\cite{Agirre2009}, each consisting of three different tasks:
\begin{enumerate}
    \item \textbf{All-words} in which (almost) all words of a text have to be disambiguated;
    \item \textbf{Lexical sample} in which only select words have to be disambiguated;
    \item \textbf{Translation} which is similar to the lexical sample task, but for which, instead of selecting a sense, the WSD system needs to select its translation into a different language, e.g. Japanese. 
\end{enumerate}
Each SensEval version also provided a target lexicon (i.e. a list of words and their senses), a sense-annotated corpus, and a coarse-graining or fine-graining of senses. Several SVM-based algorithms have been evaluated on the SensEval tasks and have achieved up to 75.2\% in SensEval-3 and up to 89.4\% in the coarse-grained version of the SensEval-2007~\cite{Iacobacci2016}.
\begin{figure}[htb!]
    \centering
    \scalebox{.85}{\begin{tabular}{|c|cccc|cccc|}
        \hline&\multicolumn{4}{c|}{Part-of-Speech of neighbours} & \multicolumn{4}{c|}{Surrounding words}\\
        &\multicolumn{4}{c|}{values: $\left\{\text{PRON}=0, \text{V}=1, \text{PREP}=2, \text{ADV}=3\right\}$} & \multicolumn{4}{c|}{basis: $\left\{account, economy, rate, take\right\}$}\\
        & $w_{-2}$ & $w_{-1}$ & $w_{+1}$ & $w_{+2}$ & $account$ & $economy$ & $rate$ & $take$\\\hline
        Vector & 3 & 1 & 2 & 0 & 0 & 0 &0 &1\\\hline
    \end{tabular}}
    \caption{Example of the vector corresponding to the word \textit{interest} in the context \textit{My brother has always taken interest in my work.} (simplified from~\cite{Zhong2010}).\label{fig:SVMexVector}}
\end{figure}

\paragraph{}Although they achieve very high accuracies, the supervised approaches are not easily extended to large-scale applications as they require a large amount of manually annotated data. Alternative methods make use of knowledge bases such as (computer-readable) dictionaries, thesauri, or the \texttt{WordNet} databse~\cite{WordNet}. 

One of the prominent approaches within this category is the class of \emph{Lesk algorithms}~\cite{Lesk1986,Kilgarriff2000,Banerjee2002,Basile2014}. The idea is that given a target word $w$ in context $c$, the overlap of the context $c$ with the \emph{gloss} of a sense (i.e. definition and possibly examples) is higher if the sense does correspond to the intended one. Here, the overlap corresponds to the intersections of the set of words in the context and the glosses. For example, given the different glosses of the word \textit{bank}\footnote{Example taken from~\cite{jurafskyspeech}}:
\begin{enumerate}
    \item a financial institution that accepts deposits and channels the money into lending activities\\
    \underline{Examples:} he cashed a cheque at the bank, that bank holds the mortgage on my home
    \item sloping land (especially the slope beside a body of water)\\
    \underline{Examples:} they pulled the canoe up on the bank, he sat on the bank of the river and watched the currents
\end{enumerate}
And given the following target context (taken from the BNC):
\begin{quote}
    \underline{Cash} includes \underline{cheque} payments, bank transfers and credit card payments .
\end{quote}
Then, the algorithm will return the correct intended sense, namely its financial institution meaning; the overlap between the gloss and the target context is underlined above. This description is known as the \emph{simplified Lesk algorithm}~\cite{Kilgarriff2000}. In contrast, the original Lesk algorithm of~\cite{Lesk1986} intends to compare the glosses of \emph{all} of the words in the phrase, which increases the algorithm complexity. 

As one may expect, only using the overlap of the glosses with the context is a bit crude. Many extensions of this approach have been proposed, for example by also considering the words in glosses of related words (obtained from \texttt{WordNet})~\cite{Banerjee2002} or by using distributional or neural representations of contexts and glosses~\cite{Basile2014} and taking the cosine of vectors as the measure of overlap.

\begin{figure}[htb!]
    \centering
    \includegraphics[width=.75\linewidth]{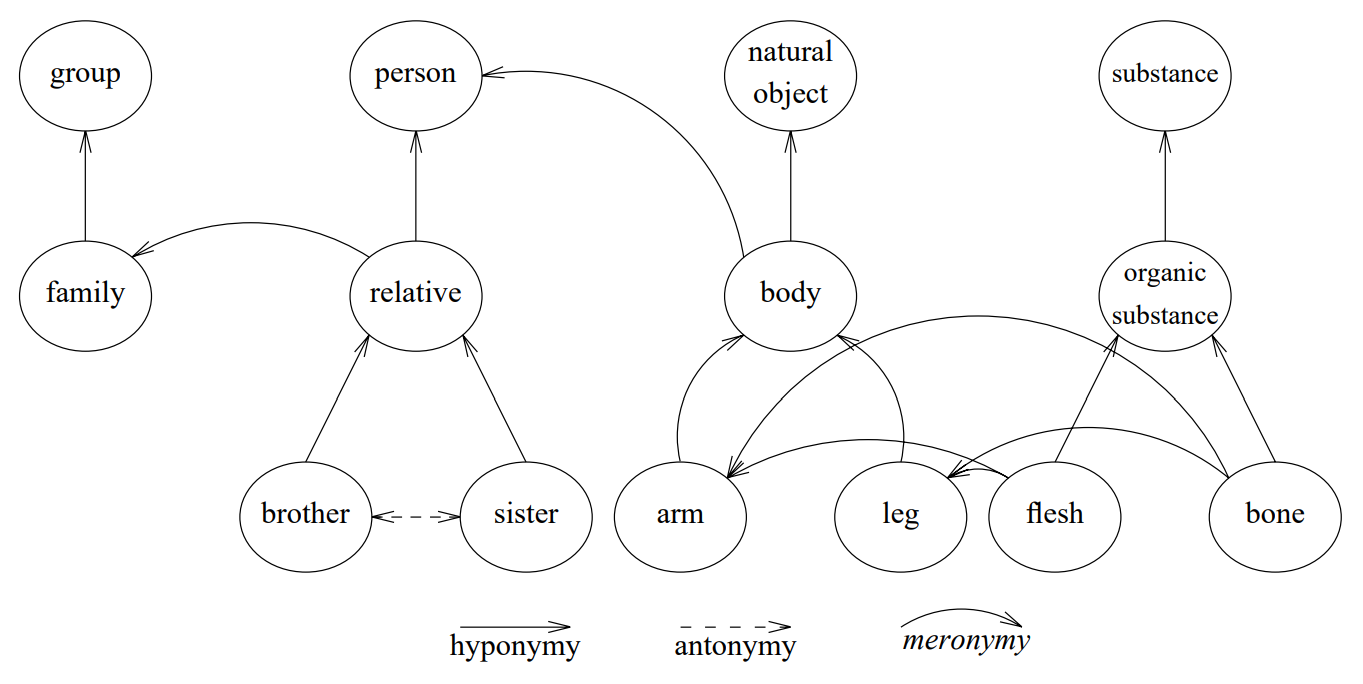}
    \caption{Illustration of the \texttt{WordNet} knowledge base. (taken from~\cite{WordNetB})\label{fig:WordNet}}
\end{figure}

\subsubsection{Neural approaches}
\paragraph{}The state-of-the-art approaches in NLP differ from the previously described methods as they vastly rely on artificial \emph{neural networks}. As for distributional approaches and SVMs, the meanings of words and sentences are stored as vectors known as embeddings. The entries of these vectors are learned from input/output pairs from a training set, where the input goes through a network of nodes (or artificial neurons) with tunable activation strength -- the details of this process depend on the architecture of the neural network. 

\paragraph{}One of the first instances of successful neural network architecture is the \emph{Recurrent Neural Network} (RNN), in which a sequence of tokens is input and processed from left to right (see Fig.~\ref{subfig:RNN}). Hence, at any stage, the network can retain information from all previous words. RNNs were a way to escape the sparsity of $n$-grams co-occurrences in corpora~\cite{Bengio2003,Mikolov2013} and were soon found to give meaningful representations of words. For instance, the \texttt{word2vec} vectors were found to predict semantic relations between words accurately~\cite{Mikolov2013}. 

However, in ``standard'' RNNs, the contribution of a given token to the gradient can either tend to $0$ or $\infty$ as training time increases, which makes them impractical for dealing with large texts. The \emph{Long-Short Term Memory} (LSTM) architecture aims to solve this problem by learning to ``forget'' information that is no longer relevant~\cite{Hochreiter1997}. 

In addition, by restricting the processing of words strictly from left to right, some of the long-distance dependencies may be lost from this forgetting mechanism. For this reason, \emph{bidirectional} LSTMs (biLSTMs) were proposed as an alternative, where two LSTMs, one going from left to right and the other from right to left, are combined.

\paragraph{}In~\cite{Melamud2016}, the authors proposed an extension of the \texttt{word2vec} word-embeddings from~\cite{Mikolov2013} by creating embeddings of contexts, referred to as \texttt{context2vec} context-embeddings. These context-vectors were then used in WSD as follows. Given a target word $w$, we can collect all of the context-embeddings associated with each of the occurrences of each sense $s$ of $w$ from a sense-annotated corpus such as SemCor~\cite{Miller1993}. Similarly,  we can create a context-vector for the target context. We then compare this target context-vector to all of the relevant context-vectors obtained from the annotated corpus, and we select the sense associated with the context-vector closest to the target context-vector. 

In~\cite{Peters2018}, the authors used a similar algorithm, but using the \texttt{ELMo} (Embeddings from Language Models) \emph{contextualised embeddings} (which represents words in contexts)  instead of the context-embeddings (which represent the context assuming all words are context-independent). The approach of~\cite{Peters2018} also differed from the one of~\cite{Melamud2016} as sense-embeddings of a target word $w$ were obtained by averaging the contextualised word-embeddings of occurrences of $w$ in SemCor which corresponded to the same sense. 

The above approach has the major flaw that only a $16.11\%$ of \texttt{WordNet} senses are found in SemCor~\cite{Loureiro2020}. The solution of~\cite{Peters2018} was always to select the most common sense for each unseen word. To obtain a representation of an unseen \texttt{WordNet} sense, the authors of~\cite{Loureiro2020} make use of the structure of the \texttt{WordNet} lexical database. Indeed, each of the senses (corresponding to a lemma, its part of speech, and its gloss) is organised in synsets, which include synonymous senses. Each synset has a set of \emph{hypernyms}, e.g. $dog^1_n$ is a hypernym of $pug^1_n$, which is, in turn, part of a larger \emph{lexname}, e.g. $dog^1_n$ is part of $noun.animal$. We can obtain the representation of each of these abstraction levels by averaging the context-embeddings associated with all of the senses included in them and are present in the annotated corpus. The representations of the missing sense would then be abstracted by the representation of its first hypernym or lexname for which a representation exists.

Furthermore, motivated by the intuition behind the Lesk algorithms, some neural approaches have also used glosses as a knowledge source. It was shown that including the gloss embeddings on top of the context-embeddings using co-attention mechanisms~\cite{Luo2018a,Luo2018b} or by simple concatenation~\cite{Loureiro2020} improves the performance of WSD algorithms.


    

\paragraph{}In 2017, Vaswani et al. introduced a novel neural network architecture known as the \emph{transformer}~\cite{attention}. In particular, this new architecture allowed parallelisation of the training process by allowing all-to-all connectivity of the artificial neurons (see Fig.~\ref{subfig:transformers}). The parallelisation of the training process opened the opportunity to train the neural network using a substantial amount of data. For example, the Google language model \texttt{BERT} (Bidirectional Encoder Representations from Transformers) was trained over 3.3 billion tokens, while the largest version of \texttt{GPT} to date was trained over 449 billion tokens. 
\begin{figure}[ht!]
    \centering
    \begin{subfigure}{.45\linewidth}
        \centering
        \includegraphics[width=\linewidth]{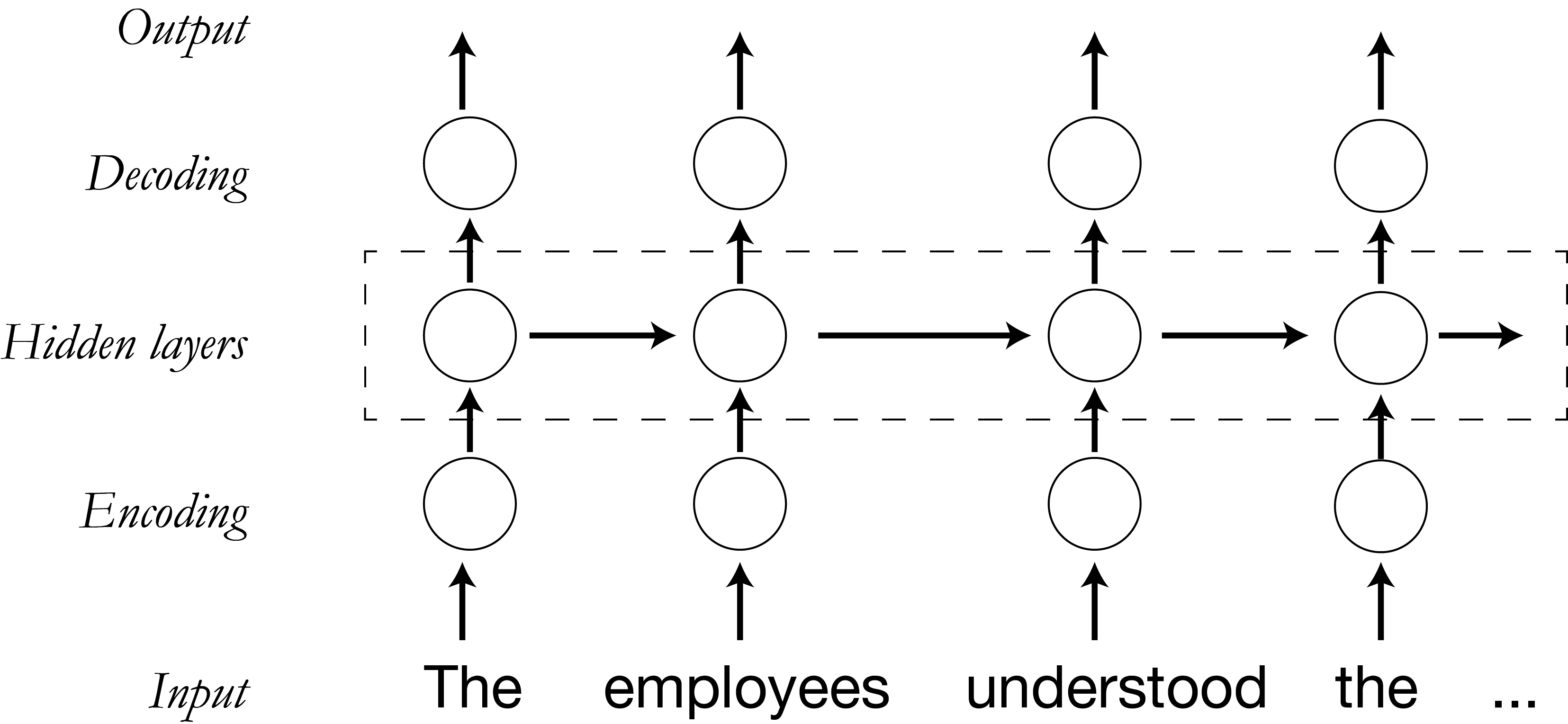}
        \caption{RNN\label{subfig:RNN}}
    \end{subfigure}\qquad%
    \begin{subfigure}{.45\linewidth}
        \centering
        \includegraphics[width=\linewidth]{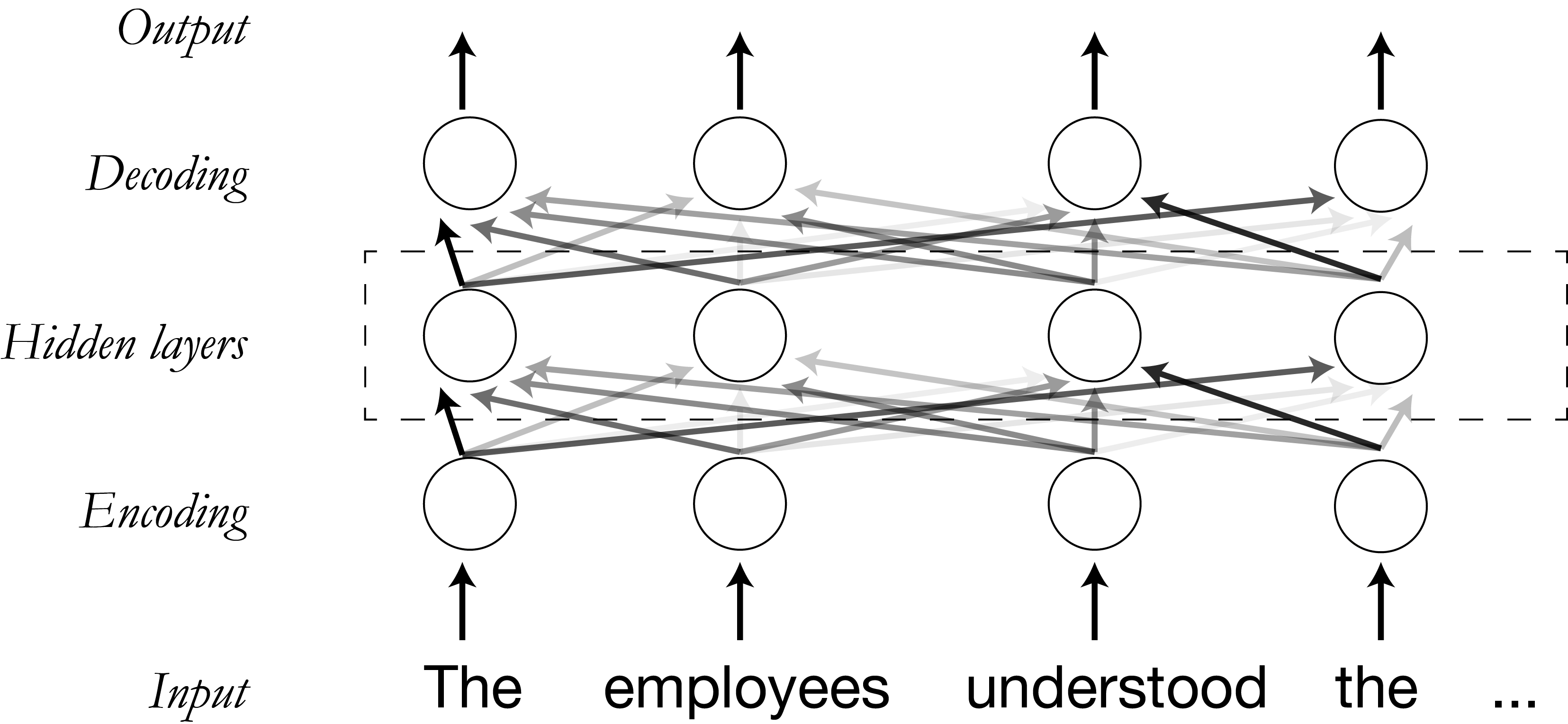}
        \caption{Transformers\label{subfig:transformers}}
    \end{subfigure}
    \caption{Schematics of the differences between recurrent and transformers neural networks architectures.}
\end{figure}
The \emph{pre-training} process conducted by Google or OpenAI is done by tuning the neural network's attention weights (i.e. the strength of the connection between two nodes) to solve a generic task. 

For example, \texttt{BERT} was trained simultaneously on mask prediction and next-sentence prediction tasks. The mask prediction task is as follows: the neural network is presented with a sentence where one or several words\footnote{To be more accurate, the masks are tokens and not words; the distinction is not important for the rest of this work.} are \emph{masked}, and its goal is to predict the values of these masks. For example, the language model could be presented the input:
\begin{equation*}
    \texttt{Paris is the [MASK] of France.}
\end{equation*}
and will attempt to predict the word \textit{capital}. In the next sentence prediction task, the language model is shown two sentences $\mathtt{S}_1$ and $\mathtt{S}_1$ and needs to decide whether $\mathtt{S}_2$ follows $\mathtt{S}_1$. For example, if:
\begin{align*}
    \mathtt{S_1} =& \texttt{The man went to the store.}\\
    \mathtt{S_2} =& \texttt{He bought a gallon of milk.}\\
\end{align*}
the neural network should output \texttt{True}, i.e. $\mathtt{S}_2$ indeed follows $\mathtt{S}_1$; but in the case of:
\begin{align*}
    \mathtt{S_1} =& \texttt{The man went to the store.}\\
    \mathtt{S_2} =& \texttt{Penguins are flightless birds.}\\
\end{align*} 
the output should be \texttt{False}. 

\paragraph{}From the pre-trained language models, two main ways exist to solve the WSD problem. First, as for the LSTMs described earlier, given a text input, the transformer neural network will return a contextualised word-embedding. Hence, we can use these embeddings as in the previously described algorithms. For example, the same Nearest Neighbour algorithm described above for LSTMs performed better using \texttt{BERT} embeddings than using \texttt{context2vec}~\cite{Melamud2016} or \texttt{ELMo}~\cite{Peters2018} embeddings~\cite{Loureiro2020}. 

Another possibility is to \emph{fine-tune} the pre-trained models by training the language model for a more specific task, such as WSD, using a much smaller training dataset than the one required for pre-training. In~\cite{Vandenbussche2021}, this approach was taken on the language model \texttt{BERT}, where the fine-tuning process was as follows. The training set consisted of tuples $(c,d,l)$ where $c$ is a context containing a target word $w$, $d$ is a definition of a sense of $w$, and $l$ is the label in $\left\{yes,no\right\}$ corresponding to whether the intended meaning of $w$ in $c$ corresponds to the definition $d$. The target context $c$ and definition $d$ are then fed into the neural network, which will then be trained to predict the correct labels $l$. This approach achieved an accuracy of 72.3\%, which is comparable to the performance of SVMs, but does not use annotated corpora.
\subsection{The human disambiguation process}\label{subsec:lexicalPsycho}
\paragraph{}In this section, we review the psycholinguistic theories on the lexical disambiguation process of humans. In particular, we will focus on the differences in the processing of ambiguous nouns and verbs and words with different ``levels of ambiguity''. Indeed, two interpretations of an ambiguous word could be completely unrelated, such as \textit{bank} in \textit{bank account} and in \textit{river bank}, or somewhat related, for example, \textit{book} in \textit{interesting book} and in \textit{hardback book}. In the case of \textit{book}, both expressions refer to the same entity, but the former relates to information content, whereas the latter interpretation relates to the properties of the physical object. 

When the interpretations of an ambiguous word are unrelated, the word is said to be \emph{homonymous}. In contrast, if its interpretations are related, it is said to be \emph{polysemous}. We will, for the rest of the thesis, adopt the terminology of the literature where \emph{meanings} will refer to unrelated interpretations, \emph{senses} will refer to related interpretations, and \emph{interpretations} can refer to both meanings or senses.

\paragraph{}Psycholinguists study the disambiguation process of lexically ambiguous words using eye-tracking data. In such settings, the participants are presented with a text on a screen and asked to read it. The eye-tracker will then record the movements of the participant's eyes and the lengths (and order) of the eye fixations on different zones of the screen (which usually coincide with each of the words of the text). 

The prominent figures of interest will be a target word's first-pass and second-pass fixation times. The former corresponds to the time spent on a word reading from left to right, and the latter corresponds to any additional fixation on the target region (i.e. when the reader has to go ``back'' to the target word).  

\subsubsection{Homonymous nouns}
\paragraph{} Most studies on lexical ambiguity in psycholinguistics focus on the processing of homonymous nouns, i.e., nouns that have unrelated meanings.

\paragraph{}The main effect detected from eye-tracking experiments over homonymous nouns relates to the frequencies at which each meaning occurs. In particular, we observe a slowdown whenever the interpretation that has to be selected is uncommon~\cite{FrazierRayner1990,Dopkins1992}. For example, in the following (taken from~\cite{FrazierRayner1990}), the sentence in~\ref{item:bandPreferred} was read faster than it~\ref{item:bandUnpreferred}.
\begin{enumerate}[label=(1\alph*)]
    \item\label{item:bandPreferred} Playing so loudly, the wedding \underline{band} upset the groom.
    \item\label{item:bandUnpreferred} Looking so tarnished, the wedding \underline{band} upset the groom.
\end{enumerate}
This suggests that a homonymous noun's most common meaning is activated more easily than uncommon ones. We will emphasise here that this bias in meaning selection is \emph{independent of the context} and only depends on the reader's experience, i.e. how often they have seen each of the different meanings in their lives. 

Moreover, a slowdown also occurs whenever the meanings of the target word were \emph{balanced}, i.e. when the meanings are roughly as common as each other, for example \textit{palm} the type of tree or \textit{palm} the part of the hand~\cite{RaynerDuffy1986}. In fact, in~\cite{RaynerDuffy1986}, the authors showed that if the meanings were balanced, the slowdown occurs when the reader encounters the ambiguous word. In contrast, in the case of non-balanced words, a slowdown only occurs in the \emph{disambiguating region} whenever the subordinate (i.e. the less common) meaning has to be activated. This result has also been replicated in~\cite{Dopkins1992}. 

This finding suggests multiple interpretations of a homonymous noun can be activated simultaneously. However, the activation levels will not be the same for all meanings, and a reader will give the most common meaning a higher ``rating''. If multiple meanings are equally probable, they will compete, and this competition will create difficulty when reading the ambiguous word. This describes a \emph{parallel ranked} process of disambiguation. Moreover, when the context suggests that a subordinate meaning has to be selected, the reader encounters difficulty when seeing the disambiguating context, as they will have to readjust the activation levels. On the other hand, if the intended meaning is the dominant one, i.e. the most common one, no difficulty should occur.  

\paragraph{}Finally, the position of the disambiguating context also plays a role in the disambiguation process. In particular, if the disambiguating context is \emph{before} the target word, the fixation times are higher on the target word (but low on the disambiguating region) than when the disambiguating region is \emph{after} the target word~\cite{FrazierRayner1990,Dopkins1992}. However, the reading time of the whole sentence was higher whenever the disambiguating context is \emph{after} the target word~\cite{FrazierRayner1990,Dopkins1992}. 

This would suggest that the disambiguation process is incremental, i.e. the reader will readjust the activation weights of each possible meaning as more information is known.

\subsubsection{Polysemous nouns and underspecification}
\paragraph{} We now look at the disambiguation process of polysemous nouns, which, somewhat surprisingly, carries some stark differences from the disambiguation process of homonymous nouns.

In particular, no difference in reading times has been observed between common and uncommon senses~\cite{FrazierRayner1990}, nor any difference between concrete and abstract senses~\cite{FrazierRayner1990}, well-known senses and senses created through rules~\cite{FrissonPickering1999}. In all of these cases, the reading times observed are comparable with the reading times of unambiguous words~\cite{FrazierRayner1990,FrissonPickering1999}. 

This suggests that, instead of having a parallel-ranked representation of the different senses of a polysemous word, the reader will always start by selecting an \emph{underspecified} interpretation of the word, which includes all of its possible senses. In the theory of underspecification, the context then shapes the salient interpretation of a polysemous word.

\paragraph{}Although the interpretation is underspecified within a sentence, if the sentential context of a polysemous word is not enough to select a sense, the possible senses behave like meanings in the following sentence, i.e. relative frequencies start to have an impact on the reading times~\cite{underspecification}. This phenomenon is referred to as a \emph{homing-in} stage. 

This observation suggests that, even though polysemous words are processed faster, commitment to an interpretation, i.e. selecting the most appropriate sense, is only done at the end of the sentence, whereas we recall that this is meaning selection occurs almost instantly in homonymous nouns.

\begin{rmk}
    Parallel unranked models of the disambiguation of polysemous words have also been proposed. However, they suffer from several drawbacks. In particular, there is no obvious way to select the most appropriate meaning from them. In addition, parallel unranked models would predict that the more possible senses a word has, the more difficult it should be to read. However, this prediction was not in line with eye-tracking data~\cite{underspecification}.
\end{rmk}

\subsubsection{Disambiguating verbs}
\paragraph{}Most of the psycholinguistics literature on lexical ambiguity focuses on ambiguous nouns, and seldom research looked at the disambiguation process of more complex ambiguous types such as verbs and adjectives. In~\cite{Tanenhaus1979}, the authors investigated words with multiple possible grammatical types (e.g. \textit{wacth} could be a noun or a verb), and the authors of~\cite{Mullaly} looked at the process of disambiguating adjectives. Here, we will mostly focus on the disambiguation of ambiguous verbs, which was studied in~\cite{PickeringFrisson2001}.

\paragraph{} In particular, in~\cite{PickeringFrisson2001}, the authors showed that the processing of ambiguous verbs (both polysemous and homonymous) differed from the processing of ambiguous nouns. Indeed, we observe a general slowdown in reading time of ambiguous verbs, compared to ambiguous nouns, which shows that ambiguous verbs are more complex to disambiguate.

The effect of frequency was much smaller for homonymous verbs than it was for homonymous nouns. Indeed, no significant difference in first-pass reading times of the target verb between common and uncommon meanings was observed. In second-pass reading, only a mild frequency effect was detected, where less backtracking is observed when the intended meaning is the dominant one. In addition, the main difficulty did not occur in the verb region but was slightly delayed until the object of the verb was encountered. This finding suggests that the disambiguation of homonymous verbs is not immediate, contrary to homonymous nouns, but the reader waits until the arguments of the verbs are known. In particular, since meaning selection is delayed, both the dominant and subordinate meanings ``have the time'' to be activated, which diminishes the frequency effects.

For polysemous verbs, the disambiguation was delayed even further, and the slowdown in first-pass mainly occurred at the end of the phrase or the sentence. In addition, similarly to polysemous nouns, no significant frequency effect could be observed during the initial analysis. This delay suggests once again that polysemous verbs will initially activate an underspecified meaning, which is then made more and more specific as information from the context emerges. Some minor frequency effects occurred on second-pass, which suggests that the reader will home-in on a chosen sense at the end of the sentence, which consequently behaves more like a homonymous verb on reanalysis. 

\paragraph{}Many hypotheses can be advanced to explain the relative difficulty of processing verbs compared to nouns. First, this general increase in difficulty is not restricted to ambiguous cases. When reading a sentence, readers will generally spend longer on the main verb of the sentence than any other words~\cite{Rayner1977}. In addition, many studies suggest that the meaning of verbs is highly dependent on their arguments~\cite{GentnerFrance2013}, making its interpretation much more variable~\cite{GentnerFrance2013,Gentner1981,Gentner1982}. For instance, in the case of a mismatch between the arguments of a verb and its standard interpretation, e.g. in~\ref{item:carUnmatch} as opposed to~\ref{item:muleMatch}, it is the verb that tends to acquire a metaphorical interpretation.
\begin{enumerate}[label=(2\alph*)]
    \item\label{item:muleMatch} The mule shivered
    \item\label{item:carUnmatch} The car shivered
\end{enumerate}
This apparent complexity might be deeply rooted in language acquisition. Indeed, it is well-established that children will learn nouns before verbs~\cite{Gentner1982,Bohannon1986} and have more difficulty using verbs~\cite{Tomasello1997}. This factor could explain the overall difficulty of disambiguating verbs.

\section{Human parsing \& garden-path sentences}\label{sec:SyntacticBackground}
\paragraph{}In this section, we will describe another type of ambiguity, namely \emph{syntactic ambiguity}. This ambiguity arises whenever several \emph{gramamtical structures} are simultaneously possible. An example of a syntactically ambiguous sentence is:
\begin{enumerate}[label=(\arabic*)]
    \item She saw a man with binoculars
\end{enumerate}
This sentence either means:
\begin{enumerate}[label=(1\alph*)]
    \item She used binoculars and saw a man 
    \item She saw a man, and he was using binoculars
\end{enumerate}
In addition, syntactic ambiguity does not always occur at the sentence level. Sometimes, a sentence can also be \emph{locally} ambiguous. For example, consider the sentence fragment:
\begin{enumerate}[label=(2)]
    \item The artist painted [...]
\end{enumerate}
The verb \textit{painted} can either be transitive, i.e. take an object such as \textit{a portrait}, intransitive, i.e. does not take any object, or even be in the passive voice, e.g. the artist is the thing that is painted. 

\paragraph{}Certain types of sentences, known as \emph{garden-path sentences}, have been used in psychological research to unmask the processes behind grammatical parsing. These are sentences for which humans initially parsed a locally ambiguous fragment incorrectly. 

For example, in the following sentence, the phrase \textit{the contract} is initially parsed as the object of the verb \textit{understood} and is eventually found to be the subject of the verb phrase \textit{would change}.
\begin{enumerate}[label=(3)]
    \item The employees \underline{understood the contract} would change
\end{enumerate}

\paragraph{}In this section, we review the existing literature on garden-path sentences. In Section~\ref{subsec:syntacticPsycho}, we introduce the different human parsing theories and the evidence supporting them. These theories will motivate our models described in Part~\ref{part:Syntactic}. In Section~\ref{subsec:surprisal}, we describe a popular framework of computational linguistics, namely surprisal theory, which has been applied to predict reading times of garden-path sentences. We give some history and motivation for this framework and its drawbacks.

\subsection{Psycholinguistics theories of parsing}\label{subsec:syntacticPsycho}
\paragraph{}In~\cite{Bever}, Thomas Bever exposed an overview of his theories of \emph{perceptual strategies}, which are mechanisms that allow humans to convert (external) linguistic structures to (internal) perceptual representations of meaning. He believed these perceptual mechanisms are more fundamental than linguistic structures like grammar and are the first to be acquired by children. Therefore, language acquisition would correspond to learning labels and rules corresponding to these perceptual strategies. Among these strategies, he explains that one of the first steps of language understanding is parsing a sequence of words and sounds (external structures) into groups associated with a fundamental role, such as actor, action, object, or modifier (internal structures). A subsequent rule is to associate a ``N...V...(N)'' sequence with the ``actor-action-object'' roles as soon as possible unless markers suggest that the passive voice is used or that the first phrase modifies the main clause. To illustrate his claim, Bever describes a series of linguistic behaviour and experiments testing them. Among these behaviours, he describes the difficulty of parsing sentences such as:
\begin{enumerate}[label=(\arabic*), series=garden]
    \item\label{item:horse} The horse raced past the barn fell
\end{enumerate}
The difficulty of parsing would be due to non-conformity with respect to those perceptual strategies. These sentences were later referred to as \emph{garden-path sentences}. 

\paragraph{}Generally speaking, a garden-path sentence is a syntactically unambiguous sentence for which the reader is  ``led down a garden-path'', i.e. they are forced to adopt the wrong syntactic structure at some initial stage. After Bever's original work, these sentences have been widely used in psycholinguistics to uncover mechanisms at the heart of human parsing by studying what induces errors for readers~\cite{Frazier87,pritchett1992,Holmes87,Mitchell87,trueswell1993,pickering1998,GARNSEY1997,SturtPick}. 

Some specific types of garden-path sentences are widely studied in psycholinguistics. These are exemplified as follows.
\begin{itemize}
    \item \textbf{NP/S sentences}\\
    \begin{enumerate}[label=(\arabic*), resume=garden]
        \item\label{item:NPS} The employees \underline{understood} the contract would change.
    \end{enumerate} 
    Here, the main verb \textit{understood} either takes a noun-phrase (NP), such as \textit{the contract}, as an object, or a sentential (S) object, such as \textit{the contract would change}.
    \item \textbf{NP/Z sentence}\\
    \begin{enumerate}[label=(\arabic*), resume=garden]
        \item\label{item:NPZ} As the woman \underline{read} the magazine amused the editors.
    \end{enumerate} 
    In this case, the main verb \textit{read}, either takes an NP as an object, e.g. \textit{the magazine}, or no object at all -- this is denoted by (Z) for ``zero''.
    \item \textbf{MV/RR sentences}\\
    \begin{enumerate}[label=(\arabic*), resume=garden]
        \item\label{item:MVRR} The soldiers \underline{warned} about the dangers conducted the raid. 
    \end{enumerate} 
    In this case, the underlined verb \textit{warned} is either the main verb (MV) or part of a relative clause (RR). The example~\ref{item:horse} is also an MV/RR sentence.
\end{itemize}
In this thesis, we will mostly focus on sentences of type NP/S and NP/Z.

\paragraph{}What has been shown unambiguously in several studies is that these different types of garden-path sentences are read with different levels of difficulty. In particular, NP/S sentences were read faster than NP/Z sentences, so NP/S sentences are more straightforward to parse than NP/Z ones~\cite{SturtPick,grodner}. What is more debated in the literature is why this is the case. One hypothesis is that it is due to the nature of the changes needed to obtain the correct parse~\cite{SturtPick}.

\paragraph{}An underlying problem is to identify why a garden-path effect occurs in the first place. 

The first thing to mention is that local ambiguity is not the main cause of difficulty. In fact, the presence of syntactic ambiguity can make a sentence \emph{faster to read} than its unambiguous variants~\cite{Traxeler1998}. This feature distinguishes syntactic disambiguation from lexical ambiguity resolution (see Section~\ref{subsec:lexicalPsycho}). The intuition is as follows. Let's assume that a given fragment has two equally likely possible syntactic structures. If the sentence is only locally ambiguous but globally unambiguous, then half of the time, the reader will initially select the ``wrong parse''. Hence, reanalysis is triggered half of the time. On the other hand, if the sentence remains globally ambiguous, then no analysis is ever required. The sentence is then read faster on average. 

However, several contributing factors to the existence of garden-path effects have been identified, among which lexical~\cite{Frazier87,Holmes87,Mitchell87,trueswell1993,GARNSEY1997}, plausibility~\cite{pickering1998,GARNSEY1997} and discourse biases~\cite{Frazier87,TrueswellTanenhaus2007}. The main quoted factor is the frequency bias of its main verb~\cite{Frazier87,Holmes87,Mitchell87,trueswell1993,GARNSEY1997}. For instance, the verb \textit{hear} would have a higher bias towards taking an NP as an object than the verb \textit{claim}. Hence, sentences with the main verb \textit{hear} would be more likely to create an NP/S or NP/Z garden-path sentence than sentences with the main verb \textit{claim}. This phenomenon is exemplified in the following, where~\ref{item:alarmNoGP} does not create considerable difficulty, whereas~\ref{item:alarmGP} does.
\begin{enumerate}[label=(5\alph*)]
    \item\label{item:alarmNoGP} The officer claimed the alarm was a surprise.
    \item\label{item:alarmGP} The officer heard the alarm was a surprise.
\end{enumerate}
Similarly, the difficulty of selecting the correct parse in garden-path sentences increases if the object of the NP reading is deemed plausible~\cite{pickering1998,GARNSEY1997}. For example,~\ref{item:magazinePlaus} is harder to read than~\ref{item:magazineImplaus}.
\begin{enumerate}[label=(6\alph*)]
    \item\label{item:magazinePlaus} As the woman read the magazine amused the editors.
    \item \label{item:magazineImplaus} As the woman sailed the magazine amused the editors.
\end{enumerate}

\paragraph{}With regards to \emph{how} do humans resolve the garden-path effects, two main categories of procedures have been proposed, namely \emph{serial processing}~\cite{Frazier87,pritchett1992} and \emph{parallel processing}~\cite{Jurafsky1996,VanGompel2005}. Advocates of the serial strategy argue that, at any given stage, the reader will create a \emph{single parse} which can be completed~\cite{Frazier87,pritchett1992}, and that certain conditions can be imposed on the partial structure (e.g. minimal attachment states readers never add unnecessary nodes to the structure under construction). 

In the case of parallel processing, when syntactic ambiguity occurs, at least two or more parses are constructed in parallel. However, since we do observe a garden-path effect in sentences such as~\ref{item:horse}-~\ref{item:MVRR}, this implies that these structures have to be \emph{ranked}. Otherwise, the less likely parse having been constructed already, it should not be hard to extend the ``correct'' parse. The difficulty then comes from some ``reranking mechanism'', weights have to be transferred from the previously likely but incorrect parse to the correct parse. 

It is quite difficult to distinguish between probabilistic serial sentence processing (i.e. if two or more parses are possible, they are chosen with their respective degree of likelihood) or ranked-parallel processing (i.e. all of the possible parses are created at the same time, but with different accessibility levels)~\cite{gibson2000}\footnote{This is similar to the fact that for a single quantum state, a superposition of basis state is indistinguishable from a probabilistic mixture of the same basis states if only basis measurements are available.}. In our model described in Chapter~\ref{chap:SyntacticModel}, we take an approach compatible with both interpretations. 

\paragraph{}Finally, there is also the question of how the reader obtains the correct parse. In~\cite{grodner}, the authors compare two different hypotheses, namely \emph{repair}, where the reader will amend a given parse to obtain the correct one, or \emph{reanalysis}, where the reader will reparse a sentence fragment in the view of incoming information. The authors of~\cite{grodner} presented evidence supporting the reanalysis hypothesis. The study uses a \emph{locality bias} (i.e. it is easier to attach clauses that are close together), such that, by adding some extra words between the start of clause marker \textit{As} in~\ref{item:NPZ} and the ambiguous NP \textit{the magazine}, the garden-path effect decreased in NP/Z sentences, whereas no such effect occurs in NP/S sentences.
\subsection{Surprisal predictions for garden-path sentences}\label{subsec:surprisal}
\paragraph{}Psycholinguistic studies have shown that one of the main factors influencing reading time is \emph{predictability} of a word in a context~\cite{EhrlichRayner}. Indeed, words are read faster if found in a context that makes them predictable than in contexts where they are not~\cite{EhrlichRayner}. For example, the word \textit{shark} is read faster in~\ref{item:shark1} than in~\ref{item:shark2}.
\begin{enumerate}[label=(7\alph*)]
    \item\label{item:shark1} The coast guard had warned that someone had seen a \ldots
    \item\label{item:shark2} The zoo keeper explained that the lifespan of a \ldots
\end{enumerate}

In~\cite{smith2013effect}, the relation between predictability and reading time was \emph{logarithmic}. This result was obtained by looking at eye-tracking times of a subset of the Dundee dataset (containing newspapers) and self-paced reading times for subsets of the Brown corpus (containing texts of various genres). The reading times correlated with the conditional probabilities of encountering a word $w_{n+1}$ in the context $c = w_1\ldots w_{n}$. This then motivates the use of the \emph{suprisal}, defined as:
\begin{equation*}
    S(w_{n+1}|w_1\ldots w_n) = -\log_2 P\left[w_{n+1}~\middle|~w_1\ldots w_n\right]
\end{equation*}
as a predictor for reading time. The authors from~\cite{smith2013effect} demonstrated that:
\begin{equation*}
    S(w_{n+1}|w_1\ldots w_n) \propto RT(w_{n+1}|w_1\ldots w_n)
\end{equation*}

Surprisal, also known as self-information, originates in Shannon's theory of information~\cite{Shannon}. In this theory, surprisal is defined as the \emph{quantity of information} entailed by knowing the value $X=w$, where $X$ is defined as a random variable selecting the following word in the context $w_1\ldots w_n$. The intuition is that a very predictable word is not surprising and, therefore, does not carry out a lot of information. On the other hand, if a word is not predictable, it should significantly increase the amount of information available to the reader. 

\paragraph{}Predictability has historically been estimated from \emph{cloze tasks}~\cite{taylorCloze}, where human participants are asked to complete a sentence or a piece of text. However, cloze tasks fail to estimate the predictability of highly improbable (probability $<5\%-10\%$) words and constructions~\cite{smith2013effect}. Therefore, data from such tasks are not reliable for studying garden-path sentences. 

In~\cite{smith2013effect}, the authors instead decided to collect statistics (trigram probabilities with a bigram cache) from text corpora (e.g. BNC) to obtain word predictability. These probabilities, however, only take local context into account and do not necessarily mirror the way humans assign probabilities~\cite{smith2013effect}. 

With the advent of neural language modelling, computational linguistics soon realised that they could use predictions from language models instead of collecting predictability from cloze tasks. 

\paragraph{}Hale~\cite{Hale2001} was the first to suggest using surprisal for predicting the slowdown in garden-path sentences. To do so, the author used the probabilities coming from a probabilistic Earley parser. However, the correlation between the magnitude of surprisal and reading time was not investigated~\cite{Hale2001}. 

In~\cite{vanSchijndelA,vanSchijndelB,huang2023,linzenCCG}, the authors studied the empirical correlations between surprisal from language models and self-paced reading times of garden-path sentences. Although surprisal calculated from language models can predict the existence of a garden-path effect (i.e. a higher reading time in the critical region in garden-path sentences compared to their unambiguous version), it consistently underestimates its magnitude~\cite{vanSchijndelA,vanSchijndelB,huang2023,linzenCCG}. In addition, although predictions are mostly lower for NP/S than for NP/Z~\cite{vanSchijndelA,huang2023}, no statistical difference has been found between the predicted garden-path effects of NP/S and NP/Z sentences. In fact, in the study presented in~\cite{vanSchijndelB}, the average garden-path effect for NP/S sentences was lower than that for NP/Z sentences.

Many possible reasons for the discrepancies between surprisal and reading times have been advanced: 
\begin{enumerate}
    \item Surprisal is not the main predictor for reading times~\cite{vanSchijndelB};
    \item The human parsing process is not strictly incremental and reanalysis or backtracking is necessary~\cite{vanSchijndelA,vanSchijndelB,huang2023,linzenCCG};
    \item The statistics used to calculate surprisal are not adequate representations of how humans assign predictions~\cite{linzenCCG}.
\end{enumerate}

In~\cite{vanSchijndelB}, authors have compared the predictions from surprisal with predictions from alternative incremental measures of dissonance that have been proposed in the past, such as entropy and entropy reduction~\cite{Hale2003,Hale2006,Frank2013}. The idea is that reading time is also modulated by how much uncertainty is removed by adding an extra word to a sentence fragment~\cite{Hale2003,Frank2013}; the higher this differential of uncertainty, the higher the reading time. What they found is that surprisal outperformed the entropy-based models by quite a distance, as the entropy-based measures did not predict a garden-path effect at all~\cite{vanSchijndelB}. This shows that if the disambiguation process is truly incremental, surprisal would be the best-known prediction factor. 

One other factor that authors of~\cite{vanSchijndelB} have explored is whether the predictions from large language models suffer from the same drawbacks as the cloze tasks, namely that low probabilities predictions are effectively not predicted -- and therefore, LLM predictions are not reliable for rare or complex constructions. Indeed, no ceiling effect was observed~\cite{vanSchijndelB}, therefore confirming that this was not the main source of error in the surprisal calculations. 

Finally, in~\cite{linzenCCG}, the influence of the \emph{syntactic surprisal}, as opposed to the lexical surprisal defined above, on the garden-path effect predictions was investigated. To do so, the authors defined a surprisal measure based on the probability of obtaining a given Combinatory Categorial Grammar (CCG) supertag for the last word of a sentence fragment, conditioned on having the usual (lexical) context consisting of the previous words. In other words, the syntactic surprisal is defined from a new probability distribution:
\begin{equation}
    P\left[c_{n+1}~\middle|~w_{1}\ldots w_{n}\right] = \sum_{w_{n+1}} P\left[t(w_{n+1})=c_{n+1}~\middle|~w_1\ldots w_{n+1}\right] \times P\left[w_{n+1}~\middle|~w_1\ldots w_{n}\right]
\end{equation}
where, as before, $w$'s runs over the vocabulary $V$, $c$'s runs over the set of CCG supertags $C$, and the map $t:V\to C$ associates a word with a CCG type. In addition, since the CCG supertag of a word is not necessarily unique, all possible supertags were considered in the syntactic surprisal calculations. Therefore, the syntactic surprisal is overall defined as:
\begin{align}
    S_{synt}\left(w_{n+1}\middle|w_{1}\ldots w_n\right) = -\log_2 \sum_{c_{n+1}} &P\left[c_{n+1}~\middle|~w_{1}\ldots w_{n}\right] \nonumber\\
    &\times P\left[t(w_{n+1})=c_{n+1}~\middle|~w_1\ldots w_{n+1}\right]
\end{align}
The syntactic surprisal provided more accurate predictions than the lexical surprisal alone, and combining both lexical and syntactic surprisal improved the predictions~\cite{linzenCCG}. However, even with these improvements, the garden-path effect was still widely underestimated~\cite{linzenCCG}, and thus, including syntactic surprisal did not fully resolve the previous issues. In Chapter~\ref{part:Syntactic}, we will give an alternative model based on sheaf-theory which achieves better predictions.

\part{Lexical Ambiguity}\label{part:Lexical}




\chapter{Aspects of the lexical disambiguation process}\label{chap:lexicalFeatures}

\paragraph{}In this chapter, we want to create a model of the human lexical disambiguation process. Moreover, our goal is to make this model quantum native, so that we can explore the potential of quantum simulations in the next chapter. Here, we study natural language data using mathematical frameworks arising from quantum mechanics (Sections~\ref{sec:lexicalContext} to~\ref{sec:lexicalCausal}). By doing so, we create a parallel description of linguistic features in terms of quantum ones. 

Moreover, we recall from Section~\ref{subsec:lexicalPsycho} that humans do not disambiguate words of different grammatical types, or different levels of ambiguity in the same way. However, the mainstream NLP approaches to word-sense disambiguation introduced in Section~\ref{subsec:WSD}, disambiguate all words in the same way, regardless of their part-of-speech, and whether the interpretations are related or unrelated. With our approach, we study the linguistic data for words of different parts-of-speech (i.e. nouns and verbs) and different levels fo ambiguity (i.e. polysemous and homonymous words). We then compare our results with the theories of psycholinguistics presented in Section~\ref{subsec:lexicalPsycho}.

More specifically, we start by looking at the potential \emph{contextuality} of lexical ambiguity data (Section~\ref{sec:lexicalContext}). In Section~\ref{sec:lexicalSignal}, we observe that the \emph{signalling} property of the collected data is a quantity of interest and study it further. In Section~\ref{sec:lexicalCausal}, this analysis is extended by studying the structure of the observed signalling, via its \emph{causal structure}. 
\section{Methodology}\label{sec:lexicalMethods}
\paragraph{}In all of the different studies carried out in this chapter, we will use a common interpretation of the analogue of quantum measurements. The idea is to take parties to represent \emph{grammatical roles} (e.g. subject, object, main verb, etc.) or \emph{grammatical types} (e.g. noun, verb, adjective, etc.) of words. In most of the following studies, we focused on subject-verb (SV) and verb-object (VO) scenarios, i.e. when we have two parties, either $S$ and $V$ or $V$ and $O$ corresponding respectively to subject and verb or verb and object. 

We then give each party a choice of inputs, corresponding to a choice of \emph{word} to fill in their corresponding part-of-speech. In analogy to the Bell scenario described in Section~\ref{sec:QDescription}, these choices of inputs can also be seen as local measurements. Let us look at an example and consider the lexicon:
\begin{equation}
    \mathcal{L} = \left\{tap, pitcher, box, cabinet\right\}
\end{equation}
and two parties $S$ and $V$ such that $S$ can choose a subject for a verb chosen by $V$. Then, $V$ is allowed to choose between the words of $\mathcal{L}$ which are verbs, in this case, \textit{tap} or \textit{box}. Similarly, in the general case, $S$ can choose between the set of words in $\mathcal{L}$ which can be nouns, i.e. the whole of $\mathcal{L}$. However, for the empirical models described in the following sections, we will decide to manually restrict each party's input choices, e.g., by letting $S$ choose between \textit{pitcher} or \textit{cabinet}. 

Finally, each measurement outcome will correspond to possible interpretations of the input words. For instance, the word \textit{pitcher} can have two possible interpretations:
\begin{enumerate}[label=\alph*.]
    \item A large jug
    \item The position in batting sports (mostly baseball) in which the player delivers the ball to the batter 
\end{enumerate}
\begin{figure}[ht!]
    \centering
    \begin{subfigure}{.45\linewidth}
        \centering
        \includegraphics[width=.5\linewidth]{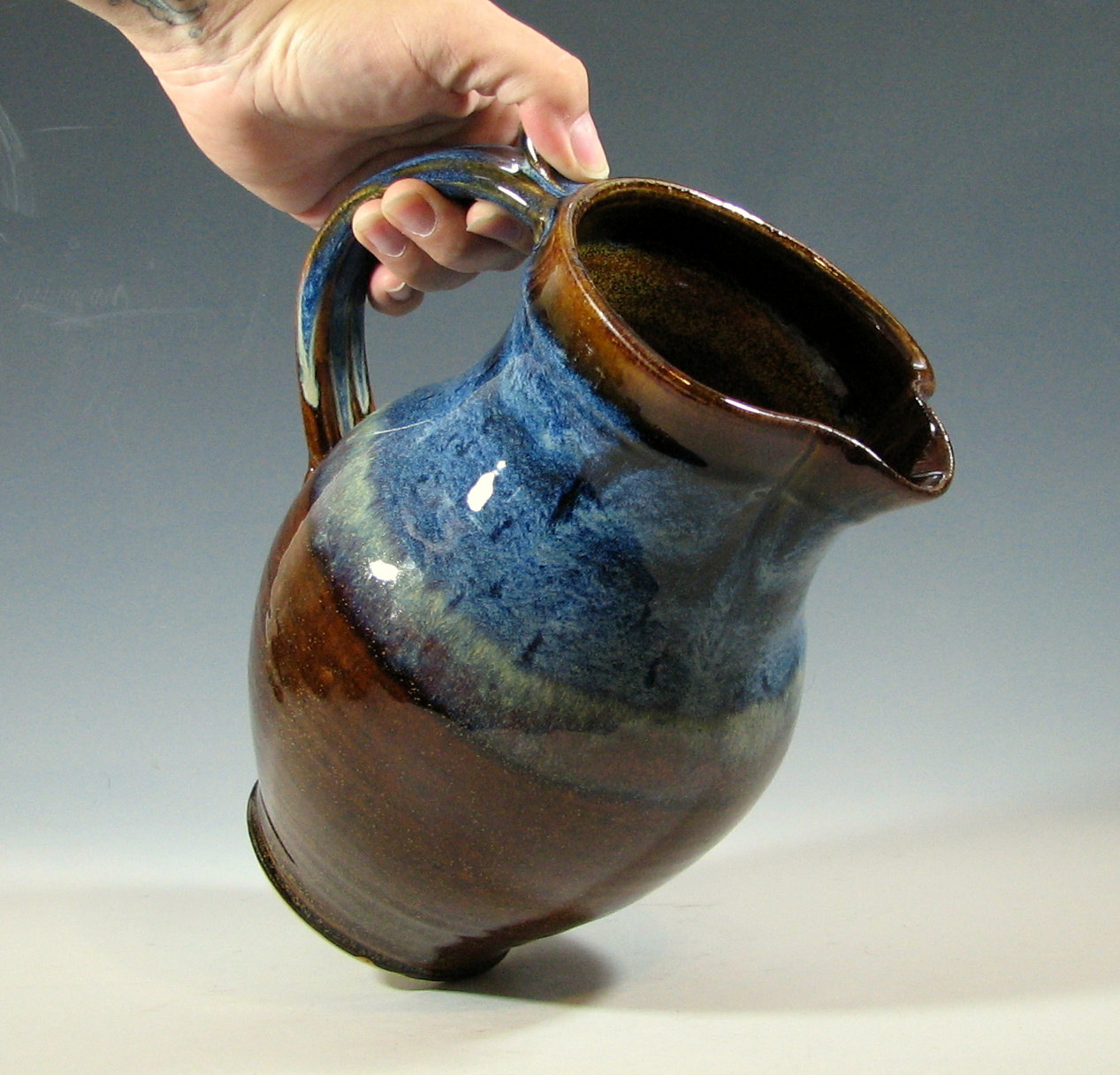}
        \caption{A jug}
    \end{subfigure}\qquad%
    \begin{subfigure}{.45\linewidth}
        \centering
        \includegraphics[width=.7\linewidth]{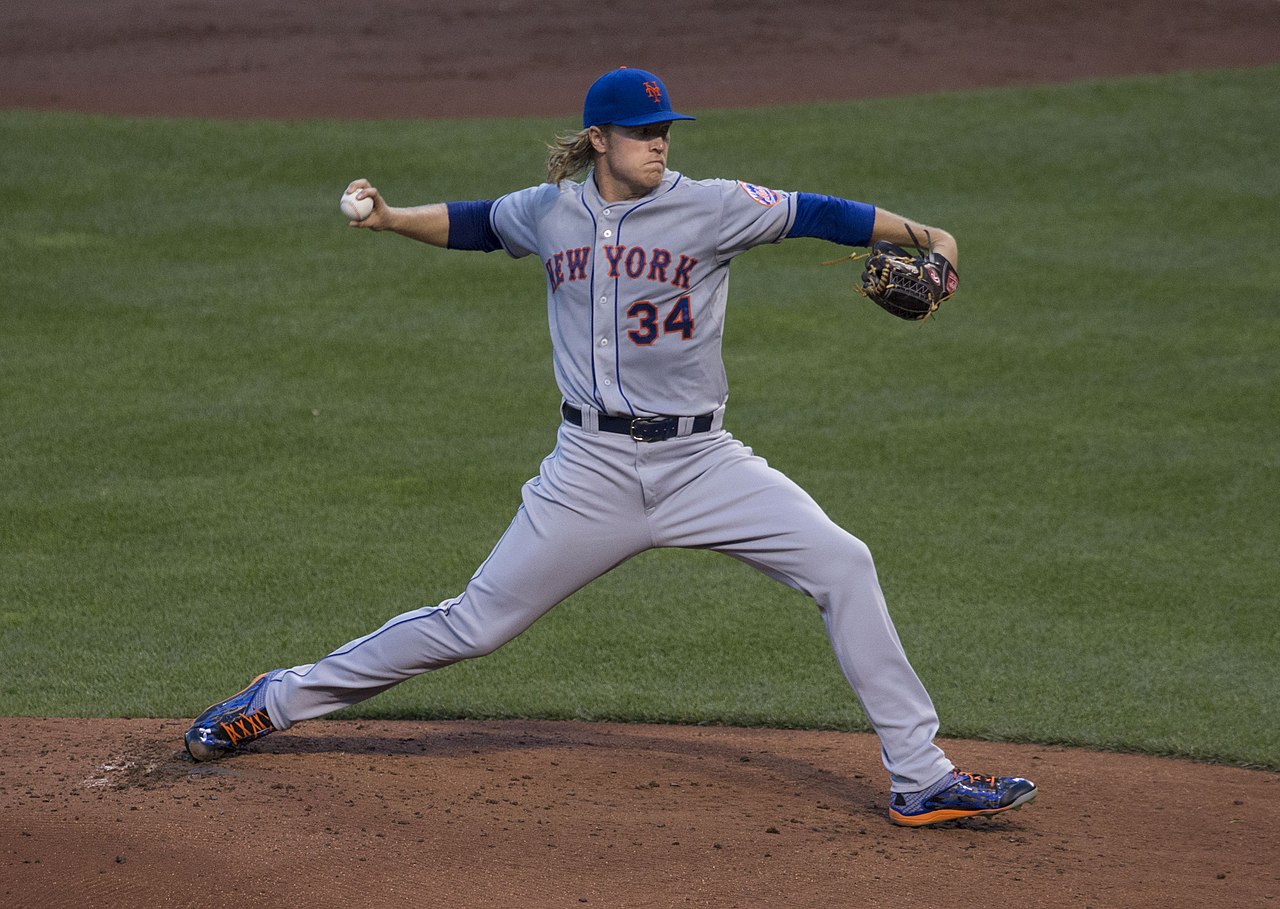}
        \caption{A baseball pitcher}
    \end{subfigure}
    \caption{The two interpretations of the word \textit{pitcher}}.
\end{figure}

If all of the words in the lexicon are lexically ambiguous, the set of outcomes for each measurement is not the singleton, and each interpretation will come with probabilities. Furthermore, these probability distributions will be \emph{dependent on their context}, meaning that the probability distributions are defined when all parties have made their choices of inputs (i.e. words). These probability distributions combined in different ways form empirical models that we will study in the next sections. 

\begin{table}[ht!]
    \centering
    \begin{tabular}{|c|c|}
        \hline\textbf{Quantum mechanics} & \textbf{Linguistics}\\\hline
        Parties & Grammatical roles/types\\\hline
        Inputs/Measurements & Words\\\hline
        Outputs/Outcomes & Interpretations\\\hline 
    \end{tabular}
    \caption{Analogy between quantum and linguistics scenarios.\label{tab:analogyQuantumLinguistics}}
\end{table}

\paragraph{} To estimate those probability distributions, we created two datasets that we will refer to as the \emph{corpus dataset} and the \emph{human judgment dataset}. For both datasets, we started with a list of homonymous and polysemous nouns from~\cite{Tanenhaus1979,RaynerDuffy1986}, and a list of homonymous and polysemous verbs from~\cite{PickeringFrisson2001,Shutova2010}. The list of these ambiguous words can be found in Appendix~\ref{app:listWords}. To simplify calculations, we also restricted the choice of meanings (resp. senses) to two distinct meanings (resp. senses) per word. For example, even though the verb \textit{tap} has multiple interpretations as a verb (e.g. touching something, secretly recording conversations, tap dancing, using up a resource, etc.), we decided to restrict to the following two meanings:
\begin{enumerate}[label=\alph*.]
    \item Hit something gently, e.g. \textit{tapping someone on the shoulder}
    \item Secretly listen or record what someone is saying using a device, e.g. \textit{tapping phones}
\end{enumerate}

\subsection{The corpus dataset}\label{subsec:corpus}

\paragraph{}Our first approach was to approximate these probability distributions using frequencies obtained from large corpora. To do so, we made use of two corpora, namely the British National Corpus (BNC)~\cite{BNC} containing 100 million words from a variety of sources and the \texttt{ukWaC} corpus~\cite{ukWaC} which contains more than 2 billion words obtained by crawling \texttt{.uk} web domains. Both corpora are part-of-speech annotated, but the semantic annotations had to be done by hand. Obtaining probabilities was then done as follows:
\begin{enumerate}
    \item As we are interested in SV and VO phrases, we recorded every occurrence in the corpus where one of the ambiguous nouns (from our list in Appendix~\ref{app:listWords}) was the subject or object of one of the ambiguous verb (also in the list in Appendix~\ref{app:listWords}).
    \item For each of these occurrences, we annotated the intended interpretation $x_v, x_n$ of the verb $v\in \mathcal{L}$ and the noun $n\in\mathcal{L}$. For instance, if we found the SV phrase \textit{the pitcher tapped} (i.e. $n=pitcher$ and $v=tap$) in the full sentence \textit{The pitcher tapped his glove and glanced over at the runner on first base}, then we would have annotated it as:
    \begin{align*}
        x_{pitcher}=& \text{ baseball player}\\
        x_{tap} =& \text{ hit something gently}
    \end{align*}
    \item For SV phrases, we then estimated the probability of the joint occurrence of the interpretations $x_v, x_n$ in the context ``$n$ is the subject of $v$'' as:
    \begin{equation}
        P\left[x_v,x_n~\middle|~n \text{ subject of }v\right] = \frac{N\left((n,v)\mapsto \left(x_n,x_v\right)\wedge n\text{ subject of }v\right)}{N\left(n\text{ subject of }v\right)}
    \end{equation}
    where $N$ records the number of occurrences of each of the events, and $(n,v)\mapsto \left(x_n,x_v\right)$ correspond to the event where the noun $n$ and verb $v$ are interpreted as $x_n$ and $x_v$ respectively. Similarly, in VO phrases, we calculated the probability distributions of measuring the joint occurrence of the interpretations $x_v$ and $x_n$ in the context ``$n$ is the object of $v$'' as:
    \begin{equation}
        P\left[x_v,x_n~\middle|~n \text{ object of }v\right] = \frac{N\left((n,v)\mapsto \left(x_n,x_v\right)\wedge n\text{ object of }v\right)}{N\left(n\text{ object of }v\right)}
    \end{equation}
\end{enumerate}
The obtained dataset is available at~\cite{corpusDataset}.

\subsubsection{Limitations of the dataset}
\paragraph{}This approach for collecting probabilities is intuitive, and the obtained probabilities are easily interpretable. In addition, large corpora are widely used in NLP and are easily (and freely) accessible. However, this approach also comes with some non-negligible drawbacks. 

\paragraph{}The most important one is the number of \emph{joint occurences} of two ambiguous words in a sentence, regardless of the grammatical relations imposed. These numbers were scarce, implying that the frequency obtained was not approaching the large number approximation of actual probabilities. 

\paragraph{}In addition, due to this small number of occurrences, not all possible combinations appeared in the corpora, and if they occurred, implausible interpretations in practice never appeared. For instance, we could easily imagine circumstances under which the phone conversations of a baseball pitcher would be recorded (e.g. if they were involved in a police inquiry). However, this meaning combination did not occur in either corpus. 

An explanation for this phenomenon is that written texts and conversations are meant to be understood as efficiently as possible. Combining ambiguous words in a single sentence may increase the sentence's overall ambiguity, making it unreadable. 

A probabilistic argument could also explain this. Namely, the probability of occurrence of a di-gram is smaller than the probability of occurrence of either word in the di-gram. Given that some of the words on the list in Appendix~\ref{app:listWords} did not occur very often to start with, e.g. the noun \textit{pitcher} and the verb \textit{to pen} only occurred 108 and 215 times respectively in the BNC, it would be unreasonable to expect the number SV or VO containing them to be high. Indeed, the phrase \textit{the pitcher pens}, although meaningful, never occurred.

\subsection{The human judgment dataset}\label{subsec:Amazon}

\paragraph{} To bypass the corpus dataset issues, we decided to ask \emph{humans} to rate the plausibility of the ambiguous phrases. In particular, this allowed us to \emph{choose} which phrases we wanted to study, and even highly unlikely meanings of phrases would be able to obtain a non-zero probability. In addition, fewer data points are necessary to get a reasonable estimate of the ``real'' probability distribution. Indeed, the judgment ratings of a single person are already approximations of the probability, whereas frequencies from corpora are dependent on the law of large numbers.

\paragraph{} The data collection proceeded as follows:
\begin{enumerate}[series=Amazon]
    \item We started with the same list of ambiguous nouns and verbs as per the corpus dataset (see Appendix~\ref{app:listWords}) and manually selected:
    \begin{itemize}
        \item 50 (noun, verb) pairs consisting of both homonymous nouns and verbs;
        \item 50 (noun, verb) pairs consisting of a homonymous noun and a polysemous verb;
        \item 50 (noun, verb) pairs consisting of a polysemous noun is polysemous and a homonymous verb;
        \item 50 (noun, verb) pairs consisting of both polysemous nouns and verbs.
    \end{itemize}
    The pairs selected were also checked to give a well-defined probability distribution for both SV and VO phrases, i.e. at least one of the meaning combinations will come with non-zero probability. These pairs can be found in the Appendix~\ref{app:human}.
    \item We then split the 400 phrases into batches of 8 randomly chosen phrases containing only SV or VO phrases. We submitted the batches on the Amazon Mechanical Turk platform, where they were sent to workers to annotate.
    \item 25 independent workers annotated each batch, and each worker was only allowed to annotate either an SV or a VO batch.
\end{enumerate}
We then presented the workers with the following task:
\begin{enumerate}[resume=Amazon]
    \item A phrase (e.g. \textit{the pitcher taps}) was shown to the annotator, who rated the plausibility of each of the meaning combinations as either:
    \begin{itemize}
        \item \textit{Impossible} (score: 0)
        \item \textit{Extremely unlikely} (score: 1)
        \item \textit{Very unlikely} (score: 2)
        \item \textit{Somewhat unlikely} (score: 3)
        \item \textit{Neutral} (score: 4)
        \item \textit{Somewhat likely} (score: 5)
        \item \textit{Very likely} (score: 6)
        \item \textit{Extremely likely} (score: 7)
    \end{itemize}
    \item For each phrase, we obtain a probability distribution from an annotator by normalising their score as follows:
    \begin{equation}
        P\left[(n,v)\mapsto \left(x_n, x_v\right)\right] = \frac{S(x_n, x_v)}{\sum_{\tilde{x}_n, \tilde{x}_v} S\left(\tilde
        {x}_n, \tilde{x}_v\right)}
    \end{equation}
    Here, $S: \mathcal{I}_n\times\mathcal{I}_v\to \{0, \ldots, 7\}$ is the function associating a score of an interpretation $(x_n,x_v)\in \mathcal{I}_n\times\mathcal{I}_v$. For instance, the set of scores corresponding to the SV phrase \textit{the pitcher taps}:
    \begin{center}
        \scalebox{.8}{\begin{tabular}{|c|c|c|c|c|}
            \hline \multirow{2}{*}{\textbf{Interpretation}} &$pitcher\mapsto jug$& $pitcher\mapsto jug$ &$pitcher\mapsto player$ &$pitcher\mapsto player$\\
            & $tap\mapsto hit$& $tap\mapsto record$&$tap\mapsto hit$& $tap\mapsto record$\\\hline
            \textbf{Score}& 5 & 1 & 7 & 3\\\hline
        \end{tabular}}
    \end{center}
    would have led to the probability distribution:
    \begin{center}
        \scalebox{.8}{\begin{tabular}{|c|c|c|c|c|}
            \hline \multirow{2}{*}{\textbf{Interpretation}} &$pitcher\mapsto jug$& $pitcher\mapsto jug$ &$pitcher\mapsto player$ &$pitcher\mapsto player$\\
            & $tap\mapsto hit$& $tap\mapsto record$&$tap\mapsto hit$& $tap\mapsto record$\\\hline
            \textbf{Probability}& $\frac{5}{16}$ & $\frac{1}{16}$ & $\frac{7}{16}$ & $\frac{3}{16}$\\\hline
        \end{tabular}}
    \end{center}
    \item The probability distributions of all of the workers who annotated the same phrase were then averaged\footnote{Note that that is the same as averaging the score and then normalising them}.
\end{enumerate}

An example of the task, as presented to the annotators, is illustrated in Fig.~\ref{fig:AmazonExample}.

\begin{figure}[p!]
    \centering
    \includegraphics[width=.9\linewidth]{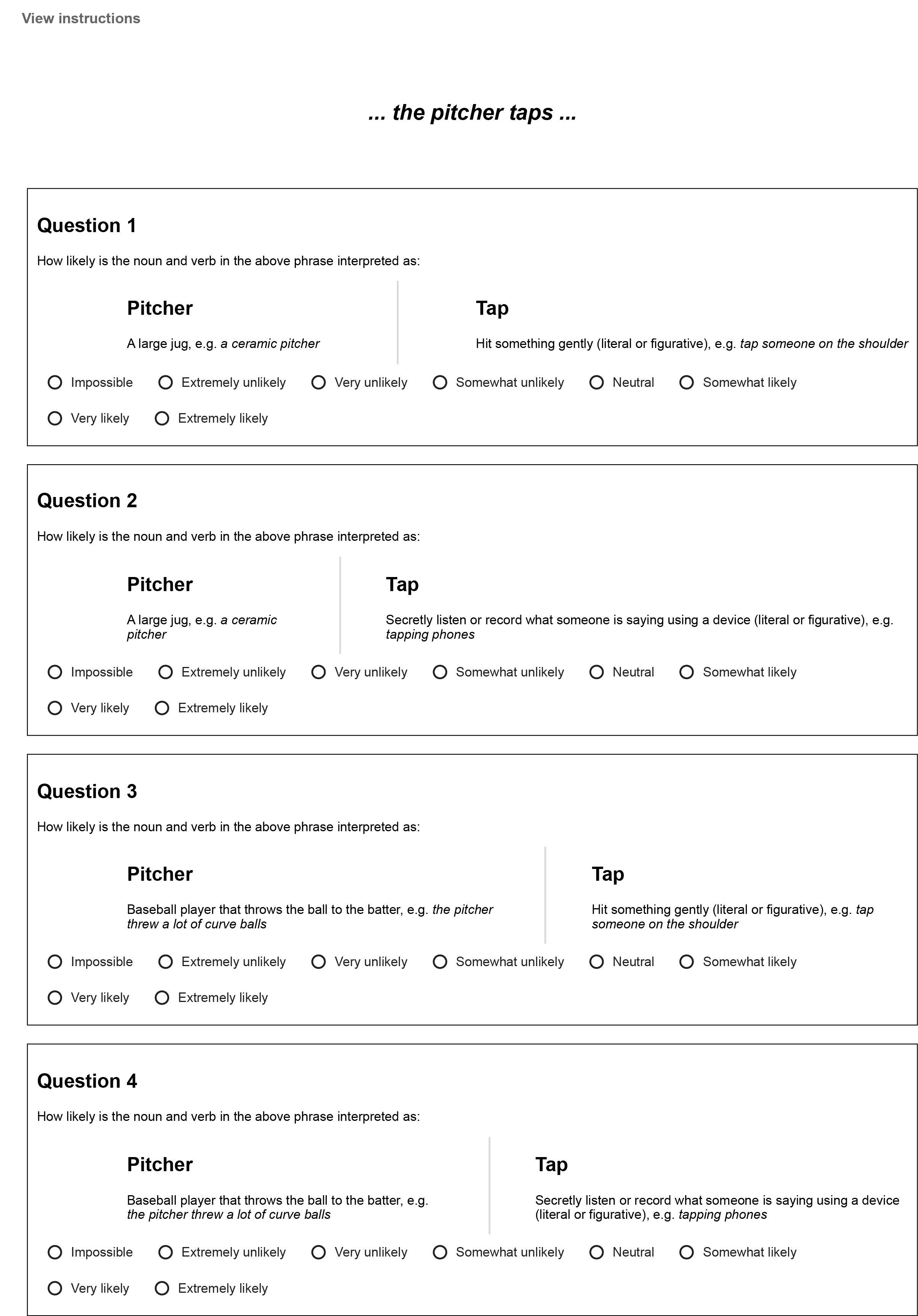}
    \caption{Example of a task seen by the Amazon Mechanical Turk workers.\label{fig:AmazonExample}}
\end{figure}
\section{On the quantum-like contextuality of ambiguous phrases}\label{sec:lexicalContext}
\paragraph{} It is often said that natural language is ``contextual'', notably in the context of lexical ambiguity. What is meant by that is that even though a single word, such as \textit{pitcher}, can have multiple interpretations (see Section~\ref{sec:lexicalMethods}), it may have a single accurate interpretation given a context. For example, consider the following sentences (taken from~\cite{FrazierRayner1990}):
\begin{enumerate}[label=(1\alph*)]
    \item\label{item:pithcerJug} Being so elegantly designed, the \textit{pitcher} pleased Mary.
    \item\label{item:pithcerBaseball} Throwing so many curve balls, the \textit{pitcher} pleased Mary.
\end{enumerate}
In the context~\ref{item:pithcerJug}, the only appropriate meaning of \textit{pitcher} is \textit{large jug}, whereas in~\ref{item:pithcerBaseball}, the only appropriate meaning is \textit{baseball player}.

However, the meaning of contextuality in quantum mechanics is different. Although, as in the linguistic sense, statistics of a system depend on their measurement contexts, the notion of contextuality is more complex than that. In quantum terminology, a system is said to be \emph{contextual} iff the dependence of the statistics on the contexts is ``essential'' in the sense that it cannot be attributed to other factors (see Section~\ref{sec:QDescription}). For instance, in the special case of non-locality, contextuality is observed if the statistics for the global measurement contexts cannot be explained entirely by its local behaviour. We are therefore interested in seeing whether the dependence of interpretation on (linguistic) context is also ``essential'' and whether we can witness quantum-like correlations between word interpretations.

\paragraph{}Historically, contextuality in quantum mechanics has been proven via inequalities, usually referred to as \emph{Bell inequalities}. These inequalities, however, depend on a crucial assumption of the system, namely that it is \emph{no-signalling}.

We recall from Section~\ref{sec:QDescription}, that a system is no-signalling iff the probability distributions agree on the intersections of their contexts, i.e. if all local sections are compatible. This condition is well motivated in the case of Bell-type experiments in Quantum Mechanics as the no-signalling property states that \emph{information} cannot be transmitted faster than light (i.e. non-locally). However, obtaining perfectly no-signalling empirical models is, in practice, unfeasible due to the finite nature of experimental results and the imperfections of the experimental apparatus. 

In addition, there is no fundamental reason why no-signalling should even apply in the case of ambiguities in natural language. In fact, the psycholinguistics literature would suggest that probability distributions arising from lexically ambiguous phrases \emph{should be signalling}. For example, the probabilities associated with the different meanings of \textit{pitcher} should be different in the phrase \textit{ceramic pitcher} to the ones in the phrase \textit{baseball pitcher}.

Several extensions of the notion of quantum contextuality have been proposed to account for signalling systems, including the Contextuality-by-Default (CbD) framework~\cite{Dzhafarov2016contextcontent} (see Section~\ref{subsec:CbD}), and the corrected Bell-inequalities approach~\cite{Emeriau2022} based on the sheaf-theoretic framework of contextuality~\cite{AbramskyBrad} (see Section~\ref{subsec:sheafContextuality}). In this section, we propose to apply these generalised frameworks to investigate the contextuality of lexically ambiguous phrases.

\subsection{Cyclic models of rank 2}\label{subsec:contextr2}

\paragraph{}We start by discussing the simplest possible models, which only contain two words, a noun $n$ and a verb $v$, and two different ways to combine them, in our case, as a subject-verb or verb-object phrase. In our previous analogy (see Table~\ref{tab:analogyQuantumLinguistics}), this means that we have two parties, corresponding to grammatical types noun ($N$) and verb ($V$), each of them having a unique choice of input, but for which we can obtain two different probability distributions (one corresponding to SV and one corresponding to VO). These models are called \emph{cyclic systems of rank 2} in the CbD literature~\cite{Dzhafarov2016contextcontent,Kujala2015,Dzhafarov2017}.

\begin{ex}
    Consider the case where $n=pitcher$ and $v=tap$, taking the interpretations from Section~\ref{sec:lexicalMethods}. This choice of words leads to a valid SV phrase, namely \textit{the pitcher taps ...}, and a VO phrase, namely \textit{... taps the pitcher}. We can then associate the probability distributions with these two phrases\footnote{This is not an empirical model obtained from the corpus or human judgment dataset. The probability distribution for the SV phrase is taken from an example in Section~\ref{sec:lexicalMethods}, and the probability distribution for the VO phrase is taken from the corpus dataset.}:
    \begin{center}
        \begin{tabular}{|c|c|c|c|c|}
            \hline$(N,V)$ & (a., a.) & (a., b.) & (b., a.) & (b., b.)\\\hline
            $(pitcher, tap)_{SV}$ & 5/16 & 1/16 & 7/16 & 3/16\\\hline
            $(pitcher, tap)_{VO}$ & 17/22 & 0 & 15/22 & 0\\\hline
        \end{tabular}
    \end{center}
    This pair of probability distributions will be our empirical model for these contexts.
\end{ex}
\begin{rmk}
    This situation is similar to the question-order effect investigated by Wang et al.~\cite{Wang2013}. The work of~\cite{Wang2013} consisted of a series of behavioural experiments where participants were asked the same set of questions in different orders, and it was shown that the answers appeared to depend on the order of the questions that were asked. In~\cite{Wang2013}, the authors argued that such experiments exhibit quantum-like contextuality. However, the authors of~\cite{Dzhafarov2016} demonstrated that the apparent contextuality was due to signalling.
\end{rmk}

\paragraph{}We note here that these models are trivially non-contextual within the generalisation of the sheaf-theoretic framework of contextuality. This can be seen as follows. Let us start with a decomposition of a given (signalling) empirical model $e$:
\begin{equation}\label{eq:exCyclic2SF}
    e = \lambda \cdot e_{NS} + (1-\lambda) e'
\end{equation}
In the case of the models described above, in any no-signalling model $e_{NS}$ satisfying \eqref{eq:exCyclic2SF}, the probability distributions of both the SV phrase and the VO phrase collapse to a single probability distribution. Therefore, any of the no-signalling empirical models $e_{NS}$ are trivially \emph{non-contextual} (i.e. a global probability distribution exists and corresponds to the unique probability distribution in the empirical model). 

However, these models can exhibit contextuality within the CbD framework. We will therefore focus on the CbD analysis of such cyclic systems of rank 2.

\paragraph{}As first demonstrated in~\cite{Wang2021a}, it is indeed possible to find instances of linguistic empirical models of this form that exhibit CbD-contextuality. First, we found two contextual examples in the corpus dataset: empirical models where $n=boxer$ and $v=adopt$, and $n=pitcher$ and $v=throw$. These empirical models are depicted in Fig.~\ref{fig:corpusContextual}. The degree of non-contextuality of both of these models can also be calculated to be respectively:
\begin{align}
    \mathsf{NCNT2}(e_{(boxer,adopt)}) =& -\frac{1}{30}\\
    \mathsf{NCNT2}(e_{(pitcher,throw)}) =& -\frac{7}{30}
\end{align}
We recall from Section~\ref{subsec:CbD} that the fact that these numbers are negative shows that these models are contextual. 
\begin{figure}[ht!]
    \centering
    \begin{minipage}[c]{\linewidth}
        \centering
        \begin{tabular}{|c|c|c|c|c|}
            \hline$(N,V)$ & $(0,0)$ & $(0,0)$ & $(0,0)$ & $(0,0)$\\\hline
            $(boxer, adopt)_{SV}$ & $1/4$ & $0$ & $0$ & $3/4$\\\hline
            $(boxer, adopt)_{VO}$ & $0$ & $29/30$ & $1/30$ & $0$\\\hline
        \end{tabular}
    \end{minipage}

    \vspace{20pt}

    \begin{minipage}[c]{\linewidth}
        \centering
        \begin{tabular}{|c|c|c|c|c|}
            \hline$(N,V)$ & $(0,0)$ & $(0,0)$ & $(0,0)$ & $(0,0)$\\\hline
            $(pitcher, throw)_{SV}$ & $0$ & $2/3$ & $1/3$ & $0$\\\hline
            $(pitcher, throw)_{VO}$ & $2/5$ & $0$ & $1/10$ & $1/2$\\\hline
        \end{tabular}
    \end{minipage}
    \caption{Empirical models of cyclic systems of rank 2 which were found to be CbD-contextual within the corpus dataset.\label{fig:corpusContextual}}
\end{figure}
\begin{rmk}
    The study described in~\cite{Wang2021a} provided the first instances of quantum-like contextuality in linguistic scenarios, which takes signalling into account. Previous work existing in the literature~\cite{Bruza2015} have claimed to have violated Bell inequalities in natural language data. However, these did not assume the no-signalling condition and were later found to be non-contextual within the CbD framework~\cite{Dzhafarov2016}.  
\end{rmk}

Other contextual cyclic systems of rank 2 emerged from the human judgment dataset. These are found in Fig.~\ref{subfig:HJContextual} with their degree of contextuality. 
\begin{figure}[ht!]
    \centering
    \begin{subfigure}[c]{.4\linewidth}
        \centering
        \begin{tabular}{|c|c|}
            \hline$(n,v)$ & $-\mathsf{NCTN}_2$\\\hline
            (pitcher, throw) & $0.233$\\\hline
            (boxer, adopt) & $0.033$\\\hline
        \end{tabular}
        \caption{Corpus dataset}
    \end{subfigure}\qquad%
    \begin{subfigure}[c]{.4\linewidth}
        \centering
        \begin{tabular}{|c|c|}
            \hline$(n,v)$ & $-\mathsf{NCTN}_2$\\\hline
            (file, admit) & $0.232$\\\hline
            (cabinet, reflect) & $0.199$\\\hline
            (volume, conduct) & $0.111$\\\hline
            (perch, file) & $0.073$\\\hline
            (plant, trap) & $0.052$\\\hline
            (press, file) & $0.042$\\\hline
            (swallow, admit) & $0.021$\\\hline
            (press, conduct) & $0.011$\\\hline
            (port, bill) & $0.008$\\\hline
            (organ, bill) & $0.001$\\\hline
        \end{tabular}
        \caption{Human judgment dataset\label{subfig:HJContextual}}
    \end{subfigure}
    \caption{All of the found examples of CbD-contextual cyclic systems of rank 2 (sorted by degree of contextuality).}
\end{figure}

\paragraph{}The presence of contextuality in quantum systems has some important consequences. From a foundational point of view, contextuality distinguishes between classical and quantum behaviours~\cite{Bell1966,Kochen1967}. From a computational point of view, it is a resource that allows quantum advantages to arise~\cite{Howard2014,Abramsky2019,Duarte2018}. The reader should refer back to Section~\ref{subsec:BackgroundContext} for a more detailed discussion. 

In the case of the CbD-contextuality, this interpretation could be clearer. The witnesses of CbD-contextual in linguistics imply that the influence of the context over meaning selection is highly non-trivial and consists of a ``truly contextual influence''. This finding concurs with the intuition that the context is the main factor in the interpretation of lexically ambiguous words. 

\subsubsection{Degrees of contextuality}
\paragraph{}In addition to finding out whether a given empirical model is contextual, it is also possible to calculate \emph{how contextual} (or equivalently how non-contextual) a given empirical model is within the CbD framework\footnote{It is also possible to calculate some measures of contextuality in the sheaf-theoretic framework, the contextual fraction $\mathsf{CF}$ being the most prominently used~\cite{Abramsky2019,Abramsky2017}. We are, however, not making use of them in this thesis.}. In this work, we will make use of $\mathsf{NCNT2}$ defined in Section~\ref{subsec:CbD} to see how the levels of ambiguity of words can influence the \emph{degree of contextuality}. 

\paragraph{Results}We calculated the degree of non-contextuality for the empirical models and classified them in terms of the ambiguity of their nouns and verbs (see Fig.~\ref{fig:NCNT2r2}). 

We found that the contextuality mostly depends on the ambiguity of the verb. Specifically, the degree of non-contextuality is higher if the verb is polysemous. Equivalently, a rank 2 system will be more contextual whenever the verb is homonymous. Although we observed this trend in both the corpus and human judgment datasets, it was only statistically significant\footnote{The definition of statistical significance is standardly defined as follows. A feature is said to be statistically significant iff the associated $p$-value is less than $0.05$.} in the latter. The $t$-test comparing the data concerning homonymous and polysemous had $p$-value $p=0.006$ in the human judgment dataset, as opposed to $p=0.314$ in the corpus dataset. 

In addition, we found no dependence on the levels of ambiguity of the nouns and the degree of (non-)contextuality in either dataset. The $p$-values were $p=0.38$ and $p=0.15$ for the corpus and the human judgment dataset, respectively.

\begin{figure}[ht!]
    \centering
    \begin{subfigure}[c]{.45\linewidth}
        \centering
        \includegraphics[width=\linewidth]{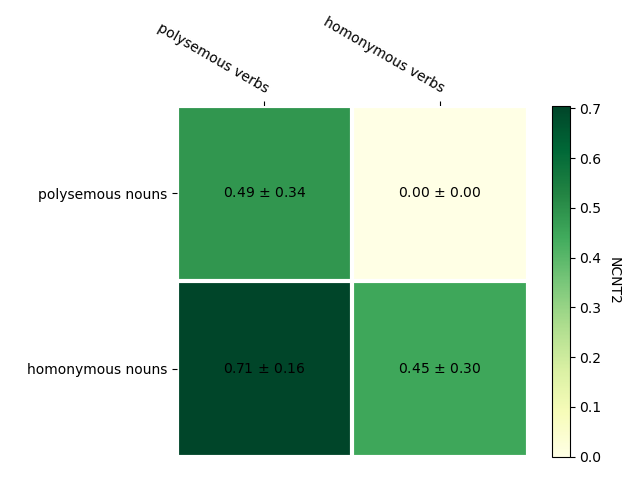}
        \caption{Corpus dataset}
    \end{subfigure}\qquad%
    \begin{subfigure}[c]{.45\linewidth}
        \centering
        \includegraphics[width=\linewidth]{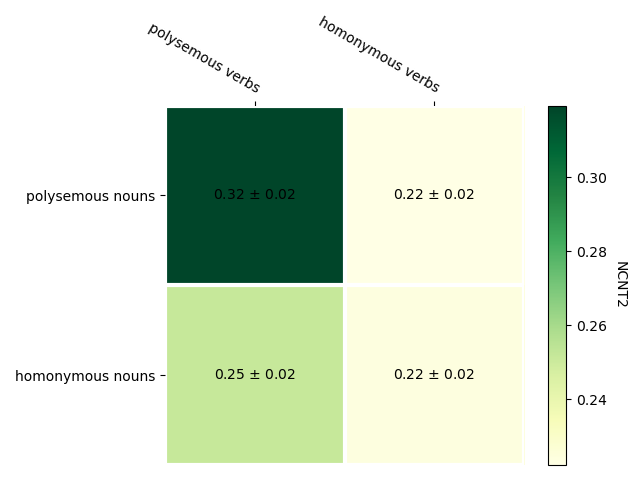}
        \caption{Human judgement dataset}
    \end{subfigure}
    \caption{Mean $\mathsf{NCNT2}$ depending on whether the noun and verb were polysemous or homonymous. The errors quoted are the standard errors of the means.\label{fig:NCNT2r2}}
\end{figure}

\paragraph{Discussion}The conclusions one can draw from the degree of (non-)contextuality are hindered by the need for a more operational interpretation of CbD-contextuality. Therefore, the interpretation of the result is not completely clear. 

A possible interpretation of finding contextual witnesses is the following. Recall that, in the case of cyclic systems of rank 2, the choices of contexts are whether the pair of words is found in an SV or VO phrase. This result suggests that a verb would use its context more whenever its possible interpretations are unrelated. 

This conclusion is consistent with the fact that homonymous verbs use the knowledge of both their subject and object in their first disambiguation stage. In contrast, polysemous verbs will use information from a broader context (if available) to be disambiguated.

\subsection{Cyclic models of rank 4}\label{subsec:contextr4}

\paragraph{}We have also studied a larger type of empirical model: cyclic systems of rank 4. These models are analogues of the (2,2,2)-Bell scenarios, i.e. scenarios that consist of 2 parties, each of these parties having 2 choices of inputs, and each local measurement having at 2 possible outcomes. In these models, we take the parties to represent specific dependency relations, e.g. main verb, subject, or object. In particular, we will focus on SV models (with parties $S$ and $V$) and VO models (with parties $V$ and $O$). In each model, the parties will choose between two words of the correct grammatical type (i.e. nouns for parties $S$ and $O$ and verbs for $V$ parties). Each word will then have two different possible interpretations.

\begin{ex}
    Let's consider an example of a VO model where $V$ is allowed to choose between the verbs in $\{tap, box\}$ and $O$ is allowed to choose the object of these verbs in the set $\{pitcher, cabinet\}$. The interpretations of \textit{tap} and \textit{pitcher} are taken to be the ones from Section~\ref{sec:lexicalMethods}. In addition, the interpretations of \textit{box} and \textit{cabinet} are respectively:
    \begin{enumerate}[label=\alph*.]
        \item To put something in a box, e.g. \textit{boxing up clothes and books}
        \item To practice the sport of boxing, e.g. \textit{He boxed professionally for years}
    \end{enumerate}
    and:
    \begin{enumerate}[label=\alph*.]
        \item A small group of the most important people in government, e.g. \textit{a cabinet minister}
        \item A piece of furniture with shelves, cupboards, or drawers, e.g. \textit{a glass-fronted cabinet}
    \end{enumerate}
    Then, from the corpus dataset, we obtained the following empirical model:
    \begin{center}
        \begin{tabular}{|c|c|c|c|c|}
            \hline$(V,O)$ & (a., a.) & (a., b.) & (b., a.) & (b., b.)\\\hline
            (tap, pitcher) & 17/22 & 15/22 & 0 & 0 \\\hline 
            (tap, cabinet) & 1/21 & 3/7 & 11/21 & 0 \\\hline
            (box, pitcher) & 3/4 & 1/4 & 0 & 0 \\\hline 
            (box, cabinet) & 3/7 & 10/21 & 2/21 & 0 \\\hline
        \end{tabular}
    \end{center}
\end{ex}

\paragraph{}These types of models have the potential to exhibit contextuality in both the Contextuality-by-Default framework and the extension of the sheaf-theoretic framework for corrected Bell inequalities. However, none of the empirical models obtained in the corpus or human judgment dataset were contextual (in either framework). 

Regarding the CbD framework, this could be explained as the probability of obtaining a contextual model decreases as the systems get bigger. We can estimate these probabilities by random sampling from a uniform distribution, and the likelihood of obtaining a CbD-contextual cyclic system of rank 2 is about $17\%$ and drops to about $0.01\%$ for cyclic systems of rank 4. In addition, violations of the corrected Bell inequalities of~\cite{Emeriau2022} is a stronger condition than the CbD notion of contextuality (see Section~\ref{subsec:CbD}). Therefore, obtaining a contextual model within this framework is also quite unlikely. 

Even though we haven't found any contextual cyclic system of rank 4, this does not mean that \emph{no such system} is contextual. A larger scale experiment will be needed to obtain witnesses of contextuality in these types of models -- we leave this to future work.

\subsubsection{Degrees of contextuality}
\paragraph{}Although we have not been able to find contextual witnesses in cyclic systems of rank 4, we can study the degrees of (non-)contextuality of the obtained empirical models, as we did for cyclic systems of rank 2. 

\paragraph{Methods}We want to know how the levels of ambiguity of words of different syntactic roles (i.e. subject, verb, or object) influence the degree of contextuality of the respective empirical models. In addition, each empirical model can have:
\begin{itemize}
    \item 0 polysemous verbs \& 2 homonymous verbs;
    \item 1 polysemous verb \& 1 homonymous verb;
    \item 2 polysemous verb \& 0 homonymous verbs;
\end{itemize}
and similarly for subjects and objects. Therefore, we will classify the SV and VO models in terms of their numbers of homonymous verbs, subjects, and objects\footnote{We could have equivalently chosen to classify them in terms of their number of polysemous verbs, subjects, and objects. However, the adopted convention fits our intuition that homonymous words are, to some extent, ``more ambiguous'' than polysemous ones.}. We will mostly focus this analysis on the human judgment dataset since the corpus dataset did not have empirical in all categories. For example, the instances recorded did not lead to any SV cyclic system of rank 4 with two polysemous subjects and two homonymous verbs.

In addition, we are interested in the monotonic relations between the number of homonymous words (of each type) and the degree of contextuality, i.e. whether contextuality increases or decreases as the number of homonymous verbs, subjects, or objects increases. On the other hand, the existence of non-monotonic relations between these quantities does not lead to easily interpretable results. For example, it is unclear what it would mean for the contextuality to be higher whenever we can choose between a polysemous and a homonymous verb. Hence, we will make use of Spearman's rank correlation coefficient $\rho$, which will assess whether a monotonic relation exists between two random variables, one being the number of homonymous verbs, subjects, or objects and the other being the degree of non-contextuality $\mathsf{NCNT2}$.

\paragraph{Results}The values of the degrees of non-contextuality for each class of empirical models can be found in Fig.~\ref{fig:NCNT2r4}. An ANalysis Of Variance (ANOVA) first shows that the degrees of non-contextuality are statistically different across the different categories ($p<10^{-4}$ for SV models and $p=0.005$ for VO models). 

In addition, we observe that in SV models, the degree of non-contextuality increases as the number of polysemous verbs and subjects increases. This can be verified using Spearman's correlation coefficient $\rho$, which was $\rho=-0.27$ with associated $p$-value $p<10^{-6}$ for the correlation with respect to the number of homonymous verbs, and $\rho=-0.20$ with $p=0.0002$ for the correlation with respect to the number of homonymous subjects. In both cases, the negativity of the $\rho$'s shows that $\mathsf{NCNT2}$ decreases as the number of homonymous verbs and subjects increases. In addition, the fact that $p$-values are both $<0.05$ shows that we are more than $95\%$ confident that these coefficients are different from $0$ (i.e. a correlation exists with a $95\%$ confidence). 

In VO models, these trends are much milder and, in fact, not statistically significant. We can see that the degree of direct influence decreases as the number of homonymous verbs decreases ($\rho=-0.11$, $p=0.053$). However, no monotonic correlation is found with respect to the number of homonymous objects ($\rho=0.03$, $p=0.52$).   
\begin{figure}[ht!]
    \centering
    \begin{subfigure}{.45\linewidth}
        \centering
        \includegraphics[width=\linewidth]{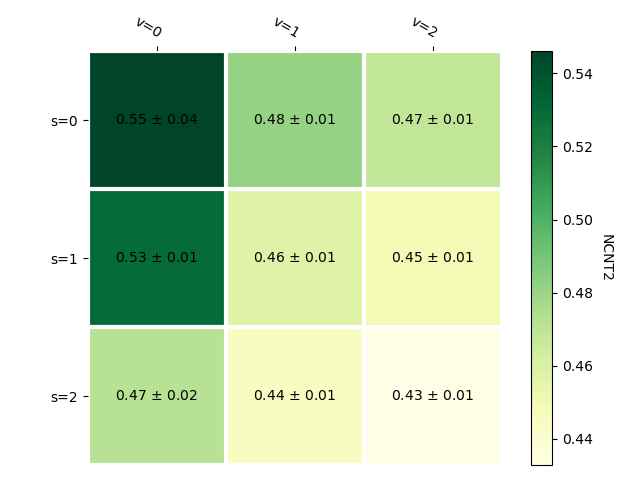}
        \caption{SV models}
    \end{subfigure}\qquad %
    \begin{subfigure}{.45\linewidth}
        \centering
        \includegraphics[width=\linewidth]{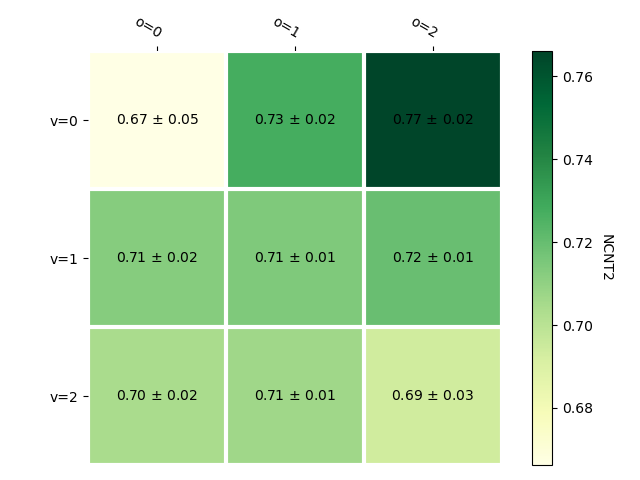}
        \caption{VO models}
    \end{subfigure}
    \caption{Averages of the $\mathsf{NCNT2}$ as a function of the number of homonymous verbs, subjects, and objects. \label{fig:NCNT2r4}}
\end{figure}

\paragraph{Discussion}As in cyclic systems of rank 2, the main factor influencing the (non-)contextuality of the systems is the levels of ambiguity of the verbs. Indeed, we have already shown that the degree of contextuality increases as the number of homonymous verbs increases. This reinforces the intuition that homonymous verbs use their arguments (here, the context) more intrinsically than their polysemous analogues. 

In addition, the same applies to homonymous nouns in SV phrases. Namely, homonymous nouns would lead to a higher amount of ``true'' contextuality in the obtained probability distributions. 

The fact that this occurs in SV models only (and not in VO models) suggests that this finding relates to the position of the disambiguating context of homonymous nouns. In particular, a slowdown in the reading time of homonymous nouns has been observed when the disambiguating context is found after the target noun~\cite{FrazierRayner1990}, but this slowdown was lesser for polysemous nouns. Hence, in the case of SV phrases, where the disambiguating context can only be the verb (situated after the noun), the observed high degree of contextuality suggests that homonymous nouns use their context in a less trivial way when the disambiguating context is found after it.

\section{Degrees of signalling and the levels of ambiguity}\label{sec:lexicalSignal}
\paragraph{}In the previous section, we saw that signalling is the main obstacle to studying contextuality in natural language data. However, the presence of signalling is not in itself a weakness of natural language data. In fact, in most linguistic studies of lexical ambiguity, the fact that the interpretation of ambiguous words changes with the context, i.e. signalling, is the focus point. Here, we propose to study the amount of signalling present in the different empirical models and see what conclusions can be drawn.

Two ways of quantifying signalling are available to us, namely $\mathsf{SF}$ coming from the sheaf-theoretic framework of contextuality, and $\Delta$ from the CbD framework. In addition, the latter can be split with respect to the individual input choices. Therefore, we can study the amount of signalling coming from a specific choice of word or, equivalently, how different the probability distributions of a given word are in the different contexts it is found in. We will then study the correlations between these different quantities and the levels of ambiguity of words in cyclic systems of rank 2 and 4.

\begin{rmk}
    Although the sheaf-theoretic and Context-by-Default frameworks are essential in the definition of the signalling fraction and the degrees of direct influence respectively, the mathematical machinery they employ are only used implicitly in this Section.
\end{rmk}

\begin{rmk}
    The results of this section may appear challenging to interpret and reason about, most of all because some results were verified in one dataset but not the other. In addition, it is not common in the linguistic literature to study phrases where more than one word is clearly (lexically) ambiguous. In particular, not much is known about what happens if the context of an ambiguous word is itself ambiguous.
\end{rmk}

\subsection{Cyclic systems of rank 2}\label{subsec:signalr2}

\paragraph{}We start with the cyclic systems of rank 2, which we recall contains a noun/verb pair and two different contexts, SV and VO. 

\subsubsection{Overall degrees of signalling}
\paragraph{}We start by looking at the total degree of direct influence $\Delta$, as well as the signalling fraction $\mathsf{SF}$, both of which measure how signalling the whole system is (we also recall from Section~\ref{subsec:CbD} that these two quantities are not unrelated). Then, as we did in the analysis of the degree of contextuality in the previous section, we partition both our datasets in terms of the levels of ambiguity of the nouns and the verbs. 

\paragraph{Results}The overall degrees of signalling for all of the different classes of empirical models are shown in Fig.~\ref{fig:DeltaSFr2}. We observe the same trend in both of these datasets, namely that the overall signalling of the system increases as the number of polysemous words increases. An ANOVA reveals a statistical difference between $\mathsf{SF}$ or $\Delta$ and the different classes of models in \emph{the human judgment dataset only}. The $p$-values were $p=0.015$ for $\Delta$ and $p<10^{-9}$ for $\mathsf{SF}$ in the human judgment dataset, as opposed to $p=0.98$ and $p=0.70$ for $\Delta$ and $\mathsf{SF}$ respectively in the corpus dataset. 

In the human judgment dataset, we can verify using a $t$-test that $\Delta$ and $\mathsf{SF}$ are higher for nouns with multiple senses ($p=0.04$ and $p=0.02$ respectively). Similarly, $\Delta$ and $\mathsf{SF}$ are also higher for verbs with multiple senses ($p=0.02$ and $p=0.03$ resp.). No statistically significant trend could be found in the corpus dataset.

\begin{figure}[ht!]
    \centering
    \begin{subfigure}{\linewidth}
        \centering
        \includegraphics[width=.45\linewidth]{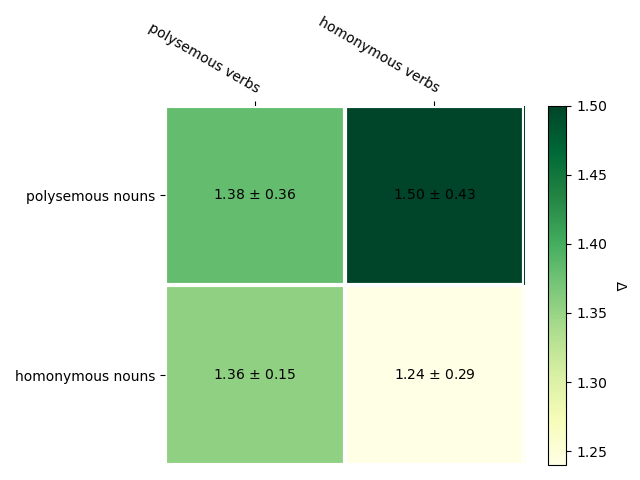}\qquad\includegraphics[width=.45\linewidth]{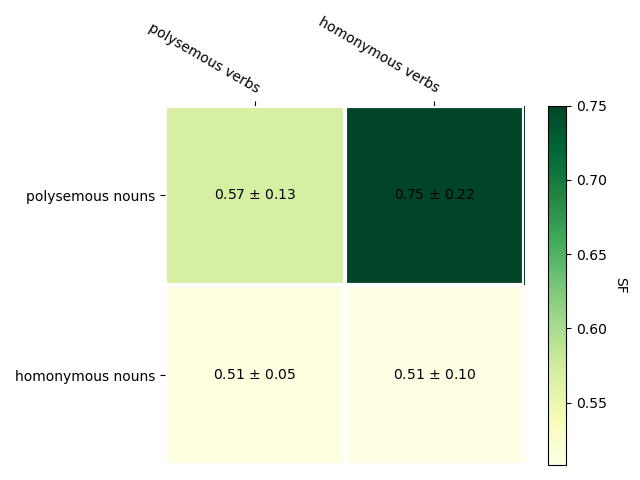}
        \caption{Corpus dataset}
    \end{subfigure}
    \begin{subfigure}{\linewidth}
        \centering
        \includegraphics[width=.45\linewidth]{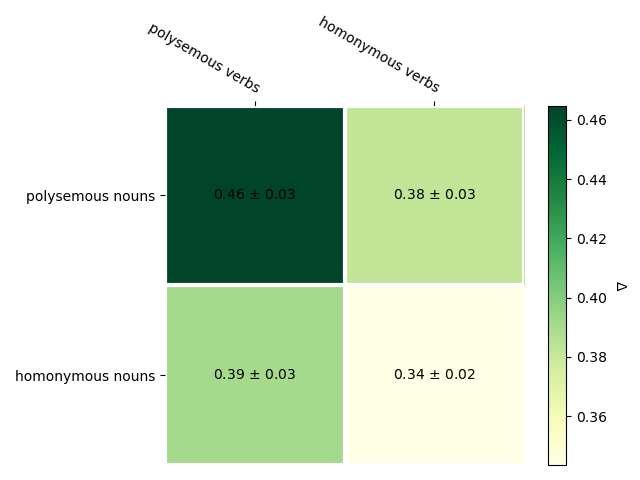}\qquad \includegraphics[width=.45\linewidth]{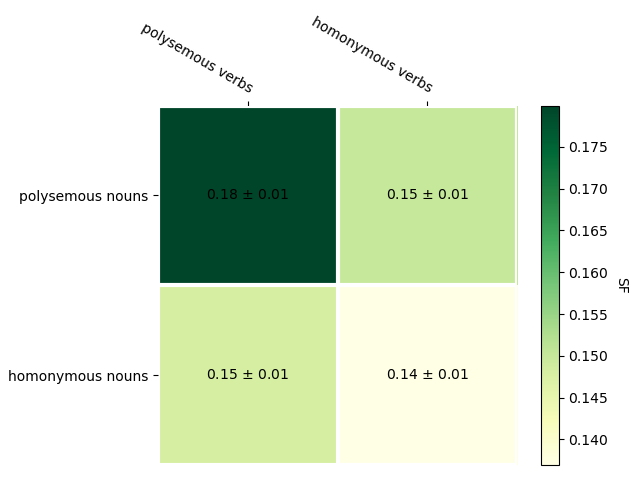}
        \caption{Human judgement dataset}
    \end{subfigure}
    \caption{Averaged $\Delta$ and $\mathsf{SF}$ of the cyclic systems of rank 2, depending on whether the noun and verb were polysemous or homonymous. The errors quoted are the standard errors of the means.\label{fig:DeltaSFr2}}
\end{figure}

\paragraph{Discussion}This phenomenon could be explained using the theory of \emph{underspecification}, in which the interpretation of polysemous words is essentially \emph{created} from its context, whereas meanings of homonymous words are \emph{selected} using contextual information~\cite{underspecification} (see also Section~\ref{subsec:lexicalPsycho}). Hence, assuming that the SV and VO contexts are unrelated, this would imply that there is potential for having different interpretations of the same polysemous word. On the other hand, one may argue that its context window is too small for homonymous words to swing widely from one meaning to the other. 

\subsubsection{Individual degrees of signalling}
\paragraph{}Given the above interpretation of the results, we also expect the individual degrees of direct influence to follow the same tendency. 

\paragraph{Results}This is, however, not quite verified at the level of individual degrees of direct influence. In particular, in the human judgment dataset, the degree of direct influence from the verb was higher whenever the verb was polysemous, which is consistent with our hypothesis. On the other hand, the degree of direct influence from the noun was higher whenever the noun was \emph{homonymous}. In both cases, the observed effect was relatively small, and the difference in individual direct influence was $0.06$ for verbs of different levels of ambiguity and $0.08$ for nouns of different levels of ambiguity. In both cases, these differences, however small, were still found to be statistically significant (with respective $p$-values $p=0.006$ and $p<10^{-4}$). 

In the corpus dataset, the reverse trend is observed (i.e. homonymous verbs and polysemous nouns had, on average, higher degrees of direct influence). However, none of the differences were statistically significant ($p=0.27$ for verbs and $p=0.78$ for nouns). 

\paragraph{Discussion}Due to the size of the effect and the fact that datasets did not agree on the findings, we could conclude that such a difference may be due to statistical fluctuations. However, if these effects were in fact ``true'', this would imply that some more complex mechanism occurs in the disambiguation process of both nouns and verbs.

\subsubsection{On the disambiguation windows of ambiguous verbs}

\paragraph{The corpus dataset}In~\cite{Wang2021b}, we showed that the \emph{proportion} of the direct influence coming from the verb was statistically significantly higher for homonymous verbs than for polysemous verbs (see Fig.~\ref{fig:JCogSciHist}) in the corpus dataset. This fact was attributed to the difference in disambiguation windows in homonymous and polysemous verbs. 

Indeed, we recall that the first disambiguation stage for homonymous verbs happens as soon as all of its arguments are known. In contrast, the reader only selects the senses of polysemous verbs at the end of the phrase or sentence. 

Now, as in cyclic systems of rank 2, we are studying the difference in distributions between subject-verb and verb-object contexts. We can expect differences in the distributions of the interpretations of homonymous verbs, as we are precisely within this first disambiguation stage. 

The fact that we did not observe any such effect for nouns was explained by the fact that the changes of contexts studied remain within the same disambiguation window for both polysemous and homonymous nouns.

\begin{figure}[ht!]
    \centering
    \includegraphics[width=\linewidth]{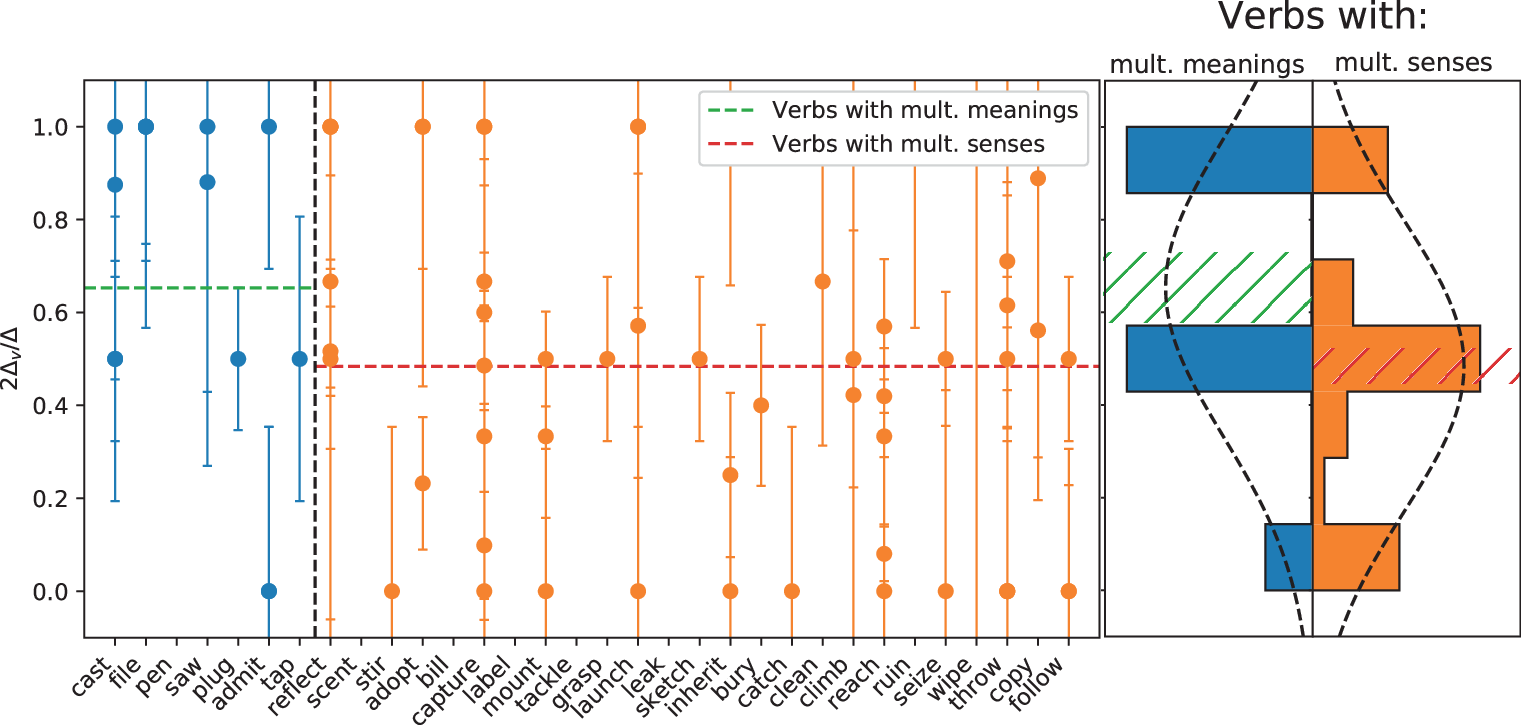}
    \caption{Relative contributions of the verb content to the overall direct influence given different levels of ambiguity for the verb or the noun. The left-hand figures correspond to the contributions of the verb; averages for each level of ambiguity are shown with dotted lines. The right-hand figures correspond to the distributions of these data points. The hatched area depicts the 66\%-confidence intervals for the means. The fitted normal distributions are also plotted.\label{fig:JCogSciHist}}
\end{figure}

\paragraph{The human judgment dataset}In the human judgment dataset, this relation was not verified. In fact, we observed the opposite, namely that the proportion of direct influence from polysemous verbs was greater than the one for homonymous verbs, but this was not statistically significant ($p=0.22$). 

On the other hand, the degree of direct influence from the verb was, on average, higher than the degree of direct influence from the noun ($p=0.003$ for a related $t$-test). This observation is consistent with the fact that verbs need their arguments to be disambiguated, whereas nouns do not. Hence, the variations in the interpretations are greater between the two contexts (SV or VO) compared to the variations in the interpretations of the nouns. However, we would have expected from the previous discussion that homonymous verbs would have higher degrees of direct influence, which is not the case in the human judgment dataset (see above). 

In addition, we should also note that this is \emph{not} at all observed in the corpus dataset ($p=0.51$). This lessens the findings of~\cite{Wang2021b}, but, on the other hand, offers alternative evidence that readers do not disambiguate verbs and nouns in the same way and that the presence of the arguments of the verb is essential in their disambiguation process.

\subsubsection{On the ambiguity of the context}

\paragraph{}Lastly, some cross-effect between the ambiguity of one word and the degree of direct influence of the other has been observed in both datasets. 

\paragraph{Results}In the human judgment dataset, $\Delta_v$ was higher whenever the verb was combined with a polysemous noun (with an average difference of $0.13$ and $p<10^{-9}$). We also observed a similar effect in the corpus dataset, but the effect size is much smaller (average difference of $0.014$) and not statistically significant ($p=0.95$). 

In addition, $\Delta_n$ was slightly higher whenever the noun in the empirical model was combined with a polysemous verb (average difference of $0.004$ in the human judgment dataset and $0.27$ in the corpus dataset). However, in this case, none of the observed differences were statistically significant ($p=0.82$ in the human judgment dataset and $p=0.21$ in the corpus dataset, respectively). 

\paragraph{Discussion}The interpretation of this fact would be related to something that has yet to be studied in the psycholinguistic literature, namely, what happens if the context itself is ambiguous? This result suggests that if the context is polysemous or underspecified, the interpretation of a target word becomes more variable. In contrast, if the context can have several unrelated interpretations, then the interpretation of the target word is also more defined. This will be made more transparent and intuitive in Section~\ref{sec:lexicalCausal} when studying the \emph{causality} of the systems instead of its signalling.

\subsection{Cyclic systems of rank 4}\label{subsec:signalr4}

\paragraph{}We now look at the degrees of signalling of the cyclic systems of rank 4. As in Section~\ref{subsec:contextr4}, we do not have enough data in the corpus dataset to cover all possible combinations of the number of polysemous and homonymous subjects, verbs, and objects. Hence, we will shift our focus to the human judgment dataset only. 

\begin{rmk}
    In terms of notation, we will denote as $\Delta_A$, $A\in \{S, V, O\}$ for the total amount of direct influence coming from the two subjects, verbs, or objects (i.e. the different parties), and $\Delta_a$ for $a\in \{s, v, o\}$ for the individual direct influence coming from a particular choice of subject, verb or object (i.e. the individual choices of inputs). For instance, in the SV empirical model with inputs $\left\{s_1, s_2, v_1, v_2\right\}$, we would have:
    \begin{align*}
        \Delta_S =& \Delta_{s_1} + \Delta_{s_2}\\
        \Delta_V =& \Delta_{v_1} + \Delta_{v_2}
    \end{align*}
    and:
    \begin{equation*}
        \Delta = \Delta_S + \Delta_V = \sum_{w\in \left\{s_1, s_2, v_1, v_2\right\}} \Delta_w
    \end{equation*}
\end{rmk}

\subsubsection{Subject-Verb vs. Verb-Object}

\paragraph{} We start by looking at the difference in signalling between SV and VO models. 

\paragraph{Results}The degree of signalling, measured by both the signalling fraction $\mathsf{SF}$ and CbD measure $\Delta$ are (statistically) significantly higher for SV models compared to the VO models (the average difference was $0.38$ for $\Delta$ and $0.08$ for $\mathsf{SF}$, with respective $p$-values $p<10^{-52}$ and $p<10^{-40}$). 

\paragraph{Discussion}This suggests that the interpretations of VO phrases are easier to obtain, and their interpretations are more well-defined than those of SV phrases. This concurs with the usual grouping of VO compounds as VPs, whereas subject-verb does not generally correspond to anything special, for example, in context-free grammars~\cite{GentnerFrance2013}. 

\begin{table}[ht!]
    \centering
    \begin{subtable}{\linewidth}
        \centering
        \begin{tabular}{ccc}
            SV & \includegraphics[width=.4\linewidth, align=c]{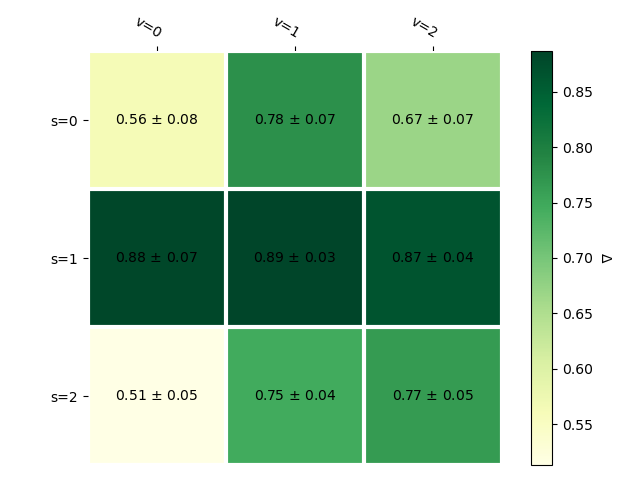} & \includegraphics[width=.4\linewidth, align=c]{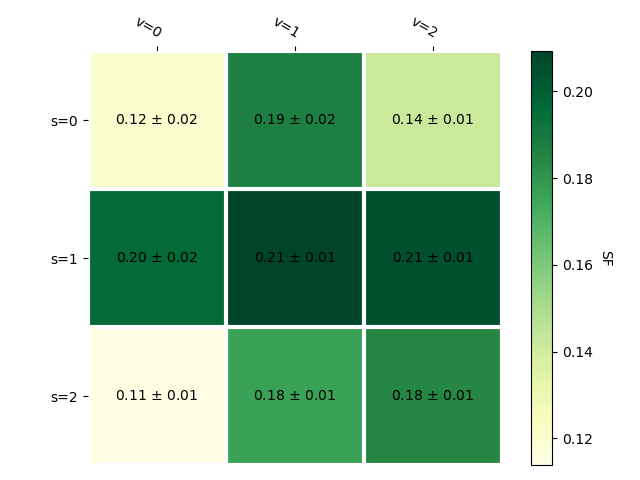}\\
            VO & \includegraphics[width=.4\linewidth, align=c]{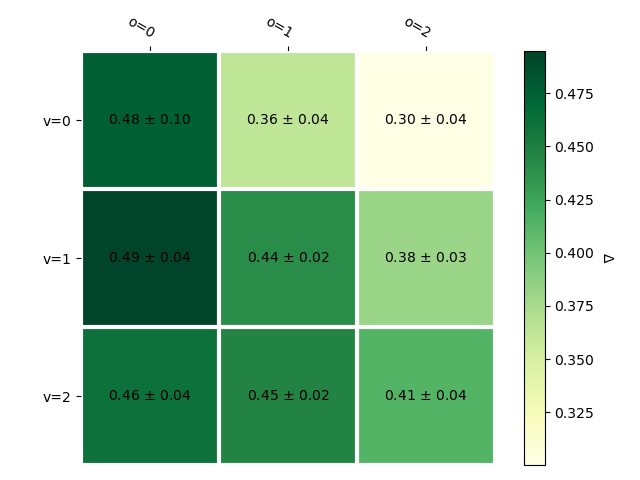} & \includegraphics[width=.4\linewidth, align=c]{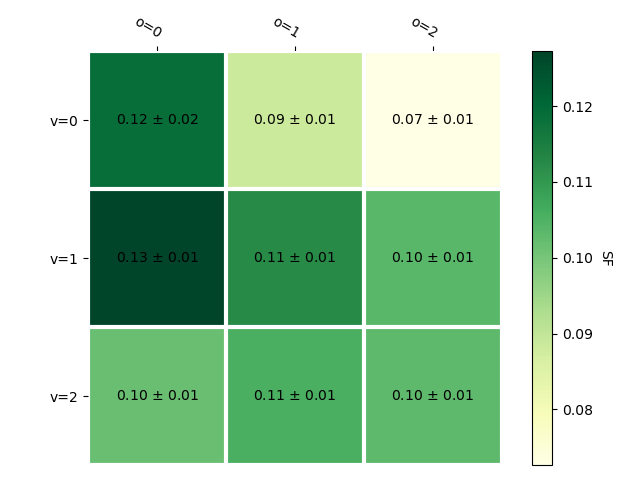}
        \end{tabular}
        \caption{$\mathsf{SF}$ and $\Delta$ as a function of the number of homonymous subjects, verbs, and objects. The errors quoted are the standard errors of the mean.}
    \end{subtable}

    \begin{subtable}{.45\linewidth}
        \centering
        \begin{tabular}{|c|c|c|c|}
            \hline&& $R$ & $p$\\\hline
            \multirow{2}{*}{$\Delta$} & Subject & 0.0170 & 0.7607\\\cline{2-4}
            & Verb & 0.0329 & 0.5568\\\hline 
            \multirow{2}{*}{$\mathsf{SF}$} & Subject & 0.0466 & 0.4045\\\cline{2-4}
            & Verb & 0.05353 & 0.3383\\\hline 
        \end{tabular}
        \caption{SV}
    \end{subtable}\qquad%
    \begin{subtable}{.45\linewidth}
        \centering
        \begin{tabular}{|c|c|c|c|}
            \hline&& $R$ & $p$\\\hline
            \multirow{2}{*}{$\Delta$} & Verb & 0.13755 & 0.01349\\\cline{2-4}
            & Object & -0.1560 & 0.005 \\\hline 
            \multirow{2}{*}{$\mathsf{SF}$} & Verb & 0.0650 & 0.2444\\\cline{2-4}
            & Object & -0.1289 & 0.0207\\\hline 
        \end{tabular}
        \caption{VO}
    \end{subtable}
    \caption{Analysis of the monotonicity of the amount of signalling and the number of homonymous words in SV and VO empirical models.\label{tab:DeltaSFr4}}
\end{table}

\subsubsection{Verbs vs. Nouns}

\paragraph{} As in the cyclic system of rank 2, the degree of direct influence from the verb, $\Delta_V$, was found to be higher than the direct influence from both the subject ($\Delta_S$) and the object ($\Delta_O$); in SV models, the average difference was $0.1961$ with associated $p$-value $p<10^{-20}$, and in VO, the average difference was $0.08$ with associated $p$-value $p<10^{-13}$. As before, we justify this finding as the verbs are, in general, not fully disambiguated at these stages, and therefore, their interpretation is more variable.

\subsubsection{Degrees of signalling for homonymous and polysemous words}
\paragraph{Ambiguity of the objects}We observe that in VO models, the degrees of signalling (both $\mathsf{SF}$ and $\Delta$) increase as the number of polysemous objects increases. The obtained Spearman's correlation coefficients are $\rho=0.16$ for $\Delta$ and $\rho = 0.13$ for $\mathsf{SF}$, and the respective $p$-values are $p=0.005$ and $p=0.02$.

This finding is consistent through all of the fine-grained measures of direct influence. Indeed, $\Delta_O$ was increasing the more polysemous objects were in the empirical models ($\rho = 0.14$, $p=0.01$), and the individual degrees of direct influence $\Delta_o$ were also higher for polysemous than homonymous objects (average difference of $0.02$, and associated $p$-value $p=0.03$). 

This resonates with the theory of underspecification which happens in polysemous words. In other words, since noun senses are more dependent on their context than noun meanings, it makes sense that the interpretation of a polysemous word would vary more than the interpretation of a homonymous word.


\paragraph{Ambiguity of other dependencies}The effect described above is only observed in objects in VO models. In the other cases, the different degrees of signalling did not appear to be (monotonically) related to the numbers of polysemous/homonymous subjects or verbs (see Table~\ref{tab:DeltaSFr4}), or if they ``exist'', their size is very small. 

We will attribute this lack of relations between the ambiguity of the different words and their degrees of direct influence to other deciding factors in the process of their disambiguation. For instance, the disambiguation of the verbs in both SV and VO models will be affected not only by their own levels of ambiguity but also by the ambiguity of their arguments (i.e., their subject or objects). Similarly, the subject could also be affected by their context, either by the position or ambiguity of the context.

\subsubsection{On the ambiguity of the context}

\paragraph{} As for cyclic systems of rank 2, it is possible to find some relations between the degree of ambiguity of the context and the degree of direct influence of a given target word, although these correlations were very moderate. 

\paragraph{Results}In both SV and VO models, the degree of direct influence of the verbs $\Delta_V$ increased as the number of polysemous arguments (subject or object) increased. Spearman's correlation was $\rho = 0.08$ in VO systems and $\rho=0.005$ in SV systems, and neither were statistically significant ($p=0.17$ and $p=90.3$ respectively). On the other hand, $\Delta_S$ and $\Delta_O$ increased as the number of homonymous verbs increased ($\rho = 0.15$ and $p=0.006$ in the case of $\Delta_S$ and $\rho=0.17$ and $p=0.002$ in the case of $\Delta_O$). 

\paragraph{Discussion}At first sight, this appears contradictory to the findings in rank 2 systems, where we recall that $\Delta_n$ increased as the number of polysemous verbs increased. However, one crucial difference here is that context is more ``fixed'' in SV and VO models, whereas it is not in cyclic systems of rank 2 previously described. For instance, consider an SV model. Taking the subject to be the target word, the context $\_~v$ is fixed in the two contexts $s_1~v$ and $s_2~v$. In contrast, in cyclic systems of rank 2, the two contexts of a noun would be $\_~v$ and $v~\_$, which are fundamentally different. 

Moreover, the influence of the ambiguity of the verb of their subject and object could be understood as follows. The disambiguation of homonymous verbs first starts when the arguments are established. Hence, the reader will begin disambiguating both the verb and the noun, which creates some interaction between the choice of meaning of the verb and the subject/object. This makes the interpretation of the subject/object more uncertain. On the other hand, if the verb is polysemous, the meaning of the nouns will be more well-defined, as the disambiguation of the verb itself will be delayed to the end of the sentence (and hence beyond the scope of our experiments). 


\subsection{Discussion of the results}\label{subsec:signalDiscussion}

\paragraph{}The signalling property of empirical models provides insight into the mechanisms at the heart of the disambiguation process. However, our analysis is made difficult and not easily interpretable from how empirical models are formed. Indeed, to have non-trivial empirical models, all of the inputs must be ambiguous. Yet, this adds complexity to the study as the distinction between (linguistic) context and target word is symmetric, i.e. a word can be both a target-word and a context-word. In addition, the notion of signalling is also bidirectional, i.e. we can only see whether a dependence exists between two variables $A$ and $B$, not whether $A$ influences $B$ or the other way around. In the next section, we will remedy these problems by examining the causal influences between the different parts-of-speech.
\section{The causality of the disambiguation process}\label{sec:lexicalCausal}
\paragraph{}In the previous section, we have seen that the signalling property of natural language systems is not necessarily a hindrance but can give us some insight into human behaviour when disambiguating lexically ambiguous words and phrases. In this section, we go one step further and study the \emph{structure} of the signalling present in natural language data. To do so, we make use of the extensions of the sheaf-theoretic framework of contextuality to account for \emph{causality}~\cite{sheafcausality,sheafcausalityB,abramskyCausality} (see Section~\ref{subsec:sheafContextuality}). In this line of research, we are interested in the \emph{direction} of signalling. For instance, does the disambiguation order follow the reading order, i.e. is disambiguation purely incremental, or is backtracking necessary?
\begin{rmk}
    In this perspective, we can see the event of choosing the input as reading a new word, associating an outcome to a choice of measurement will then correspond to the disambiguation step. Hence, having a causal order that does not follow the linear ordering of the words in a given sentence does \emph{not} mean that the reader reads the words in a different order, but rather that the \emph{understanding} process is not instant (i.e. does not follow the reading order). 
\end{rmk}

\paragraph{}As for the no-signalling property, obtaining statistics that are fully compatible with a given causal order is not feasible in practice. Therefore, instead of calculating the no-signalling fraction $\mathsf{NSF} = 1-\mathsf{SF}$ (i.e. the amount of the empirical model which is compatible with a perfectly no-signalling system), we can calculate the \emph{causal fraction} introduced in Section~\ref{subsec:sheafContextuality} which measure the amount of the empirical model which is consistent with a given causal order~\cite{sheafcausality}. In addition, similar to the previous sections, we will focus on the simplest non-trivial scenarios to minimise the calculations. Here, the smallest non-trivial system we can consider is similar to the (2,2,2)-Bell scenarios (or cyclic systems of rank 4). Indeed, in the case where only 1 party is present, there is a unique compatible causal order: the party $A$ can influence itself, and every such system would be trivially consistent with it. Similarly, if we have 2 parties, but at least one party has a unique choice of input, then this scenario reduces, without loss of generality, to the one-party case. Finally, if only 1 outcome is possible for all of the measurements, then all of the measurements are deterministic, and the analysis becomes once again trivial. 

\begin{rmk}
    This work does not assume that parties are spacelike separated entities. Instead, the notion of party corresponds to entities isolated in time or space. For instance, we could consider a single person in a lab performing a sequence of 3 sequential measurements to count as 3 different parties.    
\end{rmk}

The systems that we will consider are, as in Sections~\ref{sec:lexicalSignal} and~\ref{sec:lexicalContext}, SV and VO systems, where the two parties are either $S$ and $V$ or $V$ and $O$ and the interpretations of input and outcomes are left unchanged. In particular, we here focus on \emph{definite causal orders}, i.e. causal orders represented by direct acyclic graphs~\cite{Pearl2011}. In such graphs, the nodes correspond to random variables, which can have inputs and outputs. The directions of the arrows represent the causal relations, e.g. if $A\to B$, then this means that the input of $A$ can influence the output of $B$, and the absence of arrows shows the independence of random variables. Finally, the acyclicity condition ensures that a given event cannot influence its past. 

In a two-party system like ours, say with parties $A$ and $B$, there are only three possible definite causal orders, namely:
\begin{enumerate}[]
    \item\label{item:causalNS} \begin{tikzpicture}[baseline=-.2em]
        \node[shape=circle, draw=black, fill=white] (A) at (-1, 0){$A$};
        \node[shape=circle, draw=black, fill=white] (B) at (1, 0){$B$};
    \end{tikzpicture}
    \item\label{item:causalAtoB} \begin{tikzpicture}[baseline=-.2em]
        \node[shape=circle, draw=black, fill=white] (A) at (-1, 0){$A$};
        \node[shape=circle, draw=black, fill=white] (B) at (1, 0){$B$};
        \draw [->] (A) to (B);
    \end{tikzpicture}
    \item\label{item:causalBtoA} \begin{tikzpicture}[baseline=-.2em]
        \node[shape=circle, draw=black, fill=white] (A) at (-1, 0){$A$};
        \node[shape=circle, draw=black, fill=white] (B) at (1, 0){$B$};
        \draw [->] (B) to (A);
    \end{tikzpicture}
\end{enumerate}
The situation of~\ref{item:causalNS} corresponds to our familiar no-signalling scenarios. The cases~\ref{item:causalAtoB} and~\ref{item:causalBtoA} respectively represent situations where the party $A$ can influence $B$ and where the party $B$ can influence $A$. We will denote the causal fractions associated with causal orders~\ref{item:causalNS},~\ref{item:causalAtoB} and~\ref{item:causalBtoA} as respectively $\mathsf{CausF}_{NS}=\mathsf{NSF}$, $\mathsf{CausF}_{A\to B}$ and $\mathsf{CausF}_{B\to A}$.

\begin{rmk}
    The causal orders $A \to B$ and $B\to A$ are not mutually exclusive. Therefore, we could have a model which is highly consistent with both causal orders separately.
\end{rmk}

\begin{rmk}[On the no-signalling causal order]
    A system with parties $A,B$ is said to be no-signalling iff it is compatible with both causal models of the form $A\to B$ and $B\to A$. Therefore, it is stronger than causal orders $A\to B$ and $B\to A$. Therefore, in terms of the causal fractions, this means that:
    \begin{equation}
        \mathsf{NSF} \leq \min \left(\mathsf{CausF}_{A\to B}, \mathsf{CausF}_{B\to A}\right)
    \end{equation}
\end{rmk}

We are investigating which of these causal orders is the most relevant, i.e., explains most of the system's statistics. We do so by comparing the different causal fractions obtained for the different causal order, notably~\ref{item:causalAtoB}-~\ref{item:causalBtoA}. 

\paragraph{}Calculating the causal fraction of an arbitrary model is not straightforward, as it requires solving an optimisation problem. However, given the specific form of the models we are considering, the causal fractions for each causal order can be calculated efficiently.

\begin{prop}\label{prop:causalFracForm}
    In a (2,2,2)-Bell-type scenario with parties $A$ and $B$, the causal fraction is given by:
    \begin{equation}\label{eq:causalF}
        \mathsf{CausF}_{A\to B} = \min_{a\in \{a_1,a_2\}, o\in \{0,1\}} 1 - \left|\left.e_{(a, b_1)}\right|_{A}\left(o\right) - \left.e_{(a, b_2)}\right|_{A}\left(o\right)\right|
    \end{equation}
    The causal fraction for the $B\to A$ causal order can be obtained by applying the formula to a relabelled empirical model.
\end{prop}

This proposition's proof can be found in Appendix~\ref{app:causalFracForm}.

\paragraph{}Through initial calculations, we have found that the empirical models obtained in the corpus dataset were too sparse, and therefore did not lead to any conclusive result. In the rest of this work, we will focus on the causal analysis of the human judgment dataset, as their probability distributions are, on the whole, of better quality.

\subsection{The direction of signalling in SV and VO models}\label{subsec:causalDirection}

\paragraph{}We start by looking at the compatibility of our data with the different causal orders described above. 

\paragraph{Results}The causal fractions obtained for the SV and VO models are shown in Fig.~\ref{subfig:scatterCausFSV} and~\ref{subfig:scatterCausFVO}, respectively. 

The data reveals that SV phrases are predominantly compatible with the $S\to V$ causal order. Indeed, all of the models have a causal fraction $\mathsf{CausF}_{S\to V}>0.7$ (the causal fractions $\mathsf{CausF}_{S\to V}$ is on average $0.89$), and both the $V\to S$ and the no-signalling fractions achieve lower causal fraction values where the causal fractions are on average  $0.83$ and $0.19$ respectively, see also Fig.~\ref{subfig:histCausFSV}. 

Similarly, the VO models achieve a causal fraction with the $O\to V$ order higher than $76\%$ (and $0.93$ on average), and the other causal fractions reach significantly lower scores ($\mathsf{CausF}_{V\to O}\sim 0.91$ and $\mathsf{SF}\sim 0.11$), see Fig.~\ref{subfig:histCausFVO}. We note that in the case of the VO models, this suggests that the disambiguation order in these phrases is \emph{opposite to the reading order} (assuming the standard active voice SVO structure of English).

\begin{figure}[htbp]
    \centering
    \begin{subfigure}{.45\linewidth}
        \centering
        \includegraphics[width=\linewidth]{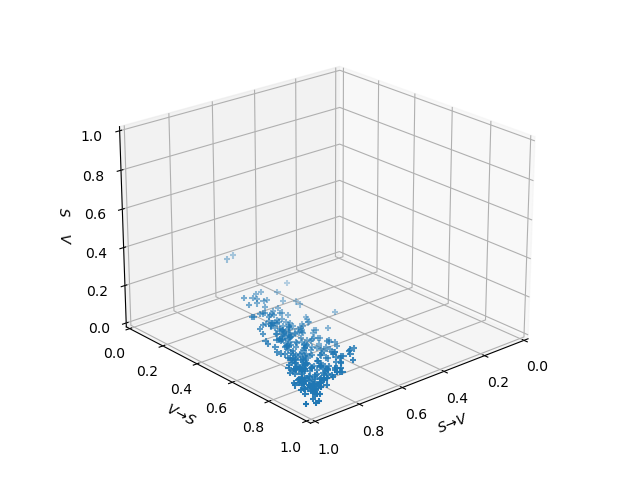}
        \caption{Scatter plot of the different causal fractions for SV empirical models\label{subfig:scatterCausFSV}}
    \end{subfigure}\qquad%
    \begin{subfigure}{.45\linewidth}
        \centering
        \includegraphics[width=\linewidth]{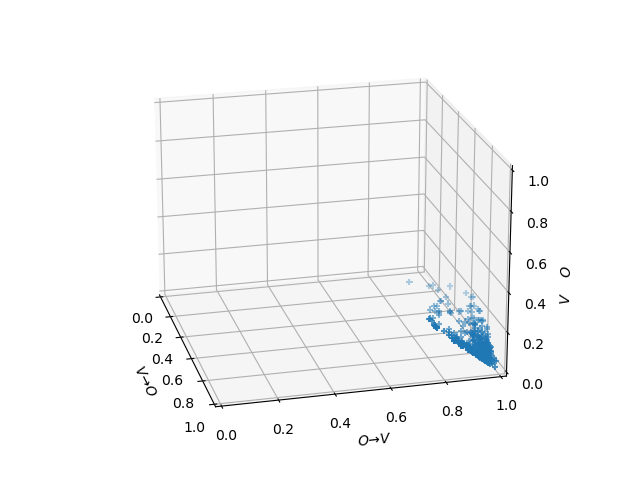}
        \caption{Scatter plot of the different causal fractions for VO empirical models\label{subfig:scatterCausFVO}}
    \end{subfigure}

    \begin{subfigure}{.45\linewidth}
        \centering
        \includegraphics[width=\linewidth]{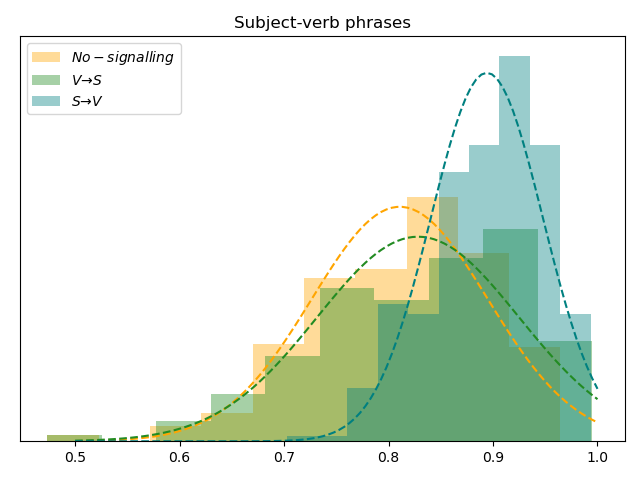}
        \caption{Histogram of the recorded causal fractions for SV empirical models\label{subfig:histCausFSV}}
    \end{subfigure}\qquad%
    \begin{subfigure}{.45\linewidth}
        \centering
        \includegraphics[width=\linewidth]{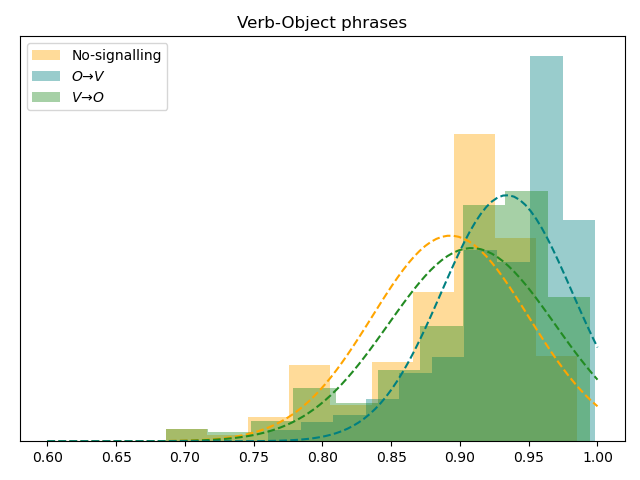}
        \caption{Histogram of the recorded causal fractions for VO empirical models\label{subfig:histCausFVO}}
    \end{subfigure}
    \caption{Distributions of the causal fractions}
\end{figure}

\paragraph{Discussion} These results show that the interpretation of ambiguous verbs is more affected by the choices of arguments (i.e. subject or object) rather than the other way around. This result is also validated by the research in psycholinguistics, which has indeed shown that the reader delays the disambiguation process until the arguments of the verb are known in the case of a homonymous verb and until the end of the phrase or sentence in the case of polysemous verbs~\cite{PickeringFrisson2001}. 

In addition, the causal fractions of VO models are generally higher and less variable (i.e. smaller standard deviation) than those obtained in SV models, see Fig.~\ref{fig:histSVvVO}. This may suggest that the disambiguation process for VO ambiguous phrases is more straightforward than for SV ones. 

\begin{figure}[ht!]
    \centering
    \includegraphics[width=.5\linewidth]{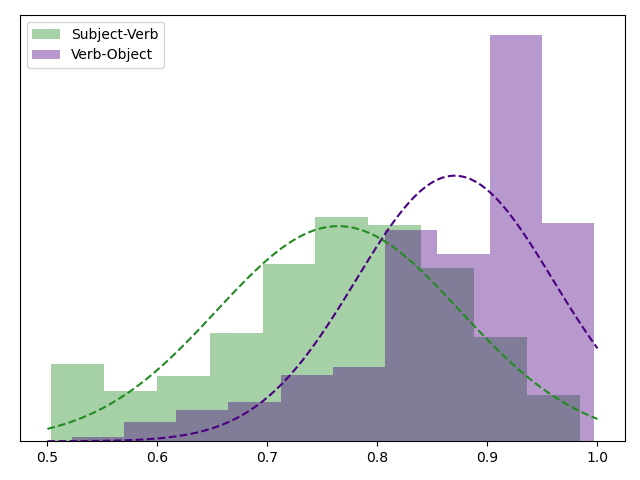}
    \caption{Comparison of the distributions of the $S\to V$ abd $O\to V$ causal fractions.\label{fig:histSVvVO}}
\end{figure}

\subsection{Models with different levels of ambiguity}\label{subsec:causalAmb}

\paragraph{}We now investigate whether we can observe a difference in behaviour between homonymous and polysemous words. We first recall that eye-tracking data suggests that readers tend to disambiguate polysemous words much later than their homonymous counterparts by selecting an underspecified interpretation instead of committing to a single sense. In SV models, this should result in having a higher causal fraction associated with $S\to V$ whenever the verbs are polysemous and the nouns are homonymous, as this case would delay the disambiguation of the verbs (compared to average), and make the disambiguation of the homonymous noun faster than the average scenario. Similarly, in VO models, we expect the causal fraction associated with $O\to V$ to be larger whenever the verbs are polysemous and the nouns are homonymous. To verify this hypothesis, we study the correlation between the causal fractions and the number of homonymous/polysemous verbs and nouns in the different models.

\paragraph{Results \& Discussion}We started by looking at the effect of the ratio of homonymous/polysemous words in the empirical models to the causal fractions. From the above description, this ratio shouldn't have a significant effect. To check this, we calculated the Spearman $\rho$ coefficients of the causal fractions and the number of homonymous words. We did not observe any apparent correlation in VO models ($\rho=0.009$, $p>87\%$). In SV models, we only observed a mild (but statistically significant) effect. In that case, we observed that the more polysemous the words in a model, the higher the $S\to V$ causal fraction ($\rho=0.15$, $p<0.7\%$). These results are depicted in Fig.~\ref{fig:boxplotGlobalAmb} and are overall consistent with our hypothesis.

\begin{figure}[ht!]
    \centering
    \begin{subfigure}{.45\linewidth}
        \centering
        \includegraphics[width=\linewidth]{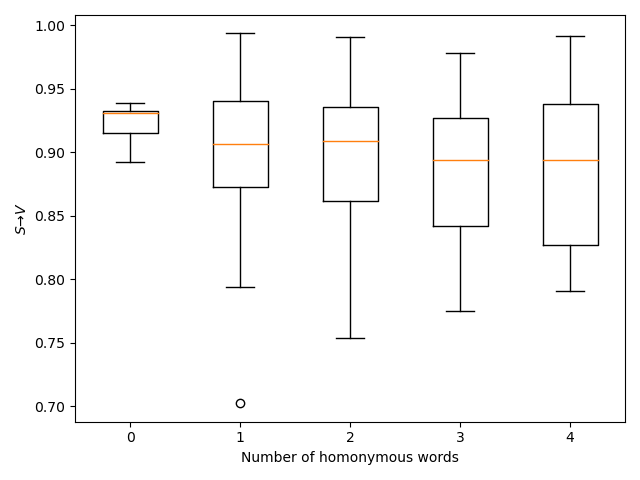}
        \caption{$\mathsf{CausF}_{S\to V}$}
    \end{subfigure}\qquad %
    \begin{subfigure}{.45\linewidth}
        \centering
        \includegraphics[width=\linewidth]{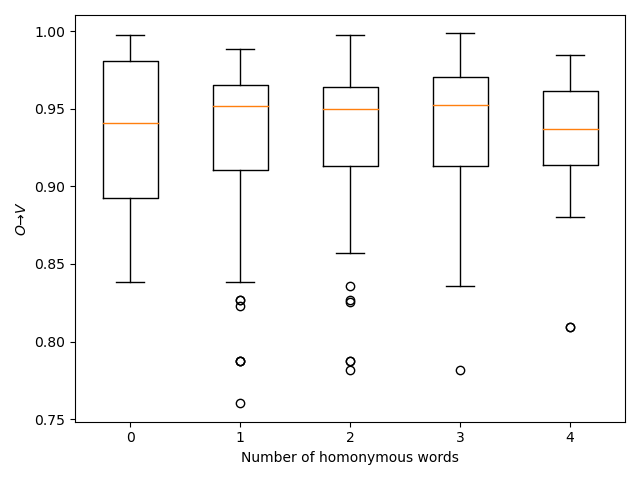}
        \caption{$\mathsf{CausF}_{O\to V}$}
    \end{subfigure}
    \caption{Boxplots of the distributions of causal fractions $\mathsf{CausF}_{S\to V}$ and $\mathsf{CausF}_{O\to V}$ as a function of the number of homonymous words of the empirical models.\label{fig:boxplotGlobalAmb}}
\end{figure}

We then sort our empirical models in terms of their number of homonymous and polysemous \emph{nouns}, and subsequently in terms of their number of homonymous and polysemous \emph{verbs}. The observed distributions of the causal fractions are depicted in Fig.~\ref{fig:localAmb}. In both SV and VO phrases, the $S\to V$ and $O\to V$ causal fractions were higher whenever the verb was polysemous. The Spearman coefficients and $p$-values were $\rho=0.17$, $p<0.2\%$ for SV phrases and $\rho=0.16$, $p<0.4\%$ for VO phrases. In addition, the causal fraction associated with $O\to V$ was higher whenever objects were homonymous  ($\rho=0.14$, $p<2\%$). These observations confirm our initial hypothesis that homonymous nouns are disambiguated faster than polysemous words.

The exception is the case of the ambiguity of the subjects in SV phrases ($\rho=0.04$, $p>50\%$). One way to interpret this difference would be to consider the relative position of the disambiguating context. It was shown in~\cite{FrazierRayner1990} that homonymous nouns were disambiguated much faster than polysemous nouns when the disambiguating context occurred before the target words. However, a significant slowdown has been observed if the disambiguating context is found after the target word. This slowdown was even exacerbated when the target word was homonymous. This nicely explains the difference between VO and SV phrases. Indeed, the only possible disambiguation context for nouns in VO phrases is the verb, which is positioned \emph{before} the object (once again assuming active voice). In the case of the disambiguating context for subjects, the only disambiguating context, which is once again the verb, is found after the subject. Therefore, the lack of correlation between the subjects' ambiguity and the causal fraction is likely caused by the balancing of the effect of the ambiguity and the added difficulty induced by a following disambiguating context. 

\begin{figure}[htp!]
    \centering
    \begin{subfigure}{.45\linewidth}
        \centering
        \includegraphics[width=\linewidth]{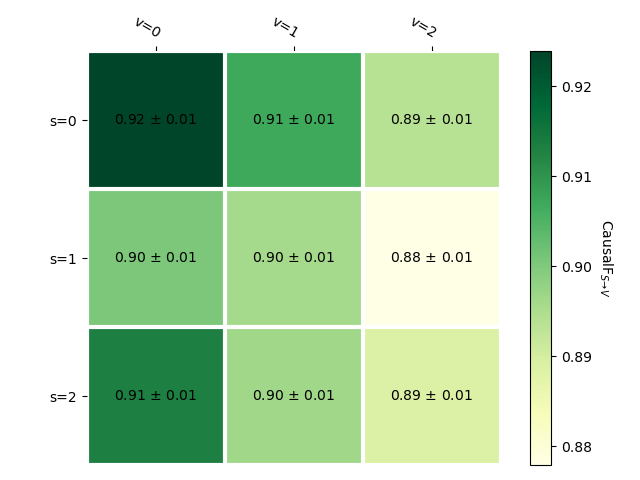}
        \caption{Averaged $\mathsf{CausF}_{S\to V}$ as the number of homonymous subjects and verbs are varied.}
    \end{subfigure}\qquad %
    \begin{subfigure}{.45\linewidth}
        \centering
        \includegraphics[width=\linewidth]{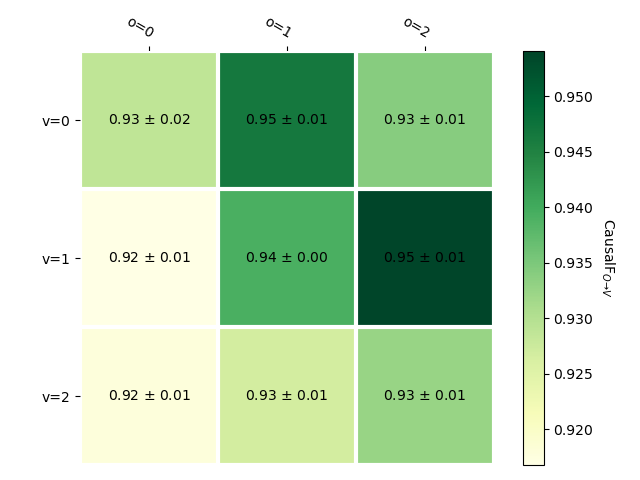}
        \caption{Averaged $\mathsf{CausF}_{O\to V}$ as the number of homonymous verbs and objects are varied.}
    \end{subfigure}

    \begin{subfigure}{.45\linewidth}
        \centering
        \includegraphics[width=\linewidth]{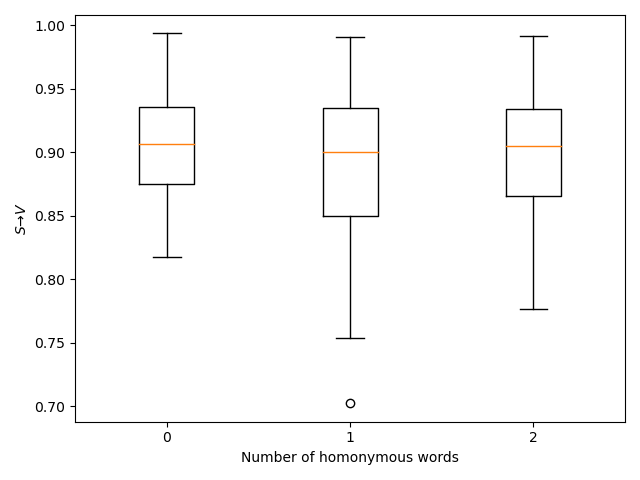}
        \caption{Distributions of $\mathsf{CausF}_{S\to V}$ as the number of homonymous subjects is varied.}
    \end{subfigure}\qquad%
    \begin{subfigure}{.45\linewidth}
        \centering
        \includegraphics[width=\linewidth]{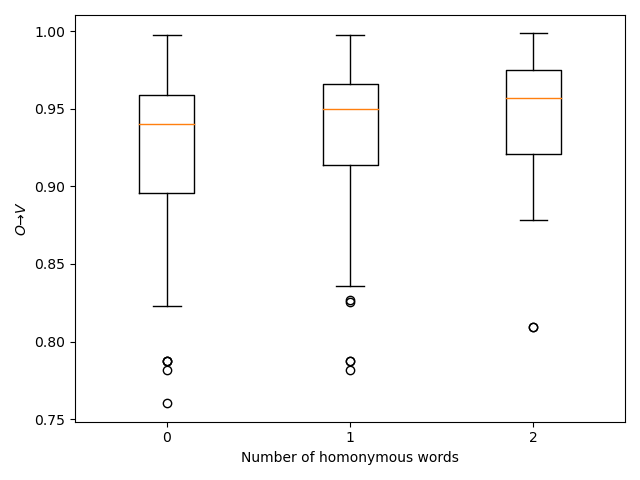}
        \caption{Distributions of $\mathsf{CausF}_{O\to V}$ as the number of homonymous object is varied.}
    \end{subfigure}

    \begin{subfigure}{.45\linewidth}
        \centering
        \includegraphics[width=\linewidth]{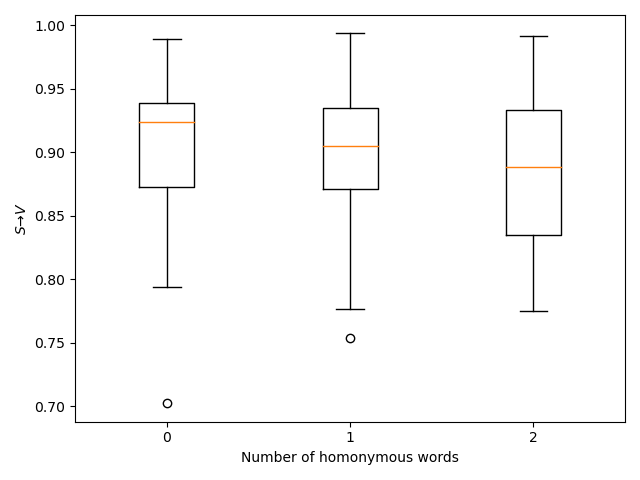}
        \caption{Distributions of $\mathsf{CausF}_{S\to V}$ as the number of homonymous verb is varied.}
    \end{subfigure}\qquad%
    \begin{subfigure}{.45\linewidth}
        \centering
        \includegraphics[width=\linewidth]{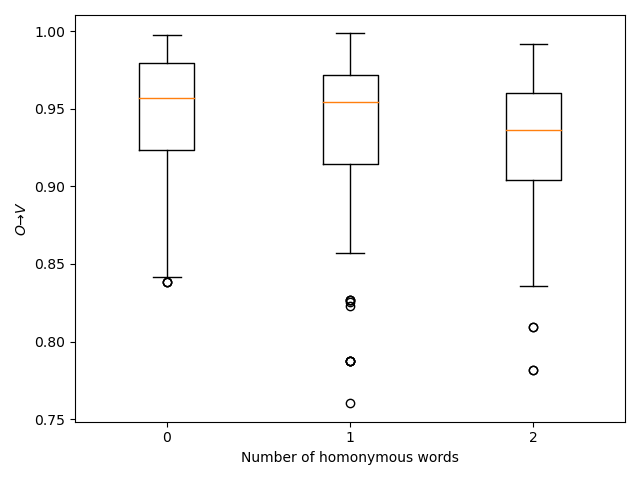}
        \caption{Distributions of $\mathsf{CausF}_{O\to V}$ as the number of homonymous verb is varied.}
    \end{subfigure}
    \caption{Causal fractions as the number of homonymous/polysemous nouns and verbs are varied\label{fig:localAmb}}
\end{figure}

The Spearman coefficients found above are relatively low ($\rho < 0.2 $), which suggests that the correlations observed are quite mild. However, the $p$-values showed that the correlations claimed in the above paragraph are statistically significant, i.e. it is highly unlikely that no correlation exists between the causal fraction and levels of ambiguity.

\section*{Summary of the Chapter}
\markboth{Summary of the chapter}{Aspects of the lexical disambiguation process}
\paragraph{}In this Chapter, we have investigated the properties of lexical ambiguity data, using the mathematics arising from the causality and contextuality quantum mechanics. We have observed:
    \begin{itemize}
        \item Contextuality-by-Default witnesses can be observed in cyclic systems of rank 2 (see Section~\ref{sec:lexicalContext} and~\cite{Wang2021a}). However, the operational interpretation of CbD is not very clear;
        \item Causal analysis of the data confirms that verbs are mostly disambiguated after their subject and object (see Section~\ref{sec:lexicalCausal})
    \end{itemize}
    
\chapter{Quantum simulations of the disambiguation process}\label{chap:lexicalCircuits}
\paragraph{}In the previous chapter, we found that about $90\%$ of the probability distributions obtained from human judgments are compatible with the $S\to V$ and $O\to V$ causal orders. Since this dataset is noisy and uses finite approximations of actual probabilities, we can be confident that these causal orders represent a good approximation of the process that occurs in humans. Here, we start from this observation and investigate whether quantum computers can simulate the disambiguation process of SV and VO phrases.

\paragraph{} In Section~\ref{sec:simulMethods} we describe how we obtained quantum circuits associated with the lexical disambiguation process of SV and VO phrases. In Section~\ref{sec:simulPred}, we investigate whether the obtained circuits can predict probability distribution instead of reproducing them. In Section~\ref{sec:simulEmb}, we describe a method for obtaining quantum word-embeddings of ambiguous words. This section also aims to examine whether we could use the obtained embeddings in NLP. Finally, in Section~\ref{sec:simulEntang} we investigate the entanglement generated by the obtained circuits.

\section{Methodology}\label{sec:simulMethods}

\paragraph{}We start from the causal orders obtained in the previous section, i.e. $S\to V$ and $O\to V$ for SV and VO phrases. Then, these causal orders naturally lead to a basic structure of the process which we need to approximate (see Section~\ref{subsec:Qcausal}), namely:
\begin{equation}\label{eq:Caus2Proc}
    S\to V \rightsquigarrow \input{Lexical/tikzit/StoV_generic.tikz} \qquad\qquad O\to V \rightsquigarrow \input{Lexical/tikzit/OtoV_generic.tikz} 
\end{equation}
These processes should be interpreted as before, i.e. the choice of inputs corresponds to a choice of word, and the outcomes will represent interpretations of these words. In addition, the parties are also defined as in quantum scenarios as ``labs'' which are allowed to do some local operations (e.g. $S$, $V$ or $O$) on the system, and causality, for example, $S\to V$, is achieved by having a subsystem of the party $S$ being used by the party $V$. 

\begin{rmk}
    The diagrams of \eqref{eq:Caus2Proc} are agnostic to which process theory they live on. In particular, we choose quantum circuits in this work, but these would also be applicable in a classical or probabilistic setting as well. Experiments involving classical systems, an hence investigation any potential quantum advantage, is left to future work.
\end{rmk}

Using this basic structure, we propose a parametric quantum circuit where we can train parameters to approximate the probability distributions obtained from human data. The details of the ansatz of the parametric circuits will be described shortly. We then optimise the parameters of our ansatz using a fairly standard hybrid quantum-classical method~\cite{QNLPinpractice,VQCClassification,VQCcircuitlearning,VQCQclassifier,VQCgenerativemodelling1}.

\begin{figure}[htbp!]
    \centering
    \input{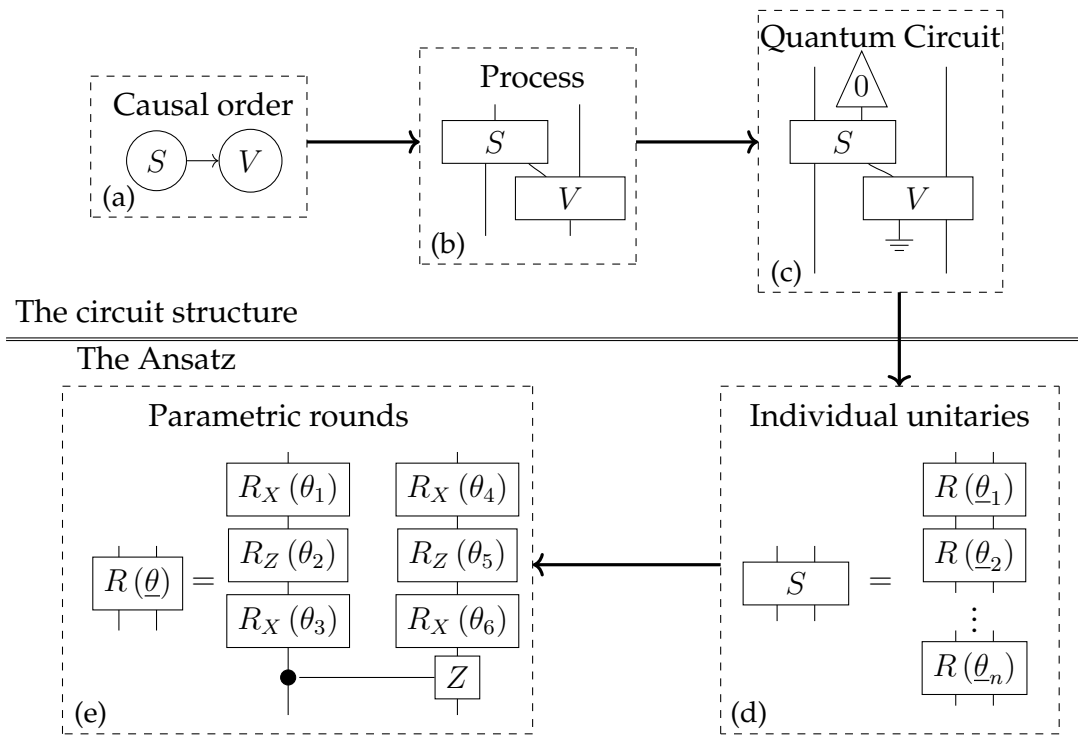}
    \caption{Summary of the approach\label{fig:simulationSummary}}
\end{figure}

\subsection{The ansatz}

\paragraph{} We start by describing the choice of ansatz. For simplicity, we take the input and output systems to be qubits. This will be enough as we only require two choices of inputs, which will be taken to be $\ket{0}$ and $\ket{1}$, and two choices of outcomes, which we will once again choose to be $\ket{0}$ and $\ket{1}$. 

In addition, we require a subsystem of the output of the noun (subject or object) operations to be fed into the input of the verb operation. We will also take this subsystem to live on $\mathbb{C}^2$. Then, to satisfy the unitary condition on the different operations, we will also require that the input system of the noun circuit is a 2-qubit system, where we choose to initialise the ancilla qubit as $\ket{0}$. Similarly, there will be an extra qubit in the output of the verb-circuit, which we will discard. The form of these circuits is illustrated in Fig.~\ref{fig:simulationSummary}(c).

\paragraph{}Then, each of the individual operations will be encoded as a parametric quantum circuit itself. These circuits will be divided into rounds of single-qubit unitaries followed by entangling gates; see Fig.~\ref{fig:simulationSummary}(d-e). Increasing the number of rounds (and therefore the number of parameters) is expected to increase the accuracy of the circuit but is also expected to take longer to be trained. We will also choose to have the same number of rounds for both parties.

\paragraph{}We then choose to define each of the rounds to be as in Fig.~\ref{fig:simulationSummary}(e). Each qubit is subject to a (parametrised) $X$-rotation, a (parametrised) $Z$-rotation, and then another (parametrised) $X$-rotation, where $X$- and $Z$-rotations are defined as follows:
\begin{equation*}
    R_X(\theta) = \begin{pmatrix}
        \cos\left(\frac{\theta}{2}\right) & -i\sin\left(\frac{\theta}{2}\right)\\ i\sin\left(\frac{\theta}{2}\right) & \cos\left(\frac{\theta}{2}\right)
    \end{pmatrix} \qquad R_Z(\theta) = \begin{pmatrix}
        1 & 0\\ 0 & e^{i\theta}
    \end{pmatrix}
\end{equation*}
This general form allows us to encode any single-qubit unitary, as this corresponds to Euler's decomposition of a generic unitary. We then apply a controlled-$Z$ gate between the two qubits to generate entanglement, defined as\footnote{Also see Example~\ref{ex:stdGates} for the defintion of general controlled gates.}:
\begin{equation*}
    \input{Lexical/tikzit/CZ.tikz} = \begin{pmatrix}
        1 & 0 & 0 & 0 \\ 0 & 1 & 0 & 0 \\ 0 & 0 & 1 & 0 \\ 0 & 0 & 0 & -1
    \end{pmatrix}
\end{equation*}

Each round then needs $2\times 3=6$ parameters to be trained. Hence, for $n$ rounds, each gate ($S, V$ or $O$) needs $6n$ parameters, giving a total of $12n$ parameters to be trained for the full circuit.
\begin{rmk}
    In most work using variational quantum circuits, the ansatz is less generic and minimises the number of parameters for each round (for example, the IQP ansatz only has 1 parameter per round). However, we decided on this ansatz as it allowed us to access a wide range of \emph{probability distributions}, and because it did converge with respect to our choice of cost function (more details below).
\end{rmk}

\subsection{The training process}

\paragraph{}To train the parameters, we apply a classical gradient descent algorithm. At each iteration of the algorithm, we update parameters $\theta$ as follows:
\begin{equation}
    \theta_n \to \theta_{n+1} = \theta_n - \gamma \nabla L\left(\theta_n\right)  
\end{equation}
for a given cost function $L$, and descent parameter $\gamma$, which we take to be fixed at $10^{-2}$. In addition, as the expression of $L$ may not be known, we employ the following finite approximation: 
\begin{equation}
    \left(\nabla L\left(\theta\right)\right)_i = \frac{L\left(\theta + \delta\mathbf{e}_i\right) - L\left(\theta - \delta\mathbf{e}_i\right)}{2\delta}
\end{equation}
where $\delta$ is chosen to be $10^{-2}$. We obtain the cost expression by simulating the quantum circuits using the Qiskit Aer platform~\cite{Qiskit}.

In this task, we are interested in reproducing the probability distributions obtained from human judgments. We, therefore, adopt as the cost function a distance between the obtained probability distribution (estimated from the counts obtained from Qiskit) and the probability distribution obtained from the human judgment dataset. In this work, we adopted the \emph{total variation} between the human and simulated probability distributions as our cost function. It is defined as:
\begin{equation}\label{eq:costFct}
    L\left(\theta\right) = \frac{1}{2} \max_{C} \sum_o \left|e\left(\theta\right)_C(o) - e_C(o)\right|
\end{equation}
\begin{rmk}
    Another choice of cost function would be the Kullback–Leibler (KL) divergence, denoted $D_{KL}(\mu||\nu)$, which measures the excepted surprisal induced from using a probability distribution $\nu$ to approximate another distribution $\mu$. It is formally defined as:
    \begin{equation}
        D_{KL}(\mu||\nu) = \sum_x \mu(x) \log\left(\frac{\mu(x)}{\nu(x)}\right)
    \end{equation}
    However, this measure is a directional measure, i.e. $D_{KL}(\mu||\nu)\neq D_{KL}(\nu||\mu)$, and therefore not a metric. More importantly, it is not defined whenever there is an outcome $x$ such that $\nu(x) = 0$ but $\mu(x)\neq 0$. For these reasons, we chose to use the total variation instead. We will leave the investigation of the circuits obtained using the KL-divergence to future work.
\end{rmk}

\subsection{Convergence}\label{subsec:simulConverg}

\paragraph{}We now look at the performance of the described ansatz subject to the proposed training process. We first note that there are some limits to how close the probability distributions obtained from the described circuits can be to the human ones. We will first specify these constraints before introducing the obtained results.

\subsubsection{Limitations of the approach}
\paragraph{}We recall from Section~\ref{subsec:causalDirection} that, although very high, the causal fractions associated with $S\to V$ and $O\to V$ causal order were not exactly $1$. As argued before, this is not necessarily because indefiniteness is necessary, in particular when the causal fraction approaches 1, but can also be due to the finiteness of the probability distribution. Hence, the probability distributions obtained by the variational circuits \emph{cannot} exactly match the probability distributions from the human judgment dataset. In other words, the minimal cost possible will be strictly greater than $0$. On the other hand, we can fix a bound on the achievable cost from the causal fraction of a given model.

\begin{prop}\label{prop:minCost}
    Given an empirical model $e$ with parties $A$ and $B$, with associated causal fraction $\mathsf{CausF}$ (with respect to causal order $A\to B$), for any empirical model $e_{\mathsf{Caus}}$ compatible with the causal order $A\to B$, we have:
    \begin{equation}\label{eq:minCost}
        \frac{1 - \mathsf{CausF}}{\mathsf{CausF}} m(e)  \leq \frac{1}{2}\max_C \sum_o \left|e_C(o) - e_{\mathsf{Caus},C} (o)\right|
    \end{equation}
    where:
    \begin{equation}
        m(e) = \min \left\{\left.e_{(a,b_1)}\right|_{A}(1), \left.e_{(a,b_2)}\right|_{A}(0)\right\}
    \end{equation}
    and $a\in \left\{a_1, a_2\right\}$ such that:
    \begin{equation*}
        \mathsf{CausF} = 1 - \left|\left.e_{(a, b_1)}\right|_{A}\left(0\right) - \left.e_{(a, b_2)}\right|_{A}\left(0\right)\right|
    \end{equation*}
    The quantity $\frac{1 - \mathsf{CausF}}{\mathsf{CausF}} m(e)$ will be refered to as the \emph{mininal cost} of an empirical model $e$.
\end{prop}

The proof of this proposition can be found in Appendix~\ref{app:minCost}.

\subsubsection{Outcome of the training process}
\paragraph{}We trained variational circuits using the above ansatz for numbers of rounds varying from 1 to 5 and set the initial set of parameters randomly. We found that all of the models converged with respect to the cost function \eqref{eq:costFct} (see Fig.~\ref{subfig:convSingModel}). In addition, the converged cost appears to get closer and closer to the minimal possible cost, dictated by \eqref{eq:minCost}, as the number of rounds increases (see Fig.~\ref{subfig:convMinCost}). This shows that the accuracy of the quantum circuits does indeed increase as the number of parameters increases. 

For the rest of this chapter, we assume that models with more parameters are more accurate and, therefore, more representative of the process under investigation.

\begin{figure}[htb!]
    \centering
    \begin{subfigure}{.45\linewidth}
        \centering
        \includegraphics[width=\linewidth]{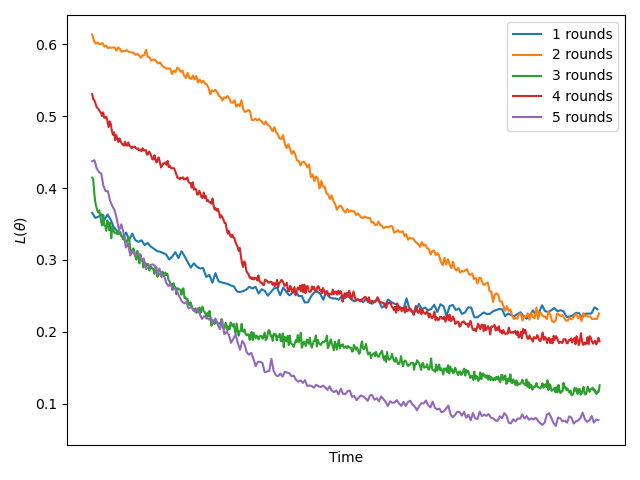}
        \caption{Cost function as a function of number of steps for the SV model with subjects taken from $\left\{press, volume\right\}$ and verbs taken from $\left\{conduct, file\right\}$, and number of rounds ranging from 1 to 5.\label{subfig:convSingModel}}
    \end{subfigure}\qquad%
    \begin{subfigure}{.45\linewidth}
        \centering
        \includegraphics[width=\linewidth]{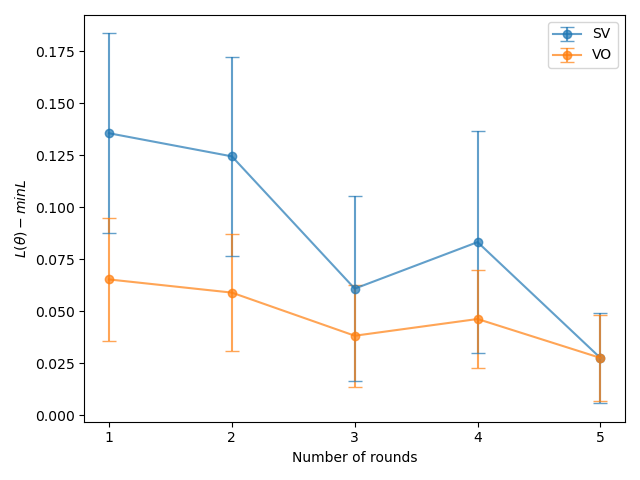}
        \caption{Average optimised cost (offset by their respective minimal costs).\label{subfig:convMinCost}}
    \end{subfigure}
    \caption{Convergence of the variational circuits}
\end{figure}

\paragraph{On the variability of the optimised parameters}We note that, even though the variational circuits converge to similar costs for different choices of initial parameters, the values of the optimised parameters are quite variable for different choices of initial parameters (see Fig.~\ref{fig:varParams}). The average distance between the parameters for different randomly sampled initial parameters was $1.57 \pm 0.91$, which is precisely the expected distance between randomly chosen parameters (which can be calculated to be $\frac{\pi}{2} \pm \frac{\pi}{2\sqrt{3}}$). 

This is to some extent expected as, although we know that there exists a minimal possible cost we can achieve, namely \eqref{eq:minCost}, there is an infinite number of (causal) empirical models $e_{\mathsf{Caus}}$ which achieves this minimal cost. This is not necessarily a problem if one is only interested in obtaining (independent) models of the disambiguations of the phrases in a given empirical model. However, in the following discussions, we will be interested in using these circuits for predictions of probability distributions, as well as investigating the use of the obtained quantum states as (quantum) word embeddings. 

Hence, from now on, we will fix the choice of initial parameters (for each number of rounds) to reduce the final parameters' variability.

\begin{figure}[htb!]
    \begin{subfigure}[c]{.45\linewidth}
        \centering
        \includegraphics[width=\linewidth]{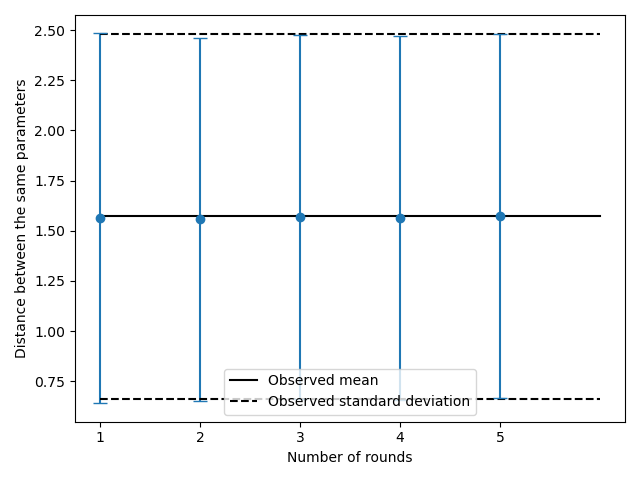}
        \caption{Distances between optimised parameters for different choices of initial parameters.}
    \end{subfigure}\quad%
    \begin{subfigure}[c]{.45\linewidth}
        \centering
        \includegraphics[width=\linewidth]{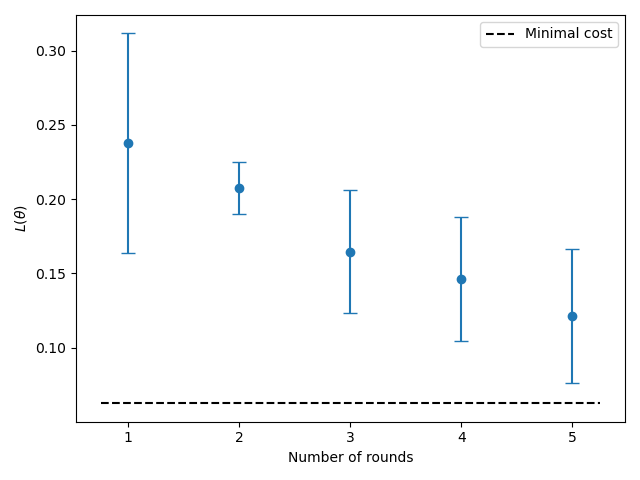}
        \caption{Optimised cost functions for different choices of initial parameters.}
    \end{subfigure}
    \caption{Comparison of the accuracies and obtained parameters for the SV model $\left\{press, volume\right\}\times \left\{conduct, file\right\}$ where the training was done using different initial parameters.\label{fig:varParams}}
\end{figure}

\section{The prediction power of the variational circuits}\label{sec:simulPred}

\paragraph{}We now investigate whether the obtained circuits can predict the activation pattern of the meanings of \emph{unseen phrases}. We then suggest the task of predicting the probability distribution of the different activation patterns for phrases that have not been explicitly trained. 

\subsection{Methods}
\paragraph{}Here, we propose to obtain predictions by splitting individual operations (i.e. subject, verb, and object) from trained circuits and combining them with operations from different models (see Fig.~\ref{fig:CircuitPredProc}). For example, given two optimised SV circuits corresponding to the empirical models with measurements $\mathcal{M}_1 = \left\{paper, plant\right\}\times \left\{bore, tap\right\}$ and $\mathcal{M}_2 = \left\{press, volume\right\}\times \left\{conduct, file\right\}$, we can create a new circuit which would correspond to an empirical model with the subjects taken from $\mathcal{M}_1$, and the verbs taken from $\mathcal{M}_2$, i.e. $\mathcal{M}' = \left\{paper, plant\right\}\times \left\{conduct, file\right\}$.

\begin{figure}[htb!]
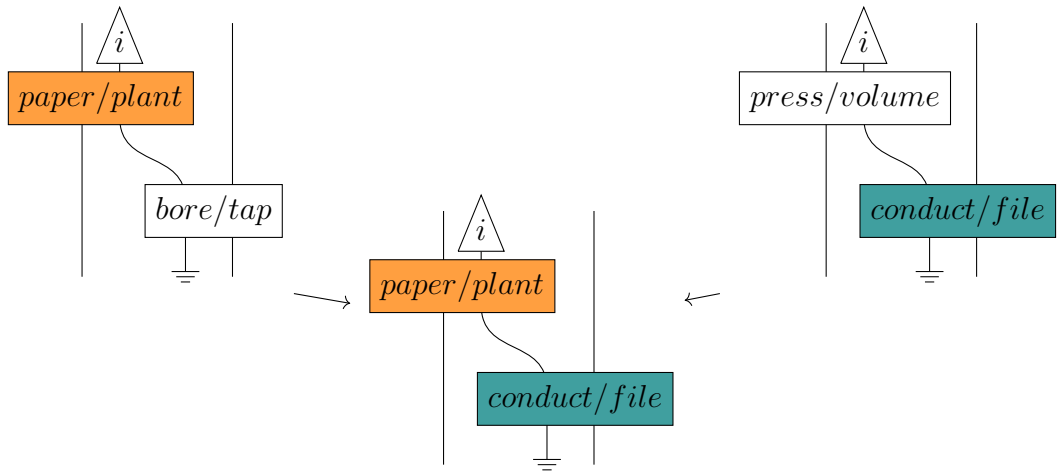

    \centering
    \begin{tikzpicture}
        \node(B) at (5,2.5){\input{Lexical/tikzit/press\_volume\_conduct\_file.tikz}};
        \node(A) at (-5,2.5){\input{Lexical/tikzit/paper\_plant\_bore\_tap.tikz}};
        \node(combined) at (0,0){\input{Lexical/tikzit/paper\_plant\_conduct\_file.tikz}};
        \draw[->] (A.south east) to (combined);
        \draw[->] (B.south west) to (combined);
    \end{tikzpicture}
    \caption{Procedure for obtaining new (untrained) circuits from trained ones.\label{fig:CircuitPredProc}}
\end{figure}

\paragraph{}We then constitute a set of 81 SV and 81 VO empirical models, which we will train using the procedure described in Section~\ref{sec:simulMethods}. These empirical models will constitute our training set. We then test the predictions obtained from recombining the subject, object, and verb circuits (as in Fig.~\ref{fig:CircuitPredProc}) to predict the probability distributions of 84 new SV and 84 VO empirical models for which we have the corresponding human data; these new empirical models will constitute our testing set. The empirical in the training and testing sets can be found in Appendix~\ref{app:predDataset}.

\begin{rmk}
    Regarding the consistitution of the traning and testing set, we were restraint by the set of randomly chosen phrases for which we collected the human plausibility judgments. Hence, not all of the possible empirical models arising from the process described above and on Fig.~\ref{fig:CircuitPredProc} could have been evaluated against a ground truth distribution (as we may not have collected it), which in turns restricts our choice of testing set. The training set on the other hand is dictated by which empirical models were needed to obtain predictions for all of the empirical models in the testing set.
\end{rmk}

\subsection{Results}

\paragraph{}We observe that the predicted circuits achieve a reasonably low cost (see Fig.~\ref{fig:CircuitPreds}). The average cost of the unseen models was $\left<L(\theta)\right> = 0.24$ (i.e. accuracy of $0.76$)  for the SV models and $\left<L(\theta)\right> = 0.14$ (i.e. accuracy of $0.86$) for the VO models. This procedure resulted in an average cost of $0.19$ (i.e. accuracy of $0.81$) for both types of models. These cost values are unsurprisingly still higher than the cost that can be achieved by training the models themselves, which were respectively $0.07\pm 0.04$ for SV and $0.06\pm 0.04$ for VO models.

\begin{figure}[htb!]
    \centering
    \includegraphics[width=.5\linewidth]{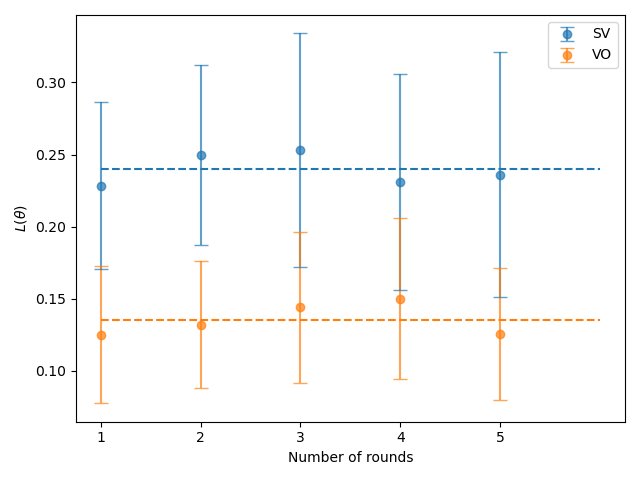}
    \caption{Cost values obtained for the predicted probability distributions.\label{fig:CircuitPreds}}
\end{figure} 

However, these predictions do not improve as the number of rounds increases. Indeed, the Pearson's $\rho$ coefficients between the number rounds and the prediction accuracies are of $\left|\rho\right|<0.01$ (with $p$-value $p=0.84$) for the SV models and $\left|\rho\right| = 0.06$ and $p=0.22$ for the VO models; such low values for $\rho$ and higher $p$-values shows that there likely \emph{no correlation} between the number of rounds and the accuracy of the predictions. On the other hand, the obtained accuracies are significantly better than the ones obtained by taking the uniform probability distributions empirical model (see Table~\ref{tab:uniform}), which corresponds to the probability distributions obtained by simply guessing outcomes without any knowledge of the system. The differences between the uniform distribution baseline and predicted accuracies were statistically different from $0$ with $p$-values $p<10^{-10}$ for each number of rounds. 

\begin{table}[htb!]
    \centering
    \begin{tabular}{|c|c|c|c|c|}
        \hline$(A,B)$ & (0, 0) & (0, 1) & (1, 0) & (1, 1)\\\hline
        $(a_1, b_1)$ & 1/4 & 1/4 & 1/4 & 1/4 \\\hline 
        $(a_1, b_2)$ & 1/4 & 1/4 & 1/4 & 1/4 \\\hline
        $(a_2, b_1)$ & 1/4 & 1/4 & 1/4 & 1/4 \\\hline 
        $(a_2, b_2)$ & 1/4 & 1/4 & 1/4 & 1/4 \\\hline
    \end{tabular}
    \caption{Empirical model containing only uniform probability distributions\label{tab:uniform}}
\end{table}

\section{Obtaining quantum word embeddings}\label{sec:simulEmb}
\paragraph{}Our next step is to see whether the circuits we have trained can be used to give a meaningful representation of words. If this is the case, this would give us a way of obtaining quantum word embedding, which could be used in NLP tasks such as word-sense disambiguation. Testing the performance of such word-state in NLP tasks is beyond the scope of this work and is left to future work. 

\subsection{Methods}
\paragraph{}Each word-state is trained for a fairly specific empirical model and is not compositional by default. Hence, we first want to check that the states which should correspond to the same word are indeed similar. If this were not the case, these word-states would not be useful anyway as a single word would have multiple representations. Here, we will assume that each word's dependency constitutes an inherent part of the word. For example, the noun \textit{pitcher} used as a subject will be considered distinct from the same word used as an object. Similarly, the verb \textit{tap} taking a subject will be considered distinct from the same verb taking an object instead. 

As we did for the prediction task, we first fixed the initial parameters for all of the variational circuits (otherwise, we would expect the different word-states to share little features). Then, using the optimised circuits, we can obtain a quantum state representation of a subject or object by fixing the input of the $S$ or $O$ individual circuits. For example, given the circuit representation of $\left\{press, volume\right\}$:
\begin{equation}\label{eq:exScircuit}
    \input{Lexical/tikzit/S_generic.tikz} = \input{Lexical/tikzit/press_volume.tikz}
\end{equation}
(obtained from, say, training for the empirical model corresponding to the SV model $\left\{press, volume \right\}\times\left\{conduct, file\right\}$),we can obtain a quantum state representation of the word \textit{volume} (as a subject) as:
\begin{equation}\label{eq:exVolumeScircuit}
    \begin{tikzpicture}
	\begin{pgfonlayer}{nodelayer}
		\node [style=elemt] (0) at (0, 1) {$\quad volume_S\quad$};
		\node [style=none] (1) at (0, -1) {};
	\end{pgfonlayer}
	\begin{pgfonlayer}{edgelayer}
		\draw (0) to (1.center);
	\end{pgfonlayer}
\end{tikzpicture}
 = \input{Lexical/tikzit/press_volume10_annotated.tikz}
\end{equation}
(recall that, in the case of SV circuits, the right-hand ancilla state is always set to $\ket{0}$ for SV circuits). Similarly, given the verb-object model associate with $\left\{conduct, file\right\}\times \left\{press, volume\right\}$, given an optimised object circuit:
\begin{equation}\label{eq:exOcircuit}
    \input{Lexical/tikzit/O_generic.tikz} = \input{Lexical/tikzit/press_volume.tikz}
\end{equation}
the quantum state representation of the word \textit{volume} (as an object) is:
\begin{equation}\label{eq:exVolumeOcircuit}
    \begin{tikzpicture}
	\begin{pgfonlayer}{nodelayer}
		\node [style=elemt] (0) at (0, 1) {$\quad volume_O\quad$};
		\node [style=none] (1) at (0, -1) {};
	\end{pgfonlayer}
	\begin{pgfonlayer}{edgelayer}
		\draw (0) to (1.center);
	\end{pgfonlayer}
\end{tikzpicture}
 = \input{Lexical/tikzit/press_volume01_annotated.tikz}
\end{equation}
(recall that, in the case of VO circuits, the right-hand ancilla state is always set to $\ket{0}$). Also note that the RHS of \eqref{eq:exScircuit} and \eqref{eq:exOcircuit} do not have to be related, so in general, \eqref{eq:exVolumeScircuit} will be different from \eqref{eq:exVolumeOcircuit}. 

For verbs, the process is a little more complicated as the representation of the verb not only depends on the choice of the verb but is also taking some information from the $S$ or $O$ circuit as well (see Fig.~\ref{fig:simulationSummary}). Hence, to obtain the verb representation, we first fix its input and then take the \emph{partial trace} over the subsystem dependent on the $S$ or $O$ output. This procedure will give us a density matrix instead of a pure quantum state. For example, given the verb circuit form the SV circuits obtained for the empirical model $\left\{press, volume\right\}\times \left\{conduct, file\right\}$:
\begin{equation}
    \input{Lexical/tikzit/V_generic.tikz} = \input{Lexical/tikzit/conduct_file.tikz}
\end{equation}
we can obtain the representation of the verb \textit{conduct} as:
\begin{equation}
    \begin{tikzpicture}
	\begin{pgfonlayer}{nodelayer}
		\node [style=elemt] (0) at (0, 0.75) {$\quad conduct\quad $};
		\node [style=none] (1) at (0, -1) {};
	\end{pgfonlayer}
	\begin{pgfonlayer}{edgelayer}
		\draw [style=thickline] (0) to (1.center);
	\end{pgfonlayer}
\end{tikzpicture}
 = \input{Lexical/tikzit/conduct_details.tikz} = \input{Lexical/tikzit/conduct_thick_annotated.tikz}
\end{equation}

\paragraph{}In order to quantify ``how similar'' two word-vectors are, we decide to calculate the inner products between them. Here, we will be interested in the inner-product between states representing the same word (and dependency) but have been trained to approximate different empirical models. For example, given two SV circuits for the empirical models associated with $\mathcal{M}_1 = \left\{press, volume\right\}\times \left\{conduct, file\right\}$ and $\mathcal{M}_2 = \left\{line, volume\right\}\times \left\{box, reflect\right\}$, we would like to compare the word-states associated with \textit{volume} in the two optimised circuits. 

In the case of nouns (i.e. subjects or objects), we recall that these are pure quantum states, so we calculate their inner product using the Born rule:
\begin{equation}
    \left|\langle\psi_{n}|\phi_{n}\rangle\right|^2 = \left|\begin{tikzpicture}
	\begin{pgfonlayer}{nodelayer}
		\node [style=elemt] (0) at (0, 1.25) {$\phi_n$};
		\node [style=coelemt] (1) at (0, -1.25) {$\psi_n$};
	\end{pgfonlayer}
	\begin{pgfonlayer}{edgelayer}
		\draw (0) to (1);
	\end{pgfonlayer}
\end{tikzpicture}
\right|^2
\end{equation}
where $\ket{\psi_n}$ and $\ket{\phi_n}$ are both representation of the noun $n$. For example, given the circuits optimised for the empirical models associated with $\mathcal{M}_1$ and $\mathcal{M}_2$ as defined above in the paragraph, we define the inner products of the two word-states associated with the word \textit{volume} as:
\begin{equation}
    \left|\input{Lexical/tikzit/volume_innerprod.tikz}\right|^2
\end{equation}

For verbs, which we recall are density matrices, we apply the generalised Born rule instead to calculate the inner product:
\begin{equation}
    Tr(\rho_v \varrho_v) = \begin{tikzpicture}
	\begin{pgfonlayer}{nodelayer}
		\node [style=elemt] (0) at (0, 1.25) {$\rho_v$};
		\node [style=coelemt] (1) at (0, -1.25) {$\varrho_v$};
	\end{pgfonlayer}
	\begin{pgfonlayer}{edgelayer}
		\draw [style=thickline] (0) to (1);
	\end{pgfonlayer}
\end{tikzpicture}

\end{equation}
where $\rho_v$ and $\varrho_v$ are both representation of a verb $v$. For example, given two SV circuits optimised to approximate the empirical models associated with :
\begin{align*}
    \mathcal{M}_1 =& \left\{press, volume\right\}\times \left\{conduct, file\right\}\\
    \mathcal{M}_3 =& \left\{letter, paper\right\}\times\left\{conduct,grasp\right\}
\end{align*}
we can define the inner-product between the two word-states of the verb \textit{conduct} as:
\begin{equation}
    \input{Lexical/tikzit/conduct_innerprod.tikz} = \input{Lexical/tikzit/conduct_innerprod_details.tikz}
\end{equation}

In both cases, an inner-product of $1$ will mean that the two states are equivalent, and an inner-product of $0$ will mean that the two states share no feature.

\subsection{Results}

\paragraph{}We found out that the pure states corresponding to nouns have a larger overlap between different models, where the average inner product was of $0.64$ for noun-states, compared to mixed verb-states where the average inner product was $0.37$. 

\paragraph{}In addition, we observed stark differences between these inner products depending on whether the word states were associated within the same or different input state, where we only consider the input state which identifies the word of interest (as opposed to the ancilla input, which is always the same). For example, the word \textit{press} in the model $(press, volume)\times (conduct, file)$ is associated with the input state $\ket{0}$, whereas the same word in the model $(line, press)\times (admit, wipe)$ will correspond to the input state $\ket{1}$.

In particular, we observed that the word states associated with words corresponding to the same input states (e.g. \textit{press} in $(press, volume)\times (conduct, file)$ and $(press, television)\times (box, label)$) have an inner product close to $1$. The average of these inner-product was $0.94$ for pure noun-states and $0.48$ for mixed verb-states.

By contrast, the word states corresponding to different input states (e.g. \textit{press} in $(press, volume)\times (conduct, file)$ and $(line, press)\times (admit, wipe)$) had a small overlap (on average the inner-product between them was $0.03$ for noun states and $0.02$ for verb states). 

This is expected as we have previously seen that the output of the training process is highly dependent on the choice of initial parameters. Hence, when words are associated with different input states, they will not necessarily correspond to the same optimised circuits. 

\paragraph{}Overall, this means that the quantum word embeddings that we obtain from these variational circuits have the potential to be useful in NLP tasks as long as one fixes the input state it corresponds to.


\section{Entanglement of phrases and words}\label{sec:simulEntang}
\paragraph{}We now investigate how much entanglement has been created using the optimised variational circuits. Entanglement is often considered the primary source of quantum correlation~\cite{Entanglement,entanglement1}. A quantum state is said to be \emph{entangled} (or \emph{non-separable}) iff it cannot be prepared using Local Operations and Classical Correlations (LOCC) alone~\cite{entanglementMixed}. If this amount of entanglement is high, in particular for the more accurate models, this suggests that training using quantum resources may be beneficial (as opposed to simply using classical probabilistic methods).

\subsection{Entanglement measures}
\paragraph{}Due to its importance in quantum information theory, entanglement needs to be quantified. Many measures of entanglement have been proposed for bipartite and multipartite states, for pure and mixed states. Here, we will focus on the bipartite measures and introduce an entanglement measure pure state and one for mixed states.

\paragraph{}For pure state, the standard measure of entanglement is the \emph{entanglement entropy}, defined as:
\begin{equation}\label{eq:Eentropy}
    E(\ket{\psi}\in \mathcal{H}_A\otimes \mathcal{H}_B) = -Tr \left(\rho_A \log_2\rho_A\right) = -Tr \left(\rho_B \log_2\rho_B\right) 
\end{equation} 
where $\rho_{A,B} = Tr_{A,B} \ket{\psi}\bra{\psi}$.

Many (non-equivalent) measures of entanglement have been proposed for mixed states. Among which is the \emph{entanglement of formation}, formally defined as:
\begin{equation}\label{eq:EFfull}
    E_F(\rho) = \inf \left\{\sum_k p_k E\left(\ket{\psi_k}\right)~\middle|~\rho = \sum_k p_k \ket{\psi_k}\bra{\psi_k}\right\} 
\end{equation}
Using the above definition alone, it is very hard to calculate the entanglement of formation for an arbitrary density matrix, as it involves finding all of the possible decompositions of the matrix $\rho$ in terms of density matrices of pure states. Fortunately, in the case of qubit systems, a closed formula has been found to calculate the entanglement of formation~\cite{EntanglementConcurrence}, namely:
\begin{equation}
    E_F(\rho) = s \left(\frac{1+\sqrt{1-C^2(\rho)}}{2}\right)
\end{equation}
where $C(\rho)$ is defined as:
\begin{equation}
    C(\rho) = \max \left\{0, \lambda_1 - \lambda_2 - \lambda_3 - \lambda_4\right\}
\end{equation}
where $(\lambda_1, \lambda_2, \lambda_3, \lambda_4)$ is the ordered set of eigenvalues of $\rho \sigma_Y\otimes \sigma_Y \rho^* \sigma_Y\otimes \sigma_Y$, and $s$ is defined as:
\begin{equation}
    s(x) = - \left(x \log_2 x\right) - \left((1-x) \log_2 (1-x)\right)
\end{equation}
\subsection{Entanglement of the optimised circuits}

\paragraph{}We start by looking at the amount of entanglement created by the whole parametric circuits. 

These parametric circuits are bipartite by design, and only make use of qubits. However, we discard a qubit system in each of the circuits. Hence, we will need to quantify entanglement for mixed states. 

\paragraph{Results \& Discussion}The amount of entanglement of formation for the trained parametric circuits is depicted in Fig.~\ref{fig:EFrounds}. As we can see, the amount of quantum correlations seems to increase as the circuit's accuracy increases, particularly for VO models. 

This would suggest that the process of disambiguation is ``truly parallel'' instead of having probabilistic mixtures of combinations of interpretations that can be selected, particularly in VO phrases. 

\begin{figure}[htb!]
    \centering
    \includegraphics[width=.5\linewidth]{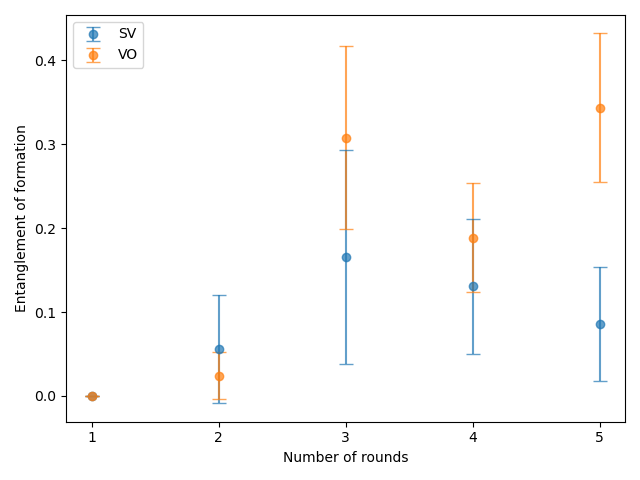}
    \caption{Entanglement of formation of the optimised parametric circuits, as the number of rounds increase.\label{fig:EFrounds}}
\end{figure}

\subsection{Entanglement of the noun-embeddings}

\paragraph{}We now want to study the amount of entanglement created for each word individually. However, since only one qubit in the output of verb-circuits is not discarded, any meaningful representation of verb vectors would be simply monopartite, and the notion of entanglement does not apply. We then primarily focus on the degrees of entanglement of nouns. In addition, since subject and object word-states are pure states by design, we can use entanglement entropy \eqref{eq:Eentropy} to measure the amount of quantum correlations created.

\paragraph{}The evolution of the entanglement entropy of noun states as the number of rounds in the ansatz increases is shown in Fig.~\ref{fig:Enouns}. The entanglement entropy of subjects is high for smaller numbers of rounds ($\left<E\right> = 0.95\pm 0.03$ for $n=1$ rounds) but decreases as the number of rounds increases ($\left<E\right> = 0.25\pm 0.18$ for $n=5$ rounds). The opposite happens for objects, that is, the amount of entanglement of noun states increases as the approximations get better and better (from $\left<E\right> = 0.55\pm 0.04$ for $n=1$ round to $\left<E\right> = 0.83\pm 0.10$ for $n=5$ rounds). 

This observation suggests that the correlations between the subject and the verb are not as strong as the correlations between the object and the verb, or that SV phrases are disambiguated more locally than VO phrases where more interaction between the two words would be needed. As before, this would require further experiments to confirm this trend.

\begin{figure}[htb!]
    \centering
    \includegraphics[width=.5\linewidth]{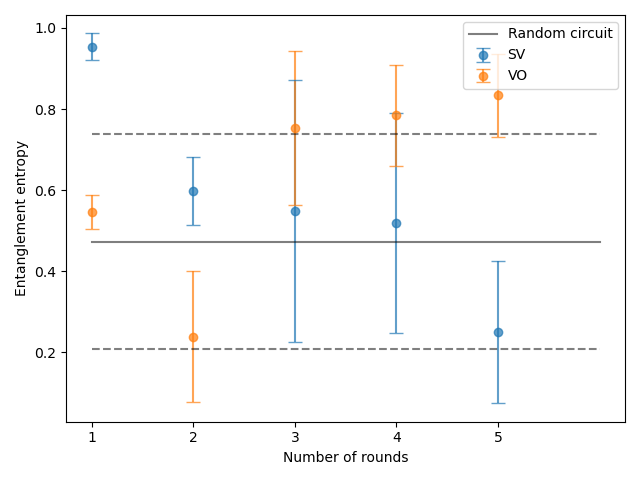}
    \caption{Evolution of the entanglement generated by the subject and object circuits, as the number of rounds increases.\label{fig:Enouns}}
\end{figure}

\paragraph{}In addition, the entanglement entropy of the nouns (subject or object) does not appear to depend on whether it is homonymous or polysemous. All of the $p$-values were $p>0.05$ for subjects and $p>0.19$ for objects. 

However, this is by design, as nouns are represented by pure states. Hence, the only way to obtain correlations between the nouns and the verb is through entanglement. It would be interesting to train density matrices in future work and check whether the ambiguity of the nouns affects the entanglement of its quantum state representation.
\section*{Summary of the Chapter}
\markboth{Summary of the chapter}{Quantum simulations of the disambiguation process}
    \begin{itemize}
        \item Using the observations of Chapter~\ref{chap:lexicalFeatures}, we proposed a quantum model of the disambiguation of subject-verb and verb-object phrases.
        \item This model was successfully implemented by variational circuits.
        \item The optimised quantum circuits were used to predict the meaning of unseen phrases.
        \item Preliminary evidence suggest that these circuits could be used as embeddings in NLP tasks.
        \item The optimised circuits generated a non-negligible amount of entanglement.
    \end{itemize}

\part{Syntactic ambiguity}\label{part:Syntactic}

\chapter{Modeling the human parsing process}\label{chap:SyntacticModel}
\paragraph{} In this chapter, we introduce our models of the human parsing process using sheaf theory. We adopt a model in which all possible parses are available at any given stage of a sentence with some probabilities (which we estimate empirically in Chapter~\ref{chap:GPPred}). This description is compatible with parallel-ranked processing, where the reader constructs parses in parallel and associates each parse with a weight, and probabilistic serial strategies, where the reader creates each parse with a given probability. 

This model leads to empirical models similar to the ones described in the previous part and the ones of quantum mechanics~\cite{AbramskyBrad,sheafcausality,sheafcausalityB,abramskyCausality}. 
Our empirical models only consider syntactic parses and no lexical or discourse information. Therefore, this is a preliminary model, and our results suggest that such models can indeed be used to represent human processes.

\paragraph{}In Section~\ref{sec:GPModel}, we present the structure of our models and how empirical statistics are collected. In Section~\ref{sec:GPfeatures}, we formally analyse the empirical models regarding quantum contextuality and causality. This section also includes intuitions about interpreting the signalling and causality in the models.

These models, as well as the predictions described in the next chapter can also be found in~\cite{WangGardenPath}.
\section{A sheaf-theoretic model of the syntax of sentences}\label{sec:GPModel}
\paragraph{}We start by describing a sheaf-theoretic model of human parsing based on the following key points:
\begin{enumerate}
    \item \textbf{Incrementality}. We want to study the evolution of the reader's mental representation as they encounter more information. The literature shows that human understanding is highly incremental. Therefore, we want to create a model that follows the linear order of the words in the sentence, i.e. how information is presented to the reader.
    \item \textbf{Grammatical structure}. This is our main object of interest. We want to capture what the reader thinks of the grammatical structure of a sentence or part of a sentence. In this work, only the grammatical structure is taken into account. In future work, we plan to include other factors such as plausibility, thematics, etc.
    \item \textbf{Statistics}. We follow the psycholinguistic hypothesis that the reader may keep in mind \emph{all} of the possible grammatical structures at each stage, but with different ratings (see Section~\ref{subsec:syntacticPsycho}). This implies that we need a way of ``rating'' different grammatical structures. In the following sections, we opt for probabilities over partial parses. This will give a parallel-ranked model of the parsing process.
\end{enumerate}
Sheaves are a promising way of combining these concepts in a single framework. The intuition is that the statistics are defined over the set of possible grammatical structures. In turn, the grammatical structures are only defined over a set of words that are presented to the reader, and evolve in a linear fashion. Let us describe this idea in more detail.

\subsection{Incrementality} \label{subsec:modelIncremental}
\paragraph{}As described above, we take our contexts to be words as appearing in a sentence or a phrase. The combination of these words forms sentence fragments by concatenation. Now, for a given sentence, e.g. \textit{The employees understood the contract}, we can define many different sentence fragments, for example, \textit{The employees}, \textit{The employees understood}, \textit{employees understood}, etc. We define a \emph{prefix order} over this set of sentence fragments as follows:
\begin{defs}
    A fragment $s_1$ of a sentence is included in another fragment $s_2$ iff $s_1$ is a prefix of $s_2$. We write:
    \begin{equation}
        s_1\leq_p s_2
    \end{equation}
    The set of fragments of a sentence $\texttt{S}$ equipped with the prefix order forms a (preorder) category $\mathcal{C}_\texttt{S}$.
\end{defs}
\begin{ex}
    In the sentence \textit{The employees understood the contract}, we have:
    \begin{equation}
        \text{\it The employees} \leq_p \text{\it The employees understood}
    \end{equation}
    However, the two fragments \textit{The employees} and \textit{employees understood} are not comparable.
\end{ex}

\begin{rmk}
    We could have chosen morphism to be the simple inclusion of subphrases, e.g. taking: $\text{\it employees}\subseteq \text{\it The employees}$. However, in the case of a purely incremental model of parsing, the prefix order appears to be the most relevant, as it models the order of information available to the reader.
\end{rmk}

In order to study the incremental evolution of the probability distributions, we consider a \emph{sequence} of empirical models. Each empirical model will consist of two \emph{consecutive stages}. For example, if we want to study the behaviour at the word level (e.g. the difficulty of reading a word), we would take stages to be words and subsequently consider sequences of empirical models containing a pair of contexts, where the contexts only differ by one word. In this case, the sentence \textit{The employees understood the contract would change} would lead to the sequence of empirical models with contexts:
\begin{align*}
    \mathcal{M}_1 =& \left\{\text{\it The}, \text{\it The employees}\right\}\\
    \mathcal{M}_2 =& \left\{\text{\it The employees}, \text{\it The~employees~understood}\right\}\\
    \mathcal{M}_3 =& \left\{\text{\it The~employees~understood}, \text{\it The~employees~understood~the}\right\}\\
    &\vdots\\
    \mathcal{M}_6 =& \left\{\text{\it The~employees~understood~the~contract~would},\right.\\
    &\left.\text{\it The~employees~understood~the~contract~would~change}\right\}\\
\end{align*}
Similarly, we could take stages to correspond to regions or phrases of the sentence, in which case we would consider sequences of empirical models in which the contexts differ by a region. An example of such a sequence of empirical models would consider the following contexts:
\begin{align*}
    \mathcal{M}_1 =& \left\{\text{\it The~employees}, \text{\it The~employees~understood~the~contract}\right\}\\
    \mathcal{M}_2 =& \left\{\text{\it The~employees~understood~the~contract},\right.\\
    &\left.\text{\it The~employees~understood~the~contract~would~change}\right\}
\end{align*}
We could also consider a more fine-grained analysis and take tokens or morphemes to be the incremental unit. In this chapter, we will study word-by-word parsing behaviour. 

\subsection{Grammatical structures}\label{subsec:modelGrammar}
\paragraph{}For each phrase or context, we want to be able to associate a grammatical structure. The grammatical structure will then be the \emph{outcomes} of our models. There are different ways to represent the grammatical structure of natural language input, including constituency trees~\cite{Chomsky}, dependency grammars~\cite{RobinsonDependency}, and categorial grammars~\cite{ajdukiewicz,BarHillel,Lambek,SteedmanCCG}.
In this work, we decide to work with dependency grammars. 

\subsubsection{Depedency grammar}
\paragraph{}In dependency grammar, each word of the sentence is associated with a \emph{head} and a \emph{dependency} or syntactic function. The main dependency structures are as follows:
\begin{itemize}
    \item The main verb of a sentence, or the main noun of a noun-phrase, is its own head and is associated with the dependency \texttt{ROOT};
    \item The head of the subject of a verb is the verb, and its dependency is \texttt{nsubj} (nominal subject);
    \item The head of the object of a transitive verb is once again the verb, and its dependency is \texttt{dobj} (direct object);
    \item The head of a determiner is its head noun and comes with dependency \texttt{det};
    \item The head of an adjective is the noun it is modifying, and its dependency is \texttt{amod} (adjective modifier);
    \item \ldots
\end{itemize}
\paragraph{}These dependency structures are usually represented as graphs, where we use labelled directed arrows $word\xrightarrow{depdency} head(word)$ to represent them (see Fig.~\ref{fig:fullDepParse} for example of such graphs).

\begin{figure}[htb!]
    \centering
    \input{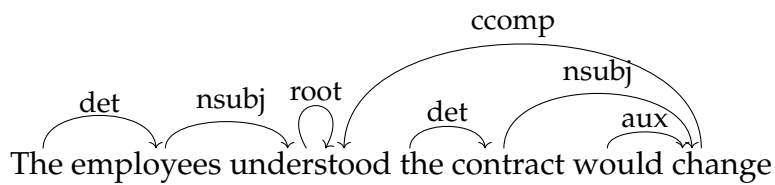}
    \caption{Depedency relations in the sentence \textit{The employees understood the contract would change}.\label{fig:fullDepParse}}
\end{figure}

\subsubsection{The presheaf of events}

\paragraph{Events}The main reason for adopting dependency grammar instead of a different paradigm is its minimality. In particular, it is easy to convert a dependency graph to a function $s: U\to \mathbb{N}\times D$, where $U$ corresponds to the \emph{ordered} set of words in a sentence, and $D$ is the set of dependency relations. 

To do so, we start by labelling the words of the sentence (or sentence fragment) by their position in the sentence (resp. sentence fragment) -- then each word has a label in $\mathbb{N}$. Then, each element of $U$ can be seen as an element of $\mathcal{V}\times \mathbb{N}$, where $\mathcal{V}$ is the vocabulary. We want a pair $(w,n)\in \mathcal{V}\times \mathbb{N}$ to represent a word $w$ being found in position $n$ in a sentence fragment. Hence, we would also require that for each element of $U$:
\begin{equation}
    (w,k), (w',k)\in U\implies w=w'
\end{equation}
This last condition states that we can only have one word at position $k$. We can then see that the objects of $\mathcal{C}_\texttt{S}$, can be expressed in this fashion as:
\begin{equation}\label{eq:UtoObCS}
    w_1\ldots w_n := \left\{\left(w_1,1\right), \ldots, \left(w_n, n\right)\right\}
\end{equation}
\begin{ex}
    Consider the sentence $\mathtt{S}=$\textit{The employees understood the contract would change}. Here, we have:
    \begin{align*}
        \text{\it The} =& \left\{(\text{\it The}, 1)\right\}\\
        \text{\it The employees} =& \left\{(\text{\it The}, 1), (\text{\it employees}, 2)\right\}\\
        \text{\it The employees understood} =& \left\{(\text{\it The}, 1), (\text{\it employees}, 2), (\text{\it understood}, 3)\right\}\\
        &\vdots\\
        \mathtt{S} =& \left\{(\text{\it The}, 1), (\text{\it employees}, 2), (\text{\it understood}, 3), (\text{\it the}, 4), \right.\\
        & \left.(\text{\it contract}, 5), (\text{\it would}, 5), (\text{\it change}, 6)\right\}\\
    \end{align*}
\end{ex}

We are now ready to define the dependency structure of a sentence fragment $U$ as a function $s:U\to \mathbb{N}\times D$, such that $(w, k) \mapsto (k', d)$ signifies that the word $w$ (at position $k$) has a head at position $k'$ with dependency $d$. So, for instance:
\begin{equation}
    \input{Syntactic/figs/1\_2\_2\_full\_dep.tikz}\equiv \begin{cases}
        (\text{\it The}, 1)&\mapsto (2, \mathtt{det})\\
        (\text{\it employees}, 2)&\mapsto (3,\mathtt{nsubj})\\
        (\text{\it understood}, 3)&\mapsto (3, \mathtt{ROOT})
    \end{cases}
\end{equation}

\paragraph{Presheaf structure}Using the correspondance \eqref{eq:UtoObCS} between the defined sets $U$ and objects of $\mathcal{C}_\mathtt{S}$, we can define the functor:
\begin{equation}
    \begin{matrix}
        \tilde{\mathcal{E}}: &\mathcal{C}_\mathtt{S}^{op}&\to &\mathbf{Sets}\\
        &U &\mapsto &\left\{\tilde{s}: U\to \mathbb{N}\times D\right\}\\
        &U\leq_p V &\mapsto &\tilde{s}_V\mapsto \tilde{s}_V|_U
    \end{matrix}
\end{equation}
where the restriction morphisms are defined $s_V\mapsto s_V|_U$ as:
\begin{equation}\label{eq:restrictParse}
    \tilde{s}_V|_U((w,k)\in U) = \tilde{s}_V((w,k)\in V)
\end{equation}
Note that this is well-defined since $U\leq_p V$ implies that $U$ and $V$ are of the form:
\begin{align*}
    U =& \left\{(w_1,1)\ldots (w_n,n)\right\}\\
    V =& \left\{(w_1,1)\ldots (w_{k}, k)\right\}
\end{align*}
where $n\leq k$. 
\begin{ex}
    Let us again take the sentence $\mathtt{S} = $ \textit{The employees understood the contract would change}. For $U=$ \textit{The employees} and $V=$ \textit{The employees understood}, we would have:
    \begin{equation}
        \left.\input{Syntactic/figs/1\_2\_2\_full\_dep.tikz}\right|_{\text{\it The employees}} = \input{Syntactic/figs/1\_2\_full\_dep.tikz}
    \end{equation}
\end{ex} 
\begin{rmk}
Since we are dealing with sentence fragments as well as sentences, the head of a word in a fragment may be \emph{undefined} or at least \emph{underspecified}. This feature is quite useful from a cognitive plausibility point of view, as it is hypothesised that humans tend to make predictions about the completion of sentences. Therefore, dependencies may not be known in advance.
\end{rmk}

\paragraph{} To simplify calculations, we restrict our data by only considering unlabelled attachments, i.e., to only consider the head of each word as a grammatical structure. One can see that this is enough to distinguish between the different syntactic structures of NP/S and NP/Z garden-path sentences (see Fig.~\ref{fig:twoParses}).
\begin{figure}[ht!]
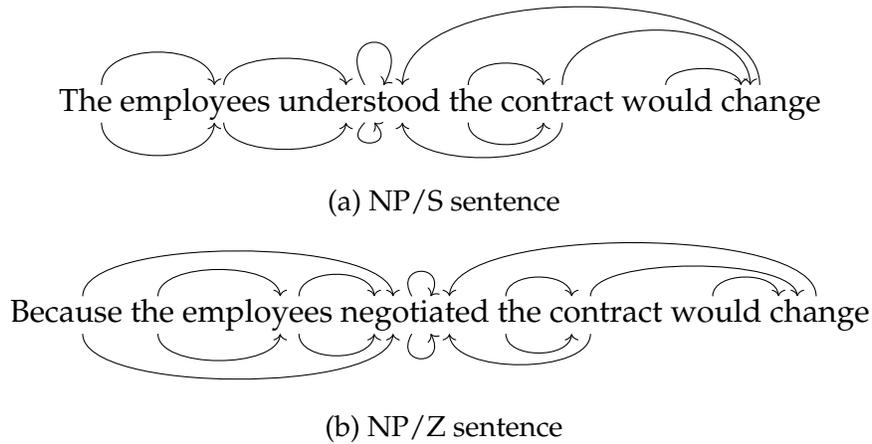

    \centering
    \begin{subfigure}{\linewidth}
        \centering
        \input{Syntactic/figs/NPS_twoparses.tikz}
        \caption{NP/S sentence}
    \end{subfigure}
    \begin{subfigure}{\linewidth}
        \centering
        \input{Syntactic/figs/NPZ_twoparses.tikz}
        \caption{NP/Z sentence}
    \end{subfigure}
    \caption{Examples of unlabelled dependency parses for NP/S and NP/Z sentences. The discarded parse is also shown (bottom/upside-down parse). \label{fig:twoParses}}
\end{figure}
Hence, we will take our \emph{presheaf of events} to be the functor:
\begin{equation}
    \begin{matrix}
        \mathcal{E}: &\mathcal{C}_\mathtt{S}^{op}&\to &\mathbf{Sets}\\
        &U &\mapsto &\left\{s: U\to \mathbb{N}\right\}\\
        &U\leq_p V &\mapsto &s_V\mapsto s_V|_U
    \end{matrix}
\end{equation}
where the parses $s:U\to \mathbb{N}\in \mathcal{E}(U)$ can be obtained from the parses $\tilde{s}: U\to \mathbb{N}\times D \in \tilde{\mathcal{E}}(U)$ as:
\begin{equation}
    s = \pi_1 \circ \tilde{s}
\end{equation}
\begin{ex}
    Taking a labelled parse to be:
    \begin{equation*}
        \tilde{s} = \input{Syntactic/figs/1\_2\_full\_dep.tikz}
    \end{equation*}
    the associated unlabelled parse is given as:
    \begin{equation*}
        s = \input{Syntactic/figs/1\_2.tikz}
    \end{equation*}
\end{ex}

\subsection{Probability distributions} \label{subsec:modelProb}
\paragraph{}We now associate probability scores with possible parses. As in previous chapters, this is done by selecting sections of the presheaf of events post-composed with the distribution monad $\mathcal{D}_{\mathbb{R}_+}: \mathbf{Sets}\to \mathbf{Sets}$. 

We recall that the functor $\mathcal{D}_{\mathbb{R}_+}$ associate with each set $A$, the set of probability distributions over $A$. Hence, for each $U\in ob\left(\mathcal{C}_\mathtt{S}\right)$, the set $\mathcal{D}_{\mathbb{R}_+}\mathcal{E}(U)$ corresponds to the set of all probability distributions over the set of possible (dependency) parses of $U$. Now, we are only interested in \emph{one} probability distribution in the set $\mathcal{D}_{\mathbb{R}_+}\mathcal{E}(U)$, namely, the probability distribution over parses corresponding the mental representation of the syntactic structure of the fragment $U$. 
\begin{ex}
    For $U=$ \textit{The employees understood}, we could single out the following probability distribution $e_{\text{\it The employees understood}}\in \mathcal{D}_{\mathbb{R}_+}\mathcal{E}(U)$:
\begin{equation}\label{eq:M5a}
    \begin{matrix}
        e_{\text{\it The employees understood}} \left(\scalebox{.8}{\input{Syntactic/figs/1\_2\_2.tikz}}\right)&=& 0.95\\
        e_{\text{\it The employees understood}} \left(\scalebox{.8}{\input{Syntactic/figs/1\_2\_7.tikz}}\right)&=& 0.02\\
        e_{\text{\it The employees understood}} \left(\text{other syntactic structures}\right)&<& 0.01\\
    \end{matrix}
\end{equation}
\end{ex}

Recall that we consider a sequence of empirical models such that in each empirical model, the measurement scenario consists of a pair of contexts differing by a single word. Consequently, an empirical model will consist of a pair of probability distributions. 
\begin{ex}
    An empirical model corresponding to $\mathcal{M}_3$ of the sentence \textit{The employees understoond the contract would change} could consist of the probability distribution of \eqref{eq:M5a} and :
\begin{equation}\label{eq:M5b}
    \begin{matrix}
        e_{\text{\it The employees understood the}} \left(\scalebox{.8}{\input{Syntactic/figs/1_2_2_4.tikz}}\right)&=& 0.37\\
        e_{\text{\it The employees understood the}} \left(\scalebox{.8}{\input{Syntactic/figs/1_2_2_5.tikz}}\right)&=& 0.35\\
        e_{\text{\it The employees understood the}} \left(\scalebox{.8}{\input{Syntactic/figs/1_2_2_6.tikz}}\right)&=& 0.26\\
        e_{\text{\it The employees understood the}} \left(\scalebox{.8}{\input{Syntactic/figs/1_2_6_4.tikz}}\right)&=& 0.01\\
        e_{\text{\it The employees understood the}} \left(\text{ other syntactic structures}\right)&<& 0.01\\
    \end{matrix}
\end{equation}
\end{ex}
\section{Contextuality, causality and signalling of the models}\label{sec:GPfeatures}
\paragraph{}We now investigate the properties of the created empirical models. In particular, we are going to focus on the properties of the signalling fraction $\mathsf{SF}$, which was first defined in~\cite{Emeriau2022} and Section~\ref{subsec:sheafContextuality}, and is going to be our main reading time predictor in Chapter~\ref{chap:GPPred}. First, we start with the impossibility of observing contextuality in this type of scenario.

\subsection{Contextuality}\label{subsec:GPcontext}
\paragraph{}Not all of the measurement scenarios are capable of hosting contextuality. This can be determined by looking at the structure of $\mathcal{M}$. In particular, it is known from Vorob'ev's theorem~\cite{Vorobev,SoaresBarbosa2014} that if the set $\mathcal{M} = \left\{M_1, \ldots, M_n\right\}$ satisfies\footnote{Topologically speaking, this means that $\mathcal{M}$ forms a simplex.}:
\begin{equation}
    M_1 \subseteq M_2 \subseteq \ldots M_n
\end{equation}
then, any compatible family $\left\{e_M~\middle|~M\in\mathcal{M}\right\}$ of probability distribution over $\mathcal{M}$ admits a global distribution over $\bigcup_{M\in \mathcal{M}} M$, which can in fact be shown to be the maximal distribution $e_{M_n}$ using the compatibility assumption. In other words, the empirical models described in Section~\ref{sec:GPModel} cannot exhibit contextuality.

\subsection{Causality and signalling}\label{subsec:GPCausSig}
\paragraph{}Here, we argue that the linguistic models correspond to both a contextuality and a causality scenario. To see this, we first note that for each empirical model, one context includes exactly one less word than the other. As a result, we can without loss of generality see empirical models as an $\{m, mw\}$ scenario. For example, in the empirical model $\mathcal{M}_2$, we have $m=$\textit{The employees} and $w=$ \textit{understood}. We can, therefore, express the compatibility relation of this model as: 
\begin{enumerate}
    \item A symmetric relation analogous to measurements that can be simultaneously measured. In this case, we interpret $m$ and $w$ as compatible, meaning they are somewhat parsed independently. In this interpretation, we have a situation similar to contextuality scenarios;
    \item An asymmetric relation analogous to causal scenarios. In this case, the compatibility relation $\preceq$ is read $m\preceq w$.
\end{enumerate}

\paragraph{}Both interpretations are possible and very much related. However, the meanings of the quantity $\mathsf{SF}$ are subtly different. 

In the symmetric interpretation, $\mathsf{SF}$ quantifies how consistent the two probability distributions are. On the other hand, in the causal interpretation, the signalling fraction is also a measure of the departure from a causal model following the linear \emph{reading} order. In other words, a high signalling fraction is evidence that parsing a particular subphrase is not incremental\footnote{In the linguistic sense, i.e. following the left-to-right reading order.} but instead should require information coming from words situated \emph{after} the phrase under consideration. 

\paragraph{}In either case, if we observe that an empirical model has a higher signalling fraction, this should signify that some reanalysis has to occur. According to psycholinguistic parsing theories, this should trigger a slowdown in reading time. 

Hence, we hypothesise that the signalling fraction $\mathsf{SF}$,  equivalently the non-causal fraction $\mathsf{NCausF}$, should correlate with human reading times. We investigated this in Section~\ref{sec:GPPred}. For uniformity purposes, we will focus on the models' signaling in the rest of this part.

\subsection{Computing $\mathsf{SF}$}
\paragraph{}Computing the signalling/causal fractions in a generic empirical model is not a trivial task, as it requires finding a solution to a linear optimisation problem~\cite{Emeriau2022} (see also the discussion in Section~\ref{sec:lexicalCausal}). However, given the specific structure of our empirical models, it is possible to find an expression of the signalling fraction $\mathsf{SF}$, which can be calculated efficiently. 
\begin{prop}\label{prop:GPSF}
    The signalling fraction can be computed via the following equation 
    \begin{equation}\label{eq:eNSA}
          \mathsf{SF} = 1 - \sum_{o} \min\left(\left.e_{mw}\right|_{m}(o), e_m(o)\right)
      \end{equation}
\end{prop}

The proof of this proposition can be found in Appendix~\ref{app:GPSF}.

\paragraph{}We argue that the signalling fraction measures parsing difficulty. This claim is motivated by the fact that $\mathsf{SF}$ can be seen as a measure of distance between probability distributions observed at different stages of the sentence. Therefore, the higher the signalling fraction, the more readers will have to readjust their mental representation of the grammatical structure. We can even say that since the contexts $m_{i}, m_{i+1}\in \mathcal{M}_i$ only differ by a single word, the signalling fraction of the empirical model $e_i$ becomes related to the difficulty of understanding the extra word. 

\begin{ex}
    For the empirical corresponding to:
\begin{equation*}
    \mathcal{M}_3 = \{\text{\it The employees understood}, \text{\it The employees understood the}\}
\end{equation*}
defined above in \eqref{eq:M5a} and \eqref{eq:M5b}, we obtain a signalling fraction of $\mathsf{SF}_3 = 0.05$, hence showing that the word \textit{the} at the end of the fragment \textit{The employees understood the} is not difficult to parse. On the other hand, if we calculate the signalling fraction for the empirical model:
\begin{align*}
    \mathcal{M}_5 = \{&\text{\it The employees understood the contract}, \\
    &\text{\it The employees understood the contract would}\}
\end{align*}
(see Fig.~\ref{fig:emCritical}), the signalling fraction can be found to be $\mathsf{SF}_5 = 0.79$, which reflects the fact the parsing the word \textit{would} is quite difficult.

\begin{figure}[hbt!]
    \begin{subfigure}{\linewidth}
        \centering
        \begin{tabular}{|m{.7\linewidth}|m{.2\linewidth}|}
            \hline Parse & Probability\\\hline
            \input{Syntactic/figs/1\_2\_2\_4\_2.tikz} & 0.44\\\hline
            \input{Syntactic/figs/1\_2\_2\_4\_5.tikz} & 0.14\\\hline
            \input{Syntactic/figs/1\_2\_2\_5\_5.tikz} & 0.12\\\hline
            \input{Syntactic/figs/1\_2\_2\_4\_6.tikz} & 0.10\\\hline
            \input{Syntactic/figs/1\_2\_2\_5\_6.tikz} & 0.04\\\hline
            Other parses & $<0.01$\\\hline
        \end{tabular}
        \caption{Probability distribution for the context \textit{The employees understood the contract}}
    \end{subfigure}
    \begin{subfigure}{\linewidth}
        \centering
        \begin{tabular}{|m{.7\linewidth}|m{.2\linewidth}|}
            \hline Parse & Probability\\\hline
            \input{Syntactic/figs/1\_2\_2\_4\_6\_6.tikz} & 0.96\\\hline
            \input{Syntactic/figs/1\_2\_2\_4\_5\_2.tikz} & 0.02\\\hline
            Other parses & $<0.01$\\\hline
        \end{tabular}
        \caption{Probability distribution for the context \textit{The employees understood the contract would}}
    \end{subfigure}
    \caption{Example of an empirical model corresponding to $\mathcal{M}_5$ in the sentence \textit{The employees understood the contract would change} (adapted from the empirical model obtained for \textit{The faithful employees understood the technical contract would be changed}, which can be found in~\cite{datasetGardenPath}).\label{fig:emCritical}}
\end{figure}
\end{ex}

\section*{Summary of the chapter}
\markboth{Summary of the chapter}{Modeling the human parsing process}
    \paragraph{}We consider a \textbf{sequence of empirical models} such that:
    \begin{itemize}
        \item Each empirical model contains a \textbf{pair of contexts} differing by a single word;
        \item For each context, we select a \textbf{probability distribution over possible parses} corresponding to the mental representation of the syntactic structure of a sentence fragment. The details of the computation of the probability distribution will be discussed in the next Chapter.
    \end{itemize}

    \paragraph{}Although we cannot observe contextuality in these empirical models, we can use the signalling fraction $\mathsf{SF}$ to quantify the difficulty of parsing incoming information.    
\chapter{Predicting garden-path effects}\label{chap:GPPred}
\paragraph{}In the previous chapter, we introduce our model of the human parsing process. This chapter aims to test the predictions arising from the model using empirical data.
\paragraph{}We start by describing the procedure for computing the reading time predictions in Section~\ref{sec:GPMethods}. In Section~\ref{sec:GPPred}, we describe the predictions from empirical models, and in Section~\ref{sec:GPComp}, we compare these results with the ones obtained from surprisal theory. 
\section{Methods}\label{sec:GPMethods}
\paragraph{}In this section, we describe the procedure employed to predict reading times. We start by describing the computational tools used. Subsequently, we explain how we used these tools to collect probabilities. Finally, we describe the datasets from which the garden-path sentences and reading times are taken.

\subsection{Tools}
\paragraph{}In this work, we approximate this probability distribution using the large language model \texttt{BERT} and the state-of-the-art dependency parser \texttt{spaCy}.

\paragraph{}\texttt{BERT}~\cite{BERT} is one of the first language models to adopt the \emph{transformer} architecture. The transformer was first introduced in~\cite{attention} in 2017 and offered an alternative to Recurrent Neural Networks (RNNs). This architecture improved both the trainability of neural networks and their performance and is still considered state-of-the-art. See Section~\ref{subsec:WSD} for a more detailed description of \texttt{BERT}.

\texttt{spaCy} is an open-source Python library developed by the company Explosion. It is widely used in NLP, in particular for linguistic annotations. Its functionalities include tokenization, lemmatization, part-of-speech tagging, sentence boundary detection, and dependency parsing. In this work, we mostly made use of the latter. The dependency parser has been evaluated independently in~\cite{choi2015depends} over the OntoNotes5 corpus containing 2.9M tokens. It was shown that the \texttt{spaCy} dependency parser predicted the head of a word (Unlabelled Attachment Score) with a 89.61\% accuracy and the head and label of a word (Labelled Attachment Score) with an accuracy of 87.92\%~\cite{choi2015depends}. 

\paragraph{}We also worked with different variations of \texttt{BERT} and \texttt{spaCy} to see how the accuracies of the predictions would vary. 

For \texttt{BERT}, we used the following flavours:
\begin{itemize}
    \item \texttt{distilBERT}: a light version of \texttt{BERT}. It only has $40\%$ of the parameters of the original \texttt{bert-base} model, but runs $60\%$ faster while preserving $95\%$ of its performance accuracies in language understanding tasks;
    \item \texttt{bert-base-cased}: the most commonly used version of \texttt{BERT}. It has 110 million parameters and was trained on the Toronto Book Corpus and the English Wikipedia, both of which distinguish between lower and upper case letters. Uncased versions of the same algorithm exist and were developed for purposes of cross-lingual learning;
    \item \texttt{bert-large-cased}: a larger version of \texttt{bert-base-case} which has 340 million parameters.
\end{itemize}
For \texttt{spaCy}, we also worked with different models, namely:
\begin{itemize}
    \item \texttt{en\_core\_web\_sm}: its standard version. It was trained using convolutional neural networks on web text consisting of blogs, news, and comments;
    \item \texttt{en\_core\_web\_lg}: its larger version. Its training procedure is similar to the \texttt{en\_core\_web\_sm} version, but also contains a word vector table with 500k unique 300-dimensional vectors;
    \item \texttt{en\_core\_web\_trf}: its newer version. It has no word vectors but is trained using state-of-the-art transformer-based neural networks.
\end{itemize}

\subsection{Method}
\paragraph{} To obtain a probability distribution over parses, we implement the following procedure:
\begin{enumerate}
    \item Given a fragment of a sentence $\mathtt{S}$, we turn it into a complete sentence by \emph{masking} all of the remaining words of the $\mathtt{S}$, see Fig.~\ref{tab:BERTcontexts} for an example.
    \begin{rmk}
        Since the task we are interested in is closely related to the task \texttt{BERT} was trained on, we did not need to fine-tune the pre-trained \texttt{BERT} models.
    \end{rmk}
    \item \texttt{BERT} then provides a list of predictions of the completion of the subphrases and a \emph{logit} score $\sigma$ for each of these predictions, which is meant to rate the likelihood of each prediction. The common practice in NLP is to use the logistic function to turn these scores into probabilities, namely:
    \begin{equation}
       p = \frac{e^\sigma}{1+e^\sigma}
    \end{equation}
    \begin{rmk}
        We could have equally used a softmax function to convert those scores into probabilities. We leave the investigation of results using softmax as future work.
    \end{rmk}
    Now, these predictions are not always words and may include punctuations. We only include the text predictions to avoid complications by dropping the punctuation and renormalising the probability distribution.
    \item We then use \texttt{spaCy} to parse each of the predictions provided by \texttt{BERT}. To obtain the syntactic structure of the specific subphrase we are working with, we restrict the full parse to the words included in that subphrase using the restriction maps from the presheaf $\mathcal{E}$ (see eq. \eqref{eq:restrictParse}).
    \item The probability of each such partial parse is obtained by summing up all the \texttt{BERT}-prediction probabilities that restrict to the same parse. For example, the predictions for the continuations of \textit{The employees understood}:
    \begin{align*}
        \input{Syntactic/figs/The\_employees\_understood\_that\_their\_salaries\_varied.tikz}\\
        \input{Syntactic/figs/The\_employees\_understood\_the\_risks\_in\_advance.tikz}\\
        \input{Syntactic/figs/The\_employees\_understood\_they\_also\_had\_freedom.tikz}
    \end{align*}
    will lead to the same partial parse when restricted to the context \textit{The employees understood}, namely:
    \begin{equation*}
        \input{Syntactic/figs/1\_2\_2.tikz}
    \end{equation*}
    Their probabilities will, therefore, be summed up in the corresponding empirical model.
\end{enumerate}

\begin{figure}[t]
    \centering
    \resizebox{\linewidth}{!}{
    \begin{tabular}{l|l}
        \textbf{Context} & \textbf{\texttt{BERT} input}\\\hline
        \textit{The} & \texttt{The [MASK] [MASK] [MASK] [MASK] [MASK] [MASK]}\\
        \textit{The employees} & \texttt{The employees [MASK] [MASK] [MASK] [MASK] [MASK]}\\
        \textit{The employees understood} & \texttt{The employees understood [MASK] [MASK] [MASK] [MASK]}\\
        \textit{The employees understood the} & \texttt{The employees understood the [MASK] [MASK] [MASK]}\\
        \textit{The employees understood the contract} & \texttt{The employees understood the contract [MASK] [MASK]}\\
        \textit{The employees understood the contract would} & \texttt{The employees understood the contract would [MASK]}\\
    \end{tabular}}
    \caption{\texttt{BERT} inputs for the sentence \textit{The employees understood the contract would change}. \label{tab:BERTcontexts}}
\end{figure}

\subsection{Description of the datasets}
\paragraph{}In this work, we make use of two reading time datasets that have been collected in psychology, namely the one of Sturt et al.~\cite{SturtPick} and the one of Grodner et al.~\cite{grodner}. In both of the studies presented in~\cite{SturtPick} and~\cite{grodner}, the authors investigated the slowdowns of garden-path sentences of types NP/S and NP/Z, and both collected (non-cumulative) \emph{self-paced reading} times. 

\paragraph{Self-paced reading} In self-paced reading experiments, the participants are presented with a sentence or text, where most words are hidden, apart from a text window (usually consisting of a word or several words). Then, upon interaction with a computer (e.g. pressing the space bar), the window is allowed to move from left to right, revealing more text to the participant. 
\begin{figure}[htb!]
    \centering
    \resizebox{\linewidth}{!}{
        \begin{tabular}{|l|c|}
        \hline& Display\\\hline
        Step 1 & \texttt{The faithful employees\_\_\_\_\_\_\_\_\_\_\_\_\_\_\_\_\_\_\_\_\_\_\_\_\_\_\_\_\_\_\_\_\_\_\_\_\_\_\_\_\_\_\_\_\_\_\_\_\_\_\_\_\_\_\_\_\_\_\_\_\_}\\\hline
        Step 2 & \texttt{\_\_\_\_\_\_\_\_\_\_\_\_\_\_\_\_\_\_\_\_\_\_\_understood the technical contract\_\_\_\_\_\_\_\_\_\_\_\_\_\_\_\_\_\_\_\_\_\_\_\_\_\_\_}\\\hline
        Step 3 & \texttt{\_\_\_\_\_\_\_\_\_\_\_\_\_\_\_\_\_\_\_\_\_\_\_\_\_\_\_\_\_\_\_\_\_\_\_\_\_\_\_\_\_\_\_\_\_\_\_\_\_\_\_\_\_\_\_\_\_would be changed\_\_\_\_\_\_\_\_\_\_}\\\hline
        Step 3 & \texttt{\_\_\_\_\_\_\_\_\_\_\_\_\_\_\_\_\_\_\_\_\_\_\_\_\_\_\_\_\_\_\_\_\_\_\_\_\_\_\_\_\_\_\_\_\_\_\_\_\_\_\_\_\_\_\_\_\_would be changed\_\_\_\_\_\_\_\_\_\_}\\\hline
        Step 4 & \texttt{\_\_\_\_\_\_\_\_\_\_\_\_\_\_\_\_\_\_\_\_\_\_\_\_\_\_\_\_\_\_\_\_\_\_\_\_\_\_\_\_\_\_\_\_\_\_\_\_\_\_\_\_\_\_\_\_\_\_\_\_\_\_\_\_\_\_\_\_\_\_\_\_\_\_very soon}\\\hline
        Step 4 & \texttt{\_\_\_\_\_\_\_\_\_\_\_\_\_\_\_\_\_\_\_\_\_\_\_\_\_\_\_\_\_\_\_\_\_\_\_\_\_\_\_\_\_\_\_\_\_\_\_\_\_\_\_\_\_\_\_\_\_\_\_\_\_\_\_\_\_\_\_\_\_\_\_\_\_\_very soon}\\\hline
    \end{tabular}
    }
    \caption{Evolution of the display presented to participants in a (region-by-region) self-paced reading experiment, with input sentence \textit{The faithful employees understood the technical contract would be changed very soon}}
\end{figure}
In a non-cumulative setting, the participant cannot access previous text once the window has moved. 

Self-paced reading experiments provide less information than eye-tracking experiments do. For example, backtracking is not an option for participants, and they are arguably less natural than the eye-tracking setting. However, they are much more interpretable since they produce fewer variables to keep track of. In addition, they are less expensive to set up since all that is required is a computer, and online crowdsourcing is also possible.

\subsubsection{The Sturt et al. dataset}
\paragraph{}The Sturt et al.~\cite{SturtPick} dataset consists of 32 pairs of sentences such as:
\begin{enumerate}[label=(1\alph*), series=sturt]
    \item \label{item:sturtNPS}The faithful employees understood the technical contract would be changed very soon.
    \item \label{item:sturtNPZ}Because the faithful employees negotiated the technical contract would be changed very soon.
\end{enumerate}
Each of these pairs contains an NP/S sentence, such as~\ref{item:sturtNPS}, and an NP/Z sentence, such as~\ref{item:sturtNPZ}, and both the sentences in a pair share overlap in vocabulary. 

In addition, each of these garden-path sentences is also associated with \emph{unambiguous version}, which is easier to parse. For NP/S sentences, this is done by adding the connective \textit{that} after the main verb. In NP/Z sentences, this is achieved by adding a comma after the main verb. For example, the following are the unambiguous versions of the sentences~\ref{item:sturtNPS} and~\ref{item:sturtNPZ} respectively:
\begin{enumerate}[label=(1\alph*), resume=sturt]
    \item The faithful employees understood \textbf{that} the technical contract would be changed very soon.
    \item Because the faithful employees negotiated\textbf{,} the technical contract would be changed very soon.
\end{enumerate}
\begin{rmk}
    In the following discussion, we sometimes use the term ``ambiguous sentences'' to describe the garden-path sentences. This is an \textit{abus de langage} since these sentences are not actually (globally) ambiguous. Similarly, the ``unambiguous sentences'' are not fully (locally) unambiguous. This terminology helps distinguish the sentences that cause difficulty parsing and those that don't. 
\end{rmk}

This gives a total of 128 sentences. 

\paragraph{}Each of these sentences is, in turn, divided into 4 regions. For instance, the sentences~\ref{item:sturtNPS} and~\ref{item:sturtNPZ} are respectively divided as:
\begin{enumerate}[label=(1\alph*), resume=sturt]
    \item The faithful employees / understood the technical contract / \underline{would be changed} / very soon.
    \item Because the faithful employees / negotiated the technical contract / \underline{would} \underline{ be changed} / very soon.
\end{enumerate}
and similarly for the unambiguous sentences. The critical regions are the ones underlined. 

The experiment described in~\cite{SturtPick} recorded the \emph{region-by-region} reading times, i.e. the participants were presented with one of these regions at a time and allowed to move one region forward at each event. The numbers reported in the study were average region reading times, averaged across sentences of the same type (i.e. NP/S ambiguous, NP/S unambiguous, NP/Z ambiguous, or NP/Z unambiguous), and across participants. These numbers can also be found in Table~\ref{tab:SturtRT}.

\begin{table}[htb!]
    \centering
    \begin{tabular}{|c|c|c|c|c|}
    \hline 
    & \multicolumn{4}{c|}{Regions}\\\cline{2-5}
    & 1 & 2 & 3 & 4\\\hline
    NP/S (ambiguous) & 990 & 1183 & 877 & 771\\ 
    NP/S (unambiguous) & 981 & 1282 & 790 & 768\\
    NP/Z (ambiguous) & 914 & 1269 & 1335 & 848\\
    NP/Z (unambiguous) & 998 & 1384 & 935 & 832\\\hline
    \end{tabular}
    \caption{Region-by-region self-paced reading times of garden-path sentences and their unambiguous variants (in ms). These numbers were taken from~\cite{SturtPick}\label{tab:SturtRT}}
\end{table}

\subsubsection{The Grodner et al. dataset}
\paragraph{}The Grodner et al. dataset~\cite{grodner} was created in a different way. As for the Sturt et al. dataset, it contains both NP/S and NP/Z sentences such as :
\begin{enumerate}[label=(2\alph*), series=grodner]
    \item\label{item:grodnerNPSunmod} The employees understood the contract would be changed to accommodate all parties.
    \item\label{item:grodnerNPZunmod} Even though the band left the party went on for at least another two hours.
\end{enumerate}

In addition, each of the NP/S and NP/Z sentences also come with a \emph{modified variant}, where a descriptive noun-phrase is added to the subject of the main verb. For example:
\begin{enumerate}[label=(2\alph*), resume=grodner]
    \item\label{item:grodnerNPSmod} The employees \textbf{who initiated the strike} understood the contract would be changed to accommodate all parties.
    \item\label{item:grodnerNPZmod} Even though the band \textbf{which played funk music} left the party went on for at least another two hours.
\end{enumerate}
\begin{rmk}
    The reason for having modified and unmodified versions of the same garden-path sentence was to identify whether human parsing strategies were more consistent with repair-based or reanalysis-based models (see Section~\ref{subsec:syntacticPsycho}). However, since we are not interested in this particular aspect of the parsing process, we ignored their distinction when analysing the garden-path effects. However, when creating a linear regression model, having a multiplicity of data points allows us to obtain a better model and reduces errors due to averaging.
\end{rmk}
The dataset taken from~\cite{grodner} contains 20 pairs of modified/unmodified NP/S garden-path sentences and 20 pairs of modified/unmodified NP/Z sentences. 

In addition, as for the Sturt et al. dataset, each sentence comes with an unambiguous version. For instance, for the unmodified versions~\ref{item:grodnerNPSunmod} and~\ref{item:grodnerNPZunmod}, these are:
\begin{enumerate}[label=(2\alph*), series=grodner]
    \item The employees understood \textbf{that} the contract would be changed to accommodate all parties.
    \item Even though the band left the party\textbf{,} went on for at least another two hours.
\end{enumerate}
Similarly, for the modified version~\ref{item:grodnerNPSmod} and~\ref{item:grodnerNPZmod}, the unambiguous variants are:
\begin{enumerate}[label=(2\alph*), series=grodner]
    \item The employees who initiated the strike understood \textbf{that} the contract would be changed to accommodate all parties.
    \item Even though the band which played funk music left the party\textbf{,} went on for at least another two hours.
\end{enumerate}
This gives us a total of 160 sentences. 

\paragraph{}Each of these sentences is also divided into regions, but contrary to the dataset of~\cite{SturtPick}, the number of regions is different for every type of sentence, and the length of regions is highly variable within a given sentence. These regions are depicted in Table~\ref{subtab:GrodnerRegions}. 

In the corresponding study, the authors collected \emph{word-by-word} self-paced reading times and reported the average word reading time for each region (aside from the last one for which numbers are omitted), averaged across sentences of the same type and across participants. The obtained averages are shown in Table~\ref{subtab:GrodnerRT}.

\begin{table}[htb!]
    \centering
    \begin{subtable}{\linewidth}
        \centering
        \resizebox{\linewidth}{!}{\begin{tabular}{|c|c|c|c|c|c|c|}
        \hline 
        & \multicolumn{6}{c|}{Regions}\\\cline{2-7}
        & 1 & 2 & 3 & 4 & 5 & 6\\\hline
        NP/S (unmod., amb.) & The employees & understood & the contract & \underline{would be changed} &&\\
        NP/S (unmod., unamb.) & The employees & understood & that & the contract & \underline{would be changed} &\\
        NP/S (mod., amb.) & The employees &who initiated the strike& understood & the contract & \underline{would be changed} &\\
        NP/S (mod., unamb.) & The employees &who initiated the strike& understood &that& the contract & \underline{would be changed}\\
        NP/Z (unmod., amb.) & Even though the band & left & the party & \underline{went on for}[...]&&\\
        NP/Z (unmod., unamb.) & Even though the band & left, & the party & \underline{went on for}[...]&&\\
        NP/Z (mod., amb.) & Even though the band &which played funk music& left & the party & \underline{went on for}[...]&\\
        NP/Z (mod., unamb.) & Even though the band &which played funk music& left, & the party & \underline{went on for}[...]&\\\hline
        \end{tabular}}
        \caption{Regions of the different sentence types in the Grodner et al. dataset. The critical regions are underlined. (Note that the last region is omitted)\label{subtab:GrodnerRegions}}
    \end{subtable}
    \begin{subtable}{\linewidth}
        \centering
        \begin{tabular}{|c|c|c|c|c|c|c|}
        \hline 
        & \multicolumn{6}{c|}{Regions}\\\cline{2-7}
        & 1 & 2 & 3 & 4 & 5 & 6\\\hline
        NP/S (unmod., amb.) & 397 & 467 & 412 & 424 &&\\
        NP/S (unmod., unamb.) & 393 & 460 & 431 & 396 & 410 &\\
        NP/S (mod., amb.) & 392 & 415 & 449 & 401 & 419 &\\
        NP/S (mod., unamb.) & 398 & 413 & 471 & 393 & 388 & 391\\
        NP/Z (unmod., amb.) & 452 & 402 & 382 & 452&&\\
        NP/Z (unmod., unamb.) & 400 & 452 & 402 & 383&&\\
        NP/Z (mod., amb.) & 433 & 407 & 464 & 415 & 432&\\
        NP/Z (mod., unamb.) & 405 & 400 & 494 & 448 & 395&\\\hline
        \end{tabular}
        \caption{Average word-by-word self-paced reading times of garden-path sentences and their unambiguous variants (in ms). (numbers taken from~\cite{grodner})\label{subtab:GrodnerRT}}
    \end{subtable}
    \caption{Description of the Grodner et al. dataset}
\end{table}

\begin{rmk}
     In addition, since the numbers quoted in the Sturt et al. dataset are region-by-region reading times, we decided to make region-by-region predictions over this dataset. Similarly, since the Grodner et al. dataset used word-by-word reading times, we decided to make word-by-word predictions over this dataset. This is at odds with the study of~\cite{vanSchijndelA} where, for uniformity purposes, the region-by-region reading times were averaged to produce word-by-word reading times. However, doing so would increase the amount of systematic error. For instance, assuming that our signalling fraction $\mathsf{SF}$ is indeed related to reading times and assuming that the buffering time associated with the change of stimulus of the screen is constant, this extra time (unrelated to the reading difficulty) is only added once per region in the region-by-region setting, but multiple times per region in the word-by-word setting. Hence, this buffering time is constant in the former but dependent on the length region in the latter. Due to these differences in the reading time collection process in the two datasets, we decided to present the predictions obtained for the Sturt et al. dataset and the ones for the Grodner et al. dataset separately.
\end{rmk}
\section{Analysis of the predictions}\label{sec:GPPred}
\paragraph{} In this section, we investigate the prediction power of the signalling fraction $\mathsf{SF}$. We start by looking at the empirical correlation between $\mathsf{SF}$ and reading times. Then, we create linear regression models of reading times from the signalling fraction. Using this model to produce prediction, we then investigate whether we can observe a garden-path effect and whether we can see a difference in this garden-path effect between NP/S and NP/Z sentences. In the following section, we will compare our results with the existing ones in the literature that use surprisal.

\begin{rmk}
    For the remainder of this chapter, we will study the \emph{linear correlations} between $\mathsf{SF}$ and reading times. We made this choice for simplicity and not according to any heuristic. Investigating other types of relations between these quantities is left to future work. 
\end{rmk}

\subsection{The Sturt et al. dataset}\label{subsec:GPpredSturt}

\begin{figure}[p]
    \begin{minipage}[c]{.45\linewidth}
        \centering
        \begin{subfigure}{\linewidth}
            \centering
            \includegraphics[width=\linewidth]{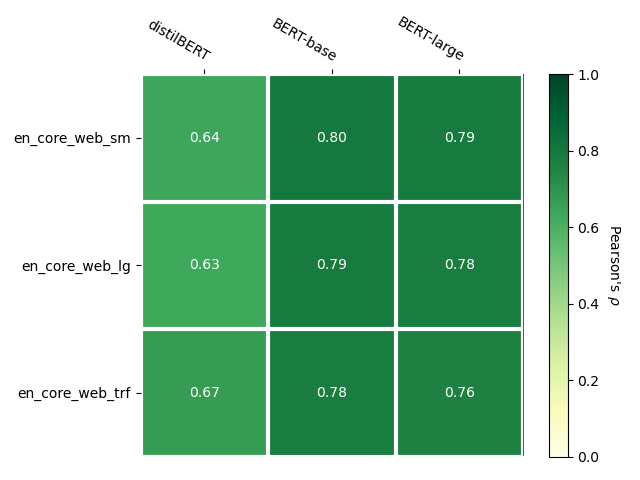}
            \caption{Pearson's $\rho$ coefficients.\label{subfig:SturtRho}}
        \end{subfigure}
        \begin{subfigure}{\linewidth}
            \centering
            \includegraphics[width=\linewidth]{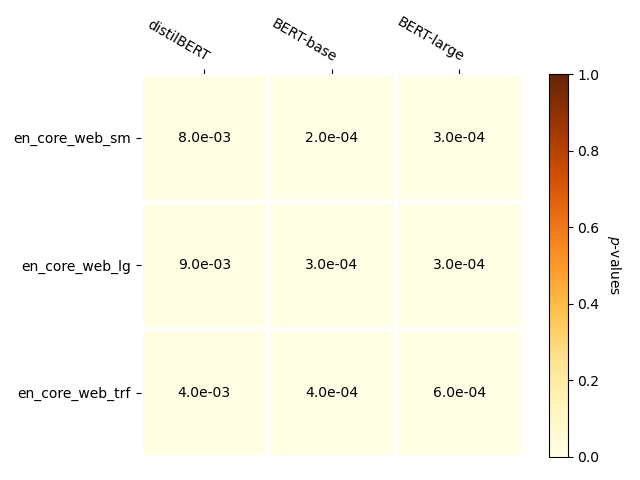}
            \caption{$p$-values associated with the Pearson's $\rho$ coefficients.\label{subfig:SturtRhoPval}}
        \end{subfigure}
    \end{minipage}\quad%
    \begin{minipage}[c]{.45\linewidth}
        \centering
        \begin{subfigure}{\linewidth}
            \centering
            \includegraphics[width=\linewidth]{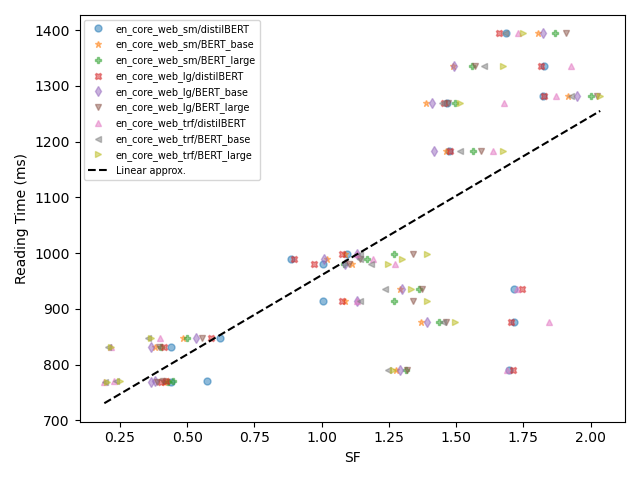}
            \caption{Linear correlation between $\mathsf{SF}$ and self-paced reading times.\label{subfig:SturtScatter}}
        \end{subfigure}
    \end{minipage}
    \begin{subfigure}{\linewidth}
        \centering
        \includegraphics[width=\linewidth]{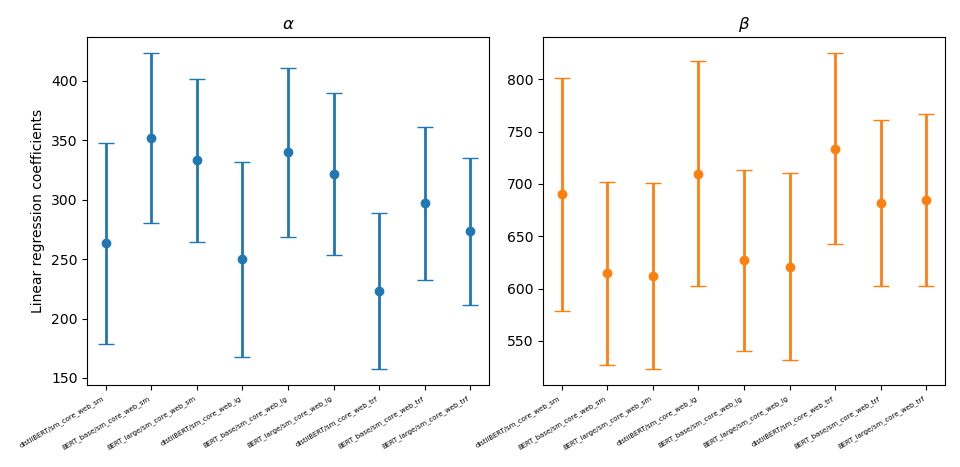}
        \caption{Coefficients of the linear regressions obtained for different \texttt{BERT} and \texttt{spaCy} variants. The standard error on these coefficients is depicted as error bars.\label{subfig:SturtLinReg}}
    \end{subfigure}
    \caption{Analysis of of the linear correlations between $\mathsf{SF}$ and reading times in the Sturt et al. dataset.}
\end{figure}

\subsubsection{The linear regression model}
\paragraph{}Starting from the assumption that the signalling fraction $\mathsf{SF}$ of an empirical model with contexts $\mathcal{M} =\{m, mw\}$ correlates with the amount of difficulty induced by reading the word $w$ (see Section~\ref{sec:GPfeatures}). From this, we expect that region reading time correlates with the sum of the signalling fractions, summed over all the words in the region. Now, given that we are studying linear relations between the signalling fractions and reading times, we expect this relation to be of the form:
\begin{equation}
    RT(r) = \alpha \sum_{w\in r} \mathsf{SF} (w) + \beta
\end{equation}
for any region $r$ and $w$ in that region. In the above equation, we denote by $\mathsf{SF} (w)$, the signalling fraction associated with the empirical model with contexts $\mathcal{M} = \{m, mw\}$. 

\paragraph{}Moreover, since~\cite{SturtPick} presented the average reading time over sentences of the same type, we thus can check whether the following holds: 
\begin{equation}\label{eq:SturtLinReg}
    \left<RT\big(r(\mathtt{S})\big)\right>_{\mathtt{S}} = \alpha \left<\sum_{w\in r(\mathtt{S})} \mathsf{SF} (w)\right>_{\mathtt{S}} + \beta
\end{equation}
where this time, the region $r(\mathtt{S})$ denotes a particular region of a given sentence $\mathtt{S}$. 

\subsubsection{Correlation with reading times}
\paragraph{} To test the above hypotheses, we calculate Pearson's $\rho$ coefficients associated with $\mathsf{SF}$ and reading times. The obtained $\rho$ coefficients and their associated $p$-values (testing the confidence level at which we have $\rho\neq 0$) are shown in Fig.~\ref{subfig:SturtRho} and~\ref{subfig:SturtRhoPval} respectively. 

\paragraph{}As we can see, these correlation coefficients are generally high (all of them $> 0.63$) and, importantly, all positive. Furthermore, the $p$-values associated with the correlation coefficients are statistically significant ($p<9\times 10^{-3}$ for any choice of \texttt{BERT} and \texttt{spaCy} variants). These results are evidence of a robust monotonic relation between $\mathsf{SF}$ and reading times. I.e. when $\mathsf{SF}$ increases, so does the reading time. 

\paragraph{Impact of the different \texttt{BERT} and \texttt{spaCy} variants}In addition, these coefficients are higher ($\rho \simeq 0.78$) for any empirical model obtained from the larger \texttt{BERT-base} or \texttt{BERT-large} as compared to the empirical models obtained from the lighter version \texttt{distilBERT}. Similarly, the $p$-values appear to be larger for models using \texttt{distilBERT} (of the order of magnitude of $p\sim 10^{-3}$) as opposed to ones obtained from the other \texttt{BERT} variants ($p\sim 10^{-4}$). This suggests that reducing the parameters of \texttt{BERT} may impact our performance accuracies. 

There is, however, no sign of the influence of the \texttt{spaCy} models by solely looking at the $\rho$ coefficients. 

\paragraph{Regression models}From the existence of a linear correlation, the use of the linear model of \eqref{eq:SturtLinReg} is justified. The coefficients $\alpha$ and $\beta$ calculated for empirical models calculated from different \texttt{BERT} and \texttt{spaCy} variants are shown in Fig.~\ref{subfig:SturtLinReg}. We can then see that the obtained coefficients are comparable for all of the \texttt{BERT} and \texttt{spaCy} models, which overall give a linear regression model around:
\begin{equation}
    \left<RT\big(r(\mathtt{S})\big)\right>_{\mathtt{S}} \simeq 295\left<\sum_{w\in r(\mathtt{S})} \mathsf{SF} (w)\right>_{\mathtt{S}} + 664
\end{equation}

For the rest of this work, however, we will take the individual regression models obtained for each of the \texttt{BERT} and \texttt{spaCy} variants. This will then ensure that we get the best possible predictions.

\subsubsection{Predicting garden-path effects}
\paragraph{}Using these regression models, we can make predictions of reading times from $\mathsf{SF}$. In turn, we can evaluate these predictions' accuracy by looking at what effect they can and cannot predict. 

\paragraph{Methods}We start by investigating whether $\mathsf{SF}$ can predict garden-path effects, i.e. whether the reading times of garden-path sentences are higher than the reading times for the equivalent unambiguous sentences (over their critical region). 

To do so, we calculate the so-called \emph{garden-path effect} of a garden-path sentence by simply taking the difference in reading time of the critical region in the ambiguous and unambiguous versions, i.e.:
\begin{equation}\label{eq:SturtGPE}
    GPE(\mathtt{S}) = RT(r_{critical}(\mathtt{S})) - RT(r_{critical}(unambiguous(\mathtt{S})))
\end{equation}
where $r_{critical}(\mathtt{S})$ isolates the critical region of $\mathtt{S}$, and $unambiguous(\mathtt{S})$ gives the unambiguous version of a sentence $\mathtt{S}$. 

\paragraph{Results}The average garden-path effect obtained for empirical models using the different variants of \texttt{BERT} and \texttt{spaCy} are shown in Fig.~\ref{subfig:SturtGPE}. We can observe that, on average, these predicted garden-path effects are indeed positive, therefore showing that $\mathsf{SF}$ predicts higher reading times for garden-path sentences than their unambiguous versions. 

To strengthen this result, we also conducted 1-sample $t$-tests testing the null hypothesis that this average is $0$ (i.e. that the reading time predictions are the same for ambiguous and unambiguous sentences). The resulting $p$-values are depicted in Fig.~\ref{subfig:SturtGPEttest}. We can see all of the $p$-values are indeed statistically significant, except for the empirical models using \texttt{distilBERT} and the \texttt{en\_core\_web\_lg} pipeline of \texttt{spaCy} (where even the $p$-value is relatively low and is $p = 0.07$). 

These results overall show that $\mathsf{SF}$ can confidently detect the existence of a garden-path effect.    

\begin{figure}[htb!]
    \centering
    \begin{subfigure}[t]{.45\linewidth}
        \centering
        \includegraphics[width=\linewidth]{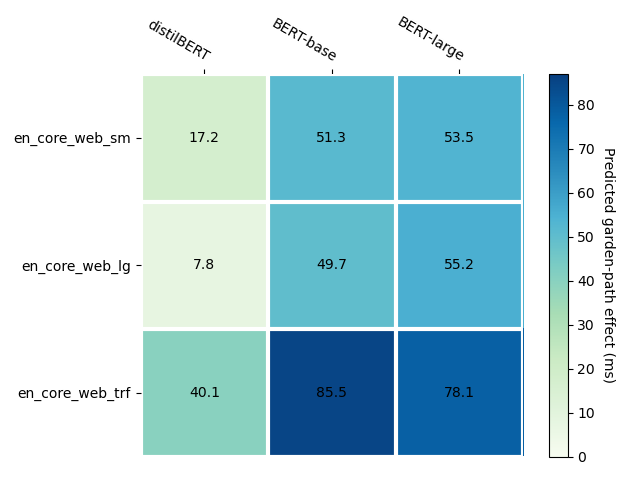}
        \caption{Average predicted garden-path effect.\label{subfig:SturtGPE}}
    \end{subfigure}\quad%
    \begin{subfigure}[t]{.45\linewidth}
        \centering
        \includegraphics[width=\linewidth]{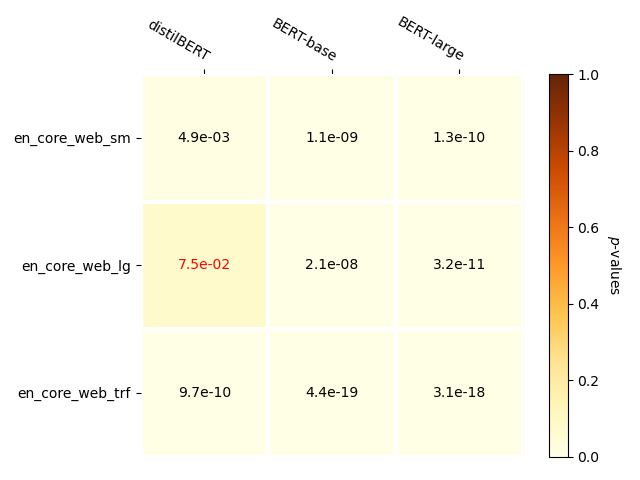}
        \caption{$p$-values associated with the 1-sample $t$-tests evaluating whether the average garden-path effect is $0$.\label{subfig:SturtGPEttest}}
    \end{subfigure}
    \caption{Analysis of the predicted garden-path effect over the Sturt et al. dataset.}
\end{figure}

\subsubsection{NP/S and NP/Z predictions}
\paragraph{}We now want to analyse the prediction for garden-path effects for NP/S and NP/Z sentences separately. 

\paragraph{}The distributions of the obtained garden-path effects for NP/S and NP/Z sentences are shown in Fig.~\ref{fig:SturtBoxplots}. As we can see, $\mathsf{SF}$ overall underestimates the garden-path effects, particularly for NP/Z sentences. This suggests that $\mathsf{SF}$ alone does not entirely explain the full difficulty of garden-paths. However, we can also observe that increasing the number of parameters of both the \texttt{BERT} and \texttt{spaCy} models improves the predictions. 

\begin{figure}[htb!]
    \centering
    \includegraphics[width=\linewidth]{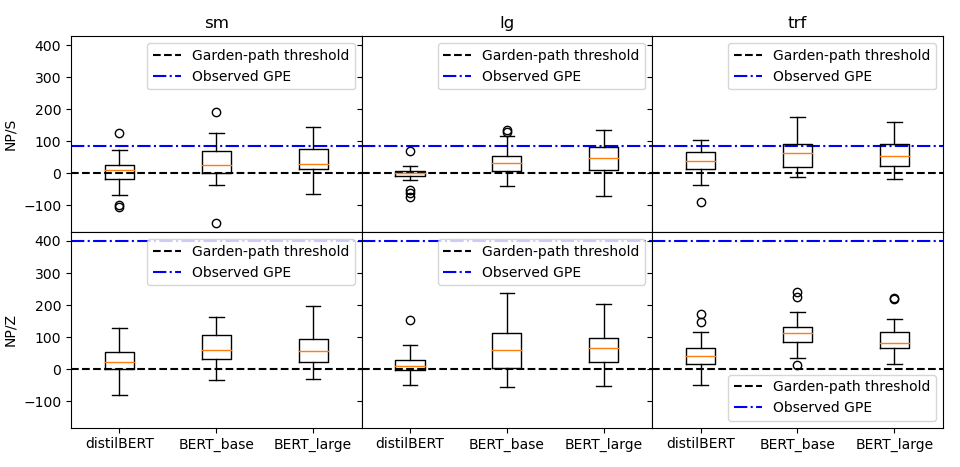}
    \caption{Boxplots of the garden-path effects predicted over the Sturt et al. dataset. The human baseline is also quoted.\label{fig:SturtBoxplots}}
\end{figure}

\paragraph{}Finally, we are also interested to see whether $\mathsf{SF}$ can detect different levels of difficulty, namely between the NP/S and NP/Z sentences. By comparing the average predicted garden-path effects of NP/S and NP/Z sentences shown in Figs.~\ref{subfig:SturtNPS} and~\ref{subfig:SturtNPZ}, one can see that the garden-path effects are generally higher for NP/S than for NP/Z sentences. 

We then tested this hypothesis by conducting $t$-test comparing them and found that this difference is statistically significant for most \texttt{BERT} and \texttt{spaCy} variants. Surprisingly, the empirical models using \texttt{BERT-large} did not perform well under these $t$-tests. This negative result could still be due to the noisiness of human data or the several averaging steps that have occurred to obtain data in the first place. 

On the whole, this gives us evidence that $\mathsf{SF}$ can identify different levels of parsing difficulty.

\begin{figure}[htb!]
    \centering
    \begin{minipage}[c]{.45\linewidth}
        \centering
        \begin{subfigure}{\linewidth}
            \centering
            \includegraphics[width=\linewidth]{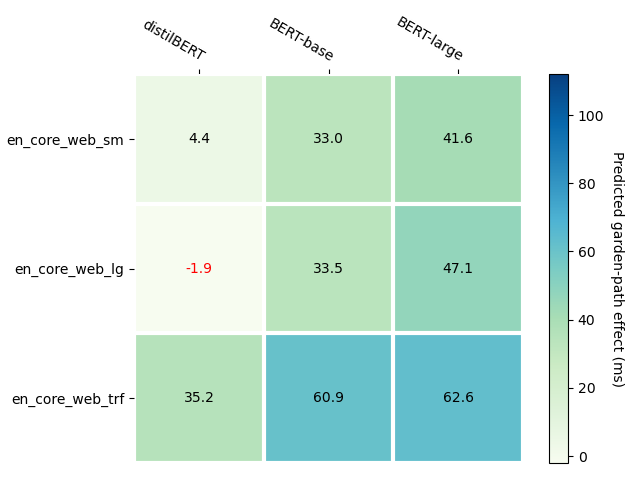}
            \caption{Average predicted garden-effects for NP/S sentences.\label{subfig:SturtNPS}}
        \end{subfigure}
        \begin{subfigure}{\linewidth}
            \centering
            \includegraphics[width=\linewidth]{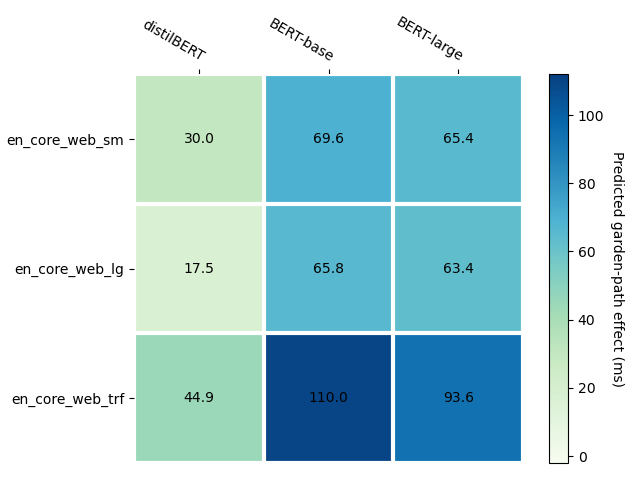}
            \caption{Average predicted garden-effects for NP/Z sentences.\label{subfig:SturtNPZ}}
        \end{subfigure}
    \end{minipage}\quad%
    \begin{minipage}[c]{.45\linewidth}
        \centering
        \begin{subfigure}{\linewidth}
            \centering
            \includegraphics[width=\linewidth]{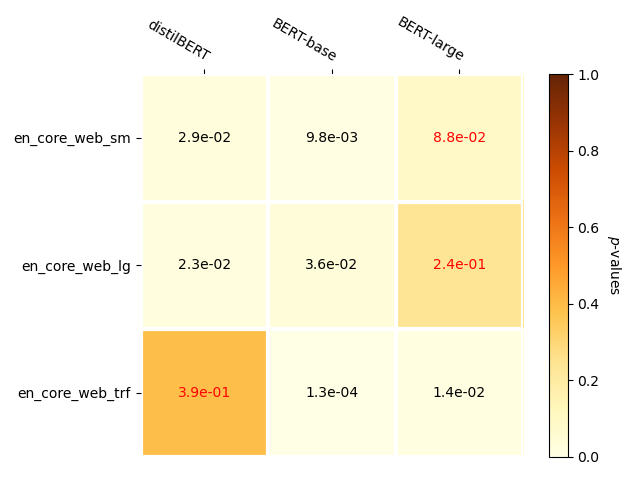}
            \caption{$p$-values associated with the $t$-test with the null hypothesis that garden-path effects from NP/S and NP/Z sentences are sampled from the same distribution.\label{subfig:SturtSZttest}}
        \end{subfigure}
    \end{minipage}
    \caption{Comparison of the predicted garden-path effects between NP/S and NP/Z sentences.}
\end{figure}

\subsection{The Grodner et al. dataset}\label{subsec:GPpredGrodner}

\begin{figure}[p]
    \begin{minipage}[c]{.45\linewidth}
        \centering
        \begin{subfigure}{\linewidth}
            \centering
            \includegraphics[width=\linewidth]{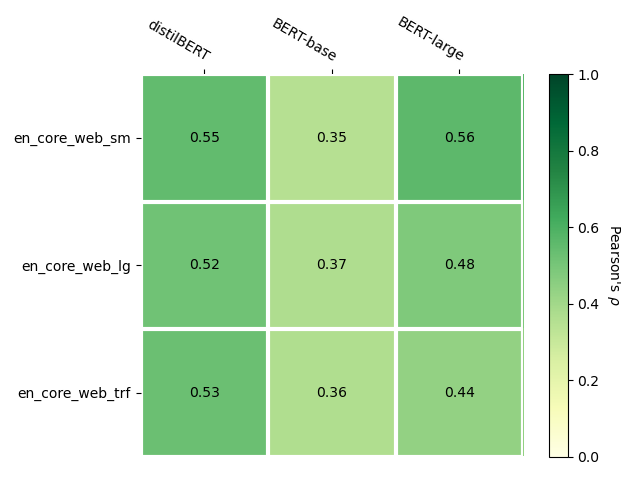}
            \caption{Pearson's $\rho$ coefficients.\label{subfig:GrodnerRho}}
        \end{subfigure}
        \begin{subfigure}{\linewidth}
            \centering
            \includegraphics[width=\linewidth]{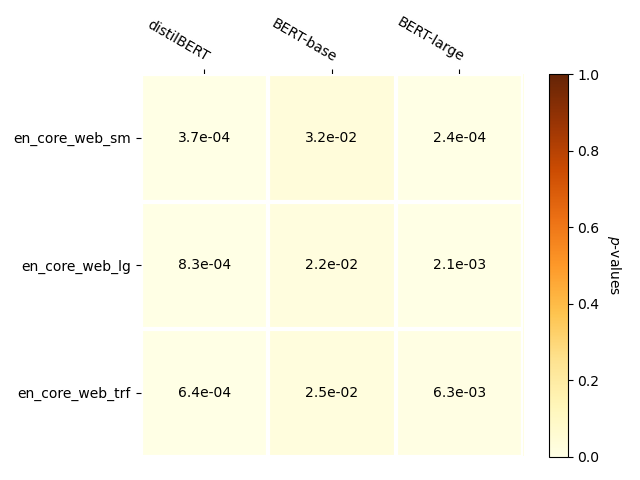}
            \caption{$p$-values associated with the Pearson's $\rho$ coefficients.\label{subfig:GrodnerRhoPval}}
        \end{subfigure}
    \end{minipage}\quad%
    \begin{minipage}[c]{.45\linewidth}
        \centering
        \begin{subfigure}{\linewidth}
            \centering
            \includegraphics[width=\linewidth]{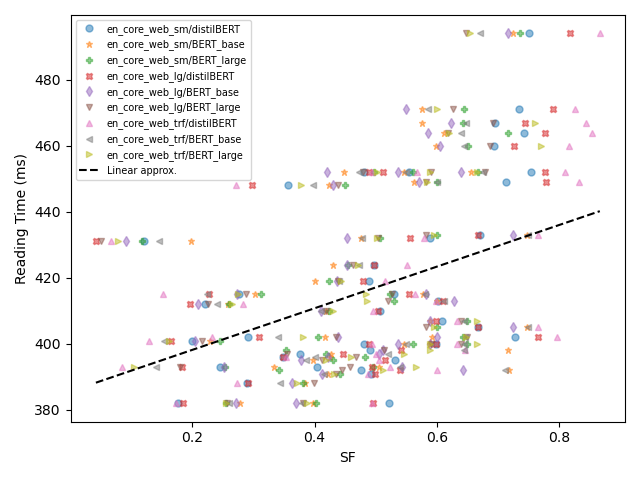}
            \caption{Linear correlation between $\mathsf{SF}$ and self-paced reading times.\label{subfig:GrodnerScatter}}
        \end{subfigure}
    \end{minipage}
    \begin{subfigure}{\linewidth}
        \centering
        \includegraphics[width=\linewidth]{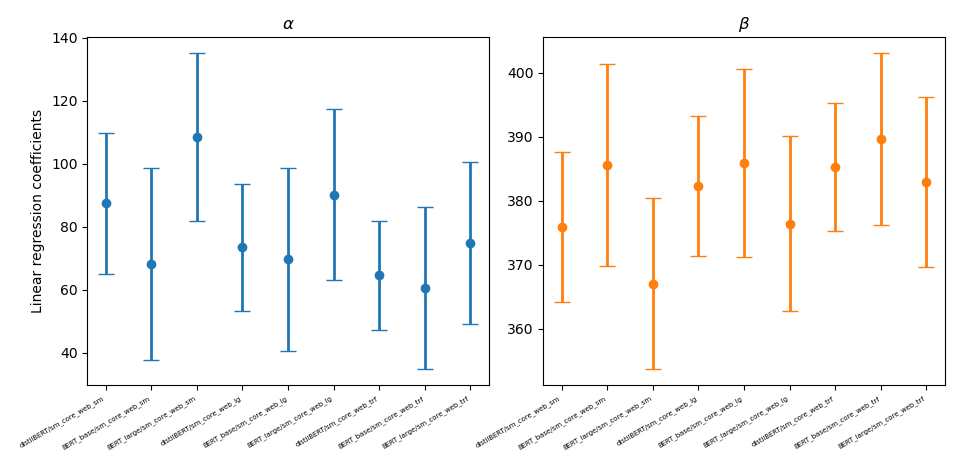}
        \caption{Coefficients of the linear regressions obtained for different \texttt{BERT} and \texttt{spaCy} variants. The standard error on these coefficients is depicted as error bars.\label{subfig:GrodnerLinReg}}
    \end{subfigure}
    \caption{Analysis of of the linear correlations between $\mathsf{SF}$ and reading times in the Grodner et al. dataset.}
\end{figure}

\subsubsection{The linear regression models}
\paragraph{}We now analyse the Grodner et al. dataset predictions. We first recall that the figures quoted in~\cite{grodner} are not region-by-region but word-by-word reading times, averaged for each region across different sentences. Hence, instead of finding a linear regression model of \eqref{eq:SturtLinReg}, we will be interested in models of the form:
\begin{equation}\label{eq:GrodnerLinReg}
    \left<RT(w)\right>_{w\in r(\mathtt{S}), \mathtt{S}} = \alpha \left<\mathsf{SF}(w)\right>_{w\in r(\mathtt{S}), \mathtt{S}} + \beta
\end{equation}

\subsubsection{Correlation with reading times}
\paragraph{}As we did for the Sturt et al. dataset, we test this hypothesis by first computing the associated Pearson's $\rho$ coefficients and associated $p$-values. These can be found in Fig.~\ref{subfig:GrodnerRho} and~\ref{subfig:GrodnerRhoPval} respectively. 

\paragraph{}The obtained correlation coefficients are found to be smaller than the ones obtained for the Sturt et al. dataset (here, $\rho >0.35$ for all \texttt{BERT} and \texttt{spaCy}) variant. The correlation is still positive, and we achieve correlations up to $\rho = 0.56$. Similarly, the $p$-values are generally higher than the ones of Fig.~\ref{subfig:SturtRhoPval}, but still statistically significant (all of the $p<0.04$). 
\paragraph{Impact of the \texttt{BERT} and \texttt{spaCy} variants}Similarly to the other dataset, we observe that the coefficients $\rho$ seem to be more affected by the choice of \texttt{BERT} model than the choice of \texttt{spaCy} variant. However, contrary to the previous results, we observe that it is the \texttt{BERT-base} model that led to the worse correlations and that the \texttt{distilBERT} empirical models lead to fairly high correlations ($\rho \sim 0.53$). 

Overall, the positive correlations and the low $p$-values strengthen our previous findings that $\mathsf{SF}$ are correlated with reading times. 

\paragraph{Linear regression}As before, we can use linear regression equations to predict reading times from signalling fractions. The different $\alpha$ and $\beta$ coefficients obtained for different choices of \texttt{BERT} and \texttt{spaCy} variants are shown in Fig.~\ref{subfig:GrodnerLinReg}. 

These coefficients are also highly similar and revolve around the following model:
\begin{equation}
    \left<RT(w)\right>_{w\in r(\mathtt{S}), \mathtt{S}} = 77 \left<\mathsf{SF}(w)\right>_{w\in r(\mathtt{S}), \mathtt{S}} + 381
\end{equation}

As done previously, we use each linear regression model for the rest of the analysis to reduce the number of errors due to averaging.

\subsubsection{Predicting garden-path effects}
\paragraph{}We then compute the predictions of the garden-path effect. Similarly to \eqref{eq:SturtGPE}, we calculate this garden-path effect as:
\begin{equation}
    GPE(\mathtt{S}) = \left<RT(w)\right>_{w\in r_{critical}(\mathtt{S})} -  \left<RT(w)\right>_{w\in r_{critical}(unambiguous(\mathtt{S}))}
\end{equation}

\paragraph{Results}As before, we want this garden-path effect to be greater than $0$. This is indeed the case (see Fig.~\ref{subfig:GrodnerGPE}). 

We furthermore conducted 1-sample $t$-tests to quantify the confidence of this finding and found that these garden-path effects are indeed statistically significantly non-zero, aside from the empirical models using \texttt{distilBERT} along with the \texttt{spaCy} models \texttt{en\_core\_web\_sm} and \texttt{en\_core\_web\_lg}. 

We therefore conclude that $\mathsf{SF}$ can also detect a garden-path effect in this dataset.

\begin{figure}[htb!]
    \centering
    \begin{subfigure}[t]{.45\linewidth}
        \centering
        \includegraphics[width=\linewidth]{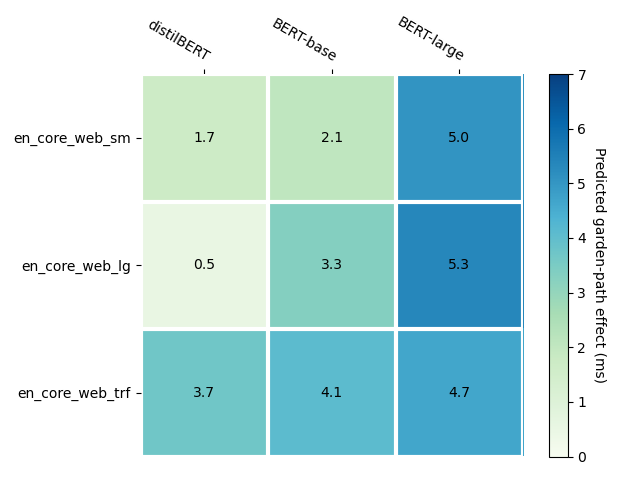}
        \caption{Average predicted garden-path effect.\label{subfig:GrodnerGPE}}
    \end{subfigure}\quad%
    \begin{subfigure}[t]{.45\linewidth}
        \centering
        \includegraphics[width=\linewidth]{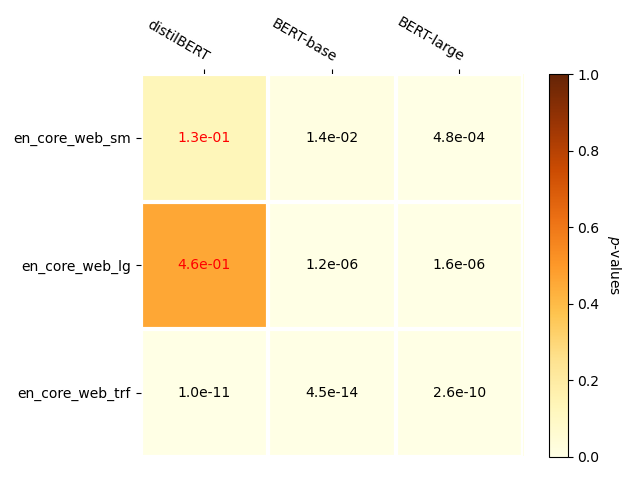}
        \caption{$p$-values associated with the 1-sample $t$-tests evaluating whether the average garden-path effect is $0$.\label{subfig:GrodnerGPEttest}}
    \end{subfigure}
    \caption{Analysis of the predicted garden-path effect over the Sturt et al. dataset.}
\end{figure}

\subsubsection{NP/S and NP/Z sentences}
\paragraph{}Now focusing on the difference in predictions between NP/S and NP/Z sentences, we depict the distributions of predicted garden-path effect for the two categories of garden-path sentences in Fig.~\ref{fig:GrodnerBoxplots}. 

\paragraph{}We observe that, as before, $\mathsf{SF}$ underestimates the garden-path effect of both NP/S and NP/Z sentences. In addition, in some NP/S empirical models, the predicted garden-path effect is \emph{negative}. The models showing a negative garden-path effect are, however, restricted to the small \texttt{BERT} (\texttt{distilBERT} or \texttt{BERT-base}), and \texttt{spaCy} (\texttt{en\_core\_web\_sm} or \texttt{en\_core\_web\_lg}) models. In the larger models, the predicted garden-path effects become positive.

\begin{figure}[htb!]
    \centering
    \includegraphics[width=\linewidth]{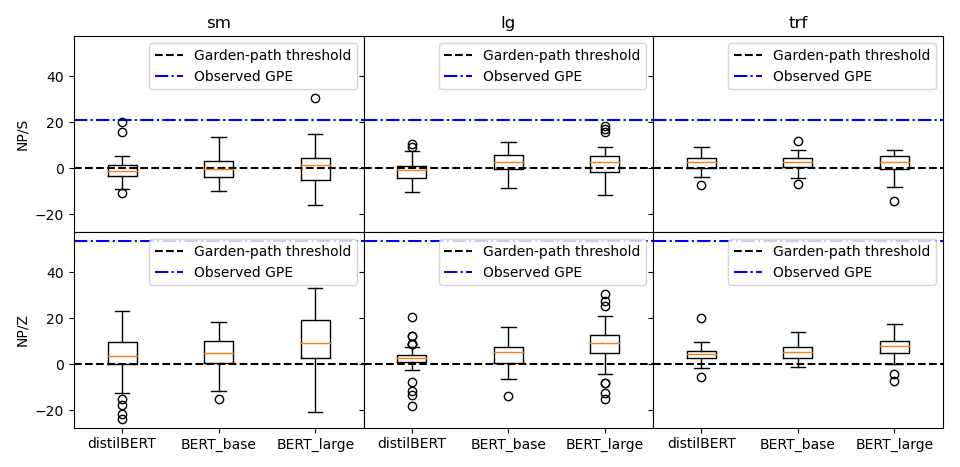}
    \caption{Boxplots of the garden-path effects predicted over the Sturt et al. dataset. The human baseline is also quoted.\label{fig:GrodnerBoxplots}}
\end{figure}

\paragraph{}It also appears that, as in the Sturt et al. dataset, the average garden-path effects are higher for NP/Z sentences than for NP/S sentences. This effect is shown empirically (see Figs.~\ref{subfig:GrodnerNPS} and~\ref{subfig:GrodnerNPZ}). 

Moreover, we can estimate the confidence levels of the claim that $\mathsf{SF}$ distinguishes between NP/S and NP/Z sentences using $t$-tests. The obtained $p$-values are depicted in Fig.~\ref{subfig:GrodnerSZttest}. These $t$-test are all statistically significant aside from the empirical models using \texttt{BERT-base} in conjunction with \texttt{en\_core\_web\_lg} pipeline of \texttt{spaCy} (and even then, the $p$-value is found to be $p<0.1$). 

Contrary to the ongoing trend, these $p$-values are more statistically significant than the ones obtained for the Sturt et al. dataset, therefore showing that the difference between garden-path effect predictions of NP/S and NP/Z sentences is \emph{more marked} in the Grodner et al. dataset.

\begin{figure}[htb!]
    \centering
    \begin{minipage}[c]{.45\linewidth}
        \centering
        \begin{subfigure}{\linewidth}
            \centering
            \includegraphics[width=\linewidth]{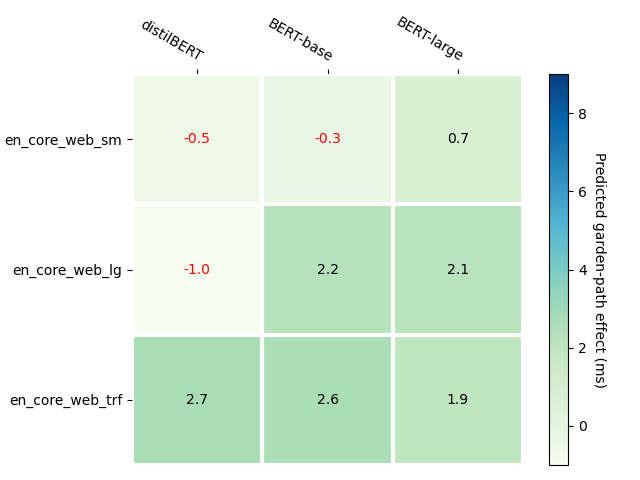}
            \caption{Average predicted garden-effects for NP/S sentences.\label{subfig:GrodnerNPS}}
        \end{subfigure}
        \begin{subfigure}{\linewidth}
            \centering
            \includegraphics[width=\linewidth]{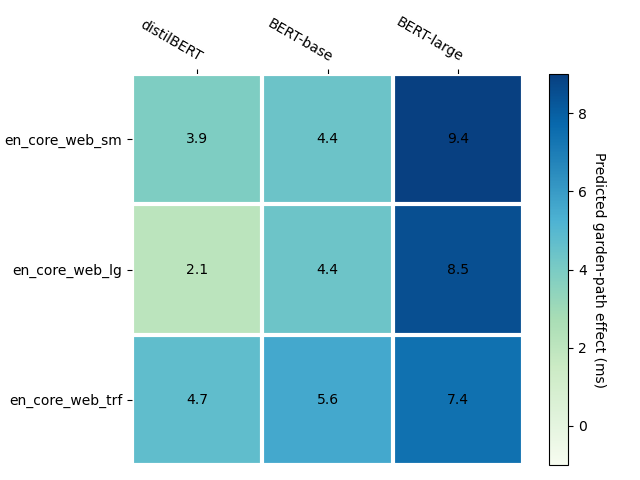}
            \caption{Average predicted garden-effects for NP/Z sentences.\label{subfig:GrodnerNPZ}}
        \end{subfigure}
    \end{minipage}\quad%
    \begin{minipage}[c]{.45\linewidth}
        \centering
        \begin{subfigure}{\linewidth}
            \centering
            \includegraphics[width=\linewidth]{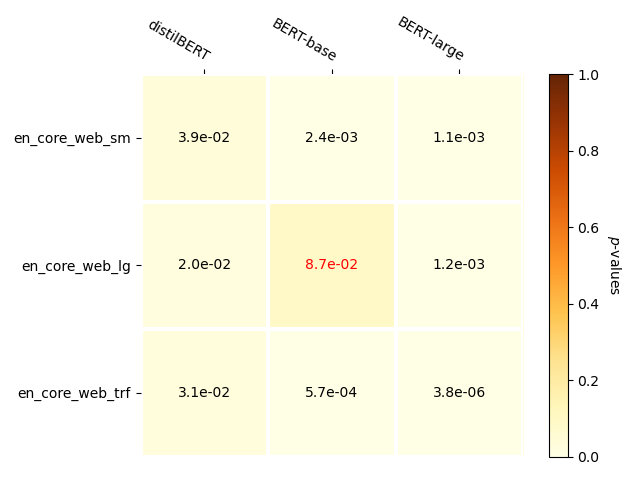}
            \caption{$p$-values associated with the $t$-test with the null hypothesis that garden-path effects from NP/S and NP/Z sentences are sampled from the same distribution.\label{subfig:GrodnerSZttest}}
        \end{subfigure}
    \end{minipage}
    \caption{Comparison of the predicted garden-path effects between NP/S and NP/Z sentences.}
\end{figure}

\subsection{General discussion}\label{subsec:GPpredDiscussion}
\paragraph{}By quantifying the amount of signalling in certain empirical models, we have obtained an alternative measure of the difficulty of parsing. We have seen that the signalling fraction $\mathsf{SF}$ does have a strong positive correlation with human reading times. Using a linear regression model lets us obtain predictions for reading times and garden-path effects. We have also found that this predicted garden-path effect is (statistically significantly) higher for NP/Z sentences than for NP/S sentences. Finding such differences has yet to be shown to be possible to identify using surprisal theory only~\cite{vanSchijndelA,vanSchijndelB,huang2023}.

\paragraph{}Although $\mathsf{SF}$ correctly identifies the relative difficulty levels of different garden-path sentences, the overall magnitude of the predicted effect is systematically underestimated. This is similar to the findings from surprisal theory~\cite{vanSchijndelA,vanSchijndelB,huang2023,linzenCCG}, and we therefore have similar hypotheses to explain why this is the case.

For instance, this could be evidence that the human parsing process is purely forward-looking. Many psycholinguistic theories suggest that backtracking or reprocessing is necessary to make sense of garden-path sentences, both of which are backward-looking processes. However, there is no clear evidence that this is indeed the case. The field of research is globally more interested in the cause of difficulty (e.g. vocabulary biases, plausibility, \ldots) and the nature of the parsing strategies (e.g. presence or absence of backtracking, parrallel or sequential parsing, \ldots), not the algorithmic parsing procedure (i.e. the specific steps the reader carry out when parsing a sentence), which is not even guaranteed to be the same for all individuals.

In addition, part of the difficulty also comes from the nature of the datasets used in this experiment. Indeed, only a few averaged data points were available to compute the linear regression models. Hence, taking those averages would have considerably increased the amount of error even before any computation was done. In particular, we have observed that the predictions obtained from the Grodner et al. dataset were almost consistently worse than those obtained from the Sturt et al. dataset. From these considerations, this could be because one more round of averaging was done in this dataset (i.e. averaging of the reading times for each region). 

Other explanations for this discrepancy could be due to the unevenness of the dataset. For instance, the number of characters per region, the number of words per region, and the vocabulary were not specifically controlled. To remedy this, the authors of~\cite{grodner} decided to use a linear regression model tailored to each participant to normalise the reading times and determine whether a word is read slower or faster than expected. These normalised reading times were referred to as \emph{residual reading times}. However, the parameters of these linear regressions are not available, and it was still more convenient to use raw reading times in our study rather than residual reading times. 

\paragraph{}To deal with those issues, we plan to use more detailed datasets in future works, such as self-paced reading time datasets collected in~\cite{prasad,huang2023}, which contain word-by-word reading times. Using a more detailed dataset also comes with its drawbacks. For instance, \emph{spillover} (i.e. delay in the observed difficulty) has to be accounted for. However, workarounds exist in the literature~\cite{vanSchijndelB,huang2023}. Our first steps will then use them. 

In addition, although surprisal is not by default capable of studying backward-looking processes, only a few modifications to the procedure described in Section~\ref{sec:GPModel} are necessary to be able to account for more complex parsing strategies. We expand this in more detail in the conclusion.
\section{Comparison with surprisal}\label{sec:GPComp}
\paragraph{}We now want to put our results in perspective and compare them with the state-of-the-art methods from computational linguistics, which use surprisal theory. To simplify this comparison, the work presented in~\cite{vanSchijndelA} used the same datasets as we did to produce garden-path effect predictions from surprisal. 

\begin{rmk}[On the fairness of the comparison] There are several reasons why the comparison between the predictions of~\cite{vanSchijndelA} and the ones presented in Section~\ref{sec:GPPred} may not be fair. Firstly, the calculations were not using the same language models. The authors of~\cite{vanSchijndelA} trained their own LSTM from the Wikipedia (2 million and 90 million tokens versions) and Wall Street Journal corpora. In contrast, we used the transformer model \texttt{BERT} (which did not exist when~\cite{vanSchijndelA} was published). The LSTMs trained in~\cite{vanSchijndelA} are not openly available. Therefore, it wasn't easy to obtain a fair comparison, and using pre-trained models was by far the most convenient solution for us.
\end{rmk}

\subsubsection{Magnitude of the garden-path effects}
\paragraph{}We start by comparing the magnitude of the garden-path effect predictions from surprisal and $\mathsf{SF}$. 

\paragraph{}The best predictions from Section~\ref{sec:GPPred}, as well as the best predictions from~\cite{vanSchijndelA} is presented in Table~\ref{tab:SurpVsSF}.

We observe that $\mathsf{SF}$ outperforms surprisal considerably better in the Sturt et al. dataset, where the predictions are up to $40\%$ more accurate than the ones of~\cite{vanSchijndelA} for NP/S sentences and about $20\%$ more accurate for NP/Z sentences. 

On the other hand, the surprisal predictions over the Grodner et al. dataset appear to be more accurate than the ones obtained using $\mathsf{SF}$. However, this accuracy decrease is only $20\%$ for NP/S sentences and $3\%$ for NP/Z sentences.

\paragraph{}What is quite interesting is that the study of van Schijndel and Linzen in~\cite{vanSchijndelA} reported that surprisal performed much better over the Grodner et al. dataset. In contrast, our investigation led to more accurate predictions over the Sturt et al. dataset. This discrepancy could explain why our results are better in the Sturt et al. dataset, whereas surprisal outperforms $\mathsf{SF}$ over the Grodner et al. dataset.

However, the cause of such discrepancy in accuracies is not clear. This suggests that surprisal and $\mathsf{SF}$ do not give the same weights to the same features.

\begin{table}[ht!]
    \centering
    \begin{tabular}{|c|c|c|c|c|}
        \hline
        && \multicolumn{2}{c|}{\textbf{Prediction} (ms)} & \multirow{2}{*}{\textbf{Observed} (ms)} \\\cline{3-4}
        && \textbf{$\mathsf{SF}$} & $S$& \\\hline
        \multirow{2}{*}{Sturt et. al}& NP/S & \textbf{62.6} &$24^*$& 87\\\cline{2-5}
        & NP/Z & \textbf{110} & $30^*$& 400\\\hline\hline
        \multirow{2}{*}{Grodner et. al}&NP/S & 2.73& \textbf{7} & 21\\\cline{2-5}
        &NP/Z & 8.52 & \textbf{10} & 53.5\\\hline
        \end{tabular}
    \caption{Comparison of the garden-path effects obtained using surprisal  (numbers taken from~\cite{vanSchijndelA}) and $\mathsf{SF}$. ${}^*$These numbers have been converted to be a region reading time from the word-by-word reading times quoted in~\cite{vanSchijndelA}.\label{tab:SurpVsSF}}
\end{table}

\subsubsection{NP/S and NP/Z predictions}
\paragraph{}Our more clear-cut results were regarding the difference in garden-path effect predictions for sentences with different levels of difficulty. 

Indeed, even though $p$-values were not quoted in the various studies using surprisal for predicting garden-path effects~\cite{vanSchijndelA,vanSchijndelB,huang2023,linzenCCG}, the authors identify the main issue with surprisal as not being able to distinguish between NP/S and NP/Z sentences. In fact, the trend observed in a follow-up study was that NP/S garden-path effects were, on average, higher than the predictions for NP/Z sentences~\cite{vanSchijndelB}. In conjunction with consistent underestimation of the slowdown prediction, this is their primary motivation for advocating backward-looking mechanisms in parsing strategies.

The fact that we can find such statistical differences using a forward-looking model does not invalidate this hypothesis. After all, our predictions are still widely underestimating the slowdowns as well. However, this may show that there might be some features that surprisal cannot detect, which opens up the question of what other quantities (even apart from $\mathsf{SF}$) could contribute to the reading difficulty of garden-path sentences.

\subsubsection{Linking $\mathsf{SF}$ and surprisal}

\paragraph{}Even though our usage of  $\mathsf{SF}$  stemmed from similar motivations to those for surprisal, it is unclear whether they are mathematically related. The reason for the better performance of $\mathsf{SF}$ is that surprisal, as used in this~\cite{vanSchijndelA}, mostly focuses on lexical items. In contrast, syntactic structures are first-class citizens for the $\mathsf{SF}$ quantity described here. 

\paragraph{}Only very recently, the role of syntactic structure in conjunction with surprisal has come into light:  in~\cite{linzenCCG}, it was shown that syntactic surprisal performs slightly better than pure lexical surprisal but still falls short when distinguishing  NP/S from NP/Z and the differences in garden-path effects. 

The results of~\cite{linzenCCG} and the ones presented here motivate the hypothesis that syntactic structures are the main deciding factor in the difficulty of garden-path sentences. 

\paragraph{}Another aspect of our work, which may have led to more accurate results, is that our model can take long-distance dependencies into account, whereas surprisal is not.
\section*{Summary of the chapter}
\markboth{Summary of the chapter}{Predicting garden-path effects}
    This Chapter used our sheaf-theoretic models to predict reading times and garden-path effects. Using two datasets from the psycholinguistic literature, we obtained the following results:
    \begin{itemize}
        \item The correlation between $\mathsf{SF}$ and reading times was positive and statistically significant;
        \item Using a linear regression model, we successfully predicted the existence of garden-path effects. However, we consistently underestimated the magnitude of the effect.
        \item We accurately predicted that the garden-path effects were higher for NP/Z than for NP/S sentences.
    \end{itemize}

    The signalling fraction clearly outperforms the surprisal predictions of the Sturt et al. dataset, both by the magnitude and by distinguishing the NP/S and NP/Z sentences. Over the Grodner et al. dataset, surprisal achieves more accurate predictions than $\mathsf{SF}$, but does not distinguish between NP/S and NP/Z sentences.

\chapter*{Conclusion}\label{chap:conclusion}
\addcontentsline{toc}{chapter}{Conclusion}
\markboth{Conclusion}{Conclusion}
\paragraph{} In this thesis, we studied natural language data from the perspective of foundational quantum mechanics. 

We observed how contextuality arises in natural language data and found uses for various quantities in psycholinguistics. To this end, we used the causal and signalling fractions, which were so far only used to describe quantum systems. Our results demonstrated that the sheaf-theoretic framework of contextuality and causality does uncover some linguistic phenomena relating to lexical and syntactic ambiguity arising from human behaviour.

\paragraph{Lexical Ambiguity} We started by looking at lexical ambiguities, where the analogy between words and quantum systems appeared more natural. Our detailed analysis showed that although contextuality is hard to obtain in the statistics of lexically ambiguous phrases, it is still possible to find witnesses of quantum-like contextuality, as defined under the Contextuality-by-Default framework. This is evidence of the essential role of the context in the disambiguation process of lexically ambiguous items. In addition, we also saw that the causal fractions of SV and VO empirical models confirmed that the observed probability distributions were primarily consistent with verb after subject and verb after object disambiguation orders. This finding showed that verbs tend to be disambiguated after their arguments, which is consistent with the psycholinguistic theories of the disambiguation of lexically ambiguous words. Using this finding, we simulated the lexical disambiguation process using variational quantum circuits. We also demonstrated that these circuits could, in turn, predict the different interpretation probability distributions associated with unseen phrases. This last result is exciting as it only required a small training set; in theory, only knowledge of 8 probability distributions should be able to predict the probability distributions of 8 new ones, and each of these probability distributions only required annotations of 25 participants, which is much lower than the resources needed to train a large language model.

To our knowledge, this project is the first to study the disambiguation process of phrases containing two target words of different grammatical roles and with explicit syntactic relations between them. Moreover, the proof-of-principle that variational quantum circuits can simulate the (human) lexical disambiguation process opens several questions regarding their performance in standard NLP tasks, including word-sense disambiguation. In addition, although we have only focused on two possible interpretations of each word, the approach can easily be extended to an arbitrary number of possible interpretations.

\paragraph{Syntactic ambiguity}We then turned our attention to syntactic ambiguities and notably focused on particular sentences, namely garden-path sentences, that are important for studying the human syntactic parsing process. 

We first observed high correlations between the signalling fraction and reading times of garden-path sentences. This result showed that the signalling fraction correlates with the difficulty of parsing a given sentence fragment. We then use such correlation to produce a linear model of reading times in terms of the signalling fraction. This linear regression model allowed us to predict reading times and garden-path effects associated with different sentences. These predictions outperformed the current state-of-the-art predictions of computational linguistics using surprisal theory. Among these, the most crucial improvement from $\mathsf{SF}$ was to find statistically significant differences in NP/S and NP/Z sentences, where the former is significantly easier to parse than the latter. This may be evidence that our (significantly simplified) parsing model may be closer to the actual human parsing process than surprisal theory.

\paragraph{}We believe this project paves the way for better quantum-based NLP algorithms, as it sheds some light on how different aspects of natural language ambiguities would benefit from quantum advantages. The mathematical frameworks we used led to meaningful results and provided proof of concept that they can be used to talk about linguistic phenomena.

\section*{Future work}\label{sec:future}
\markboth{Future Work}{Conclusion}
\paragraph{}This approach adopted in this project offers more possible research lines. We here describe a few of the possible extensions of this project. 

\subsubsection*{General improvements}
\paragraph{}The first and most obvious way to extend our approach is to loosen some simplifications imposed on the different empirical models. For instance, it will be worth expanding our lexical ambiguity empirical models to include all of the possible interpretations of each word, e.g. using its different \texttt{WordNet} senses. Similarly, it would be interesting to see whether adding the labels of the dependencies in syntactic empirical models would impact the accuracy of the reading time predictions. Furthermore, we could also consider expanding from having two possible choices of words for subjects, verbs, and objects to having the full vocabulary in lexical ambiguity models. In the case of the syntactic models, we could consider all possible ways a sentence fragment can be completed instead of restricting ourselves to the observed continuation. By doing so, we may have a closer link with surprisal theory.

\paragraph{}In addition, it will be interesting to combine the syntax and semantics models, which were described independently in Parts~\ref{part:Lexical} and~\ref{part:Syntactic}. At the moment, we can envisage two ways of doing so. The first would consist of concatenating syntactic and semantic empirical models and possibly adding an ad-hoc notion of interaction between the two. On the other hand, the psychology literature suggests that syntactic information has more influence on the semantic level than the other way around. Hence, another (more complex) possibility would be to have a higher-order causal order in which syntactic empirical models could influence any semantic process (but not necessarily the other way around). We could extend this further by considering other knowledge sources, such as pragmatic information and plausibility.

\subsubsection*{Lexical Ambiguity}
\paragraph{Proof of quantum advantage}Regarding lexical ambiguity empirical models, even though we demonstrated the existence of contextual witnesses in lexical ambiguity data, it is still not clear that quantum systems are \emph{necessary} or would even provide a computational advantage, in simulating the disambiguation process -- further investigation will be needed in this regard. 

Besides, the simulations conducted as part of the project were merely proof of principle, and the obtained results have yet to be compared with their classical analogues. 

\paragraph{Investigating other parts-of-speech}The next step should be to investigate the behaviour of other grammatical types (e.g. adjective, adverbs, \ldots) and more compex phrases and sentences, e.g. subject-verb-object sentences. However, the need for more psycholinguistics research may be a hindrance. 

By doing so, we expect to provide a new compositional way of processing natural language data, which, although it comes from a different motivation, may be highly related to the approach of DisCoCat~\cite{DisCoCat} or DisCoCirc~\cite{Coecke2021} formalisms.

\paragraph{Improving the variational circuits}In addition, even though we observe differences in data from words of different levels of ambiguity or different grammatical types, they are, in the empirical models, treated in the same way. This may particularly affect Chapter~\ref{chap:lexicalCircuits} simulations. 

Examples of possible improvement may be allowing words to be represented as mixed states (i.e. probabilistic mixture of pure states) or pure states (i.e. superposition of states) -- note that at the moment, nouns are only represented as pure states. Using the intuition of~\cite{Piedeleu2015}, we would expect homonymous nouns to be represented as mixed states and polysemous nouns as pure states. 

Furthermore, the accuracy of the predictions and simulations may also increase by having an extra ancilla for verbs, thus allowing the verb to take in information from both the subject and object, even though one argument is not known. We could then represent underspecification by taking the partial trace over the system for which no information is provided. This representation for verbs would then be similar to the DisCoCat representation of a transitive verb, and by adopting this structure, we can train verb-states compatible with the DisCoCat formalism.

\paragraph{Including indefinite causal orders}It is also quite clear that the process of disambiguating, even SV or VO phrases, is not entirely one-way (i.e., the probability distributions associated with the activation of subject and objects depend on the choice of verbs, as the verb provides context for the ambiguous nouns). Therefore, to fully describe the disambiguation process, we must introduce the notion of \emph{indefinite causal order}. 

Indefinite causal orders have benefited from an increasing amount of research interest in the quantum foundations community, notably since a causal order can not only be probabilistic but also in \emph{superposition} in (higher-order) processes such as the quantum switch. 

The first question on the linguistic side is whether the disambiguation process is \emph{causally separable} (i.e. correspond to the probabilistic mixture of causal orders) or \emph{causally inseparable} (i.e. correspond to the superposition of causal orders). In the latter's case, this would provide an additional (and possibly more interpretable) advantage in using quantum resources. We could study this by calculating the so-called \emph{causal separability fraction} introduced in~\cite{sheafcausalityB}.

\subsubsection*{Syntactic ambiguity}
\paragraph{Further investigate the properties of our model}Regarding our syntactic model, our line of research offers excellent promises relating to modelling cognitive processes using sheaves and presheaves. There are still many avenues to explore, e.g., the nature of the correlations between $\mathsf{SF}$ and difficulty or the model's applicability to a broader class of sentences. In addition, our model still underestimates the garden-path effects of both NP/S and NP/Z sentences. It is, therefore, imperative to identify the reason for this discrepancy, i.e. whether it be because of the choice of regression or due to a more fundamental factor such as backtracking. Finally, the uncertainity introduced from averaging data in the psycholinguistic datasets used widely hindered our preliminary results. Using different and more detailed datasets is a way to address this point.

\paragraph{Introducing an edit distance}One possible criticism of our framework is that the different parses are treated as completely unrelated. In reality, this is not the case, as transformations between certain parses may be easy or hard. For instance, moving the head of a determiner by one place should be easier than changing the head of the whole sentence. These transformations between parses can also occur at different levels, for instance, within the same sentence fragment or across different fragments. Defining a measure of discrepancy between probability distributions over parses that considers this ``transformation difficulty'' is left as future work. 

\paragraph{Studying reanalysis}In addition, the main hypothesis behind garden-path effect underestimation, in surprisal theory at least, corresponds to backtracking or other related non-incremental processes. Contrary to surprisal, however, it would be fairly easy to alter our current parsing model to study non-incremental processes. In particular, this could be done by extending the morphisms from prefix order to standard inclusions and changing the choice of cover. This would amount to changing the causal order of interest. By comparing the causal fractions associated with different causal orders, we expect that the one(s) with the highest causal fraction would show up in eye-tracking data as the trajectories adopted by the different readers. In addition, it is not clear that only one causal order would be more advantageous as compared to others; in fact, we would expect multiple causal orders to have comparably high causal fractions. Hence, we would expect that the different causal fractions might predict which \emph{reanalysis patterns} could be employed by readers and at which frequency each of the possible patterns is adopted.

\bibliographystyle{plain}
\addcontentsline{toc}{chapter}{Bibliography}
\bibliography{others,quantum,psycho,compLing,nlp,qnlp}

\begin{thebibliography}{100}

\bibitem{senseval3}
{\em Proceedings of {SENSEVAL}-3, the Third International Workshop on the
  Evaluation of Systems for the Semantic Analysis of Text}, Barcelona, Spain,
  July 2004. Association for Computational Linguistics.

\bibitem{EMpBA}
Samson Abramsky and Rui~Soares Barbosa.
\newblock The logic of contextuality.
\newblock Schloss Dagstuhl – Leibniz-Zentrum für Informatik, 2021.

\bibitem{Abramsky2019}
Samson Abramsky, Rui~Soares Barbosa, Martti Karvonen, and Shane Mansfield.
\newblock A comonadic view of simulation and quantum resources.
\newblock volume 2019-June, 2019.

\bibitem{Abramsky2017}
Samson Abramsky, Rui~Soares Barbosa, and Shane Mansfield.
\newblock Contextual fraction as a measure of contextuality.
\newblock {\em Physical Review Letters}, 119, 2017.

\bibitem{abramskyCausality}
Samson Abramsky, Rui~Soares Barbosa, and Amy Searle.
\newblock Combining contextuality and causality: a game semantics approach,
  2023.

\bibitem{AbramskyBrad}
Samson Abramsky and Adam Brandenburger.
\newblock The sheaf-theoretic structure of non-locality and contextuality.
\newblock {\em New J. Phys.}, 13:113036, 2011.

\bibitem{Abramsky2004}
Samson Abramsky and Bob Coecke.
\newblock A categorical semantics of quantum protocols.
\newblock volume~19, 2004.

\bibitem{Abramsky2014semanticUnification}
Samson Abramsky and Mehrnoosh Sadrzadeh.
\newblock Semantic unification: A sheaf theoretic approach to natural language.
\newblock {\em Lecture Notes in Computer Science (including subseries Lecture
  Notes in Artificial Intelligence and Lecture Notes in Bioinformatics)}, 8222,
  2014.

\bibitem{Agirre2009}
Eneko Agirre, Llu{\'i}s M{\`a}rquez, and Richard Wicentowski.
\newblock Computational semantic analysis of language: Semeval-2007 and beyond.
\newblock {\em Language Resources and Evaluation}, 43(2):97--104, Jun 2009.

\bibitem{ahrens2013fundamental}
Johan Ahrens, Elias Amselem, Adan Cabello, and Mohamed Bourennane.
\newblock Two fundamental experimental tests of nonclassicality with qutrits,
  2013.

\bibitem{ajdukiewicz}
Kazimierz Ajdukiewicz.
\newblock Die syntaktische konnexitat.
\newblock {\em Studia philosophica}, pages 1--27, 1935.

\bibitem{Amaral2019}
Barbara Amaral.
\newblock Resource theory of contextuality.
\newblock {\em Philosophical Transactions of the Royal Society A: Mathematical,
  Physical and Engineering Sciences}, 377(2157):20190010, 2019.

\bibitem{Araujo2013}
Mateus Ara\`{u}jo, Marco~Túlio Quintino, Costantino Budroni, Marcelo~Terra
  Cunha, and Adán Cabello.
\newblock All noncontextuality inequalities for the $n$-cycle scenario.
\newblock {\em Physical Review A}, 88(2), August 2013.

\bibitem{linzenCCG}
Suhas Arehalli, Brian Dillon, and Tal Linzen.
\newblock Syntactic surprisal from neural models predicts, but underestimates,
  human processing difficulty from syntactic ambiguities.
\newblock In {\em Proceedings of the 26th Conference on Computational Natural
  Language Learning (CoNLL)}, pages 301--313, Abu Dhabi, United Arab Emirates
  (Hybrid), December 2022. Association for Computational Linguistics.

\bibitem{Arute2019}
Frank Arute, Kunal Arya, Ryan Babbush, Dave Bacon, Joseph~C. Bardin, Rami
  Barends, Rupak Biswas, Sergio Boixo, Fernando~G.S.L. Brandao, David~A. Buell,
  Brian Burkett, Yu~Chen, Zijun Chen, Ben Chiaro, Roberto Collins, William
  Courtney, Andrew Dunsworth, Edward Farhi, Brooks Foxen, Austin Fowler, Craig
  Gidney, Marissa Giustina, Rob Graff, Keith Guerin, Steve Habegger, Matthew~P.
  Harrigan, Michael~J. Hartmann, Alan Ho, Markus Hoffmann, Trent Huang,
  Travis~S. Humble, Sergei~V. Isakov, Evan Jeffrey, Zhang Jiang, Dvir Kafri,
  Kostyantyn Kechedzhi, Julian Kelly, Paul~V. Klimov, Sergey Knysh, Alexander
  Korotkov, Fedor Kostritsa, David Landhuis, Mike Lindmark, Erik Lucero, Dmitry
  Lyakh, Salvatore Mandrà, Jarrod~R. McClean, Matthew McEwen, Anthony Megrant,
  Xiao Mi, Kristel Michielsen, Masoud Mohseni, Josh Mutus, Ofer Naaman, Matthew
  Neeley, Charles Neill, Murphy~Yuezhen Niu, Eric Ostby, Andre Petukhov,
  John~C. Platt, Chris Quintana, Eleanor~G. Rieffel, Pedram Roushan,
  Nicholas~C. Rubin, Daniel Sank, Kevin~J. Satzinger, Vadim Smelyanskiy,
  Kevin~J. Sung, Matthew~D. Trevithick, Amit Vainsencher, Benjamin Villalonga,
  Theodore White, Z.~Jamie Yao, Ping Yeh, Adam Zalcman, Hartmut Neven, and
  John~M. Martinis.
\newblock Quantum supremacy using a programmable superconducting processor.
\newblock {\em Nature}, 574, 2019.

\bibitem{babbush2023exponential}
Ryan Babbush, Dominic~W Berry, Robin Kothari, Rolando~D Somma, and Nathan
  Wiebe.
\newblock Exponential quantum speedup in simulating coupled classical
  oscillators.
\newblock {\em arXiv preprint arXiv:2303.13012}, 2023.

\bibitem{BaezLauda2011}
John~C. Baez and Aaron~D. Lauda.
\newblock {\em A Prehistory of n-Categorical Physics}, page 13–128.
\newblock Cambridge University Press, 2011.

\bibitem{Banerjee2002}
Satanjeev Banerjee and Ted Pedersen.
\newblock An adapted lesk algorithm for word sense disambiguation using
  wordnet.
\newblock volume 2276, 2002.

\bibitem{BarHillel}
Yehoshua Bar-Hillel and Rudolf Carnap.
\newblock Semantic information.
\newblock {\em British Journal for the Philosophy of Science}, 4:147--157,
  1953.

\bibitem{SoaresBarbosa2014}
Rui~Soares Barbosa.
\newblock On monogamy of non-locality and macroscopic averages: examples and
  preliminary results.
\newblock {\em Electronic Proceedings in Theoretical Computer Science},
  172:36--55, dec 2014.

\bibitem{Basile2014}
Pierpaolo Basile, Annalina Caputo, and Giovanni Semeraro.
\newblock An enhanced lesk word sense disambiguation algorithm through a
  distributional semantic model.
\newblock 2014.

\bibitem{Bell1966}
John~S. Bell.
\newblock On the problem of hidden variables in quantum mechanics.
\newblock {\em Reviews of Modern Physics}, 38, 1966.

\bibitem{Bengio2003}
Yoshua Bengio, Réjean Ducharme, Pascal Vincent, and Christian Jauvin.
\newblock A neural probabilistic language model.
\newblock volume~3, 2003.

\bibitem{berry2007efficient}
Dominic~W Berry, Graeme Ahokas, Richard Cleve, and Barry~C Sanders.
\newblock Efficient quantum algorithms for simulating sparse hamiltonians.
\newblock {\em Communications in Mathematical Physics}, 270:359--371, 2007.

\bibitem{Bever}
Thomas~G Bever.
\newblock The cognitive basis for linguistic structures.
\newblock {\em Cognition and the development of language}, 1970.

\bibitem{Bohannon1986}
John~Neil Bohannon, Barbara Landau, and Lila Gleitman.
\newblock Language and experience: Evidence from the blind child.
\newblock {\em Language}, 62, 1986.

\bibitem{BohmI}
David Bohm.
\newblock A suggested interpretation of the quantum theory in terms of "hidden"
  variables. i.
\newblock {\em Phys. Rev.}, 85:166--179, Jan 1952.

\bibitem{BohmII}
David Bohm.
\newblock A suggested interpretation of the quantum theory in terms of "hidden"
  variables. ii.
\newblock {\em Phys. Rev.}, 85:180--193, Jan 1952.

\bibitem{Bohr}
Niels Bohr.
\newblock {On The Notions of Causality and Complementarity}.
\newblock In J{\o}rgen Kalckar, editor, {\em Foundations of Quantum Physics II
  (1933–1958)}, volume~7 of {\em Niels Bohr Collected Works}, pages 325--338.
  Elsevier, 1996.

\bibitem{Boser1992}
Bernhard~E. Boser, Isabelle~M. Guyon, and Vladimir~N. Vapnik.
\newblock Training algorithm for optimal margin classifiers.
\newblock 1992.

\bibitem{Bovi2015}
Claudio~Delli Bovi, Luca Telesca, and Roberto Navigli.
\newblock Large-scale information extraction from textual definitions through
  deep syntactic and semantic analysis.
\newblock {\em Transactions of the Association for Computational Linguistics},
  3, 2015.

\bibitem{Bruza2015}
Peter~D. Bruza, Kirsty Kitto, Brentyn~J. Ramm, and Laurianne Sitbon.
\newblock A probabilistic framework for analysing the compositionality of
  conceptual combinations.
\newblock {\em Journal of Mathematical Psychology}, 67, 2015.

\bibitem{cabello2010noncontextuality}
Adan Cabello, Simone Severini, and Andreas Winter.
\newblock (non-)contextuality of physical theories as an axiom, 2010.

\bibitem{Cabello1996}
Adán Cabello, José~M. Estebaranz, and Guillermo García-Alcaine.
\newblock Bell-kochen-specker theorem: A proof with 18 vectors.
\newblock {\em Physics Letters, Section A: General, Atomic and Solid State
  Physics}, 212, 1996.

\bibitem{cartan1950}
Henri Cartan.
\newblock Id{\'e}aux et modules de fonctions analytiques de variables
  complexes.
\newblock {\em Bulletin de la Soci{\'e}t{\'e} math{\'e}matique de France},
  78:29--64, 1950.

\bibitem{Cavalcanti2018}
Eric~G. Cavalcanti.
\newblock Classical causal models for bell and kochen-specker inequality
  violations require fine-tuning.
\newblock {\em Physical Review X}, 8, 2018.

\bibitem{Caves2002}
Carlton~M. Caves, Christopher~A. Fuchs, and Rüdiger Schack.
\newblock Quantum probabilities as bayesian probabilities.
\newblock {\em Physical Review A}, 65(2), January 2002.

\bibitem{choi2015depends}
Jinho~D Choi, Joel Tetreault, and Amanda Stent.
\newblock It depends: Dependency parser comparison using a web-based evaluation
  tool.
\newblock In {\em Proceedings of the 53rd Annual Meeting of the Association for
  Computational Linguistics and the 7th International Joint Conference on
  Natural Language Processing (Volume 1: Long Papers)}, pages 387--396, 2015.

\bibitem{Chomsky}
N.~{Chomsky}.
\newblock Three models for the description of language.
\newblock {\em IRE Transactions on Information Theory}, 2(3):113--124, Sep.
  1956.

\bibitem{Clauser1969}
John~F. Clauser, Michael~A. Horne, Abner Shimony, and Richard~A. Holt.
\newblock Proposed experiment to test local hidden-variable theories.
\newblock {\em Physical Review Letters}, 23, 1969.

\bibitem{Coecke2021}
Bob Coecke.
\newblock {\em The Mathematics of Text Structure}, volume~20.
\newblock 2021.

\bibitem{DisCoCat}
Bob Coecke, Mehrnoosh Sadrzadeh, and Stephen Clark.
\newblock Mathematical foundations for a compositional distributional model of
  meaning, 2010.

\bibitem{BNC}
BNC Consortium.
\newblock The british national corpus.
\newblock Oxford Text Archive, 2007.
\newblock XML Edition.

\bibitem{curry2015topological}
Justin Curry.
\newblock Topological data analysis and cosheaves, 2015.

\bibitem{Deligne}
Pierre Deligne.
\newblock La conjecture de {Weil} : {I}.
\newblock {\em Publications Math\'ematiques de l'IH\'ES}, 43:273--307, 1974.

\bibitem{BERT}
Jacob Devlin, Ming-Wei Chang, Kenton Lee, and Kristina Toutanova.
\newblock Bert: Pre-training of deep bidirectional transformers for language
  understanding, 2019.

\bibitem{Dopkins1992}
Stephen Dopkins, Robin~K. Morris, and Keith Rayner.
\newblock Lexical ambiguity and eye fixations in reading: A test of competing
  models of lexical ambiguity resolution.
\newblock {\em Journal of Memory and Language}, 31, 1992.

\bibitem{Duarte2018}
Cristhiano Duarte and Barbara Amaral.
\newblock Resource theory of contextuality for arbitrary prepare-and-measure
  experiments.
\newblock {\em Journal of Mathematical Physics}, 59, 2018.

\bibitem{Dzhafarov2016}
E.~N. Dzhafarov, Ru~Zhang, and Janne Kujala.
\newblock Is there contextuality in behavioural and social systems?
\newblock {\em Philosophical Transactions of the Royal Society A: Mathematical,
  Physical and Engineering Sciences}, 374, 2016.

\bibitem{Szhafarov2023sheaf}
Ehtibar~N. Dzhafarov.
\newblock The contextuality-by-default view of the sheaf-theoretic approach to
  contextuality, 2023.

\bibitem{Dzhafarov2017}
Ehtibar~N. Dzhafarov, Víctor~H. Cervantes, and Janne~V. Kujala.
\newblock Contextuality in canonical systems of random variables.
\newblock {\em Philosophical Transactions of the Royal Society A: Mathematical,
  Physical and Engineering Sciences}, 375, 2017.

\bibitem{Dzhafarov2016contextcontent}
Ehtibar~N. Dzhafarov and Janne~V. Kujala.
\newblock Context–content systems of random variables: The
  contextuality-by-default theory.
\newblock {\em Journal of Mathematical Psychology}, 74, 2016.

\bibitem{Dzhafarov2015}
Ehtibar~N. Dzhafarov, Janne~V. Kujala, and Jan Åke Larsson.
\newblock Contextuality in three types of quantum-mechanical systems.
\newblock {\em Foundations of Physics}, 45, 2015.

\bibitem{Edmonds2002}
Philip Edmonds.
\newblock Senseval: The evaluation of word sense disambiguation systems.
\newblock 2002.

\bibitem{Edmonds2001}
Philip Edmonds and Scott Cotton.
\newblock Sensev al-2: Overview.
\newblock 2001.

\bibitem{EhrlichRayner}
Susan~F. Ehrlich and Keith Rayner.
\newblock Contextual effects on word perception and eye movements during
  reading.
\newblock {\em Journal of Verbal Learning and Verbal Behavior}, 20(6):641--655,
  1981.

\bibitem{Eilenberg1945}
Samuel Eilenberg and Saunders MacLane.
\newblock General theory of natural equivalences.
\newblock {\em Transactions of the American Mathematical Society}, 58, 1945.

\bibitem{EPR}
A.~Einstein, B.~Podolsky, and N.~Rosen.
\newblock Can quantum-mechanical description of physical reality be considered
  complete?
\newblock {\em Phys. Rev.}, 47:777--780, May 1935.

\bibitem{QAOA}
Edward Farhi, Jeffrey Goldstone, and Sam Gutmann.
\newblock A quantum approximate optimization algorithm, 2014.

\bibitem{QAOAAdvantage}
Edward Farhi, Jeffrey Goldstone, and Sam Gutmann.
\newblock A quantum approximate optimization algorithm applied to a bounded
  occurrence constraint problem, 2015.

\bibitem{VQCClassification}
Edward Farhi and Hartmut Neven.
\newblock {Classification with Quantum Neural Networks on Near Term
  Processors}, 2018.

\bibitem{ukWaC}
Adriano Ferraresi, Eros Zanchetta, Marco Baroni, and Silvia Bernardini.
\newblock Introducing and evaluating ukwac , a very large web-derived corpus of
  english.
\newblock 2008.

\bibitem{Fine1982}
Arthur Fine.
\newblock Hidden variables, joint probability, and the bell inequalities.
\newblock {\em Phys. Rev. Lett.}, 48:291--295, Feb 1982.

\bibitem{Firth1957}
J.R. Firth.
\newblock A synopsis of linguistic theory 1930-55.
\newblock {\em Studies in Linguistic Analysis: Special Volume of the
  Philological Society}, 1957.

\bibitem{Flamini2020}
Fulvio Flamini, Arne Hamann, Sofiène Jerbi, Lea~M. Trenkwalder,
  Hendrik~Poulsen Nautrup, and Hans~J. Briegel.
\newblock Photonic architecture for reinforcement learning.
\newblock {\em New Journal of Physics}, 22, 2020.

\bibitem{fong2018seven}
Brendan Fong and David~I Spivak.
\newblock Seven sketches in compositionality: An invitation to applied category
  theory, 2018.

\bibitem{Frank2013}
Stefan~L. Frank.
\newblock Uncertainty reduction as a measure of cognitive load in sentence
  comprehension.
\newblock {\em Topics in Cognitive Science}, 5(3):475--494, 2013.

\bibitem{Frazier87}
Lyn Frazier.
\newblock {\em Sentence processing: A tutorial review.}, pages 559--586.
\newblock Attention and performance 12: The psychology of reading. Lawrence
  Erlbaum Associates, Inc, Hillsdale, NJ, US, 1987.

\bibitem{FrazierRayner1990}
Lyn Frazier and Keith Rayner.
\newblock Taking on semantic commitments: Processing multiple meanings vs.
  multiple senses.
\newblock {\em Journal of Memory and Language}, 29, 1990.

\bibitem{FREYDchoice}
Peter Freyd.
\newblock The axiom of choice.
\newblock {\em Journal of Pure and Applied Algebra}, 19:103--125, 1980.

\bibitem{underspecification}
Steven Frisson.
\newblock Semantic underspecification in language processing.
\newblock {\em Linguistics and Language Compass}, 3, 2009.

\bibitem{FrissonPickering1999}
Steven Frisson and Martin~J. Pickering.
\newblock The processing of metonymy: Evidence from eye movements.
\newblock {\em Journal of Experimental Psychology: Learning Memory and
  Cognition}, 25, 1999.

\bibitem{GARNSEY1997}
Susan~M. Garnsey, Neal~J. Pearlmutter, Elizabeth Myers, and Melanie~A. Lotocky.
\newblock The contributions of verb bias and plausibility to the comprehension
  of temporarily ambiguous sentences.
\newblock {\em Journal of Memory and Language}, 37(1):58--93, 1997.

\bibitem{Gentner1982}
D~Gentner.
\newblock Why nouns are learned before verbs: Linguistic relativity versus
  natural partitioning.
\newblock {\em Language development: Vol. 2. Language, thought, and culture},
  2, 1982.

\bibitem{Gentner1981}
Dedre Gentner.
\newblock Some interesting differences.
\newblock {\em Cognition and brain theory}, 4, 1981.

\bibitem{GentnerFrance2013}
Dedre Gentner and Ilene~M. France.
\newblock {\em The verb mutability effect: Studies of the combinatorial
  semantics of nouns and verbs}.
\newblock 2013.

\bibitem{gibson2000}
Edward Gibson and Neal~J Pearlmutter.
\newblock Distinguishing serial and parallel parsing.
\newblock {\em Journal of Psycholinguistic Research}, 29:231--240, 2000.

\bibitem{giustina2015significant}
Marissa Giustina, Marijn~AM Versteegh, S{\"o}ren Wengerowsky, Johannes
  Handsteiner, Armin Hochrainer, Kevin Phelan, Fabian Steinlechner, Johannes
  Kofler, Jan-{\AA}ke Larsson, Carlos Abell{\'a}n, et~al.
\newblock Significant-loophole-free test of bell’s theorem with entangled
  photons.
\newblock {\em Physical review letters}, 115(25):250401, 2015.

\bibitem{sheafcausality}
Stefano Gogioso and Nicola Pinzani.
\newblock {The Sheaf-Theoretic Structure of Definite Causality}.
\newblock {\em Electronic Proceedings in Theoretical Computer Science},
  343:301–324, Sep 2021.

\bibitem{sheafcausalityB}
Stefano Gogioso and Nicola Pinzani.
\newblock The geometry of causality, 2023.

\bibitem{topoi}
Robert Goldblatt.
\newblock {\em Topoi: The Categorial Analysis of Logic}.
\newblock Dover Publications, 1983.

\bibitem{gottesman1998}
Daniel Gottesman.
\newblock The heisenberg representation of quantum computers, 1998.

\bibitem{grodner}
Daniel Grodner, Edward Gibson, Vered Argaman, and Maria Babyonyshev.
\newblock Against repair-based reanalysis in sentence comprehension.
\newblock {\em Journal of Psycholinguistic Research}, 32:141--166, 2003.

\bibitem{Grothendieck1965}
Alexander Grothendieck.
\newblock Formule de lefschetz et rationalit{\'e} des fonctions \$l\$.
\newblock 1966.

\bibitem{Grothendieck1957}
Alexandre Grothendieck.
\newblock Sur quelques points d'alg{\`e}bre homologique.
\newblock {\em Tohoku Mathematical Journal, Second Series}, 9(2):119--183,
  1957.

\bibitem{Grover}
Lov~K. Grover.
\newblock A fast quantum mechanical algorithm for database search.
\newblock In {\em Proceedings of the Twenty-Eighth Annual ACM Symposium on
  Theory of Computing}, STOC '96, page 212–219, New York, NY, USA, 1996.
  Association for Computing Machinery.

\bibitem{Hale2001}
John Hale.
\newblock {A probabilistic Earley parser as a psycholinguistic model}.
\newblock In {\em Second meeting of the north american chapter of the
  association for computational linguistics}, 2001.

\bibitem{Hale2003}
John Hale.
\newblock The information conveyed by words in sentences.
\newblock {\em Journal of Psycholinguistic Research}, 32(2):101--123, Mar 2003.

\bibitem{Hale2006}
John Hale.
\newblock Uncertainty about the rest of the sentence.
\newblock {\em Cognitive Science}, 30(4):643--672, 2006.

\bibitem{Harris1954}
Zellig~S. Harris.
\newblock Distributional structure.
\newblock {\em WORD}, 10, 1954.

\bibitem{hensen2015loophole}
Bas Hensen, Hannes Bernien, Ana{\"\i}s~E Dr{\'e}au, Andreas Reiserer, Norbert
  Kalb, Machiel~S Blok, Just Ruitenberg, Raymond~FL Vermeulen, Raymond~N
  Schouten, Carlos Abell{\'a}n, et~al.
\newblock Loophole-free bell inequality violation using electron spins
  separated by 1.3 kilometres.
\newblock {\em Nature}, 526(7575):682--686, 2015.

\bibitem{Hochreiter1997}
Sepp Hochreiter and Jürgen Schmidhuber.
\newblock Long short-term memory.
\newblock {\em Neural Computation}, 9, 1997.

\bibitem{Holmes87}
V.~M. Holmes.
\newblock {\em Syntactic parsing: In search of the garden path.}, pages
  587--599.
\newblock Attention and performance 12: The psychology of reading. Lawrence
  Erlbaum Associates, Inc, Hillsdale, NJ, US, 1987.

\bibitem{Howard2014}
Mark Howard, Joel Wallman, Victor Veitch, and Joseph Emerson.
\newblock Contextuality supplies the 'magic' for quantum computation.
\newblock {\em Nature}, 510, 2014.

\bibitem{huang2023}
Kuan-Jung Huang, Suhas Arehalli, Mari Kugemoto, Christian Muxica, Grusha
  Prasad, Brian Dillon, and Tal Linzen.
\newblock Surprisal does not explain syntactic disambiguation difficulty:
  evidence from a large-scale benchmark.
\newblock 2023.

\bibitem{Iacobacci2016}
Ignacio Iacobacci, Mohammad~Taher Pilehvar, and Roberto Navigli.
\newblock Embeddings for word sense disambiguation: An evaluation study.
\newblock volume~2, 2016.

\bibitem{jacobs2019causal}
Bart Jacobs, Aleks Kissinger, and Fabio Zanasi.
\newblock Causal inference by string diagram surgery, 2019.

\bibitem{PTJelephant}
Peter~T Johnstone.
\newblock {\em Sketches of an Elephant: A Topos Theory Compendium}, volume~2.
\newblock Oxford University Press, 2002.

\bibitem{Jones2019}
Matt Jones.
\newblock Relating causal and probabilistic approaches to contextuality.
\newblock {\em Philosophical Transactions of the Royal Society A: Mathematical,
  Physical and Engineering Sciences}, 377, 2019.

\bibitem{Joos1950}
Martin Joos.
\newblock Description of language design.
\newblock {\em Journal of the Acoustical Society of America}, 22, 1950.

\bibitem{Jurafsky1996}
Daniel Jurafsky.
\newblock A probabilistic model of lexical and syntactic access and
  disambiguation.
\newblock {\em Cognitive Science}, 20(2):137--194, 1996.

\bibitem{jurafskyspeech}
Daniel Jurafsky and James~H Martin.
\newblock Speech and language processing: An introduction to natural language
  processing, computational linguistics, and speech recognition.

\bibitem{Kilgarriff2000}
A.~Kilgarriff and J.~Rosenzweig.
\newblock Framework and results for english senseval.
\newblock {\em Language Resources and Evaluation}, 34, 2000.

\bibitem{KissingerHobanCoecke}
Aleks Kissinger, Matty Hoban, and Bob Coecke.
\newblock Equivalence of relativistic causal structure and process terminality,
  2017.

\bibitem{Kissinger2019}
Aleks Kissinger and Sander Uijlen.
\newblock A categorical semantics for causal structure.
\newblock {\em Logical Methods in Computer Science}, 15, 2019.

\bibitem{KCBS}
Alexander~A. Klyachko, M.~Ali Can, Sinem Binicioğlu, and Alexander~S.
  Shumovsky.
\newblock Simple test for hidden variables in spin-1 systems.
\newblock {\em Physical Review Letters}, 101(2), July 2008.

\bibitem{Kochen1967}
Simon Kochen and E.~Specker.
\newblock The problem of hidden variables in quantum mechanics.
\newblock {\em Indiana University Mathematics Journal}, 17, 1967.

\bibitem{Kripke1965}
Saul~A. Kripke.
\newblock Semantical analysis of intuitionistic logic i.
\newblock In J.N. Crossley and M.A.E. Dummett, editors, {\em Formal Systems and
  Recursive Functions}, volume~40 of {\em Studies in Logic and the Foundations
  of Mathematics}, pages 92--130. Elsevier, 1965.

\bibitem{Kujala2016}
Janne~V. Kujala and Ehtibar~N. Dzhafarov.
\newblock Proof of a conjecture on contextuality in cyclic systems with binary
  variables.
\newblock {\em Foundations of Physics}, 46, 2016.

\bibitem{Kujala2019}
Janne~V. Kujala and Ehtibar~N. Dzhafarov.
\newblock Measures of contextuality and non-contextuality.
\newblock {\em Philosophical Transactions of the Royal Society A: Mathematical,
  Physical and Engineering Sciences}, 377, 2019.

\bibitem{Kujala2015}
Janne~V. Kujala, Ehtibar~N. Dzhafarov, and Jan Åke Larsson.
\newblock Necessary and sufficient conditions for an extended noncontextuality
  in a broad class of quantum mechanical systems.
\newblock {\em Physical Review Letters}, 115, 2015.

\bibitem{Lambek}
Joachim Lambek.
\newblock The mathematics of sentence structure.
\newblock {\em The American Mathematical Monthly}, 65(3):154--170, 1958.

\bibitem{lambek1980lambda}
Joachim Lambek.
\newblock From lambda-calculus to cartesian closed categories.
\newblock {\em To HB Curry: essays on combinatory logic, lambda calculus and
  formalism}, pages 375--402, 1980.

\bibitem{LambekScott}
Joachim Lambek and Philip~J Scott.
\newblock {\em Introduction to higher-order categorical logic}, volume~7.
\newblock Cambridge University Press, 1988.

\bibitem{Lawvere1964}
F~William Lawvere.
\newblock An elementary theory of the category of sets.
\newblock {\em Proceedings of the national academy of sciences},
  52(6):1506--1511, 1964.

\bibitem{lawvere1970quantifiers}
F~William Lawvere.
\newblock Quantifiers and sheaves.
\newblock In {\em Actes du congres international des mathematiciens, Nice},
  volume~1, pages 329--334, 1970.

\bibitem{Lee2002}
Yoong~Keok Lee and Hwee~Tou Ng.
\newblock An empirical evaluation of knowledge sources and learning algorithms
  for word sense disambiguation.
\newblock 2002.

\bibitem{Leray1945}
Jean Leray.
\newblock Sur la forme des espaces topologiques et sur les points fixes des
  repr{\'e}sentations.
\newblock {\em Journal de Math{\'e}matiques Pures et Appliqu{\'e}es},
  24:95--167, 1945.

\bibitem{Lesk1986}
Michael Lesk.
\newblock Automatic sense disambiguation using machine readable dictionaries.
\newblock 1986.

\bibitem{Lo2022}
Kin~Ian Lo, Mehrnoosh Sadrzadeh, and Shane Mansfield.
\newblock A model of anaphoric ambiguities using sheaf theoretic quantum-like
  contextuality and bert.
\newblock volume 366, 2022.

\bibitem{Lo2023}
Kin~Ian Lo, Mehrnoosh Sadrzadeh, and Shane Mansfield.
\newblock Generalised winograd schema and its contextuality.
\newblock {\em Electronic Proceedings in Theoretical Computer Science}, 384,
  2023.

\bibitem{QNLPinpractice}
Robin Lorenz, Anna Pearson, Konstantinos Meichanetzidis, Dimitri Kartsaklis,
  and Bob Coecke.
\newblock {QNLP in Practice: Running Compositional Models of Meaning on a
  Quantum Computer}, 2021.

\bibitem{lorenz2023causal}
Robin Lorenz and Sean Tull.
\newblock Causal models in string diagrams, 2023.

\bibitem{Loureiro2020}
Daniel Loureiro and Alípio~Mário Jorge.
\newblock Language modelling makes sense: Propagating representations through
  wordnet for full-coverage word sense disambiguation.
\newblock 2020.

\bibitem{Luo2018a}
Fuli Luo, Tianyu Liu, Zexue He, Qiaolin Xia, Zhifang Sui, and Baobao Chang.
\newblock Leveraging gloss knowledge in neural word sense disambiguation by
  hierarchical co-attention.
\newblock 2018.

\bibitem{Luo2018b}
Fuli Luo, Tianyu Liu, Qiaolin Xia, Baobao Chang, and Zhifang Sui.
\newblock Incorporating glosses into neural word sense disambiguation.
\newblock volume~1, 2018.

\bibitem{CategoriesWorkingMathematician}
Saunders Mac~Lane.
\newblock {\em Categories for the working mathematician}, volume~5.
\newblock Springer Science \& Business Media, 2013.

\bibitem{maclane2012sheaves}
Saunders MacLane and Ieke Moerdijk.
\newblock {\em Sheaves in geometry and logic: A first introduction to topos
  theory}.
\newblock Springer Science \& Business Media, 2012.

\bibitem{Madsen2022}
Lars~S. Madsen, Fabian Laudenbach, Mohsen~Falamarzi Askarani, Fabien Rortais,
  Trevor Vincent, Jacob~F.F. Bulmer, Filippo~M. Miatto, Leonhard Neuhaus,
  Lukas~G. Helt, Matthew~J. Collins, Adriana~E. Lita, Thomas Gerrits, Sae~Woo
  Nam, Varun~D. Vaidya, Matteo Menotti, Ish Dhand, Zachary Vernon, Nicolás
  Quesada, and Jonathan Lavoie.
\newblock Quantum computational advantage with a programmable photonic
  processor.
\newblock {\em Nature}, 606, 2022.

\bibitem{MansfieldSequential}
Shane Mansfield and Elham Kashefi.
\newblock Quantum advantage from sequential-transformation contextuality.
\newblock {\em Physical Review Letters}, 121(23), dec 2018.

\bibitem{Melamud2016}
Oren Melamud, Jacob Goldberger, and Ido Dagan.
\newblock context2vec: Learning generic context embedding with bidirectional
  lstm.
\newblock 2016.

\bibitem{Mermin1990}
N.~David Mermin.
\newblock Simple unified form for the major no-hidden-variables theorems.
\newblock {\em Physical Review Letters}, 65, 1990.

\bibitem{Mikolov2013}
Tomas Mikolov, Kai Chen, Greg Corrado, and Jeffrey Dean.
\newblock Efficient estimation of word representations in vector space.
\newblock 2013.

\bibitem{WordNet}
George~A. Miller.
\newblock Wordnet: A lexical database for english.
\newblock {\em Commun. ACM}, 38(11):39–41, nov 1995.

\bibitem{WordNetB}
George~A Miller, Richard Beckwith, Christiane Fellbaum, Derek Gross, and
  Katherine~J Miller.
\newblock Introduction to wordnet: An on-line lexical database.
\newblock {\em International journal of lexicography}, 3(4):235--244, 1990.

\bibitem{Miller1993}
George~A. Miller, Claudia Leacock, Randee Tengi, and Ross~T. Bunker.
\newblock A semantic concordance.
\newblock 1993.

\bibitem{VQCcircuitlearning}
K.~Mitarai, M.~Negoro, M.~Kitagawa, and K.~Fujii.
\newblock Quantum circuit learning.
\newblock {\em Phys. Rev. A}, 98:032309, Sep 2018.

\bibitem{Mitchell87}
D.~C. Mitchell.
\newblock {\em Lexical guidance in human parsing: Locus and processing
  characteristics.}, pages 601--618.
\newblock Attention and performance 12: The psychology of reading. Lawrence
  Erlbaum Associates, Inc, Hillsdale, NJ, US, 1987.

\bibitem{Mullaly}
Allison Mullaly, Christina Gagné, Thomas Spalding, and Kristan Marchak.
\newblock Examining ambiguous adjectives in adjective-noun phrases: Evidence
  for representation as a shared core-meaning with sense specialization.
\newblock {\em The Mental Lexicon}, 5:87--114, 06 2010.

\bibitem{VandeNest2013}
Maarten Van~Den Nest.
\newblock Universal quantum computation with little entanglement.
\newblock {\em Physical Review Letters}, 110, 2013.

\bibitem{nielsen2010quantum}
Michael~A Nielsen and Isaac~L Chuang.
\newblock {\em Quantum computation and quantum information}.
\newblock Cambridge university press, 2010.

\bibitem{Pearl2011}
Judea Pearl.
\newblock {\em Causality: Models, reasoning, and inference, second edition}.
\newblock 2011.

\bibitem{Peres1991}
A.~Peres.
\newblock Two simple proofs of the kochen-specker theorem.
\newblock {\em Journal of Physics A: General Physics}, 24, 1991.

\bibitem{entanglementMixed}
Asher Peres.
\newblock Separability criterion for density matrices.
\newblock {\em Phys. Rev. Lett.}, 77:1413--1415, Aug 1996.

\bibitem{Peters2018}
Matthew~E. Peters, Mark Neumann, Mohit Iyyer, Matt Gardner, Christopher Clark,
  Kenton Lee, and Luke Zettlemoyer.
\newblock Deep contextualized word representations.
\newblock volume~1, 2018.

\bibitem{PickeringFrisson2001}
Martin~J. Pickering and Steven Frisson.
\newblock Processing ambiguous verbs: Evidence from eye movements.
\newblock {\em Journal of Experimental Psychology: Learning Memory and
  Cognition}, 27, 2001.

\bibitem{pickering1998}
Martin~J Pickering and Matthew~J Traxler.
\newblock Plausibility and recovery from garden paths: An eye-tracking study.
\newblock {\em Journal of Experimental Psychology: Learning, Memory, and
  Cognition}, 24(4):940, 1998.

\bibitem{Piedeleu2015}
Robin Piedeleu, Dimitri Kartsaklis, Bob Coecke, and Mehrnoosh Sadrzadeh.
\newblock Open system categorical quantum semantics in natural language
  processing.
\newblock volume~35, 2015.

\bibitem{PiedeleuZanasi}
Robin Piedeleu and Fabio Zanasi.
\newblock An introduction to string diagrams for computer scientists, 2023.

\bibitem{Entanglement}
Martin~B. Plenio and Vlatko Vedral.
\newblock Teleportation, entanglement and thermodynamics in the quantum world.
\newblock {\em Contemporary Physics}, 39(6):431--446, nov 1998.

\bibitem{prasad}
Grusha Prasad and Tal Linzen.
\newblock Rapid syntactic adaptation in self-paced reading: Detectable, but
  only with many participants.
\newblock {\em Journal of Experimental Psychology: Learning, Memory, and
  Cognition}, 47(7):1156, 2021.

\bibitem{pritchett1992}
Bradley~L Pritchett.
\newblock {\em Grammatical competence and parsing performance}.
\newblock University of Chicago Press, 1992.

\bibitem{Qiskit}
{Qiskit contributors}.
\newblock Qiskit: An open-source framework for quantum computing, 2023.

\bibitem{Ramakrishnan2003}
Ganesh Ramakrishnan, Apurva Jadhav, Ashutosh Joshi, Soumen Chakrabarti, and
  Pushpak Bhattacharyya.
\newblock Question answering via bayesian inference on lexical relations.
\newblock 2003.

\bibitem{Rayner1977}
Keith Rayner.
\newblock Visual attention in reading: Eye movements reflect cognitive
  processes.
\newblock {\em Memory \& Cognition}, 5, 1977.

\bibitem{RaynerDuffy1986}
Keith Rayner and Susan~A. Duffy.
\newblock Lexical complexity and fixation times in reading: Effects of word
  frequency, verb complexity, and lexical ambiguity.
\newblock {\em Memory \& Cognition}, 14, 1986.

\bibitem{Rebentrost2014}
Patrick Rebentrost, Masoud Mohseni, and Seth Lloyd.
\newblock Quantum support vector machine for big data classification.
\newblock {\em Physical Review Letters}, 113, 2014.

\bibitem{RobinsonDependency}
Jane~J. Robinson.
\newblock Dependency structures and transformational rules.
\newblock {\em Language}, 46(2):259--285, 1970.

\bibitem{rosenfeld2017event}
Wenjamin Rosenfeld, Daniel Burchardt, Robert Garthoff, Kai Redeker, Norbert
  Ortegel, Markus Rau, and Harald Weinfurter.
\newblock Event-ready bell test using entangled atoms simultaneously closing
  detection and locality loopholes.
\newblock {\em Physical review letters}, 119(1):010402, 2017.

\bibitem{VQCQclassifier}
Maria Schuld, Alex Bocharov, Krysta~M. Svore, and Nathan Wiebe.
\newblock Circuit-centric quantum classifiers.
\newblock {\em Physical Review A}, 101(3), mar 2020.

\bibitem{Schutze1998}
Hinrich Sch\"{u}tze.
\newblock Automatic word sense discrimination.
\newblock {\em Computational Linguistics}, 24, 1998.

\bibitem{Selinger2007}
Peter Selinger.
\newblock Dagger compact closed categories and completely positive maps.
\newblock {\em Electronic Notes in Theoretical Computer Science}, 170:139--163,
  03 2007.

\bibitem{serre1955faisceaux}
Jean-Pierre Serre.
\newblock Faisceaux alg{\'e}briques coh{\'e}rents.
\newblock {\em Annals of Mathematics}, pages 197--278, 1955.

\bibitem{shalm2015strong}
Lynden~K Shalm, Evan Meyer-Scott, Bradley~G Christensen, Peter Bierhorst,
  Michael~A Wayne, Martin~J Stevens, Thomas Gerrits, Scott Glancy, Deny~R
  Hamel, Michael~S Allman, et~al.
\newblock Strong loophole-free test of local realism.
\newblock {\em Physical review letters}, 115(25):250402, 2015.

\bibitem{Shannon}
C.~E. Shannon.
\newblock A mathematical theory of communication.
\newblock {\em The Bell System Technical Journal}, 27(3):379--423, 1948.

\bibitem{shor1999polynomial}
Peter~W Shor.
\newblock Polynomial-time algorithms for prime factorization and discrete
  logarithms on a quantum computer.
\newblock {\em SIAM review}, 41(2):303--332, 1999.

\bibitem{Shutova2010}
Ekaterina Shutova.
\newblock Automatic metaphor interpretation as a paraphrasing task.
\newblock 2010.

\bibitem{smith2013effect}
Nathaniel~J Smith and Roger Levy.
\newblock The effect of word predictability on reading time is logarithmic.
\newblock {\em Cognition}, 128(3):302--319, 2013.

\bibitem{Steane2003}
A.~M. Steane.
\newblock A quantum computer only needs one universe.
\newblock {\em Studies in History and Philosophy of Science Part B - Studies in
  History and Philosophy of Modern Physics}, 34, 2003.

\bibitem{SteedmanCCG}
Mark Steedman.
\newblock {\em Combinators and Grammars}, pages 417--442.
\newblock Springer Netherlands, Dordrecht, 1988.

\bibitem{Stinespring1955}
W.~Forrest Stinespring.
\newblock Positive functions on $c^*$ -algebras.
\newblock {\em Proceedings of the American Mathematical Society}, 6, 1955.

\bibitem{Storz2023}
Simon Storz, Josua Sch{\"a}r, Anatoly Kulikov, Paul Magnard, Philipp Kurpiers,
  Janis L{\"u}tolf, Theo Walter, Adrian Copetudo, Kevin Reuer, Abdulkadir Akin,
  Jean-Claude Besse, Mihai Gabureac, Graham~J. Norris, Andr{\'e}s Rosario,
  Ferran Martin, Jos{\'e} Martinez, Waldimar Amaya, Morgan~W. Mitchell, Carlos
  Abellan, Jean-Daniel Bancal, Nicolas Sangouard, Baptiste Royer, Alexandre
  Blais, and Andreas Wallraff.
\newblock Loophole-free bell inequality violation with superconducting
  circuits.
\newblock {\em Nature}, 617(7960):265--270, May 2023.

\bibitem{SturtPick}
Patrick Sturt, Martin~J Pickering, and Matthew~W Crocker.
\newblock Structural change and reanalysis difficulty in language
  comprehension.
\newblock {\em Journal of Memory and Language}, 40(1):136--150, 1999.

\bibitem{Tanenhaus1979}
Michael~K. Tanenhaus, James~M. Leiman, and Mark~S. Seidenberg.
\newblock Evidence for multiple stages in the processing of ambiguous words in
  syntactic contexts.
\newblock {\em Journal of Verbal Learning and Verbal Behavior}, 18, 1979.

\bibitem{taylorCloze}
Wilson~L Taylor.
\newblock ``cloze procedure'': A new tool for measuring readability.
\newblock {\em Journalism quarterly}, 30(4):415--433, 1953.

\bibitem{Tierney2011}
M.~Tierney.
\newblock {\em Axiomatic Sheaf Theory : Some Constructions and Applications},
  pages 249--326.
\newblock Springer Berlin Heidelberg, Berlin, Heidelberg, 2011.

\bibitem{TierneyContinuum}
Myles Tierney.
\newblock Sheaf theory and the continuum hypothesis.
\newblock In F.~W. Lawvere, editor, {\em Toposes, Algebraic Geometry and
  Logic}, pages 13--42, Berlin, Heidelberg, 1972. Springer Berlin Heidelberg.

\bibitem{Tomasello1997}
Michael Tomasello, Nameera Akhtar, , Kelly Dodson, and Laura Rekau.
\newblock Differential productivity in young children's use of nouns and verbs.
\newblock {\em Journal of Child Language}, 24, 1997.

\bibitem{Traxeler1998}
Matthew~J. Traxler, Martin~J. Pickering, and Charles Clifton.
\newblock Adjunct attachment is not a form of lexical ambiguity resolution.
\newblock {\em Journal of Memory and Language}, 39(4):558--592, 1998.

\bibitem{TrueswellTanenhaus2007}
John~C. Trueswell and Michael~K. Tanenhaus.
\newblock Tense, temporal context and syntactic ambiguity resolution.
\newblock {\em Language and Cognitive Processes}, 6(4):303--338, 1991.

\bibitem{trueswell1993}
John~C Trueswell, Michael~K Tanenhaus, and Christopher Kello.
\newblock Verb-specific constraints in sentence processing: separating effects
  of lexical preference from garden-paths.
\newblock {\em Journal of Experimental psychology: Learning, memory, and
  Cognition}, 19(3):528, 1993.

\bibitem{Trugenberger2002}
Carlo~A. Trugenberger.
\newblock Quantum pattern recognition.
\newblock {\em Quantum Information Processing}, 1(6):471--493, Dec 2002.

\bibitem{Emeriau2022}
Kim Vall\'ee, Pierre-Emmanuel Emeriau, Boris Bourdoncle, Adel Sohbi, Shane
  Mansfield, and Damian Markham.
\newblock Corrected bell and non-contextuality inequalities for realistic
  experiments.
\newblock {\em Philosophical Transactions of the Royal Society A: Mathematical,
  Physical and Engineering Sciences}, 382(2268):20230011, 2024.

\bibitem{VanGompel2005}
Roger~P.G. {van Gompel}, Martin~J. Pickering, Jamie Pearson, and Simon~P.
  Liversedge.
\newblock Evidence against competition during syntactic ambiguity resolution.
\newblock {\em Journal of Memory and Language}, 52(2):284--307, 2005.

\bibitem{vanSchijndelA}
Marten Van~Schijndel and Tal Linzen.
\newblock Modeling garden path effects without explicit hierarchical syntax.
\newblock In {\em CogSci}, 2018.

\bibitem{vanSchijndelB}
Marten van Schijndel and Tal Linzen.
\newblock Single-stage prediction models do not explain the magnitude of
  syntactic disambiguation difficulty.
\newblock {\em Cognitive Science}, 45(6):e12988, 2021.

\bibitem{attention}
Ashish Vaswani, Noam Shazeer, Niki Parmar, Jakob Uszkoreit, Llion Jones,
  Aidan~N Gomez, \L~ukasz Kaiser, and Illia Polosukhin.
\newblock Attention is all you need.
\newblock In I.~Guyon, U.~Von Luxburg, S.~Bengio, H.~Wallach, R.~Fergus,
  S.~Vishwanathan, and R.~Garnett, editors, {\em Advances in Neural Information
  Processing Systems}, volume~30. Curran Associates, Inc., 2017.

\bibitem{entanglement1}
V.~Vedral, M.~B. Plenio, K.~Jacobs, and P.~L. Knight.
\newblock {Statistical inference, distinguishability of quantum states, and
  quantum entanglement}.
\newblock {\em Phys. Rev. A}, 56:4452--4455, Dec 1997.

\bibitem{Vorobev}
N.~N. Vorob'ev.
\newblock Consistent families of measures and their extensions.
\newblock {\em Theory of Probability \& Its Applications}, 7(2):147--163, 1962.

\bibitem{Wagner2023}
Rafael Wagner, Roberto~D. Baldijão, Alisson Tezzin, and Bárbara Amaral.
\newblock Using a resource theoretic perspective to witness and engineer
  quantum generalized contextuality for prepare-and-measure scenarios, 2023.

\bibitem{corpusDataset}
Daphne Wang.
\newblock The corpus dataset.
\newblock \url{https://github.com/wangdaphne/Cyclic_models}, 2022.

\bibitem{datasetGardenPath}
Daphne Wang.
\newblock Empirical models and signalling fractions of garden-path sentences.
\newblock \url{https://github.com/wangdaphne/garden-path-SF-dataset}, 2023.

\bibitem{WangGardenPath}
Daphne Wang and Mehrnoosh Sadrzadeh.
\newblock Causality and signalling of garden-path sentences.
\newblock {\em Philosophical Transactions of the Royal Society A: Mathematical,
  Physical and Engineering Sciences}, 382(2268):20230013, 2024.

\bibitem{Wang2021a}
Daphne Wang, Mehrnoosh Sadrzadeh, Samson Abramsky, and Víctor~H. Cervantes.
\newblock On the quantum-like contextuality of ambiguous phrases.
\newblock 2021.

\bibitem{Wang2021b}
Daphne Wang, Mehrnoosh Sadrzadeh, Samson Abramsky, H.~Víctor, and Cervantes.
\newblock Analysing ambiguous nouns and verbs with quantum contextuality tools.
\newblock {\em Journal of Cognitive Science}, 22, 2021.

\bibitem{Wang2013}
Zheng Wang and Jerome~R. Busemeyer.
\newblock A quantum question order model supported by empirical tests of an a
  priori and precise prediction.
\newblock {\em Topics in Cognitive Science}, 5, 2013.

\bibitem{weaver1952translation}
Warren Weaver.
\newblock Translation.
\newblock In {\em Proceedings of the Conference on Mechanical Translation},
  1952.

\bibitem{WilsonGhicaZanasi}
Paul Wilson, Dan Ghica, and Fabio Zanasi.
\newblock String diagrams for non-strict monoidal categories, 2022.

\bibitem{EntanglementConcurrence}
William~K. Wootters.
\newblock Entanglement of formation and concurrence.
\newblock {\em Quantum Info. Comput.}, 1(1):27–44, jan 2001.

\bibitem{Vandenbussche2021}
Pierre yves Vandenbussche, Tony Scerri, and Ron~Daniel Jr.
\newblock Word sense disambiguation with transformer models.
\newblock {\em Proceedings of the 6th Workshop on Semantic Deep Learning
  (SemDeep-6)}, 2021.

\bibitem{Zeng2016}
William Zeng and Bob Coecke.
\newblock Quantum algorithms for compositional natural language processing.
\newblock volume 221, 2016.

\bibitem{zhong2020quantum}
Han-Sen Zhong, Hui Wang, Yu-Hao Deng, Ming-Cheng Chen, Li-Chao Peng, Yi-Han
  Luo, Jian Qin, Dian Wu, Xing Ding, Yi~Hu, et~al.
\newblock Quantum computational advantage using photons.
\newblock {\em Science}, 370(6523):1460--1463, 2020.

\bibitem{Zhong2010}
Zhi Zhong and Hwee~Tou Ng.
\newblock It makes sense: A wide-coverage word sense disambiguation system for
  free text.
\newblock 2010.

\bibitem{VQCgenerativemodelling1}
D.~Zhu, N.~M. Linke, M.~Benedetti, K.~A. Landsman, N.~H. Nguyen, C.~H.
  Alderete, A.~Perdomo-Ortiz, N.~Korda, A.~Garfoot, C.~Brecque, L.~Egan,
  O.~Perdomo, and C.~Monroe.
\newblock Training of quantum circuits on a hybrid quantum computer.
\newblock {\em Science Advances}, 5(10):eaaw9918, 2019.

\end{thebibliography}

\appendix
\part{Appendix}
\chapter{Sections of a presheaf and the sheafication of a presheaf}\label{app:sections}
\paragraph{Bundles} A \emph{bundle} is an alternative presentation of families of sets $\left\{X_i\right\}_{i\in I}$ as a map $p: E\to I$, where $E = \bigsqcup_{i\in I} X_i = \left\{(i, x)|x\in X_i, i\in I\right\}$, and each of the $X_i$ is a set. The map $p$ in this case is simply defined as:
\begin{equation}
    p:: (i, x) \mapsto i
\end{equation}
The space $E$ is refered to as the \emph{espace \'{e}tal\'{e}} or the \emph{total space}, while $I$ is refered to as the \emph{base space} of the bundle. The equivalence of between families of sets and the bundle can be seen as each of the $X_i$ can be retrieved from the map $p$ as:
\begin{equation}
    X_i = p^{-1}(i) 
\end{equation}
It is usually said that the set $X_i$ ``sits on top'' of the point $i$ in the base space. The sets $X_i$ are called the \emph{stalks} or \emph{fibres} of $p$ at $i\in I$ and each of the elements of a given $X_i$ are called the \emph{germs} for the stalk $X_i$. This terminology comes from a vegetal analogy where, stalks of a plant, e.g. say wheat, grows up from the soil (here the base space), each each of the stalk consists germs (see Fig.~\ref{fig:bundle}). We will also define a \emph{section} of the bundle $p$ as a map $s: I\to E$ s.t. $p\circ s = id_{I}$; the idea is that sections will select a single germ in each of the fibres of $p$.

\begin{figure}[htb!]
    \centering
    \includegraphics[width=.5\linewidth]{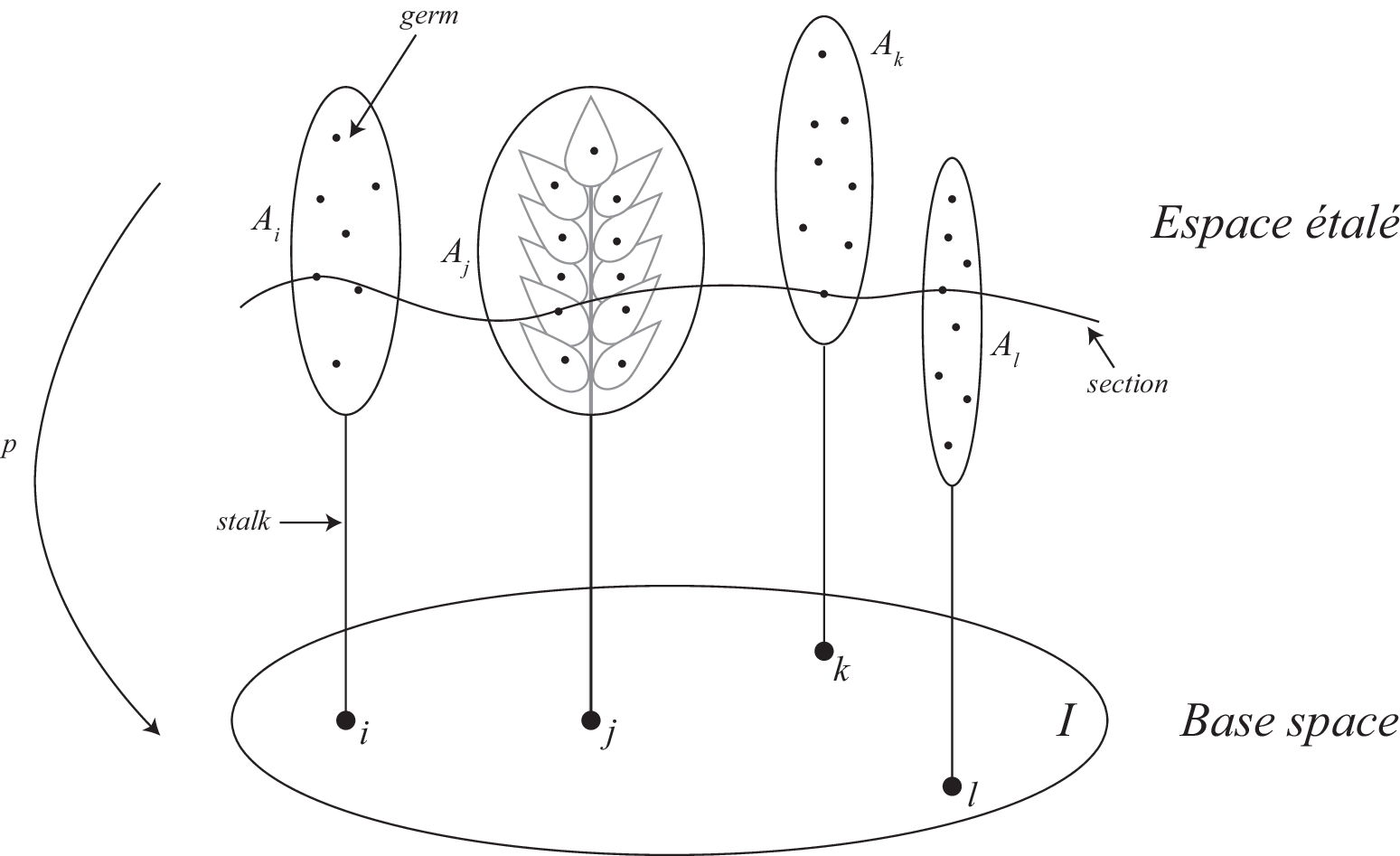}
    \caption{Illustration of a bundle.\label{fig:bundle}}
\end{figure}

Furthermore, for every function $p: E\to I$ (for any sets $E$ and $I$) defines a family of sets $\left\{X_i\right\}_{i\in I}$ where the $X_i$'s are obtained exactly as $X_i = p^{-1}(i)$. 

So far, no structure is assumed on the base and \'etal\'e space, and notably all of the points of $I$ are considered as unrelated. Let us now consider the case when $E$ and $I$ are topological spaces. Then, if $p: E\to I$ is a continuous map, then the associated bundle will have some nice properties as well. In particular, the continuity condition (with respect to the topology of $E$) implies that stalks are ``glued together'' in the sense that any two open neighbourhoods of a point $x\in E$ will be mapped to (open) sets in $I$ which contains the point $p(x)$.
The notion of section is then extend as follows. Given a continuous map $p: E\to I$, we define a section of the open set $U\subseteq I$ as a continuous map $s: U\to E$ s.t. the following pullback square commutes in $\mathbf{Top}$.

\begin{equation}
    \begin{tikzpicture}
        \node (preim) at (-1, 1){$p^{-1}U$};
        \node (E) at (1, 1){$E$};
        \node (I) at (1, -1){$I$};
        \node (U) at (-1, -1){$U$};
        \node (corner) at (preim.south east){$\lrcorner$};
        \draw [right hook ->] (U) to (I);
        \draw [right hook ->] (preim) to (E);
        \draw [->] (E) to node[right]{$p$} (I);
        \draw [->] (preim) to node[left]{$p_U$} (U);
        \draw [->, dashed] (U) to node[above]{$s$} (E); 
    \end{tikzpicture}
\end{equation}

In turn, these sections give rise to the \emph{sheaf of sections} associated each any continuous bundle $p: E\to I$, as the functor:
\begin{equation}
    \begin{matrix}
        \Gamma_p: &\mathcal{T}(I)^{op} & \to & \mathbf{Sets}\\
        & U & \mapsto & \left\{s: U\to E \middle| U\xrightarrow{s} E \xrightarrow{p} I = U \hookrightarrow X\right\} 
    \end{matrix}
\end{equation}
And the morphisms are defined as restriction morhpisms as described in Section~\ref{subsec:sheaves}. By continuity of $p$, this will indeed satisfy the sheaf condition. In addition, given a continuous bundle $p$, the sections over $U$ are indeed the elements of $\Gamma_p$.


\paragraph{Sheafication} We have seen that the sections of a bundle indeed correspond to the elements of the images of a certain (pre)sheaf associated with the bundle. We now try to go in the reverse direction, namely, the elements of $PU$ for an arbitrary $P: \mathcal{T}(X)^{op}\to \mathbf{Sets}$ will corresponds to the sections of a bundle $p: E\to X$. As a bonus, the construction of $p$ also gives us a construction of a sheaf from an arbitrary presheaf, which satisfy a universal property; this construction is therefore known as \emph{sheafication}.

We start by defining the \emph{germs} at a point $x\in U$ of an element $s\in PU$ as the following set:
\begin{align}
    germ_x s = &\left\{t\in PV| V\text{ open neighbourhood of }x \right.\nonumber\\
    &\left.\wedge \exists W \text{ open neighbourhood of } x. \left.s\right|_W = \left.t\right|_{W}\right\}
\end{align} 

We can then define the \emph{stalks} of a presheaf $P$ at $x\in X$ as:
\begin{equation}
    P_x = \left\{germ_x s \middle| \exists U \text{ open neighbourhood of } x. s\in PU\right\}
\end{equation}

Then, by defining the set:
\begin{equation}
    \Lambda_P = \bigsqcup_{x\in X} P_x
\end{equation}
we can define the following bundle:
\begin{equation}
    \begin{matrix}
        p: & \Lambda_P &\to& X\\
        & (x, germ_x s) &\mapsto& x 
    \end{matrix}
\end{equation}

Furthermore, for any $s\in PU$, we can define the following map:
\begin{equation}
    \begin{matrix}
        \tilde{s}: & U & \to & \Lambda_P\\
        & x &\mapsto& (x, germ_x s)
    \end{matrix}
\end{equation}
and it is not hard to verify that this is indeed a section of the bundle $p$. This bundle will, in turn, give rise to a sheaf of sections. This completes the sheafication process.
\chapter{Proof of the CHSH inequality}\label{app:CHSH}
In order to prove \eqref{eq:CHSH}, we start by obtaining a bound for the quantity $\left|\left<a'b\right> - \left<a'b'\right>\right|$. Since the hidden variable model should give back the observed probability distributions we have:
\begin{equation}
    \left<a'b\right> = \int_{\Lambda}d\lambda p(\lambda) A(a', \lambda) B(b,\lambda)
\end{equation}
where $A: \{a,a'\}\times \Lambda\to \{\pm1\}$ and $B: \{b,b'\}\times\Lambda\to \{\pm1\}$ are function associating a pair of input and hidden-variable with the deterministic outcome this environment gives out. Similarly, we have:
\begin{equation}
    \left<a'b'\right> = \int_{\Lambda}d\lambda p(\lambda) A(a', \lambda) B(b',\lambda)
\end{equation}
and:
\begin{equation}
    \left|\left<a'b\right> - \left<a'b'\right>\right| = \left|\int_{\Lambda}d\lambda p(\lambda) \left(A(a', \lambda) B(b,\lambda) - A(a', \lambda) B(b',\lambda)\right)\right|
\end{equation}
Now, since the hidden variables determine the values of $a, a', b, b'$ simultaneously, there is nothing stopping us from writing:
\begin{align}
    \left|\left<a'b\right> - \left<a'b'\right>\right| =& \bigg|\int_{\Lambda}d\lambda p(\lambda) \left(A(a', \lambda) B(b,\lambda) - A(a', \lambda) B(b',\lambda) \right.\nonumber\\
    &\left.A(a, \lambda) B(b,\lambda)A(a', \lambda) B(b',\lambda) - A(a, \lambda) B(b,\lambda)A(a', \lambda) B(b',\lambda)\right) \bigg|\\
    =& \bigg|\int_{\Lambda}d\lambda p(\lambda)A(a',\lambda)B(b,\lambda) \left(1+A(a,\lambda)B(b',\lambda)\right) \nonumber\\
    &- \int_\Lambda d\lambda p(\lambda)A(a',\lambda)B(b',\lambda) \left(1+A(a,\lambda)B(b,\lambda)\right)\bigg|\label{eq:CHSHproofA}\\
    =& \bigg|\int_{\Lambda}d\lambda p(\lambda)A(a',\lambda)B(b,\lambda) \left(1-A(a,\lambda)B(b',\lambda)\right) \nonumber\\
    &- \int_\Lambda d\lambda p(\lambda)A(a',\lambda)B(b',\lambda) \left(1-A(a,\lambda)B(b,\lambda)\right)\bigg|\label{eq:CHSHproofB}
\end{align}

Focusing on \eqref{eq:CHSHproofA}, for now, we can apply the triangle inequality (twice) to obtain:
\begin{align}
    \left|\left<a'b\right> - \left<a'b'\right>\right| \leq& \bigg|\int_{\Lambda}d\lambda p(\lambda)A(a',\lambda)B(b,\lambda) \left(1+A(a,\lambda)B(b',\lambda)\right)\bigg| \nonumber\\
    &+ \bigg|\int_\Lambda d\lambda p(\lambda)A(a',\lambda)B(b',\lambda) \left(1+A(a,\lambda)B(b,\lambda)\right)\bigg|\\
    \leq & \int_{\Lambda}d\lambda \bigg|p(\lambda)A(a',\lambda)B(b,\lambda) \left(1+A(a,\lambda)B(b',\lambda)\right)\bigg| \nonumber\\
    &+ \int_\Lambda d\lambda \bigg|p(\lambda)A(a',\lambda)B(b',\lambda) \left(1+A(a,\lambda)B(b,\lambda)\right)\bigg|\label{eq:CHSHproofAcont}
\end{align}
Now, since $A(a, \lambda),A(a',\lambda),B(b,\lambda),B(b',\lambda)\in \{\pm1\}$ for all $\lambda\in\Lambda$, then:
\begin{align}
    \left|A(a',\lambda)B(b,\lambda)\right| = \left|A(a',\lambda)B(b',\lambda)\right| = 1
\end{align}
And:
\begin{align}
    p(\lambda) \geq& 0\\
    1 \pm A(a,\lambda)B(b,\Lambda), 1 \pm A(a,\lambda)B(b',\Lambda) \geq& 0
\end{align}
So, \eqref{eq:CHSHproofAcont} becomes:
\begin{equation}
    \left|\left<a'b\right> - \left<a'b'\right>\right|\leq \int_{\Lambda}d\lambda p(\lambda)\left(1+A(a,\lambda)B(b',\lambda)\right) + \int_\Lambda d\lambda p(\lambda)\left(1+A(a,\lambda)B(b,\lambda)\right)
\end{equation}
Now, using:
\begin{equation}
    \int_{\Lambda}d\lambda p(\lambda) = 1
\end{equation}
we get:
\begin{equation}\label{eq:CHSHproofAend}
    \left|\left<a'b\right> - \left<a'b'\right>\right|\leq 2 + \int_{\Lambda}d\lambda p(\lambda)A(a,\lambda)B(b',\lambda) + \int_\Lambda d\lambda p(\lambda)A(a,\lambda)B(b,\lambda) = 2 + \left<ab'\right> + \left<ab\right>
\end{equation}
Similarly, starting from \eqref{eq:CHSHproofB}, we can adopt a similar reasoning to get:
\begin{equation}\label{eq:CHSHproofBend}
    \left|\left<a'b\right> - \left<a'b'\right>\right|\leq 2 - \left( \left<ab'\right> + \left<ab\right>\right)
\end{equation}
Now, using both \eqref{eq:CHSHproofAend} and \eqref{eq:CHSHproofBend}, this gives:
\begin{equation}
    \left|\left<a'b\right> - \left<a'b'\right>\right| \leq 2 - \left|\left<ab'\right> + \left<ab\right>\right|
\end{equation}
which by rearraging gives the CHSH equation:
\begin{equation}
    \left|\left<ab'\right> + \left<ab\right> + \left<a'b\right> - \left<a'b'\right>\right| \leq 2
\end{equation}
\chapter{Original proofs}
\section{Proof of proposition~\ref{prop:deltas}}\label{app:deltas}
\paragraph{}In order to prove Proposition~\ref{prop:deltas} we need some results from CbD and M-contextuality. In the CbD framework, given a cyclic system, or more generally a system for which every content is part of exactly 2 contexts, we want to minimise the probability $P\left[S^c_q=S^{c'}_q\right]=\sum_{o\in O}P\left[S^c_q=S^{c'}_q=o\right]$ (where $O$ is the set of possible outcomes) for a globally imposed joint distribution $S$ across all contexts (coupling), which agrees with the observed distributions.

\begin{lemma}
    Given a content $q$ and contexts $c,c'$ containing $q$ and outcome $o$, the maximum of $P\left[S^c_q=S^{c'}_q=o\right]$ for any coupling of the system is given by :
    \begin{equation}
        \min \left(P\left[R^c_q=o\right], P\left[R^{c'}_q=o\right]\right)
    \end{equation}
\end{lemma}

\begin{proof}
    We need a coupling to be compatible with the observed probability distributions, i.e. that the marginals of $S$ coincide with the original distributions. This condition means that:
    \begin{equation}
        \sum_{o'\in O} P\left[S^c_q = o, S^{c'}_q = o'\right] = P\left[R^c_q = o\right]
    \end{equation}
    for each context $c, c'$ sharing the content $q$, and for every value $o\in O$. In particular, this implies both of the following inequalities:
    \begin{align}
        P\left[S^c_q = o, S^{c'}_q = o\right] \leq& P\left[R^c_q = o\right]\\
        P\left[S^c_q = o, S^{c'}_q = o\right] \leq& P\left[R^{c'}_q = o\right]
    \end{align}
    and so:
    \begin{equation}
        P\left[S^c_q = o, S^{c'}_q = o\right] \leq \min\left(P\left[R^c_q = o\right], P\left[R^{c'}_q = o\right]\right)
    \end{equation}
    In addition, given any system with content $q$, it is always possible to construct a coupling for which $P\left[S^c_q=S^{c'}_q\right]$ does attain its maximum (Theorem 3.3 of~\cite{Dzhafarov2016contextcontent}). The above bound is therefore saturated.
\end{proof}

One consequence of this is that:
    \begin{equation}
        \min P\left[S^c_q\neq S^{c'}_q\right] = 1 - \max P\left[S^c_q=S^{c'}_q\right]
    \end{equation}
    
    We now use one of the main results about the correspondence between CbD and M-contextuality.

    \begin{prop}[Proposition 8.4 of~\cite{Jones2019}]
        Given a measurement system (i.e. context-content system with associated probability distributions), for each compatible canonical model $\mathcal{M}$, there exists a coupling $S$ such that:
    \begin{equation}\label{eq:Mcont-CBD}
        \Delta_{c,c'}\left(F_q\right)=P\left[S^c_q \neq S^{c'}_q\right]
    \end{equation}
    for every content $q$. Conversely, for every coupling $S$, there exists a canonical model $\mathcal{M}$ such that \eqref{eq:Mcont-CBD} is satisfied.
    \end{prop}
    
    \begin{cor}
        The minimum of direct influence given a content $q$ and pair of contexts $c,c'$, coincides with the minimum for $P\left[S^c_q\neq S^{c'}_q\right]$.
    \end{cor}
    
    We can now prove Proposition 1.
    \begin{proof}[Proof of Proposition~\ref{prop:deltas}]
        By definition, we have:
        \begin{equation}
            \Delta = \sum_q \left|\left<R^{c_q}_q\right> - \left<R^{c'_q}_q\right>\right|
        \end{equation}
        Since only binary variables are considered for this definition to make sense, each individual term of the sum is given by:
        \begin{align}
            \left|\left<R^{c_q}_q\right> - \left<R^{c'_q}_q\right>\right| =& \left|P\left[R^{c_q}_q=+1\right] - P\left[R^{c_q}_q=-1\right] - P\left[R^{c'_q}_q=+1\right] + P\left[R^{c'_q}_q=-1\right]\right|\nonumber\\
            =& 2 \left|P\left[R^{c_q}_q=+1\right] - P\left[R^{c'_q}_q=+1\right]\right|
        \end{align}
        Now, let 
        \[
        m_{q-} = \min \left(P\left[R^{c_q}_q=-1\right],P\left[R^{c'_q}_q=-1\right]\right)\]
         and respectively 
         \[m_{q+} = \min \left(P\left[R^{c_q}_q=+1\right],P\left[R^{c'_q}_q=+1\right]\right)\]
          Then, each of the above terms reduces to:
        \begin{equation}
            \left|\left<R^{c_q}_q\right> - \left<R^{c'_q}_q\right>\right| = 2 \left(1 - \left(m_{q+} + m_{q-}\right)\right)
        \end{equation}
        Hence, following our previous corollary, the result follows.
    \end{proof}
\section{Proof of proposition~\ref{prop:SFDelta}}\label{app:SFDelta}
What we want to prove is that for most empirical models we have:

\begin{equation}\label{eq:thm}
    \max_X \Delta^*_{C,C'}(X) = \sigma
\end{equation}

\section*{Definitions}
First, we need to define all the terms in \eqref{eq:thm}, and unify the notation used.

We start from a list of contexts $\mathcal{C}$ (we want a rather definition of ``context'', so a \emph{context} will include the list of measurements + potential dependence to variables that depends on each individual context; that means that two identical lists of measurements can refer to two different contexts) from which we define the empirical model $e=\left(e^C\right)_{C\in \mathcal{C}}$
\begin{itemize}
    \item Each of the $e^C$ are probability distribution over possible outcomes in context $C$
    \item We will denote $\mathcal{O}_C$ the set of possible outcomes in context $C$
    \item Given an observable $X$ in the measurement context of $C$ (write $X \in C$), we write $e^C_{X} = \left.e^{C}\right|_{X}$ the marginal distribution corresponding to the observable $X$ in the context $C$. Similarly, we define the set $\mathcal{O}_X$ as the set of possible outcomes of the observable $X$.
\end{itemize}

A hidden variable model (HVM) of an empirical model $e$ is here defined as $\Omega = \left(h = \left(h^{\lambda}\right)_{\lambda \in \Lambda}, p_\Lambda\right)$ where:
\begin{itemize}
    \item $\Lambda$ is the set of hidden/latent variables in the HVM
    \item For all hidden variable $\lambda$ and context $C$, $h^{\lambda,C}$ is a probability distribution over $\mathcal{O}_C$; therefore $\left(h^{\lambda,C}\right)_{C\in \mathcal{C}}$ forms an empirical model.
    \item $p_\Lambda$ is a probability distribution over $\Lambda$
    \item For every context $C$ we have:
    \begin{equation}\label{eq:eFromLambda}
        e^C = \sum_{\lambda\in \Lambda} p_\Lambda(\lambda)h^{\lambda,C}
    \end{equation}
\end{itemize}

For all hidden variable $\lambda$ in a HVM, we can decompose $h^\lambda$ as:
\begin{equation}
    h^\lambda = c_{NS}^\lambda h^{\lambda}_{NS} + \left(1-c_{NS}^\lambda\right) {h}'^{\lambda}
\end{equation}
where $h^\lambda_{NS}$ is no-signalling, and ${h'}^\lambda$ can be any empirical model.

We then define the \emph{signalling fraction} $\sigma$ as:
\begin{equation}
    \sigma = \min_{HVM} \max_{\lambda\in \Lambda} 1-c^\lambda_{NS}
\end{equation}

In the M-contextuality framework (framework fundamentally related to the CbD framework), we are interested in \emph{canonical causal models} $\mathcal{M}$ in which each of the individual observable $X$ is associated with a random variable; the (different choices of) contexts are also modelled as a single random variable which can influence (all of the different) observable variables. In addition, we also define a latent variable $\Lambda$ which is independent of the context variable but can also influence all of the observable variables. These models can themselves be viewed as hidden variable models in the sense described above by setting:
\begin{equation}
    h^{\lambda,C}(o) = Pr_\mathcal{M}\left[o~\middle|~C,\lambda\right]
\end{equation}
and from where we can also recover \eqref{eq:eFromLambda}. When obvious we will drop the $\mathcal{M}$ subscript.

Without loss of generality, it is also enough to restrict ourselves to canonical models $\mathcal{M}$ such that for all $\lambda\in\Lambda, C\in \mathcal{C}, X\in C$ and $x\in \mathcal{O}_X$, we have $h^{\lambda, C}_X(x) \in \{0,1\}$. As each of the $h^{\lambda, C}_X$ are probability distributions over $\mathcal{O}_X$, we can therefore define for each pair $(\lambda, C)$ and observable $X\in C$ a function $F_X: \Lambda\times \mathcal{C}\to \mathcal{O}_X$ such that:
\begin{equation}
    F_X(\lambda, C) = x \quad \iff \quad h^{\lambda, C}_X(x) = 1
\end{equation}

Given a canonical model $\mathcal{M}$, we can define the degree of direct influence from the (change of) context $C\leftrightarrow C'$ on the observable variable $X \in C\cup C'$ as:
\begin{equation}
    \Delta_{C,C'}(X) = Pr\left[\left\{\lambda~\middle|~F_X(\lambda,C)\neq F_X(\lambda, C'\right\}\right]
\end{equation}

We now introduce a couple of results from CbD and M-contextuality:
\begin{itemize}
    \item For all observables $X$, we define:
    \begin{equation}
        Pr\left[e^C_{X} = e^{C'}_{X}\right] = \sum_{o\in \mathcal{O}_X} \min_{\tilde{C}\in \{C, C'\}} e^{\tilde{C}}_X(o)
    \end{equation}
    \item The above equation can also be extended to the situation where more than two contexts intersect at the observable $X$ as follows:
    \begin{equation}
        Pr\left[e^{C_1}_{X} = e^{C_2}_{X} = \ldots = e^{C_n}_{X}\right] = \min_{(i,j) \in \{1, 2, \ldots n\}^2, i\neq j} Pr\left[e^{C_i}_{X} = e^{C_j}_{X}\right]
    \end{equation}
    \item For all canonical models $\mathcal{M}$, we always have:
    \begin{equation}\label{eq:maxDI}
        \Delta_{C,C'}(X) \leq Pr\left[e^C_X\neq e^{C'}_X\right] = 1 - Pr\left[e^C_X= e^{C'}_X\right]
    \end{equation}
    \item For any empirical model, for any observable $X$, there exist a canonical model such that:
    \begin{equation}
        \Delta_{C,C'}(X) = Pr\left[e^C_X\neq e^{C'}_X\right] 
    \end{equation}
    i.e. the maximum in \eqref{eq:maxDI} can always be attained in a canonical model. We will also write:
    \begin{equation}
        \Delta^*_{C,C'}(X) = Pr\left[e^C_X\neq e^{C'}_X\right] = \max_{\mathcal{M}}\Delta_{C,C'}(X)
    \end{equation}
\end{itemize}

\section*{Inequality 1 (the general case)}
We first prove that for an empirical model:
\begin{equation}\label{eq:ineq1}
    \max_X \Delta^*_{C,C'}(X) \leq  \sigma
\end{equation}

\begin{proof}
Suppose that there exists an observable $X$ in an empirical model for which:
\begin{equation}
    \Delta^*_{C,C'}(X) > \sigma = \max_\lambda 1 - c^\lambda_{NS}
\end{equation}
From M-contextuality, there exists a canonical model $\mathcal{M}$ in which:
\begin{equation}
    \Delta_{C,C'}(X) = \Delta^*_{C,C'}(X) = 1 - \sum_{x\in \mathcal{O}_X}\min_{\tilde{C}\in \{C, C'\}} e^{\tilde{C}}_X (x)
\end{equation}
Now, using the HVM corresponding to this canonical model, we have:
\begin{equation}
    e^{\tilde{C}}_X (x) = \sum_{\lambda\in \Lambda} p_\lambda(\lambda)h^{\lambda, \tilde{C}}_X (x)
\end{equation}
so:
\begin{equation}
    \Delta_{C,C'}(X) =1 - \sum_{x\in \mathcal{O}_X}\min_{\tilde{C}\in \{C, C'\}} \sum_{\lambda\in \Lambda} p_\lambda(\lambda)h^{\lambda, \tilde{C}}_X (x) > \max_\lambda 1 - c^\lambda_{NS} = 1 - \min_\lambda c^\lambda_{NS}
\end{equation}
In turns that implies that:
\begin{equation}
    \sum_{x\in \mathcal{O}_X}\min_{\tilde{C}\in \{C, C'\}} \sum_{\lambda\in \Lambda} p_\lambda(\lambda)h^{\lambda, \tilde{C}}_X (x) < \min_\lambda c^\lambda_{NS}
\end{equation}
We then decompose $h^\lambda$ as the convex sum of a no-signalling and a signalling part:
\begin{equation}
    \sum_{x\in \mathcal{O}_X}\min_{\tilde{C}\in \{C, C'\}} \sum_{\lambda\in \Lambda} p_\lambda(\lambda)\left[c^\lambda_{NS}h^{\lambda, \tilde{C}}_{NS,X} (x) + \left(1-c^\lambda_{NS}\right) h^{'\lambda, C}_X (x)\right] < \min_\lambda c^\lambda_{NS}
\end{equation}
Now, since $\left(1-c^\lambda_{NS}\right) h^{'\lambda, C}_X (x)>0$, this implies that:
\begin{equation}\label{eq:ineq1eq7}
     \sum_{x\in \mathcal{O}_X}\min_{\tilde{C}\in \{C, C'\}} \sum_{\lambda\in \Lambda} p_\lambda(\lambda) c^\lambda_{NS}h^{\lambda, \tilde{C}}_{NS,X} (x) < \min_\lambda c^\lambda_{NS}
\end{equation}
We then use the fact that $h^\lambda_{NS}$ is no-signalling, and therefore:
\begin{equation}
    h^{\lambda, C}_{NS,X} (x) = h^{\lambda, C'}_{NS,X} (x)
\end{equation}
So \eqref{eq:ineq1eq7} simplifies to:
\begin{equation}
    \sum_{x\in \mathcal{O}_X}\sum_{\lambda\in \Lambda} p_\lambda(\lambda) c^\lambda_{NS}h^{\lambda, C}_{NS,X} (x) = \sum_{\lambda\in \Lambda} p_\lambda(\lambda) c^\lambda_{NS}\sum_{x\in \mathcal{O}_X} h^{\lambda, C}_{NS,X} (x) < \min_\lambda c^\lambda_{NS}
\end{equation}
Now, $h^{\lambda, C}_{NS, X}$ is a probability distribution over $\mathcal{O}_X$, so $\sum_x h^{\lambda, C}_{NS, X}(x) = 1$, and:
\begin{equation}
    \sum_{\lambda\in \Lambda} p_\lambda(\lambda) c^\lambda_{NS} < \min_\lambda c^\lambda_{NS}
\end{equation}
In addition, for all $\lambda$, we have $\min_\lambda c^\lambda_{NS}\leq c^\lambda_{NS}$, so:
\begin{equation}
    \sum_{\lambda\in \Lambda} p_\lambda(\lambda) \min_{\lambda'}c^{\lambda'}_{NS} \leq \sum_{\lambda\in \Lambda} p_\lambda(\lambda) c^\lambda_{NS}< \min_\lambda c^\lambda_{NS}
\end{equation}
and:
\begin{equation}
    \sum_{\lambda\in \Lambda} p_\lambda(\lambda) \min_{\lambda'}c^{\lambda'}_{NS} = \min_{\lambda'}c^{\lambda'}_{NS} \sum_{\lambda\in \Lambda} p_\lambda(\lambda) = \min_{\lambda'}c^{\lambda'}_{NS}
\end{equation}
So we then obtain the contradiction:
\begin{equation}
    \min_{\lambda'}c^{\lambda'}_{NS} < \min_{\lambda}c^{\lambda}_{NS}
\end{equation}
\end{proof}

\section*{Inequality 2}
The second inequality, i.e.:
\begin{equation}\label{eq:ineq2}
    \max_X \Delta^*_{C,C'}(X) \geq \sigma
\end{equation}
only holds in cases when the notion of consistent connectedness and no-signalling coincides.
\begin{proof}
We want to construct an HVM such that:
\begin{equation}
    \max_{\lambda} c^\lambda_{NS}= \max_X \Delta^*_{C, C'} (X)
\end{equation}

We isolate the observable $X$ such that $\max_{\tilde{X}} \Delta^*_{C, C'}(\tilde{X}) = \Delta^*_{C, C'}(X)$, and will denote for simplicity $\Delta = \Delta^*_{C, C'}(X)$. Then, from M-contextuality, we know that there exist a canonical model $\mathcal{M}$ such that $\Delta_{C, C'}(X) = \Delta$. We will use this canonical model to create a HVM with a single hidden variable $\lambda$ where:
\begin{equation}
    h^{\lambda, D}_{NS, Y} (y) = h^{\lambda, D}_{NS, Y} (y) = \frac{1}{1-\Delta_{D, D'} (Y)}Pr_{\mathcal{M}}\left[\left\{\lambda~\middle|~F_Y(\lambda, D) = F_Y(\lambda, D') = y\right\}\right]
\end{equation}
\textbf{Note:} if the observable $Y$ is part of more than 2 contexts $D_1, D_2, D_3\ldots$, we will replace $\Delta_{D, D'}(Y)$ by $\min_{i, j} \Delta_{D_i, D_j}(Y)$, and $Pr_{\mathcal{M}}\left[\left\{\lambda~\middle|~F_Y(\lambda, D) = F_Y(\lambda, D') = y\right\}\right]$ by: $$Pr_{\mathcal{M}}\left[\left\{\lambda~\middle|~F_Y(\lambda, D_1) = F_Y(\lambda, D_2) = F_Y(\lambda, D_3) = \ldots = y \right\}\right]$$.

Since we are only working with a single hidden variable, we will also drop the $\lambda$ superscript.

We then show that this actually leads to a valid HVM.

\begin{claim}
\begin{enumerate}
    \item $h_{NS, Y}$ is a probability distribution over $\mathcal{O}_Y$
    \item For all $y\in \mathcal{O}_Y$:
    \begin{equation}
        \left(1 - \Delta\right) h_{NS, Y}(y) \leq e^D_Y(y), e^{D'}_Y(y)
    \end{equation}
\end{enumerate}
\end{claim}
\begin{proof}
1. By definition:
\begin{align}
    \Delta_{D, D'}(Y) =& 1- Pr_\mathcal{M}\left[\left\{\lambda~\middle|~F_Y(\lambda, D) = F_Y(\lambda, D')\right\}\right]\\
    =& 1-\sum_{y\in \mathcal{O}_Y} Pr_\mathcal{M} \left[\left\{\lambda~\middle|~F_Y(\lambda, D) = F_Y(\lambda, D') = y\right\}\right]
\end{align}
Then
\begin{equation}
    \sum_{y\in \mathcal{O}_Y} h_{NS, Y}(y) = \frac{1-\sum_{y} Pr_\mathcal{M} \left[\left\{\lambda~\middle|~F_Y(\lambda, D) = F_Y(\lambda, D') = y\right\}\right]}{1- \sum_{y'} Pr_\mathcal{M} \left[\left\{\lambda~\middle|~F_Y(\lambda, D) = F_Y(\lambda, D') = y'\right\}\right]} = 1
\end{equation}

2. First recall that in the HVM $\left(\left(g^\lambda\right)_{\lambda\in \Lambda}, q_\Lambda\right)$ we have:
\begin{equation}
    e^D_Y(y) = \sum_\lambda q_\Lambda(\lambda)g^{\lambda, D}_{Y}(y)
\end{equation}
where $g^{\lambda, D}_Y(y)=1$ iff $F_Y(\lambda, D) = y$ and $g^{\lambda,D}_Y(y)=0$ otherwise. So, in fact:
\begin{equation}
    e^D_Y(y) = Pr_\mathcal{M}\left[\left\{\lambda~\middle|~F_Y(\lambda, D)=y\right\}\right]
\end{equation}
In addition, we have:
\begin{equation}
    \Delta = \max_Y \max_{\mathcal{M}} \Delta_{D, D'}(Y)
\end{equation}
So:
\begin{align}
    \Delta \geq \Delta_{D, D'} (Y) \quad &\implies \quad 1 - \Delta \leq 1 - \Delta_{D, D'} (Y)\\
    &\implies \quad \frac{1 - \Delta}{1 - \Delta_{D, D'} (Y)} \leq 1
\end{align}

Now:
\begin{align}
    (1 - \Delta) h_{NS, Y} (y) =& \frac{1-\Delta}{1 - \Delta_{D, D'}(Y)}Pr_\mathcal{M} \left[\left\{\lambda~\middle|~F_Y(\lambda, D) = F_Y(\lambda, D') = y\right\}\right]\\
    \leq& Pr_\mathcal{M} \left[\left\{\lambda~\middle|~F_Y(\lambda, D) = F_Y(\lambda, D') = y\right\}\right]\\
    \leq& Pr_\mathcal{M} \left[\left\{\lambda~\middle|~F_Y(\lambda, D) = y\right\}\right] = e^D_Y(y)
\end{align}
since $F_Y(\lambda, D) = F_Y(\lambda, D') = y \implies F_Y(\lambda, D) = y$.

A similar proof can be done for $D'$.
\end{proof}

\begin{claim}
We can define $h^D_{NS, Y}$, $h^{D'}_{NS, Y}$ distributions over $\mathcal{O}_D$ and $\mathcal{O}_{D'}$ respectively such that:
\begin{enumerate}
    \item $h^D_{NS, Y}\left(\left.o\right|_Y\right) = h^{D'}_{NS, Y}\left(\left.o\right|_Y\right) = h_{NS, Y}\left(\left.o\right|_Y\right)$
    \item For all $o\in \mathcal{O}_D$:
    \begin{equation}
        (1-\Delta)h^D_{NS}(o) \leq e^D(o)
    \end{equation}
    (and similarly for $D'$).
\end{enumerate}
\end{claim}

\begin{proof}
We'll only show this for the context $D$ (but the same applies for $D'$).

We define:
\begin{equation}
    h^D_{NS}(o) = \frac{e^D(o)}{e^D_Y\left(\left.o\right|_Y\right)} h_{NS, Y}\left(\left.o\right|_Y\right)
\end{equation}
This defines a probability distribution since:
\begin{align}
    \sum_{o\in \mathcal{O}_D} h^D_{NS} (o) =& \sum_{o\in \mathcal{O}_D} \frac{e^D(o)}{e^D_Y\left(\left.o\right|_Y\right)}h_{NS, Y}\left(\left.o\right|_Y\right)\\
    =& \sum_{y\in \mathcal{O}_Y}\sum_{o\in \mathcal{O}_D|\left.o\right|_Y = y} \frac{e^D(o)}{e^D_Y\left(y\right)}h_{NS, Y}\left(y\right)\\
    =& \sum_{y\in \mathcal{O}_Y}\frac{h_{NS, Y}\left(y\right)}{e^D_Y\left(y\right)}\sum_{o\in \mathcal{O}_D|\left.o\right|_Y = y}e^D(o) = \sum_{y\in \mathcal{O}_Y}\frac{h_{NS, Y}\left(y\right)}{e^D_Y\left(y\right)}e^D\left(y\right)\\
    =& \sum_{y\in \mathcal{O}_Y}h_{NS, Y}\left(y\right) = 1
\end{align}

1. We can just check that:
\begin{equation}
    \left.h^D_{NS}\right|_Y (y) = \sum_{o\in \mathcal{O}_D|\left.o\right|_Y=y} \frac{e^D(o)}{e^D_Y(y)} h_{NS, Y}(y) = h_{NS, Y}(y)
\end{equation}

2. Similarly, we just check that:
\begin{align}
    (1-\Delta) h^D_{NS, Y}(o) =& \frac{1-\Delta}{1-\Delta_{D, D'}(Y)} \frac{e^D(o)}{e^D_Y\left(\left.o\right|_Y\right)}Pr_\mathcal{M}\left[\left\{\lambda~\middle|~F_Y(\lambda, D) = F_Y(\lambda, D')=y\right\}\right]\\
    \leq& e^D(o)\frac{Pr_\mathcal{M}\left[\left\{\lambda~\middle|~F_Y(\lambda, D) = F_Y(\lambda, D')=y\right\}\right]}{e^D_Y\left(\left.o\right|_Y\right)}\\
    \leq& e^D(o)
\end{align}
\end{proof}

\begin{cor} There is some empirical model $h'$ such that:
\begin{equation}
    e = (1-\Delta) h_{NS} + \Delta h'
\end{equation}
\end{cor}
\begin{proof}
We have already shown that $(1-\Delta)h_{NS}\leq e$. We then define:
\begin{equation}
    h' = \frac{1}{\Delta} \left(e - (1-\Delta)h_{NS}\right) \geq 0
\end{equation}
This is an empirical mode since for all contexts $D$, we have:
\begin{equation}
    \sum_{o\in \mathcal{O}_D} h'^D(o) = \frac{1}{\Delta}\left[\sum_{o\in \mathcal{O}_D}e^D(o) - (1-\Delta)\sum_{o\in \mathcal{O}_D}h^D_{NS}(o)\right] = \frac{1}{\Delta}\left[1 - (1 - \Delta)\right] = 1
\end{equation}
\end{proof}

\end{proof}
\subsection*{Why it doesn't work when $\left|C\cap C'\right|>1$}
The proof of \eqref{eq:ineq2} relies on the fact that there is for each observable $X$ a function $F_X(\lambda, C): \Lambda\times \mathcal{C}\to X$ which defines the probability distribution of the hidden-variable distributions $h^\lambda$. Now, if we were in a situation when $\left|C\cap C'\right|\geq 2$, then we would not only (in the sheaf-theoretic framework) check the no-signalling condition on each of the $X\in C\cap C'$ but also for each $S\subseteq C\cap C'$.

The above reasoning could easily be extended if there was a function $F_S: \Lambda\times \mathcal{C}\to \prod_{X\in S} \mathcal{O}_X$, which restrict to $F_X$ by post-composing with a projection operator (recall that we want \emph{all} restrictions to be well-defined). But then, we would also have:
\begin{align}
    1 - \Delta_{C, C'} \left(\{X, Y\}\right) &= Pr\left[\left\{\lambda|F_{\{X, Y\}}(\lambda, C) = F_{\{X, Y\}}(\lambda, C')\right\}\right]\\
    &\leq Pr\left[\left\{\lambda|\pi_X\circ F_{\{X, Y\}}(\lambda, C) = \pi_X\circ F_{\{X, Y\}}(\lambda, C')\right\}\right] \nonumber\\
    &= \Delta_{C, C'}(X)
\end{align}
And therefore, if $\mathcal{X}$ denotes the set of observables:
\begin{equation}
    \min_{S\in \mathcal{P}(\mathcal{X})} \Delta_{C, C'}(S) \geq \min_{X\in \mathcal{X}} \Delta_{C, C'}(X)
\end{equation}
in every canonical model. Hence, there could exist a model such that:
\begin{equation}
    \max_\mathcal{M}\max_{S\subseteq \mathcal{X}} \Delta_{C, C'}(S)\geq\max_{\lambda} 1-c^\lambda_{NS}> \max_\mathcal{M}\max_{X\in \mathcal{X}} \Delta_{C, C'}(X)
\end{equation}
\section{Equivalence between causality notions}\label{app:causality}
\paragraph{}We start by discussing the equivalence between distributions of causal \emph{functions} and causality as defined in terms of causal \emph{processes} for bipartite systems, and we will (without loss of generality) consider the causal order $A\preceq B$. We will then extend this to arbitrary causal orders.

\paragraph{}First, we need to state that distributions of causal functions in $\mathcal{D}_{\mathbb{R}_+}\left(\mathcal{L}_{A\preceq B}(\left\{A,B\right\})\right)$ live in classical probability spaces while the output of quantum processes such as the ones in~\eqref{eq:terminalityA} or~\eqref{eq:terminalityB} are quantum states. Hence, we will here show an equivalence between causality as described in Section~\ref{sec:QDescription} and the \emph{measurement statistics} of causal processes. Namely, for a quantum circuit:
\begin{equation*}
    \begin{tikzpicture}[circuit ee IEC]
	\begin{pgfonlayer}{nodelayer}
		\node [style=box] (0) at (0, 0.25) {$~~f~~$};
		\node [style=none] (1) at (-0.75, -1.75) {};
		\node [style=none] (2) at (0.75, -1.75) {};
		\node [style=none] (3) at (-0.75, 2) {};
		\node [style=none] (4) at (0.75, 2) {};
		\node [style=none] (5) at (-1.5, 1.75) {$A_{in}$};
		\node [style=none] (6) at (-1.75, -1.5) {$A_{out}$};
		\node [style=none] (7) at (1.5, 1.75) {$B_{in}$};
		\node [style=none] (8) at (1.75, -1.25) {$B_{out}$};
	\end{pgfonlayer}
	\begin{pgfonlayer}{edgelayer}
		\draw [style=thickline] (3.center) to (1.center);
		\draw [style=thickline] (4.center) to (2.center);
	\end{pgfonlayer}
\end{tikzpicture}

\end{equation*}
and the interpretation of input $a,b$ and outputs $o_a, o_b$ as quantum states $\ket{a},\ket{b},\ket{o_a}, \ket{o_b}$, we define the conditional probability:
\begin{equation}
    P\left[o_a, o_b~\middle|~a,b\right] = \begin{tikzpicture}[circuit ee IEC]
	\begin{pgfonlayer}{nodelayer}
		\node [style=box] (0) at (0, 0) {$\quad f\quad$};
		\node [style=coelemt] (1) at (-1, -2.75) {$o_a$};
		\node [style=coelemt] (2) at (1, -2.75) {$o_b$};
		\node [style=elemt] (3) at (-1, 2.75) {$a$};
		\node [style=elemt] (4) at (1, 2.75) {$b$};
		\node [style=none] (5) at (-2, 1.25) {$A_{in}$};
		\node [style=none] (6) at (-2.25, -1.25) {$A_{out}$};
		\node [style=none] (7) at (2, 1.25) {$B_{in}$};
		\node [style=none] (8) at (2.25, -1.25) {$B_{out}$};
	\end{pgfonlayer}
	\begin{pgfonlayer}{edgelayer}
		\draw [style=thickline] (3) to (1);
		\draw [style=thickline] (4) to (2);
	\end{pgfonlayer}
\end{tikzpicture}

\end{equation}

Moreover, assuming that the set $\left\{\ket{o_{B}}\right\}$ spans the entire space $B_{out}$, i.e. that:
\begin{equation}
    \sum_{o_b} \begin{tikzpicture}
	\begin{pgfonlayer}{nodelayer}
		\node [style=coelemt] (0) at (0, -1) {$o_b$};
		\node [style=none] (1) at (0, 1) {};
	\end{pgfonlayer}
	\begin{pgfonlayer}{edgelayer}
		\draw [style=thickline] (1.center) to (0);
	\end{pgfonlayer}
\end{tikzpicture}
 = \begin{tikzpicture}[circuit ee IEC]
	\begin{pgfonlayer}{nodelayer}
		\node [style=discard] (0) at (0, -0.75) {};
		\node [style=none] (1) at (0, 0.75) {};
	\end{pgfonlayer}
	\begin{pgfonlayer}{edgelayer}
		\draw [style=thickline] (0) to (1.center);
	\end{pgfonlayer}
\end{tikzpicture}

\end{equation}
This is reasonable as we expect that, at each run, some outcome is detected. Furthermore, it should also be the case that:
\begin{equation}
    \begin{tikzpicture}[circuit ee IEC]
	\begin{pgfonlayer}{nodelayer}
		\node [style=discard] (0) at (0, -0.75) {};
		\node [style=elemt] (1) at (0, 0.75) {$b$};
	\end{pgfonlayer}
	\begin{pgfonlayer}{edgelayer}
		\draw [style=thickline] (0) to (1);
	\end{pgfonlayer}
\end{tikzpicture}
 = 1
\end{equation}
i.e. discarding a state is deterministic. 

\paragraph{}Using the above assumptions, the operational condition of equation~\eqref{eq:terminalityA} for the causal order $A\preceq B$ implies that:
\begin{equation}
    \begin{tikzpicture}[circuit ee IEC]
	\begin{pgfonlayer}{nodelayer}
		\node [style=box] (0) at (0, 0) {$\quad f\quad$};
		\node [style=coelemt] (1) at (-1, -2.75) {$o_a$};
		\node [style=discard] (2) at (1, -2.75) {};
		\node [style=elemt] (3) at (-1, 2.75) {$a$};
		\node [style=elemt] (4) at (1, 2.75) {$b$};
		\node [style=none] (5) at (-2, 1.25) {$A_{in}$};
		\node [style=none] (6) at (-2.25, -1.25) {$A_{out}$};
		\node [style=none] (7) at (2, 1.25) {$B_{in}$};
		\node [style=none] (8) at (2.25, -1.25) {$B_{out}$};
	\end{pgfonlayer}
	\begin{pgfonlayer}{edgelayer}
		\draw [style=thickline] (3) to (1);
		\draw [style=thickline] (4) to (2);
	\end{pgfonlayer}
\end{tikzpicture}
 = \sum_{o_b} \begin{tikzpicture}[circuit ee IEC]
	\begin{pgfonlayer}{nodelayer}
		\node [style=box] (0) at (0, 0) {$\quad f\quad$};
		\node [style=coelemt] (1) at (-1, -2.75) {$o_a$};
		\node [style=coelemt] (2) at (1, -2.75) {$o_b$};
		\node [style=elemt] (3) at (-1, 2.75) {$a$};
		\node [style=elemt] (4) at (1, 2.75) {$b$};
		\node [style=none] (5) at (-2, 1.25) {$A_{in}$};
		\node [style=none] (6) at (-2.25, -1.25) {$A_{out}$};
		\node [style=none] (7) at (2, 1.25) {$B_{in}$};
		\node [style=none] (8) at (2.25, -1.25) {$B_{out}$};
	\end{pgfonlayer}
	\begin{pgfonlayer}{edgelayer}
		\draw [style=thickline] (3) to (1);
		\draw [style=thickline] (4) to (2);
	\end{pgfonlayer}
\end{tikzpicture}
 = \begin{tikzpicture}[circuit ee IEC]
	\begin{pgfonlayer}{nodelayer}
		\node [style=box] (0) at (-0.75, 0.25) {$\tilde{f}$};
		\node [style=coelemt] (1) at (-0.75, -2) {$o_a$};
		\node [style=discard] (2) at (0.75, -1.5) {};
		\node [style=elemt] (3) at (-0.75, 2.75) {$a$};
		\node [style=elemt] (4) at (0.75, 2.75) {$b$};
		\node [style=none] (5) at (-1.5, 1.5) {$A_{in}$};
		\node [style=none] (6) at (-1.75, -1) {$A_{out}$};
		\node [style=none] (7) at (1.5, 1.5) {$B_{in}$};
	\end{pgfonlayer}
	\begin{pgfonlayer}{edgelayer}
		\draw [style=thickline] (3) to (1);
		\draw [style=thickline] (4) to (2);
	\end{pgfonlayer}
\end{tikzpicture}

\end{equation}
which, in terms of conditional probabilities translate as:
\begin{equation}
    \sum_{o_b}P\left[o_a,o_b~\middle|~a,b\right] = P\left[o_a~\middle|a\right]
\end{equation}

Now, let $F$ being a set of functions $f: A\times B\to O_A\times O_B$ corresponding to events then, for any section $\mu\in \mathcal{D}_{\mathbb{R}_+}\left(F\right)$, we have:
\begin{equation}
    P\left[o_a,o_b~\middle|~a,b\right] = \sum_{f~s.t.~f(a,b)=(o_a,o_b)} \mu(f)
\end{equation}
and:
\begin{equation}
    P\left[o_a~\middle|~a\right] = \sum_{f~s.t.~\left.f\right|_A(a)=(o_a)} \mu(f)
\end{equation}
Therefore, we have $\mu$ compatible with the causal order $A\preceq B$ with respect to the operational definition of equation~\eqref{eq:terminalityA} iff:
\begin{equation}\label{eq:AprecBeqTarget}
    \sum_{o_b} \sum_{f~s.t.~f(a,b)=(o_a,o_b)} \mu(f) = \sum_{f~s.t.~\left.f\right|_A(a)=(o_a)} \mu(f)
\end{equation}
Now, since the LHS of equation~\eqref{eq:AprecBeqTarget} is marginalising $o_b$, the free variables on the LHS are $a,b,o_a$, whereas the free variables in the RHS are $a,o_A$. Therefore, the above equation is satisfied iff the values of $\left.f\right|_A$ do not depend on the value of $b$, which is exactly the condition for causality of functions with respect to the causal order $A\preceq B$.

\paragraph{}In order to extend this result to arbitrary causal order, we employ a trick from \cite{KissingerHobanCoecke}, namely to split any given causal order $\left(\Sigma, \preceq\right)$ as a coarse-grained causal order $A\preceq B$, where $A$ consists on the parties $\Sigma\/\left\{B\right\}$, and $B$ is a \emph{maximal box}, i.e. a party which does not influence any other one. Then, we know from the above result that a circuit is compatible with the causal order $A\preceq B$ iff the statistics of the measurements correspond to a section $\mu\in \mathcal{D}_{\mathbb{R}_+}\left(F\right)$, where $F$ is a set of causal functions with respect to $A\preceq B$. Therefore, we need to check that $\left.\mu\right|_A$ is still compatible with the causal order $\Sigma$. We then proceed to isolate a new $B'\in A$ such that $B'$ is maximal in the reduced causal order $\Sigma\/\{B\}$. We then keep going until the reduced scenario only consist of 2 parties.
\section{Proof of proposition~\ref{prop:causalFracForm}}\label{app:causalFracForm}
In order to prove the formula for the causal fraction in (2,2,2)-Bell scenarios, we start by proving a more general equation that the causal fraction needs to satisfy, in any empirical model.

\begin{prop}
    For a family of probability distributions where the causal order is not known, an upper bound of the causal fraction can be calculated as follows\footnote{Note: the order of the restrictions is from left to right.}:
\begin{equation}
\label{eq:causual2}
    \gamma \leq \min_{U,V} 1 - \left|\left.\left.e_{\underline{i}}\right|_U\right|_{U\cap V}\left(\underline{o}\right) - \left.\left.e_{\underline{i}}\right|_V\right|_{U\cap V}\left(\underline{o}\right)\right|
\end{equation}
where $\left.\left.e_{\underline{i}}\right|_U\right|_{U\cap V}$ corresponds to the restriction of $e_{\underline{i}}$ to first $U$ and then from $U$ to $U\cap V$ (and similarly for $\left.\left.e_{\underline{i}}\right|_V\right|_{U\cap V}$).
\end{prop}
\begin{proof}
    For every causal empirical causal model $e^\Omega$ with respect to a causal scenario $\Sigma = \left(\Omega, \underline{I}, \underline{O}\right)$, if we have $\gamma \cdot e^\Omega \preceq e$, then both:
    \begin{equation}
        \gamma \cdot \left.e^\Omega_{\underline{i}}\right|_{U\cap V}(\underline{o})\leq \left.\left.e_{\underline{i}}\right|_U\right|_{U\cap V}(\underline{o})
    \end{equation}
    and
    \begin{equation}
        \gamma \cdot\left.e^\Omega_{\underline{i}}\right|_{U\cap V}(\underline{o})\leq \left.\left.e_{\underline{i}}\right|_V\right|_{U\cap V}(\underline{o})
    \end{equation}
    So:
    \begin{equation}\label{eq:p1e1}
        \gamma \cdot \left.e^\Omega_{\underline{i}}\right|_{U\cap V}(\underline{o})\leq \min_{X\in \{U,V\}}\left.\left.e_{\underline{i}}\right|_X\right|_{U\cap V}(\underline{o})
    \end{equation}
    Now, since $e_{\underline{i}}$ are probability distributions:
    \begin{equation}
        1 - \left.\left.e_{\underline{i}}\right|_X\right|_{U\cap V}(\underline{o}) = \sum_{\underline{o'}\neq \underline{o}} \left.\left.e_{\underline{i}}\right|_X\right|_{U\cap V}(\underline{o'})
    \end{equation}
    and similarly for $e^\Omega$. Therefore, using $\gamma e^\Omega\preceq e$ once again:
    \begin{equation}\label{eq:p1e2}
        \gamma \left(1 - \left.\left.e^\Omega_{\underline{i}}\right|_X\right|_{U\cap V}(\underline{o})\right) \leq \min_{X\in \{U,V\}} 1 - \left.\left.e_{\underline{i}}\right|_X\right|_{U\cap V}(\underline{o}) = 1 - \max_{X\in \{U,V\}}\left.\left.e_{\underline{i}}\right|_X\right|_{U\cap V}(\underline{o})
    \end{equation}
    Then, writing $m_- = \min_{X\in \{U,V\}}\left.\left.e_{\underline{i}}\right|_X\right|_{U\cap V}(\underline{o})$ and $m_+ =\max_{X\in \{U,V\}}\left.\left.e_{\underline{i}}\right|_X\right|_{U\cap V}(\underline{o})$ for simplicity, we use \eqref{eq:p1e1} and \eqref{eq:p1e2} to get:
    \begin{equation}
        \gamma \leq 1 - m_+ + m_-
    \end{equation}
    Now, using binary minima and maxima this reduces to:
    \begin{equation}
        \gamma \leq 1 - \left|\left.\left.e_{\underline{i}}\right|_U\right|_{U\cap V}\left(\underline{o}\right) - \left.\left.e_{\underline{i}}\right|_V\right|_{U\cap V}\left(\underline{o}\right)\right|
    \end{equation}
    And since this has to be the case for all $U, V \in \mathcal{L}$, the claimed inequality has to hold.
\end{proof}

Let's now describe a construction of a causal empirical model $e^\Omega$ which satisfies $\gamma\cdot e^\Omega\preceq e$, for any given (2,2,2) Bell-type  model $e$, where $\gamma$ is given as in \eqref{eq:causalF}. This would therefore give proof that the above inequality can be saturated, and therefore that the causal fraction of such models can be known with certainty. 

    We start by constructing a probability distribution for the event $A$ as follows. For any $a\in I_A$, we select $o^*_A \in O$ such that:
    \begin{equation}
        \min_{b\in I_B}\left.e_{\left(a,b\right)}\right|_{A}(o^*_A) = \min_{o\in O}\min_{b\in I_B}\left.e_{\left(a,b\right)}\right|_{A}(o)
    \end{equation}
    and set:
    \begin{equation}
        \left.e^\Omega_{(a, b)}\right|_{A}\left(o^*_A\right) = \frac{\min_{b\in I_B} \left.e_{(a, b)}\right|_{A}\left(o^*_A\right)}{\gamma}
    \end{equation}
    and $\left.e^\Omega_{(a, b)}\right|_{A}\left(\neg o^*_A\right) = 1 - \left.e^\Omega_{(a, b)}\right|_{A}\left(o^*_A\right)$. Then we have:
    \begin{equation}
        \gamma \cdot \left.e^\Omega_{(a, b)}\right|_{A}\left(o\right) \leq \left.e_{(a, b)}\right|_{A}\left(o\right)
    \end{equation}
    for all $\left(a, b\right)\in I_A\times I_B$, and for all possible outcome $o\in O$.
    
    One can then extend this distribution to the lowerset $A\to B = \Omega$ by setting, for example:
    \begin{equation}
        e^\Omega_{(a, b)} \left(o_A, o_B\right) = \frac{e^\Omega_{(a, b)}\left(o_A, o_B\right)}{\left.e_{(a, b)}\right|_{A}\left(o_A\right)}\left.e^\Omega_{(a, b)}\right|_{A}\left(o_A\right)
    \end{equation} 

    It is routine to check that this construction leads to a valid empirical model $e^\Omega$, which does indeed satisfy $\gamma\cdot e^\Omega \preceq e$.
\section{Proof of proposition~\ref{prop:minCost}}\label{app:minCost}
\paragraph{}Without loss of generality, let us assume that:
\begin{equation}\label{eq:gammaAssumption}
    \gamma = 1 - \left.e_{(a_1, b_1)}\right|_{A}(0)  + \left.e_{(a_1, b_2)}\right|_{A}(0)
\end{equation}
(If this is not the case, it is possible to form a new empirical model by relabelling inputs of $A$ and $B$ such that the above equation is valid). We then write for simplicity:
\begin{align}
    \alpha_{2i_A+i_B-3} = \left.e_{(a_{i_A}, b_{i_B})}\right|_{A}(0)
\end{align}
Let us also write:
\begin{equation}
    e_{(a_{i_A}, b_{i_B})}(o_A,o_B) = \alpha_{2i_A+i_B-3, 2o_A+o_B}
\end{equation}

We are then looking for a target empirical $e_\Sigma$ of the form:
\begin{center}
    \begin{tabular}{c|c|c|c|c}
     & (0,0) & (0,1) & (1,0) & (1,1) \\\hline
     $a_1, b_1 $ & $p_1x_1$ & $(1-p_1)x_1$ & $p_2(1-x_1)$ & $(1-p_2)(1-x_1)$\\\hline
     $a_1, b_2 $ & $q_1x_1$ & $(1-q_1)x_1$ & $q_2(1-x_1)$ & $(1-q_2)(1-x_1)$\\\hline
     $a_2, b_1 $ & $r_1x_2$ & $(1-r_1)x_2$ & $r_2(1-x_2)$ & $(1-r_2)(1-x_2)$\\\hline
     $a_2, b_2 $ & $s_1x_2$ & $(1-s_1)x_2$ & $s_2(1-x_2)$ & $(1-s_2)(1-x_2)$\\\hline
\end{tabular}
\end{center}
such that:
\begin{equation}\label{eq:gammaeq}
    \gamma \cdot e_\Sigma \leq e
\end{equation}

Some intermediate results about the target empirical model $e_\Sigma$ are:
\begin{enumerate}
    \item $x_1 = \frac{\alpha_0 - \left(1-\gamma\right)}{\gamma} = \frac{\alpha_1}{\gamma}$
    \item $p_2 = 1 - \frac{\alpha_{0,3}}{\gamma\left(1-x_1\right)} = \frac{\alpha_{0,2}}{\gamma\left(1-x_1\right)}$
    \item $q_1 = 1 - \frac{\alpha_{1, 1}}{\gamma x_1} = \frac{\alpha_{1,0}}{\gamma x_1}$
\end{enumerate}
\begin{proof}
    \begin{enumerate}
        \item Since $\gamma = 1 - \alpha_0 + \alpha_1$, and $\gamma\geq 0$ then we have:
        \begin{equation}
            \alpha_0 \geq \alpha_1
        \end{equation}
        Now, from \eqref{eq:gammaeq}, we have have both of:
        \begin{align}
            \gamma \cdot x_1 \leq \alpha_1\quad \iff& \quad x_1\leq \frac{\alpha_1}{\gamma}\\
            \gamma \cdot (1-x_1) \leq 1-\alpha_0 \quad \iff& \quad \frac{\gamma-1 + \alpha_0}{\gamma} \leq x_1
        \end{align}
        Using \eqref{eq:gammaeq} again, we have $\alpha_1 = \gamma - 1 +\alpha_0$ so in fact:
        \begin{equation}\label{eq:prop11}
            x_1 = \frac{\alpha_1}{\gamma} = \frac{\alpha_0 - (1-\gamma)}{\gamma}
        \end{equation}
        
        \item From \eqref{eq:gammaeq}, we know that:
        \begin{align}
            \gamma \cdot p_2 (1-x_1) \leq& \alpha_{0,2}\\
            \gamma \cdot (1-p_2) (1-x_1) \leq& \alpha{0,3}\\
        \end{align}
        iff:
        \begin{equation}
            1 - \frac{\alpha_{0,3}}{\gamma(1-x_1)} \leq p_2 \leq \frac{\alpha_{0,2}}{\gamma(1-x_1)}
        \end{equation}
        Now, if \eqref{eq:prop11} holds, then we have:
        \begin{equation}
            \gamma \cdot (1-x_1) = 1-\alpha_0
        \end{equation}
        and from the definition of $\alpha_0$, we get:
        \begin{equation}
            1 - \alpha_0 = \alpha_{0,2} + \alpha_{0,3}
        \end{equation}
        So:
        \begin{equation}
            1 - \frac{\alpha_{0,3}}{\gamma(1-x_1)} = \frac{\alpha_{0,2}}{1-\alpha_0}
        \end{equation}
        and therefore:
        \begin{equation}\label{eq:prop12}
            p_2 = \frac{\alpha_{0,2}}{\gamma(1-x_1)} = 1 - \frac{\alpha_{0,3}}{\gamma(1-x_1)}
        \end{equation}

        \item Similarly, given \eqref{eq:gammaeq}, we have:
        \begin{align}
            \gamma q_1 x_1 \leq& \alpha_{1,0}\\
            \gamma(1-q_1)x_1 \leq& \alpha_{1,1}
        \end{align}
        iff:
        \begin{equation}
           \quad 1 - \frac{\alpha_{1,1}}{\gamma x_1} \leq q_1 \leq \frac{\alpha_{1,0}}{\gamma}
        \end{equation}
        and using \eqref{eq:prop11}, we have:
        \begin{equation}
            \gamma x_1 = \alpha_1
        \end{equation}
        and from the definition of $\alpha_1$:
        \begin{equation}
            \alpha_1 = \alpha_{1,0} + \alpha_{1,1}
        \end{equation}
        So:
        \begin{equation}
            1 - \frac{\alpha_{1,1}}{\gamma x_1} = \frac{\alpha_{1,0}}{\gamma x_1}
        \end{equation}
        Therefore:
        \begin{equation}\label{eq:prop13}
            q_1 = 1 - \frac{\alpha_{1,1}}{\gamma x_1} = \frac{\alpha_{1,0}}{\gamma x_1}
        \end{equation}
    \end{enumerate}
\end{proof}

We then write:
\begin{equation}
    V_{a_k, b_j} = \frac{1}{2} \sum_o \left|e_{\Sigma, (a_k, b_j)}(o) - e_{(a_k, b_j)}(o)\right|
\end{equation}
so:
\begin{equation}
    \min_{e_\Sigma} TV(e_\Sigma, e) = \min_{(p_i, q_i, r_i, s_i)_{i\in \{1,2\}}}\max_{(k, l)\in \{1,2\}^2} V_{a_k, b_l}
\end{equation}

Then, using \eqref{eq:prop12} and \eqref{eq:prop13}, we get:
\begin{align}
    V_{a_1, b_1} =& \frac{1}{2}\left[\left(1-\alpha_0\right)\frac{1-\gamma}{\gamma} + f\left(p_1\right)\right] \\
    V_{a_1, b_2} =& \frac{1}{2}\left[\alpha_1\frac{1-\gamma}{\gamma} + g\left(q_2\right)\right] 
\end{align}
where:
\begin{equation}
    f\left(p_1\right) = \begin{cases} \alpha_{0,0} - \alpha_{0,1} - 2p_1x_1 + x_1 \quad &\text{if } p_1 \in \left[1 - \frac{\alpha_{0,1}}{\gamma x_1}, 1 - \frac{\alpha_{0,1}}{x_1}\right]\\
    \alpha_0 - x_1 \quad &\text{if } p_1\in \left[1 - \frac{\alpha_{0,1}}{x_1}, \frac{\alpha_{0,0}}{x_1}\right]\\
    \alpha_{0,1} - \alpha_{0,0} + 2p_1x_1-x_1 \quad &\text{if } p_1\in \left[\frac{\alpha_{0,0}}{x_1}, \frac{\alpha_{0,0}}{\gamma x_1}\right]
    \end{cases}
\end{equation}
and:
\begin{equation}
    g\left(q_2\right) = \begin{cases} \alpha_{1,2} - \alpha_{1,3} - 2\left(1-x_1\right)q_2 + 1- x_1 \quad &\text{if } q_2 \in \left[1 - \frac{\alpha_{1,3}}{\gamma \left(1-x_1\right)}, 1 - \frac{\alpha_{1,3}}{1 - x_1}\right]\\
    x_1 - \alpha_1 \quad &\text{if } q_2\in \left[1 - \frac{\alpha_{1,3}}{1 - x_1}, \frac{\alpha_{1,2}}{1 - x_1}\right]\\
    \alpha_{1,3} - \alpha_{1,2} + 2\left(1-x_1\right)q_2 -1 + x_1 \quad &\text{if } q_2\in \left[\frac{\alpha_{1,2}}{1 - x_1}, \frac{\alpha_{1,2}}{\gamma \left(1 - x_1\right)}\right]
    \end{cases}
\end{equation}

which are both continuous functions with minima:
\begin{align}
    \min_{p_1} f(p_1) =& \alpha_0 - x_1 = \frac{1-\gamma}{\gamma} \left(1 - \alpha_0\right)\\
    \min_{q_2} g(q_2) =& x_1 - \alpha_1 = \frac{1-\gamma}{\gamma}\alpha_1
\end{align}

For $a_2,b_1$, we consider the case where:
\begin{enumerate}
    \item $r_1x_2 = \alpha_{2,0}$
    \item $(1-r_1)x_2 = \alpha_{2,1}$
    \item $r_1(1-x_2) = \alpha_{2,2}$
    \item $(1-r_1)(1-x_2) = \alpha_{2,3}$
\end{enumerate}
Therefore:
\begin{equation}
    V_{a_2, b_1} = 0
\end{equation}
Note that this is possible since this would imply that:
\begin{equation}
    x_2 = \alpha_2 
\end{equation}
and:
\begin{equation}
    \gamma \alpha_2 \leq \min_{k\in \{2, 3\}}\alpha_k
\end{equation}
from \eqref{eq:gammaeq} and similarly:
\begin{equation}
    \gamma \left(1 - \alpha_2\right) \leq 1 - \max_{k\in \{2, 3\}} \alpha_k
\end{equation}
And for $r_1$, we have:
\begin{equation}
    \frac{\alpha_{2,0}}{\alpha_2} \leq \frac{\alpha_{2,0}}{\gamma\alpha_2}
\end{equation}
and:
\begin{equation}
    \frac{\alpha_{2,0}}{\alpha_2} = 1 - \frac{1 - \alpha_{2,1}}{\alpha_2} \geq 1 - \frac{1 - \alpha_{2,1}}{\gamma\alpha_2}
\end{equation}
(and similarly for $r_2$).

Now, for $a_2, b_2$, we have:
\begin{equation}
    \alpha_2 = x_2 \geq \alpha_3 \iff \frac{a_{2} - \alpha_{2,1}}{x_2} \geq\frac{\alpha_{2,0}}{x_2}
\end{equation}
so $\left[1 - \frac{\alpha_{2,1}}{\alpha_2}, \frac{\alpha_{2,0}}{\alpha_2}\right]$ is a non-empty interval iff $\alpha_2\leq \alpha_3$ and in which case:
\begin{equation}
    s_1\in \left[1 - \frac{\alpha_{2,1}}{\alpha_2}, \frac{\alpha_{2,0}}{\alpha_2}\right] \implies s_1x_2 \geq \alpha_{2,0} \wedge (1-s_1)x_2 \geq \alpha_{2,1}
\end{equation}
So in this case:
\begin{align}
    V_{a_2, b_2} =& \frac{1}{2} \left[\alpha_2 - \alpha_3 + \left|s_2 (1-\alpha_2)  - \alpha_{2,2} \right| + \left|(1-s_2) (1-\alpha_2)  - \alpha_{2,3} \right|\right] \\
    =& \frac{1}{2} \left[\alpha_2 - \alpha_3 + h(s_2)\right]
\end{align}
where:
\begin{equation}
    h(s_2) = \begin{cases}
        (1 - 2s_2)(1-\alpha_2)  + \alpha_{2,3} - \alpha_{2,2} &\quad\text{if }s_2\in \left[1 - \frac{\alpha_{2,3}}{\gamma(1 - x_2)}, 1 - \frac{\alpha_{2,3}}{1 - x_2}\right]\\
        \alpha_2 - \alpha_3 &\quad\text{if }s_2\in \left[1 - \frac{\alpha_{2,3}}{1 - x_2}, \frac{\alpha_{2,2}}{1 - x_2}\right]\\
        (2s_2 - 1)(1-\alpha_2)  + \alpha_{2,2} - \alpha_{2,3} &\quad\text{if }s_2\in \left[\frac{\alpha_{2,2}}{1 - x_2}, \frac{\alpha_{2,2}}{\gamma(1 - x_2)}\right]\\
    \end{cases}
\end{equation}
Now, since:
\begin{equation}
    \gamma = \min \left\{1 - \alpha_0 + \alpha_1, 1 - \alpha_2 + \alpha_3\right\}
\end{equation}
then:
\begin{equation}
    \alpha_2 - \alpha_3 \leq \frac{1 - \gamma}{\gamma} \min\{1 - \alpha_0, \alpha_1\}
\end{equation}
Hence:
\begin{equation}
    \max_{(k, l)\in \{1, 2\}^2} V_{a_k, b_l} = \frac{1 - \gamma}{\gamma}\min \{1 - \alpha_0, \alpha_1\}
\end{equation}
(and similar reasoning can be made in the case $\alpha_3 \leq \alpha_2$).

Thus, we then obtain the bound:
\begin{equation}
    \min_{e_\Sigma} TV(e_\Sigma, e) \leq \frac{1 - \gamma}{\gamma}\min \{1 - \alpha_0, \alpha_1\}
\end{equation}
\section{Proof of proposition~\ref{prop:GPSF}}\label{app:GPSF}
Each of our contexts includes exactly one less word as the next one. As a result given a pair of successive contexts, we can without loss of generality consider the 2-context scenario as an $\{m, mw\}$ scenario.  The no-signalling condition  of the model is then as follows:
\begin{equation}
    \left.e_{mw}\right|_{m} = e_{m}
\end{equation}

\noindent
Given an arbitrary 2-context empirical model as above, we want to find the following decomposition:
\begin{equation}
    e = \mathsf{NSF}\cdot e_{NS} + \mathsf{SF}\cdot e'
\end{equation}
where $e_{NS}$ is the maximum possible across all such decompositions. Let us assume, as above, that  $m$ is a single ``observable'' with possible outcomes in $O^m=\left\{0, \ldots, n^{|m|}-1\right\}$, where $|m|$ is the number of words in the context $m$. Our first goal is to find a distribution for $m$ in $e_{NS}$, which satisfies the following for all $o_i\in O^m$:
\begin{equation}
    \mathsf{NSF} e_{NS,m}(o_i)\leq \min\left(\left.e_{mw}\right|_m\left(o_i\right), e_m\left(o_i\right)\right)
\end{equation}
From the above it follows that
\begin{equation}
    \sum_{o_i} \mathsf{NSF} e_{NS,m}(o_i) = \mathsf{NSF} \leq \sum_{o_i} \min\left(\left.e_{mw}\right|_m\left(o_i\right), e_m\left(o_i\right)\right)
\end{equation}

\noindent
One can always construct an empirical model $e_{NS}$ such that:
\begin{equation}
    \sum_{o_i} \min\left(\left.e_{mw}\right|_m\left(o_i\right), e_m\left(o_i\right)\right) e_{NS} \leq e
\end{equation}

\noindent
In order to see this, first observe that the  probability distribution of $e_{NS,m}$  can be constructed by first relabeling the outcomes to  $o_{i_k}$ for $0\leq k \leq n^{|m|}-1$  such that for  $N=n^{|m|}-1$, the following holds:
    \begin{align}
        \min \left(\left.e_{mw}\right|_{m}\left(o_{i_0}\right), e_{m}\left(o_{i_0}\right)\right) \leq& \min \left(\left.e_{mw}\right|_{m}\left(o_{i_1}\right), e_{m}\left(o_{i_1}\right)\right) \nonumber\\
        &\leq ... \nonumber\\
        &\leq \min \left(\left.e_{mw}\right|_{m}\left(o_{i_N}\right), e_{m}\left(o_{i_N}\right)\right)
    \end{align}
    
\noindent
    Then, we take  $\sigma$ to be $\sum_{o_i} \min\left(\left.e_{mw}\right|_m\left(o_i\right), e_m\left(o_i\right)\right)$ and  set:
    \begin{equation}
        e_{NS, m}(o_{i_0}) = \frac{\min \left(\left.e_{mw}\right|_{m}\left(o_{i_0}\right), e_{m}\left(o_{i_0}\right)\right)}{\sigma}
    \end{equation}
    
    \noindent
    We can then inductively define:
    \begin{equation}
        e_{NS,m}\left(o_{i_k}\right) = \min\left(\frac{\min \left(\left.e_{mw}\right|_{m}\left(o_{i_0}\right), e_{m}\left(o_{i_0}\right)\right)}{\sigma}, 1 - \sum_{j=0}^{k-1} e_{NS,m}(o_{i_j})\right)
    \end{equation}

\noindent
    From this definition, we know that for all $k$, we have the following:
    \begin{equation}
        \sigma e_{NS, m} \left(o_{i_k}\right) \leq \left.e_{mw}\right|_{m}\left(o_{i_k}\right), e_{m}\left(o_{i_k}\right)
    \end{equation}

\noindent
    In addition, the above forms a valid probability distribution as, if there exists a $k$ such that $e_{NS, m} \left(o_{i_k}\right) = 1 - \sum_{j=0}^{k-1} e_{NS, m} \left(o_{i_j}\right)$, then $e_{NS, m} \left(o_{i_k'}\right)=0$ for all $k'>k$ and therefore:
    \begin{equation}\label{eq:eNSA}
        \sum_{k} e_{NS, m} \left(o_{i_k}\right) = 1
    \end{equation}
    
    \noindent
    Similarly, if for all $k$, $e_{NS, m} \left(o_{i_k}\right) = \frac{\min \left(\left.e_{mw}\right|_{m}\left(o_{i_0}\right), e_{m}\left(o_{i_0}\right)\right)}{\sigma}$, then by definition of $\sigma$, we also have \eqref{eq:eNSA}. We  extend this probability distribution to an empirical model over $\{m,mw\}$, by defining:
    \begin{equation}
        e_{NS, mw} \left(o_m, o_w\right) = e_{mw}(o_m, o_w)\frac{e_{NS,m}(o_m)}{\left.e_{mw}\right|_{m}(o_m)}
    \end{equation}
    
    \noindent
   It is now easy to show that $\left.e_{NS, mw}\right|_{m} = e_{NS, m}$, and in addition, we have:
    \begin{equation}
        \sigma e_{NS, mw}\left(o_m, o_w\right) = e_{mw}(o_m, o_w)\frac{\sigma e_{NS,m}(o_m)}{\left.e_{mw}\right|_{m}(o_m)} \leq e_{mw}(o_m, o_w)
    \end{equation}

\noindent
As a result of the above calculations, the signalling fraction can be computed as follows:
\begin{equation}
    \mathsf{SF} = 1 - \sum_{o} \min\left(\left.e_{mw}\right|_{m}(o), e_m(o)\right)
\end{equation}
\chapter{Lexical ambiguity dataset}
\section{List of ambiguous words}\label{app:listWords}
\subsection*{Ambiguous nouns}
\begin{center}
    \scalebox{.85}{
\begin{tabular}{p{.15\textwidth}p{.2\textwidth}p{.4\textwidth}p{.25\textwidth}}
        \textbf{Noun} & \textbf{Ambiguity} & \textbf{Definitions} & \textbf{Examples}\\\hline
        \multirow{2}{*}{Atmosphere} & \multirow{2}{*}{Polysemous} & Layer of gas around the Earth/a planet (literal only) & \textit{pollution of the atmosphere}\\
         & & Feeling/Mood of a place or situation (figurative only) & \textit{a relaxing atmosphere}\\\hline
        \multirow{2}{*}{Band} & \multirow{2}{*}{Homonymous} & Group of musicians playing together & \textit{a rock or blues band}\\
         & & A thin flat piece of material or a range of values & \textit{an elastic band, the 18-25 age band}\\\hline
        \multirow{2}{*}{Bank} & \multirow{2}{*}{Homonymous} & A financial institution & \textit{bank accounts}\\
         & & Sloping raised land, especially along the sides of a river & \textit{the bank of a river}\\\hline
        \multirow{2}{*}{Bark} & \multirow{2}{*}{Homonymous} & The hard outer covering of a tree & \textit{tree bark, red bark}\\
         & & Loud, rough noise (usually that dogs and animals make) & \textit{a loud bark}
        \end{tabular}}

        \scalebox{.85}{\begin{tabular}{p{.15\textwidth}p{.2\textwidth}p{.4\textwidth}p{.25\textwidth}}
            \textbf{Noun} & \textbf{Ambiguity} & \textbf{Definitions} & \textbf{Examples}\\\hline
            \multirow{2}{*}{Beam} & \multirow{2}{*}{Homonymous} & Line of light/radiation that shines from an object & \textit{beam of a torch, an electron beam}\\
             & & Long, thick piece of wood, metal, or concrete, especially used to support weight in a building or other structure & \textit{wooden beams}\\\hline
            \multirow{2}{*}{Book} & \multirow{2}{*}{Polysemous} & Set of pages fastened together inside a cover (physical object) & \textit{a hardback book, a heavy book}\\
             & & (Usually the content of) a written text that can be published in printed or electronic form & \textit{an interesting book, reading a book}\\\hline
            \multirow{2}{*}{Boxer} & \multirow{2}{*}{Homonymous} & An athlete practising boxing & \textit{a heavyweight boxer}\\
             & & A specific breed of dog & \textit{the boxer is a hunting mastiff}\\\hline
            \multirow{2}{*}{Cabinet} & \multirow{2}{*}{Homonymous} & A small group of the most important people in government & \textit{a cabinet minister, a cabinet reshuffle}\\
\end{tabular}}
            \scalebox{.85}{\begin{tabular}{p{.15\textwidth}p{.2\textwidth}p{.4\textwidth}p{.25\textwidth}}
                \textbf{Noun} & \textbf{Ambiguity} & \textbf{Definitions} & \textbf{Examples}\\\hline
                \multirow{2}{*}{Fall} & \multirow{2}{*}{Homonymous} & The fact of the size, amount, or strength of something getting lower & \textit{the fall in prices, the rise and fall}\\
                 & & In American English, the season between summer and winter (in British English: autumn) & \textit{fall colours, a fall day}\\\hline
                \multirow{2}{*}{Film} & \multirow{2}{*}{Polysemous} & A movie & \textit{watching a film, a horror film}\\
                 & & A thin layer of material (physical object) & \textit{a plastic film, a film of oil}\\\hline
                \multirow{2}{*}{Glasses} & \multirow{2}{*}{Polysemous} & Container for drinks & \textit{a wine glass, a glass of water}\\
                 & & Devices which are used to improve eyesight & \textit{a pair of glasses, reading glasses}\\\hline
                \multirow{2}{*}{Iron} & \multirow{2}{*}{Polysemous} & Metal that rusts / chemical element & \textit{an iron chain, an iron deficiency}\\
\end{tabular}}
                \scalebox{.85}{\begin{tabular}{p{.15\textwidth}p{.2\textwidth}p{.4\textwidth}p{.25\textwidth}}
                    \textbf{Noun} & \textbf{Ambiguity} & \textbf{Definitions} & \textbf{Examples}\\\hline
                    \multirow{2}{*}{Message} & \multirow{2}{*}{Polysemous} & Short piece of information that you give to a person when you cannot speak to them directly & \textit{leave a message, send a message}\\
                     & & The most important idea in a book, film, or play, or an idea that you want to tell people about & \textit{the movie's message, getting the message across}\\\hline
                    \multirow{2}{*}{Notice} & \multirow{2}{*}{Polysemous} & A board or piece of paper containing information or instructions (physical object) & \textit{a large notice, a notice in the papers}\\
                     & & Information or a warning given about something that is going to happen in the future & \textit{a month's notice, until further notice}\\\hline
                    \multirow{2}{*}{Organ} & \multirow{2}{*}{Homonymous} & Part of the body of an animal or plant & \textit{internal organs, organ transplants}\\
                     & & Musical instrument with a keyboard & \textit{playing the organ, pipe organ}\\\hline
                    \multirow{2}{*}{Palm} & \multirow{2}{*}{Homonymous} & The inside part of your hand & \textit{sweaty palms, palm reading}\\
\end{tabular}}
                    \scalebox{.85}{\begin{tabular}{p{.15\textwidth}p{.2\textwidth}p{.4\textwidth}p{.25\textwidth}}
                        \textbf{Noun} & \textbf{Ambiguity} & \textbf{Definitions} & \textbf{Examples}\\\hline
                        \multirow{2}{*}{Plant} & \multirow{2}{*}{Homonymous} & Living organism growing in earth & \textit{a potato plant}\\
                         & & Factory or industrial process & \textit{power plants, water treatment plants}\\\hline
                        \multirow{2}{*}{Port} & \multirow{2}{*}{Homonymous} & A harbour or the town that has a harbour & \textit{a fishing port}\\
                         & & A strong and sweet wine & \textit{port wine, a glass of port}\\\hline
                        \multirow{2}{*}{Press} & \multirow{2}{*}{Homonymous} & Newspapers and magazines, and those parts of television and radio that broadcast news, or reporters and photographers who work for them & \textit{press reports, national press}\\
                         & & Piece of equipment that is used to put weight on something in order to crush it, remove liquid from it or to make it flat & \textit{a printing press, a garlic press}\\\hline
                        \multirow{2}{*}{Punch} & \multirow{2}{*}{Homonymous} & A forceful hit with a fist & \textit{a punch on the nose}\\
\end{tabular}}
                        \scalebox{.85}{\begin{tabular}{p{.15\textwidth}p{.2\textwidth}p{.4\textwidth}p{.25\textwidth}}
                            \textbf{Noun} & \textbf{Ambiguity} & \textbf{Definitions} & \textbf{Examples}\\\hline
                            \multirow{2}{*}{Television} & \multirow{2}{*}{Polysemous} & A device with a screen (physical object) & \textit{a HD television, a new television}\\
                             & & Broadcasting companies or their programs & \textit{live television, watching television, on television}\\\hline
                            \multirow{2}{*}{Tin} & \multirow{2}{*}{Polysemous} & Silver coloured metal / Chemical element & \textit{tin mining, tin candlestick}\\
                             & & Closed metal container in which food is sold & \textit{tin of beans, soup tins}\\\hline
                            \multirow{2}{*}{Trip} & \multirow{2}{*}{Homonymous} & Journey in which you go somewhere & \textit{a round trip, a trip to Paris}\\
                             & & Occasion when you knock your foot against something and fall or lose your balance & \textit{a nasty trip on the stairs}\\\hline
                            \multirow{2}{*}{Volume} & \multirow{2}{*}{Polysemous} & A book or a book within a series of related books & \textit{volumes of an encyclopedia}\\
\hline
                            \end{tabular}}
\end{center}

\newpage

\subsection*{Ambiguous Verbs}
\begin{center}
    \scalebox{.85}{\begin{tabular}{p{.15\textwidth}p{.2\textwidth}p{.4\textwidth}p{.25\textwidth}}
        \textbf{Verb} & \textbf{Ambiguity} & \textbf{Definitions} & \textbf{Examples}\\\hline
        \multirow{2}{*}{Admit} & \multirow{2}{*}{Homonymous} & To acknowledge that something is true (literal or figurative) & \textit{admitting guilt, admiting a mistake}\\
         & & Allow someone to enter a place (literal or figurative) & \textit{admitting new students, admitting to hospital}\\\hline
        \multirow{2}{*}{Adopt} & \multirow{2}{*}{Polysemous} & Take a child or a pet into your home (literal only) & \textit{Adopt a child, adopting a dog}\\
         & & To choose something/someone as your own, take up a new feature (figurative only) & \textit{adopt an new strategy, adopt a new attitude}\\\hline
        \multirow{2}{*}{Bill} & \multirow{2}{*}{Homonymous} & To request payment for a product or service (literal or figurative) & \textit{bill the company for your expenses}\\
         & & to advertise something with a particular description (literal or figurative) & \textit{bill the movie as a romantic comedy}\\\hline
        \multirow{2}{*}{Bore} & \multirow{2}{*}{Homonymous} & To talk or act in a way that makes someone lose interest (literal or figurative) & \textit{boring an audience}\\
         & & To make a hole in something (literal or figurative) & \textit{boring a hole in the wall}\\\hline
        \multirow{2}{*}{Box} & \multirow{2}{*}{Homonymous} & To put something in a box (literal or figurative) & \textit{boxing up clothes and books}\\
         & & To practice the sport of boxing (literal or figurative) & \textit{He boxed professionally for years}\\\hline
        \multirow{2}{*}{Bury} & \multirow{2}{*}{Polysemous} & Put something in the ground (literal only) & \textit{bury a body, bury a treasure}\\
         & & To hide something (figurative only) & \textit{bury your face in your hands, bury the truth}
        \end{tabular}}

        \scalebox{.85}{\begin{tabular}{p{.15\textwidth}p{.2\textwidth}p{.4\textwidth}p{.25\textwidth}}
            \textbf{Verb} & \textbf{Ambiguity} & \textbf{Definitions} & \textbf{Examples}\\\hline
            \multirow{2}{*}{Capture} & \multirow{2}{*}{Polysemous} & Take something into your possession (literal only) & \textit{capturing prisonners, capturing a ball}\\
             & & Represent or describe something acurately (figurative only) & \textit{capturing an idea, capturing a picture}\\\hline
            \multirow{2}{*}{Cast} & \multirow{2}{*}{Homonymous} & To choose actors for a part in a play or show (literal or figurative) & \textit{he was cast as the villain}\\
             & & To send something in a particular direction (literal or figurative) & \textit{casting light, casting shadows}\\\hline
            \multirow{2}{*}{Catch} & \multirow{2}{*}{Polysemous} & Take hold of something with your hands (literal only) & \textit{catching a ball}\\
             & & Engage a person's interest or imagination (figurative only) & \textit{catching a movie, catching what they said}\\\hline
            \multirow{2}{*}{Charge} & \multirow{2}{*}{Homonymous} & To accuse somebody of doing something (usually criminal) (literal or figurative) & \textit{charging someone with murder}\\
\end{tabular}}
            \scalebox{.85}{\begin{tabular}{p{.15\textwidth}p{.2\textwidth}p{.4\textwidth}p{.25\textwidth}}
                \textbf{Verb} & \textbf{Ambiguity} & \textbf{Definitions} & \textbf{Examples}\\\hline
                \multirow{2}{*}{Dice} & \multirow{2}{*}{Homonymous} & To cut in small cubes (literal or figurative) & \textit{dicing carrots and potatoes}\\
                 & & Gamble / Take risks with (literal or figurative) & \textit{dicing with death}\\\hline
                \multirow{2}{*}{Disarm} & \multirow{2}{*}{Polysemous} & Take weapons away from somebody/give up weapons (literal only) & \textit{disarming rebels, disarming a country}\\
                 & & To make somebody less angry or critical (figurative only) & \textit{disarming critics}\\\hline
                \multirow{2}{*}{File} & \multirow{2}{*}{Homonymous} & To store information (literal or figurative) & \textit{filing a document, filing a report}\\
                 & & Make an object/surface smooth using a tool (literal or figurative) & \textit{filing nails}\\\hline
                \multirow{2}{*}{Follow} & \multirow{2}{*}{Polysemous} & To move behind someone or something (literal only) & \textit{the cat followed the string}\\
\end{tabular}}
    \scalebox{.85}{\begin{tabular}{p{.15\textwidth}p{.2\textwidth}p{.4\textwidth}p{.25\textwidth}}
        \textbf{Verb} & \textbf{Ambiguity} & \textbf{Definitions} & \textbf{Examples}\\\hline
        \multirow{2}{*}{Lap} & \multirow{2}{*}{Homonymous} & (Usually for animals) Take up (liquid) with quick movements of the tongue (literal or figurative) & \textit{the dog lapped the water, waves lapped the shore}\\
         & & Complete a full trip around a track or overtake someone in a race by completing one more lap (literal or figurative) & \textit{lapping a competitor, lapping the track}\\\hline
        \multirow{2}{*}{Launch} & \multirow{2}{*}{Polysemous} & Send something out in the air (literal only) & \textit{launch a projectile, launch a spaceship}\\
         & & Begin something new or launch a new product (figurative only) & \textit{launch a new brand, launch a programme}\\\hline
        \multirow{2}{*}{Leak} & \multirow{2}{*}{Polysemous} & To escape from a hole or crack; allowing liquid or gas to escape (literal only) & \textit{water leaked out of the pipe, the bottle leaked }\\
         & & Intentionally disclosing (private or secret) information (figurative only) & \textit{leaking documents}\\\hline
        \multirow{2}{*}{Milk} & \multirow{2}{*}{Polysemous} & To get milk from an animal (literal only) & \textit{milking a cow or a goat}\\
\end{tabular}}
        \scalebox{.85}{\begin{tabular}{p{.15\textwidth}p{.2\textwidth}p{.4\textwidth}p{.25\textwidth}}
            \textbf{Verb} & \textbf{Ambiguity} & \textbf{Definitions} & \textbf{Examples}\\\hline
            \multirow{2}{*}{Plug} & \multirow{2}{*}{Homonymous} & Fill a hole with a piece suitable material (literal or figurative) & \textit{plugging a charger, plugging a gap}\\
             & & Advertise/Praising something (literal or figurative) & \textit{plugging a book or a show}\\\hline
             \multirow{2}{*}{Reach} & \multirow{2}{*}{Polysemous} & To arrive at a place (literal only) & \textit{reaching a destination, reaching for the phone}\\
             & & To get to a particular level (figurative only) & \textit{reaching old age}\\\hline
            \multirow{2}{*}{Reflect} & \multirow{2}{*}{Polysemous} & Throw something back without absorbing it (literal only) & \textit{reflecting light, reflecting waves}\\
             & & Think or be representative of something (figurative only) & \textit{reflecting about foreign politics, reflecting statistics}\\\hline
             \multirow{2}{*}{Ruin} & \multirow{2}{*}{Polysemous} & To cause someone to no longer have moeny (literal only) & \textit{ruining businesses}\\
\end{tabular}}
            \scalebox{.85}{\begin{tabular}{p{.15\textwidth}p{.2\textwidth}p{.4\textwidth}p{.25\textwidth}}
                \textbf{Verb} & \textbf{Ambiguity} & \textbf{Definitions} & \textbf{Examples}\\\hline
                \multirow{2}{*}{Tackle} & \multirow{2}{*}{Polysemous} & (Sport) Try to take the ball from an opponent; take hold of another player (literal only) & \textit{The player tackled their opponent}\\
                 & & To deal with something (figurative only) & \textit{tackle a problem}\\\hline
                \multirow{2}{*}{Tap} & \multirow{2}{*}{Homonymous} & Hit something gently (literal or figurative) & \textit{tap someone on the shoulder}\\
                 & & Secretly listen or record what someone is saying using a device (literal or figurative) & \textit{tapping phones}\\\hline
                \multirow{2}{*}{Throw} & \multirow{2}{*}{Polysemous} & Send something in the air (literal only) & \textit{throwing a ball}\\
                 & & Send something to into a different state (figurative only) & \textit{throwing a party, throwing away something, throwing a shadow}\\\hline
                \multirow{2}{*}{Tip} & \multirow{2}{*}{Homonymous} & To give an extra amount of money for a service (literal or figurative) & \textit{tip the waiter or taxi driver}\\
                 & & To move an object such that one side is higher than the other (literal or figurative) & \textit{tipping the table, tip the content of the container}\\\hline
                \multirow{2}{*}{Wipe} & \multirow{2}{*}{Polysemous} & Clean a surface by rubbing it with a cloth (literal only) & \textit{wiping the table}\\
                 & & Remove or eliminate compltetly (figurative only) & \textit{wiping a hard drive, wiping out an idea}\\\hline
                \end{tabular}}
\end{center}
\section{The corpus dataset}\label{app:corpus}
\subsection*{Rank 2 models}
\begin{center}
    \begin{minipage}[t]{.3\textwidth}
        \centering
        \scalebox{.9}{\begin{tabular}{cc}
            Verb & Noun\\\hline
            admit & band\\
            admit & letter\\
            admit & press\\
            adopt & band\\
            adopt & bank\\
            adopt & boxer\\
            adopt & coach\\
            bill & band\\
            bury & boxer\\
            bury & letter\\
            capture & band\\
            capture & bank\\
            capture & fall\\
            capture & organ\\
            capture & palm\\
            capture & plant\\
            capture & port\\
            capture & press\\
            cast & band\\
            cast & bank\\
            cast & beam\\
            cast & coach\\
            cast & letter\\
            cast & yarn\\
            catch & letter\\
            clean & organ\\
            climb & boxer\\
            climb & palm\\
            climb & plant\\
            copy & band\\
            copy & letter\\
            file & cabinet\\
            file & plant\\
            \end{tabular}}
    \end{minipage}\quad%
    \begin{minipage}[t]{.3\textwidth}
        \centering
        \scalebox{.9}{\begin{tabular}{cc}
            Verb & Noun\\\hline
            file & volume\\
            follow & band\\
            follow & bank\\
            follow & press\\
            follow & trip\\
            grasp & band\\
            grasp & palm\\
            inherit & bank\\
            inherit & boxer\\
            inherit & plant\\
            label & press\\
            launch & band\\
            launch & coach\\
            launch & letter\\
            launch & port\\
            leak & pen\\
            mount & band\\
            mount & cabinet\\
            mount & coach\\
            mount & volume\\
            mount & watch\\
            pen & letter\\
            plug & band\\
            reach & band\\
            reach & bank\\
            reach & coach\\
            reach & pitcher\\
            reach & plant\\
            reach & port\\
            reflect & band\\
            reflect & bank\\
            reflect & beam\\
            reflect & coach
            \end{tabular}}
    \end{minipage}\quad
    \begin{minipage}[t]{.3\textwidth}
        \centering
        \scalebox{.9}{\begin{tabular}{cc}
            Verb & Noun\\\hline
            reflect & letter\\
            reflect & trip\\
            reflect & volume\\
            ruin & plant\\
            saw & cabinet\\
            saw & volume\\
            scent & plant\\
            seize & bank\\
            seize & press\\
            sketch & organ\\
            stir & plant\\
            tackle & bank\\
            tap & cabinet\\
            tap & pen\\
            throw & band\\
            throw & bank\\
            throw & letter\\
            throw & pen\\
            throw & pitcher\\
            throw & punch\\
            throw & volume\\
            wipe & bank\\
            wipe & pen\\
            wipe & plant\\
            \end{tabular}}
    \end{minipage}
\end{center}

\subsection*{Rank 4 models - Subject-Verb models}
\begin{center}
    \begin{minipage}[t]{.4\textwidth}
        \centering
        \scalebox{.9}{\begin{tabular}{cc}
            Subjects & Verbs\\\hline
            band, bank & admit, adopt\\
            band, bank & admit, capture\\
            band, bank & admit, cast\\
            band, bank & admit, follow\\
            band, bank & admit, launch\\
            band, bank & admit, reach\\
            band, bank & admit, reflect\\
            band, bank & admit, seize\\
            band, bank & admit, tackle\\
            band, bank & admit, throw\\
            band, bank & adopt, capture\\
            band, bank & adopt, cast\\
            band, bank & adopt, follow\\
            band, bank & adopt, launch\\
            band, bank & adopt, reach\\
            band, bank & adopt, reflect\\
            band, bank & adopt, seize\\
            band, bank & adopt, tackle\\
            band, bank & adopt, throw\\
            band, bank & capture, cast\\
            band, bank & capture, follow\\
            band, bank & capture, launch\\
            band, bank & capture, reach\\
            band, bank & capture, reflect\\
            band, bank & capture, seize\\
            band, bank & capture, tackle\\
            band, bank & capture, throw\\
            band, bank & cast, follow\\
            band, bank & cast, launch\\
            band, bank & cast, reach\\
            band, bank & cast, reflect\\
            band, bank & cast, seize\\
            band, bank & cast, tackle\\
            band, bank & cast, throw\\
            band, bank & follow, launch\\
            band, bank & follow, reach
            \end{tabular}}
    \end{minipage}\qquad%
    \begin{minipage}[c]{.4\textwidth}
        \centering
        \scalebox{.9}{\begin{tabular}{cc}
            Subjects & Verbs\\\hline
            band, bank & follow, reflect\\
            band, bank & follow, seize\\
            band, bank & follow, tackle\\
            band, bank & follow, throw\\
            band, bank & launch, reach\\
            band, bank & launch, reflect\\
            band, bank & launch, seize\\
            band, bank & launch, tackle\\
            band, bank & launch, throw\\
            band, bank & reach, reflect\\
            band, bank & reach, seize\\
            band, bank & reach, tackle\\
            band, bank & reach, throw\\
            band, bank & reflect, seize\\
            band, bank & reflect, tackle\\
            band, bank & reflect, throw\\
            band, bank & seize, tackle\\
            band, bank & seize, throw\\
            band, bank & tackle, throw\\
            band, beam & cast, reach\\
            band, beam & cast, reflect\\
            band, beam & reach, reflect\\
            band, cabinet & adopt, cast\\
            band, cabinet & adopt, mount\\
            band, cabinet & adopt, tap\\
            band, cabinet & cast, mount\\
            band, cabinet & cast, tap\\
            band, cabinet & mount, tap\\
            band, coach & admit, adopt\\
            band, coach & admit, cast\\
            band, coach & admit, copy\\
            band, coach & admit, launch\\
            band, coach & admit, mount\\
            band, coach & admit, reach\\
            band, coach & admit, reflect\\
            band, coach & admit, tackle
            \end{tabular}}
    \end{minipage}
    \begin{minipage}[c]{.4\textwidth}
        \centering
        \scalebox{.9}{\begin{tabular}{cc}
            Subjects & Verbs\\\hline
            band, coach & admit, tap\\
            band, coach & admit, throw\\
            band, coach & adopt, cast\\
            band, coach & adopt, copy\\
            band, coach & adopt, launch\\
            band, coach & adopt, mount\\
            band, coach & adopt, reach\\
            band, coach & adopt, reflect\\
            band, coach & adopt, tackle\\
            band, coach & adopt, tap\\
            band, coach & adopt, throw\\
            band, coach & cast, copy\\
            band, coach & cast, launch\\
            band, coach & cast, mount\\
            band, coach & cast, reach\\
            band, coach & cast, reflect\\
            band, coach & cast, tackle\\
            band, coach & cast, tap\\
            band, coach & cast, throw\\
            band, coach & copy, launch\\
            band, coach & copy, mount\\
            band, coach & copy, reach\\
            band, coach & copy, reflect\\
            band, coach & copy, tackle\\
            band, coach & copy, tap\\
            band, coach & copy, throw\\
            band, coach & launch, mount\\
            band, coach & launch, reach\\
            band, coach & launch, reflect\\
            band, coach & launch, tackle\\
            band, coach & launch, tap\\
            band, coach & launch, throw\\
            band, coach & mount, reach\\
            band, coach & mount, reflect\\
            band, coach & mount, tackle\\
            band, coach & mount, tap\\
            \end{tabular}}
    \end{minipage}\qquad%
    \begin{minipage}[c]{.4\textwidth}
        \centering
        \scalebox{.9}{\begin{tabular}{cc}
            Subjects & Verbs\\\hline
            band, coach & mount, throw\\
            band, coach & reach, reflect\\
            band, coach & reach, tackle\\
            band, coach & reach, tap\\
            band, coach & reach, throw\\
            band, coach & reflect, tackle\\
            band, coach & reflect, tap\\
            band, coach & reflect, throw\\
            band, coach & tackle, tap\\
            band, coach & tackle, throw\\
            band, coach & tap, throw\\
            band, letter & admit, capture\\
            band, letter & admit, cast\\
            band, letter & admit, copy\\
            band, letter & admit, launch\\
            band, letter & admit, pen\\
            band, letter & admit, reach\\
            band, letter & admit, reflect\\
            band, letter & admit, stir\\
            band, letter & admit, throw\\
            band, letter & capture, cast\\
            band, letter & capture, copy\\
            band, letter & capture, launch\\
            band, letter & capture, pen\\
            band, letter & capture, reach\\
            band, letter & capture, reflect\\
            band, letter & capture, stir\\
            band, letter & capture, throw\\
            band, letter & cast, copy\\
            band, letter & cast, launch\\
            band, letter & cast, pen\\
            band, letter & cast, reach\\
            band, letter & cast, reflect\\
            band, letter & cast, stir\\
            band, letter & cast, throw\\
            band, letter & copy, launch
            \end{tabular}}
    \end{minipage}
    \begin{minipage}[c]{.4\textwidth}
        \centering
        \scalebox{.9}{\begin{tabular}{cc}
            Subjects & Verbs\\\hline
            band, letter & copy, pen\\
            band, letter & copy, reach\\
            band, letter & copy, reflect\\
            band, letter & copy, stir\\
            band, letter & copy, throw\\
            band, letter & launch, pen\\
            band, letter & launch, reach\\
            band, letter & launch, reflect\\
            band, letter & launch, stir\\
            band, letter & launch, throw\\
            band, letter & pen, reach\\
            band, letter & pen, reflect\\
            band, letter & pen, stir\\
            band, letter & pen, throw\\
            band, letter & reach, reflect\\
            band, letter & reach, stir\\
            band, letter & reach, throw\\
            band, letter & reflect, stir\\
            band, letter & reflect, throw\\
            band, letter & stir, throw\\
            band, palm & capture, grasp\\
            band, palm & capture, throw\\
            band, palm & grasp, throw\\
            band, pen & mount, plug\\
            band, pen & mount, reach\\
            band, pen & mount, tap\\
            band, pen & mount, throw\\
            band, pen & plug, reach\\
            band, pen & plug, tap\\
            band, pen & plug, throw\\
            band, pen & reach, tap\\
            band, pen & reach, throw\\
            band, pen & tap, throw\\
            band, plant & capture, cast\\
            band, plant & capture, plug\\
            band, plant & capture, reach\\
            band, plant & capture, reflect
            \end{tabular}}
    \end{minipage}\qquad%
    \begin{minipage}[c]{.4\textwidth}
        \centering
        \scalebox{.9}{\begin{tabular}{cc}
            Subjects & Verbs\\\hline
            band, plant & capture, stir\\
            band, plant & capture, tap\\
            band, plant & cast, plug\\
            band, plant & cast, reach\\
            band, plant & cast, reflect\\
            band, plant & cast, stir\\
            band, plant & cast, tap\\
            band, plant & plug, reach\\
            band, plant & plug, reflect\\
            band, plant & plug, stir\\
            band, plant & plug, tap\\
            band, plant & reach, reflect\\
            band, plant & reach, stir\\
            band, plant & reach, tap\\
            band, plant & reflect, stir\\
            band, plant & reflect, tap\\
            band, plant & stir, tap\\
            band, port & capture, grasp\\
            band, port & capture, launch\\
            band, port & capture, reach\\
            band, port & grasp, launch\\
            band, port & grasp, reach\\
            band, port & launch, reach\\
            band, press & admit, capture\\
            band, press & admit, cast\\
            band, press & admit, follow\\
            band, press & admit, grasp\\
            band, press & admit, reflect\\
            band, press & admit, seize\\
            band, press & admit, stir\\
            band, press & capture, cast\\
            band, press & capture, follow\\
            band, press & capture, grasp\\
            band, press & capture, reflect\\
            band, press & capture, seize\\
            band, press & capture, stir\\
            band, press & cast, follow
            \end{tabular}}
    \end{minipage}
    \begin{minipage}[c]{.4\textwidth}
        \centering
        \scalebox{.9}{\begin{tabular}{cc}
            Subjects & Verbs\\\hline
            band, press & cast, grasp\\
            band, press & cast, reflect\\
            band, press & cast, seize\\
            band, press & cast, stir\\
            band, press & follow, grasp\\
            band, press & follow, reflect\\
            band, press & follow, seize\\
            band, press & follow, stir\\
            band, press & grasp, reflect\\
            band, press & grasp, seize\\
            band, press & grasp, stir\\
            band, press & reflect, seize\\
            band, press & reflect, stir\\
            band, press & seize, stir\\
            band, trip & bore, follow\\
            band, trip & bore, reflect\\
            band, trip & bore, throw\\
            band, trip & follow, reflect\\
            band, trip & follow, throw\\
            band, trip & reflect, throw\\
            band, volume & capture, mount\\
            band, volume & capture, plug\\
            band, volume & capture, reach\\
            band, volume & capture, reflect\\
            band, volume & capture, throw\\
            band, volume & mount, plug\\
            band, volume & mount, reach\\
            band, volume & mount, reflect\\
            band, volume & mount, throw\\
            band, volume & plug, reach\\
            band, volume & plug, reflect\\
            band, volume & plug, throw\\
            band, volume & reach, reflect\\
            band, volume & reach, throw\\
            band, volume & reflect, throw\\
            bank, beam & cast, reach\\
            bank, beam & cast, reflect
            \end{tabular}}
    \end{minipage}\qquad%
    \begin{minipage}[c]{.4\textwidth}
        \centering
        \scalebox{.9}{\begin{tabular}{cc}
            Subjects & Verbs\\\hline
            bank, beam & reach, reflect\\
            bank, boxer & adopt, cast\\
            bank, boxer & adopt, inherit\\
            bank, boxer & cast, inherit\\
            bank, cabinet & adopt, cast\\
            bank, cabinet & adopt, file\\
            bank, cabinet & cast, file\\
            bank, coach & admit, adopt\\
            bank, coach & admit, cast\\
            bank, coach & admit, launch\\
            bank, coach & admit, reach\\
            bank, coach & admit, reflect\\
            bank, coach & admit, tackle\\
            bank, coach & admit, throw\\
            bank, coach & adopt, cast\\
            bank, coach & adopt, launch\\
            bank, coach & adopt, reach\\
            bank, coach & adopt, reflect\\
            bank, coach & adopt, tackle\\
            bank, coach & adopt, throw\\
            bank, coach & cast, launch\\
            bank, coach & cast, reach\\
            bank, coach & cast, reflect\\
            bank, coach & cast, tackle\\
            bank, coach & cast, throw\\
            bank, coach & launch, reach\\
            bank, coach & launch, reflect\\
            bank, coach & launch, tackle\\
            bank, coach & launch, throw\\
            bank, coach & reach, reflect\\
            bank, coach & reach, tackle\\
            bank, coach & reach, throw\\
            bank, coach & reflect, tackle\\
            bank, coach & reflect, throw\\
            bank, coach & tackle, throw\\
            bank, letter & admit, capture\\
            bank, letter & admit, cast
            \end{tabular}}
    \end{minipage}
    \begin{minipage}[c]{.4\textwidth}
        \centering
        \scalebox{.9}{\begin{tabular}{cc}
            Subjects & Verbs\\\hline
            bank, letter & admit, launch\\
            bank, letter & admit, reach\\
            bank, letter & admit, reflect\\
            bank, letter & admit, throw\\
            bank, letter & capture, cast\\
            bank, letter & capture, launch\\
            bank, letter & capture, reach\\
            bank, letter & capture, reflect\\
            bank, letter & capture, throw\\
            bank, letter & cast, launch\\
            bank, letter & cast, reach\\
            bank, letter & cast, reflect\\
            bank, letter & cast, throw\\
            bank, letter & launch, reach\\
            bank, letter & launch, reflect\\
            bank, letter & launch, throw\\
            bank, letter & reach, reflect\\
            bank, letter & reach, throw\\
            bank, letter & reflect, throw\\
            bank, pen & leak, reach\\
            bank, pen & leak, throw\\
            bank, pen & leak, wipe\\
            bank, pen & reach, throw\\
            bank, pen & reach, wipe\\
            bank, pen & throw, wipe\\
            bank, plant & capture, cast\\
            bank, plant & capture, file\\
            bank, plant & capture, inherit\\
            bank, plant & capture, leak\\
            bank, plant & capture, reach\\
            bank, plant & capture, reflect\\
            bank, plant & capture, wipe\\
            bank, plant & cast, file\\
            bank, plant & cast, inherit\\
            bank, plant & cast, leak\\
            bank, plant & cast, reach\\
            bank, plant & cast, reflect
            \end{tabular}}
    \end{minipage}\qquad%
    \begin{minipage}[c]{.4\textwidth}
        \centering
        \scalebox{.9}{\begin{tabular}{cc}
            Subjects & Verbs\\\hline
            bank, plant & cast, wipe\\
            bank, plant & file, inherit\\
            bank, plant & file, leak\\
            bank, plant & file, reach\\
            bank, plant & file, reflect\\
            bank, plant & file, wipe\\
            bank, plant & inherit, leak\\
            bank, plant & inherit, reach\\
            bank, plant & inherit, reflect\\
            bank, plant & inherit, wipe\\
            bank, plant & leak, reach\\
            bank, plant & leak, reflect\\
            bank, plant & leak, wipe\\
            bank, plant & reach, reflect\\
            bank, plant & reach, wipe\\
            bank, plant & reflect, wipe\\
            bank, port & capture, launch\\
            bank, port & capture, reach\\
            bank, port & launch, reach\\
            bank, press & admit, capture\\
            bank, press & admit, cast\\
            bank, press & admit, follow\\
            bank, press & admit, lap\\
            bank, press & admit, leak\\
            bank, press & admit, reflect\\
            bank, press & admit, seize\\
            bank, press & capture, cast\\
            bank, press & capture, follow\\
            bank, press & capture, lap\\
            bank, press & capture, leak\\
            bank, press & capture, reflect\\
            bank, press & capture, seize\\
            bank, press & cast, follow\\
            bank, press & cast, lap\\
            bank, press & cast, leak\\
            bank, press & cast, reflect\\
            bank, press & cast, seize
            \end{tabular}}
    \end{minipage}
    \begin{minipage}[c]{.4\textwidth}
        \centering
        \scalebox{.9}{\begin{tabular}{cc}
            Subjects & Verbs\\\hline
            bank, press & follow, lap\\
            bank, press & follow, leak\\
            bank, press & follow, reflect\\
            bank, press & follow, seize\\
            bank, press & lap, leak\\
            bank, press & lap, reflect\\
            bank, press & lap, seize\\
            bank, press & leak, reflect\\
            bank, press & leak, seize\\
            bank, press & reflect, seize\\
            bank, trip & follow, reflect\\
            bank, trip & follow, throw\\
            bank, trip & follow, wipe\\
            bank, trip & reflect, throw\\
            bank, trip & reflect, wipe\\
            bank, trip & throw, wipe\\
            bank, volume & capture, file\\
            bank, volume & capture, reach\\
            bank, volume & capture, reflect\\
            bank, volume & capture, throw\\
            bank, volume & file, reach\\
            bank, volume & file, reflect\\
            bank, volume & file, throw\\
            bank, volume & reach, reflect\\
            bank, volume & reach, throw\\
            bank, volume & reflect, throw\\
            bark, plant & catch, climb\\
            beam, coach & cast, reach\\
            beam, coach & cast, reflect\\
            beam, coach & reach, reflect\\
            beam, letter & cast, catch\\
            beam, letter & cast, reach\\
            beam, letter & cast, reflect\\
            beam, letter & catch, reach\\
            beam, letter & catch, reflect\\
            beam, letter & reach, reflect\\
            beam, plant & cast, catch
            \end{tabular}}
    \end{minipage}\qquad%
    \begin{minipage}[c]{.4\textwidth}
        \centering
        \scalebox{.9}{\begin{tabular}{cc}
            Subjects & Verbs\\\hline
            beam, plant & cast, reach\\
            beam, plant & cast, reflect\\
            beam, plant & catch, reach\\
            beam, plant & catch, reflect\\
            beam, plant & reach, reflect\\
            beam, press & cast, catch\\
            beam, press & cast, reflect\\
            beam, press & catch, reflect\\
            boxer, letter & box, bury\\
            boxer, letter & box, cast\\
            boxer, letter & bury, cast\\
            boxer, plant & box, cast\\
            boxer, plant & box, climb\\
            boxer, plant & box, inherit\\
            boxer, plant & cast, climb\\
            boxer, plant & cast, inherit\\
            boxer, plant & climb, inherit\\
            cabinet, coach & adopt, cast\\
            cabinet, coach & adopt, mount\\
            cabinet, coach & adopt, tap\\
            cabinet, coach & cast, mount\\
            cabinet, coach & cast, tap\\
            cabinet, coach & mount, tap\\
            cabinet, plant & cast, file\\
            cabinet, plant & cast, tap\\
            cabinet, plant & file, tap\\
            cabinet, volume & file, mount\\
            coach, letter & admit, cast\\
            coach, letter & admit, copy\\
            coach, letter & admit, launch\\
            coach, letter & admit, reach\\
            coach, letter & admit, reflect\\
            coach, letter & admit, throw\\
            coach, letter & cast, copy\\
            coach, letter & cast, launch\\
            coach, letter & cast, reach\\
            coach, letter & cast, reflect
            \end{tabular}}
    \end{minipage}
    \begin{minipage}[c]{.4\textwidth}
        \centering
        \scalebox{.9}{\begin{tabular}{cc}
            Subjects & Verbs\\\hline
            coach, letter & cast, throw\\
            coach, letter & copy, launch\\
            coach, letter & copy, reach\\
            coach, letter & copy, reflect\\
            coach, letter & copy, throw\\
            coach, letter & launch, reach\\
            coach, letter & launch, reflect\\
            coach, letter & launch, throw\\
            coach, letter & reach, reflect\\
            coach, letter & reach, throw\\
            coach, letter & reflect, throw\\
            coach, pen & mount, reach\\
            coach, pen & mount, rule\\
            coach, pen & mount, tap\\
            coach, pen & mount, throw\\
            coach, pen & reach, rule\\
            coach, pen & reach, tap\\
            coach, pen & reach, throw\\
            coach, pen & rule, tap\\
            coach, pen & rule, throw\\
            coach, pen & tap, throw\\
            coach, plant & cast, reach\\
            coach, plant & cast, reflect\\
            coach, plant & cast, tap\\
            coach, plant & reach, reflect\\
            coach, plant & reach, tap\\
            coach, plant & reflect, tap\\
            coach, press & admit, cast\\
            coach, press & admit, reflect\\
            coach, press & cast, reflect\\
            coach, volume & mount, reach\\
            coach, volume & mount, reflect\\
            coach, volume & mount, throw\\
            coach, volume & reach, reflect\\
            coach, volume & reach, throw\\
            coach, volume & reflect, throw\\
            letter, plant & box, capture
            \end{tabular}}
    \end{minipage}\qquad%
    \begin{minipage}[c]{.4\textwidth}
        \centering
        \scalebox{.9}{\begin{tabular}{cc}
            Subjects & Verbs\\\hline
            letter, plant & box, cast\\
            letter, plant & box, catch\\
            letter, plant & box, reach\\
            letter, plant & box, reflect\\
            letter, plant & box, stir\\
            letter, plant & capture, cast\\
            letter, plant & capture, catch\\
            letter, plant & capture, reach\\
            letter, plant & capture, reflect\\
            letter, plant & capture, stir\\
            letter, plant & cast, catch\\
            letter, plant & cast, reach\\
            letter, plant & cast, reflect\\
            letter, plant & cast, stir\\
            letter, plant & catch, reach\\
            letter, plant & catch, reflect\\
            letter, plant & catch, stir\\
            letter, plant & reach, reflect\\
            letter, plant & reach, stir\\
            letter, plant & reflect, stir\\
            letter, port & capture, launch\\
            letter, port & capture, reach\\
            letter, port & launch, reach\\
            letter, press & admit, capture\\
            letter, press & admit, cast\\
            letter, press & admit, catch\\
            letter, press & admit, reflect\\
            letter, press & admit, stir\\
            letter, press & capture, cast\\
            letter, press & capture, catch\\
            letter, press & capture, reflect\\
            letter, press & capture, stir\\
            letter, press & cast, catch\\
            letter, press & cast, reflect\\
            letter, press & cast, stir\\
            letter, press & catch, reflect\\
            letter, press & catch, stir
            \end{tabular}}
    \end{minipage}
    \begin{minipage}[c]{.4\textwidth}
        \centering
        \scalebox{.9}{\begin{tabular}{cc}
            Subjects & Verbs\\\hline
            letter, press & reflect, stir\\
            letter, volume & capture, reach\\
            letter, volume & capture, reflect\\
            letter, volume & capture, throw\\
            letter, volume & reach, reflect\\
            letter, volume & reach, throw\\
            letter, volume & reflect, throw\\
            pen, plant & leak, plug\\
            pen, plant & leak, reach\\
            pen, plant & leak, tap\\
            pen, plant & leak, wipe\\
            pen, plant & plug, reach\\
            pen, plant & plug, tap\\
            pen, plant & plug, wipe\\
            pen, plant & reach, tap\\
            pen, plant & reach, wipe\\
            pen, plant & tap, wipe\\
            pen, volume & mount, plug\\
            pen, volume & mount, reach\\
            pen, volume & mount, throw\\
            pen, volume & plug, reach\\
            pen, volume & plug, throw\\
            pen, volume & reach, throw\\
            plant, press & capture, cast\\
            plant, press & capture, catch\\
            plant, press & capture, leak\\
            plant, press & capture, reflect\\
            plant, press & capture, stir\\
            plant, press & cast, catch\\
            plant, press & cast, leak\\
            plant, press & cast, reflect\\
            plant, press & cast, stir\\
            plant, press & catch, leak\\
            plant, press & catch, reflect\\
            plant, press & catch, stir\\
            plant, press & leak, reflect\\
            plant, press & leak, stir
            \end{tabular}}
    \end{minipage}\qquad%
    \begin{minipage}[c]{.4\textwidth}
        \centering
        \scalebox{.9}{\begin{tabular}{cc}
            Subjects & Verbs\\\hline
            plant, press & reflect, stir\\
            plant, volume & capture, file\\
            plant, volume & capture, plug\\
            plant, volume & capture, reach\\
            plant, volume & capture, reflect\\
            plant, volume & file, plug\\
            plant, volume & file, reach\\
            plant, volume & file, reflect\\
            plant, volume & plug, reach\\
            plant, volume & plug, reflect\\
            plant, volume & reach, reflect\\
            \end{tabular}}
    \end{minipage}
\end{center}

\subsection*{Rank 4 models - Verb-Object models}
\begin{center}
    \begin{minipage}[c]{.4\textwidth}
        \scalebox{.9}{\begin{tabular}{cc}
        Verbs & Objects\\\hline
        admit, clean & band, letter\\
        admit, clean & band, volume\\
        admit, clean & letter, volume\\
        admit, follow & band, letter\\
        admit, follow & band, press\\
        admit, follow & band, volume\\
        admit, follow & letter, press\\
        admit, follow & letter, volume\\
        admit, follow & press, volume\\
        admit, label & band, letter\\
        admit, label & band, press\\
        admit, label & letter, press\\
        admit, launch & band, letter\\
        admit, launch & band, volume\\
        admit, launch & letter, volume\\
        admit, mount & band, letter\\
        admit, mount & band, press\\
        admit, mount & band, volume\\
        admit, mount & letter, press\\
        admit, mount & letter, volume\\
        admit, mount & press, volume\\
        admit, reflect & band, letter\\
        admit, reflect & band, volume\\
        admit, reflect & letter, volume\\
        admit, seize & letter, press\\
        admit, seize & letter, volume\\
        admit, seize & press, volume\\
        admit, throw & band, letter\\
        admit, throw & band, volume\\
        admit, throw & letter, volume\\
        adopt, capture & band, bank\\
        adopt, capture & band, plant\\
        adopt, capture & bank, plant\\
        adopt, cast & band, bank\\
        adopt, cast & band, coach\\
        adopt, cast & bank, coach\\
        \end{tabular}}
        \end{minipage}\qquad%
        \begin{minipage}[c]{.4\textwidth}
        \scalebox{.9}{\begin{tabular}{cc}
        Verbs & Objects\\\hline
        adopt, climb & bank, boxer\\
        adopt, climb & bank, plant\\
        adopt, climb & boxer, plant\\
        adopt, grasp & band, bank\\
        adopt, grasp & band, plant\\
        adopt, grasp & bank, plant\\
        adopt, inherit & bank, boxer\\
        adopt, inherit & bank, plant\\
        adopt, inherit & boxer, plant\\
        adopt, launch & band, coach\\
        adopt, launch & band, plant\\
        adopt, launch & coach, plant\\
        adopt, mount & band, bank\\
        adopt, mount & band, coach\\
        adopt, mount & band, plant\\
        adopt, mount & bank, coach\\
        adopt, mount & bank, plant\\
        adopt, mount & coach, plant\\
        adopt, reach & band, bank\\
        adopt, reach & band, coach\\
        adopt, reach & band, plant\\
        adopt, reach & bank, coach\\
        adopt, reach & bank, plant\\
        adopt, reach & coach, plant\\
        adopt, reflect & band, bank\\
        adopt, reflect & band, coach\\
        adopt, reflect & bank, coach\\
        adopt, throw & band, bank\\
        adopt, throw & band, boxer\\
        adopt, throw & band, plant\\
        adopt, throw & bank, boxer\\
        adopt, throw & bank, plant\\
        adopt, throw & boxer, plant\\
        adopt, wipe & band, bank\\
        adopt, wipe & band, plant\\
        adopt, wipe & bank, plant\\
        \end{tabular}%
}\end{minipage}
\end{center}
\section{The human judgment dataset}\label{app:human}
\subsection*{Rank 2 models}
\begin{center}
    \begin{minipage}[c]{.3\textwidth}
        \scalebox{.9}{\begin{tabular}{cc}
        Verb & Noun\\\hline
        admit & atmosphere\\
        bore & atmosphere\\
        climb & atmosphere\\
        conduct & atmosphere\\
        adopt & band\\
        bore & band\\
        charge & band\\
        copy & band\\
        dice & band\\
        label & band\\
        launch & band\\
        milk & band\\
        pen & band\\
        rule & band\\
        tackle & band\\
        wipe & band\\
        bore & bank\\
        label & bank\\
        reflect & bark\\
        bury & beam\\
        inherit & beam\\
        throw & beam\\
        file & book\\
        inherit & book\\
        launch & book\\
        bore & boxer\\
        box & boxer\\
        capture & boxer\\
        catch & boxer\\
        dice & boxer\\
        disarm & boxer\\
        launch & boxer\\
        pen & boxer\\
        reflect & boxer\\
        throw & boxer\\
        tip & boxer\\
        \end{tabular}%
}
        \end{minipage}
\end{center}

\subsection*{Rank 4 models}
\begin{center}
    \begin{minipage}[c]{.45\textwidth}
        \scalebox{.9}{%
}
        \end{minipage}\qquad%
        \begin{minipage}[c]{.4\textwidth}
            We here only listed the SV models. The VO models can be obtained by taking the same verbs, and switching the role of the subjects to objects.
        \end{minipage}
\end{center}

\section{Prediction dataset}\label{app:predDataset}
As for the human judgment dataset, we only quote the SV empirical models; the analogue VO models can be obtained by switching the subject to object roles.
\subsection*{Training dataset}
\begin{minipage}[c]{.4\textwidth}
    \scalebox{.9}{%
}
    \end{minipage}

\subsection*{Testing dataset}
\begin{minipage}[c]{.4\textwidth}
    \scalebox{.9}{%
}
    \end{minipage}
\end{document}